\documentclass[10pt,journal,compsoc]{IEEEtran}

\usepackage{microtype}
\usepackage{graphicx}
\usepackage{subfigure}
\usepackage{booktabs} 
\usepackage{multirow}
\usepackage{color}
\usepackage{colortbl}
\usepackage{enumitem}
\usepackage{pifont}
\usepackage{bbding}
\usepackage[framemethod=TikZ]{mdframed}
\usepackage{placeins}
\usepackage{soul}
\usepackage{ulem}
\usepackage{ragged2e}

\usepackage[colorlinks,
linkcolor=blue,
anchorcolor=orange,
citecolor=orange
]{hyperref}

\usepackage{amsmath}
\usepackage{amssymb}
\usepackage{mathtools}
\usepackage{amsthm}
\usepackage{lipsum}
\usepackage{titletoc}

\usepackage[capitalize,noabbrev]{cleveref}

\crefname{figure}{Fig.}{Figs.}
\Crefname{figure}{Fig.}{Figs.}
\crefname{equation}{Eq.}{Eqs.}
\Crefname{equation}{Eq.}{Eqs.}
\crefformat{equation}{#2Eq.~(#1)#3}
\Crefformat{equation}{#2Eq.~(#1)#3}

\crefname{table}{Tab.}{Tabs.}
\Crefname{table}{Tab.}{Tabs.}

\crefname{section}{Sec.}{Secs.}
\Crefname{section}{Sec.}{Secs.}

\crefname{lemma}{Lem.}{Lems.}
\Crefname{lemma}{Lem.}{Lems.}

\crefname{remark}{Remark}{Rem.}
\Crefname{remark}{Remark}{Rems.}

\crefname{proposition}{Proposition}{Prop.}
\Crefname{proposition}{Proposition}{Props.}

\crefname{definition}{Definition}{Def.}
\Crefname{definition}{Definition}{Defs.}
\crefname{appendix}{App.}{App.}

\usepackage[most]{tcolorbox}
\usepackage{makecell}
\tcbuselibrary{skins}
\newenvironment{qbox}
{\begin{tcolorbox}[enhanced jigsaw, drop shadow=black!50!white,colback=white, width=0.95\linewidth, center, left=2pt,right=2pt,top=1pt,bottom=1pt]}
{\end{tcolorbox}}
\usepackage{xcolor}

\theoremstyle{plain}
\newtheorem{theorem}{Theorem}[section]
\newtheorem{proposition}[theorem]{Proposition}
\newtheorem{lemma}[theorem]{Lemma}

\theoremstyle{definition}

\newtheorem{definetitle}[theorem]{Definition}

\theoremstyle{remark}

\definecolor{myblue}{RGB}{235, 244, 246}
\definecolor{mygreen}{RGB}{234, 240, 225}

\newcommand{\MAE}{\mathsf{MAE}}
\newcommand{\F}{\mathsf{F}}
\newcommand{\SMAE}{\mathsf{SI\text{-}MAE}}
\newcommand{\SF}{\mathsf{SI\text{-}F}}
\newcommand{\SAUC}{\mathsf{SI\text{-}AUC}}
\newcommand{\AUC}{\mathsf{AUC}}
\newcommand{\E}{\mathsf{E}}
\newcommand{\Sm}{\mathsf{S}}

\newcommand{\BCE}{\mathsf{BCE}}
\newcommand{\MSE}{\mathsf{MSE}}
\newcommand{\IOU}{\mathsf{IOU}}
\newcommand{\TP}{\mathsf{TP}}
\newcommand{\FP}{\mathsf{FP}}
\newcommand{\FN}{\mathsf{FN}}
\newcommand{\TN}{\mathsf{TN}}
\newcommand{\TPR}{\mathsf{TPR}}
\newcommand{\FPR}{\mathsf{FPR}}
\newcommand{\FNR}{\mathsf{FNR}}

\newcommand{\SI}{\mathsf{SI}}

\usepackage[textsize=tiny]{todonotes}

\begin{document}

\title{Towards Size-invariant Salient Object Detection: A Generic Evaluation and Optimization Approach}
\author{Shilong~Bao,
        Qianqian~Xu*,~\IEEEmembership{Senior~Member,~IEEE},
        Feiran~Li,
        Boyu~Han,
        Zhiyong~Yang,
        \\ Xiaochun~Cao,~\IEEEmembership{Senior~Member,~IEEE},
        and~Qingming~Huang*,~\IEEEmembership{Fellow,~IEEE}%
        \IEEEcompsocitemizethanks{%
            \IEEEcompsocthanksitem Shilong Bao and Zhiyong Yang are with the School of Computer Science and Technology,
            University of Chinese Academy of Sciences, Beijing 101408, China (e-mail: \texttt{\{baoshilong,yangzhiyong21\}@ucas.ac.cn}).
            \IEEEcompsocthanksitem Qianqian Xu is with the State Key Laboratory of AI Safety, Institute of Computing Technology,
            Chinese Academy of Sciences, Beijing 100190, China, and also with Peng Cheng Laboratory, Shenzhen 518055, China
            (e-mail: \texttt{xuqianqian@ict.ac.cn}).
            \IEEEcompsocthanksitem Feiran Li is with State Key Laboratory of Information Security (SKLOIS),
            Institute of Information Engineering, Chinese Academy of Sciences, Beijing 100093, China, and also with
            School of Cyber Security, University of Chinese Academy of Sciences, Beijing 100049, China
            (e-mail: \texttt{lifeiran@iie.ac.cn}).
            \IEEEcompsocthanksitem Boyu Han is with the State Key Laboratory of AI Safety, Institute of Computing Technology,
            Chinese Academy of Sciences, Beijing 100190, China, and with the School of Computer Science and Technology,
            University of Chinese Academy of Sciences, Beijing 101408, China (e-mail: \texttt{hanboyu23z@ict.ac.cn}).
            \IEEEcompsocthanksitem Xiaochun Cao is with School of Cyber Science and Technology, Shenzhen Campus of
            Sun Yat-sen University, Shenzhen 518107, China (e-mail: \texttt{caoxiaochun@mail.sysu.edu.cn}).
            \IEEEcompsocthanksitem Qingming Huang is with the School of Computer Science and Technology,
            University of Chinese Academy of Sciences, Beijing 101408, China, and also with the State Key Laboratory of AI Safety,
            Institute of Computing Technology, Chinese Academy of Sciences, Beijing 100190, China
            (e-mail: \texttt{qmhuang@ucas.ac.cn}).
            \IEEEcompsocthanksitem *~Corresponding authors.
        }%
    }
	\markboth{IEEE TRANSACTIONS ON PATTERN ANALYSIS AND MACHINE INTELLIGENCE}%
	{Shell \MakeLowercase{\textit{et al.}}: Bare Demo of IEEEtran.cls for Computer Society Journals}		
    
\IEEEtitleabstractindextext{%
\begin{abstract}
\justifying
This paper investigates a fundamental yet underexplored issue in Salient Object Detection (SOD): the size-invariant property for evaluation protocols, particularly in scenarios when multiple salient objects of significantly different sizes appear within a single image. We first present a novel perspective to expose the inherent size sensitivity of existing widely used SOD metrics. Through careful theoretical derivations, we show that the evaluation outcome of an image under current SOD metrics can be essentially decomposed into a sum of several separable terms, with the contribution of each term being directly proportional to its corresponding region size. Consequently, the prediction errors would be dominated by the larger regions, while smaller yet potentially more semantically important objects are often overlooked, leading to biased performance assessments and practical degradation. To address this challenge, a generic Size-Invariant Evaluation (SIEva) framework is proposed. The core idea is to evaluate each separable component individually and then aggregate the results, thereby effectively mitigating the impact of size imbalance across objects. Building upon this, we further develop a dedicated optimization framework (SIOpt), which adheres to the size-invariant principle and significantly enhances the detection of salient objects across a broad range of sizes. Notably, SIOpt is model-agnostic and can be seamlessly integrated with a wide range of SOD backbones. Theoretically, we also present generalization analysis of SOD methods and provide evidence supporting the validity of our new evaluation protocols. Finally, comprehensive experiments speak to the efficacy of our proposed approach. The code is available at \href{https://github.com/Ferry-Li/SI-SOD}{https://github.com/Ferry-Li/SI-SOD}.



\end{abstract}
\begin{IEEEkeywords}
			Salient Object Detection, Size-Invariant Evaluation, AUC Optimization, Learning Theory
\end{IEEEkeywords}}

	\maketitle
	\IEEEdisplaynontitleabstractindextext
	\IEEEpeerreviewmaketitle

\section{Introduction}\label{pami:intro}
\IEEEPARstart{S}alient object detection (SOD), also known as salient object segmentation, is a fundamental computer vision task that emulates human visual attention mechanisms to identify the most visually salient objects or regions within complex scenes \cite{IDSurvey_2022, Borji_2019_survey}. To do this, a SOD model typically takes an image as input and generates a saliency map quantifying the salient degree of each pixel. Unlike other segmentation tasks that require fine-grained object classification, SOD is inherently class-agnostic, making it highly versatile for a wide range of downstream tasks, including image segmentation \cite{AUCSeg,ji2023multispectral}, object recognition \cite{DBLP:journals/pami/TanLLYYHO23,mIOU}, medical analysis \cite{DBLP:conf/iccv/Yuan0SY21, DBLP:journals/sensors/LiHWYYL23} and adversarial learning \cite{DBLP:journals/corr/abs-2301-13862,DBLP:journals/tits/ZhengLGZSWCZW24}. 

Over the past few years, SOD has witnessed remarkable progress \cite{tracking_2009,scene_2014,tang2017Tri,li2019deep,zhang2020causal,jiang2023Hierarchical,tang2024remote}, largely driven by the substantial advancements in deep learning. In general, similar to other machine learning tasks, the success of SOD fundamentally hinges on two critical factors: \textbf{the constant evolution of model architectures} and \textbf{the principled design of evaluation and optimization strategies}. The former plays a pivotal role in capturing diverse object feature patterns, evolving from early handcrafted feature-based methods \cite{tracking_2009} to modern deep-based approaches. Notable architectures include Convolutional Neural Networks (CNNs) \cite{EDN,Luo_2017_CVPR,MENet,Ma_2021_AAAI,U-Net} and Transformer \cite{zhang_2021_nips,SwinNet,DBLP:journals/tmm/WuHLX24,VST2}, both of which have demonstrated promising SOD performance. Moreover, given the strong generalization ability of large-scale pretrained foundation models in computer vision \cite{SAM}, integrating these models into SOD has emerged as a promising research direction, with SAM-based adaptations \cite{SAM1, SAMAd, SAM3} gaining increasing attention. The second factor determines the selection and deployment of the best models for practical applications. Generally, a well-performing SOD model should simultaneously embrace a high \text{True Positive Rate ($\TPR$)} and a low \text{False Positive Rate ($\FPR$)} \cite{Borji_2019_survey,pAUC}. To facilitate a more comprehensive assessment of SOD models, various evaluation metrics have been explored \cite{chen_2021_ijcv,sun_2022_ijcv}, spanning pixel-wise errors (e.g., Mean Absolute Error ($\MAE$)), ranking-based measures (e.g., Area Under the ROC Curve ($\AUC$) \cite{MAUC,AUC_eval,AUC_rank}), and region-aware metrics (e.g., F-measure \cite{F_measure}, E-measure \cite{Emeasure}, and S-measure \cite{S_measure1,S_measure2}). Meanwhile, corresponding loss functions~\cite{chen_2021_ijcv,sun_2022_ijcv} are increasingly integrated into end-to-end training frameworks to enhance SOD performance concerning these metrics.


Despite significant advances, this paper identifies a fundamental limitation in existing SOD evaluation metrics—namely, their \textbf{inherent sensitivity to object size}. As illustrated in the motivating study in \cref{fig:msod_num_area}, we can see that real-world SOD scenarios typically involve \textbf{multiple salient objects} (\cref{fig:msod_num}) with \textbf{diverse sizes} (\cref{fig:msod_area}).  In such cases, existing evaluation methods often \ul{struggle to detect small yet semantically crucial objects}, as prediction errors are proportionately influenced by larger regions. For example, in \cref{fig:eg}, traditional $\MAE$ scores tend to favor predictions that capture large objects while entirely missing smaller ones (see \textcolor[rgb]{0.72, 0.42, 0.94}{\textbf{Example 2}}). In contrast, a prediction (the second-best in the first row in \textcolor[rgb]{0.72, 0.42, 0.94}{\textbf{Example 2}}) that successfully covers both large and small objects may receive a worse $\MAE$, despite being more consistent with human perception. This size-related bias can lead to misleading model evaluations and degraded performance in practical applications, particularly in applications where small objects are critical (e.g., traffic lights in automatic driving). Motivated by this issue, this paper centers on the following key  question: 

\begin{figure}[t]
\centering
\subfigure[Average size of objects]{   
\begin{minipage}{0.47\linewidth}
\includegraphics[width=\linewidth]{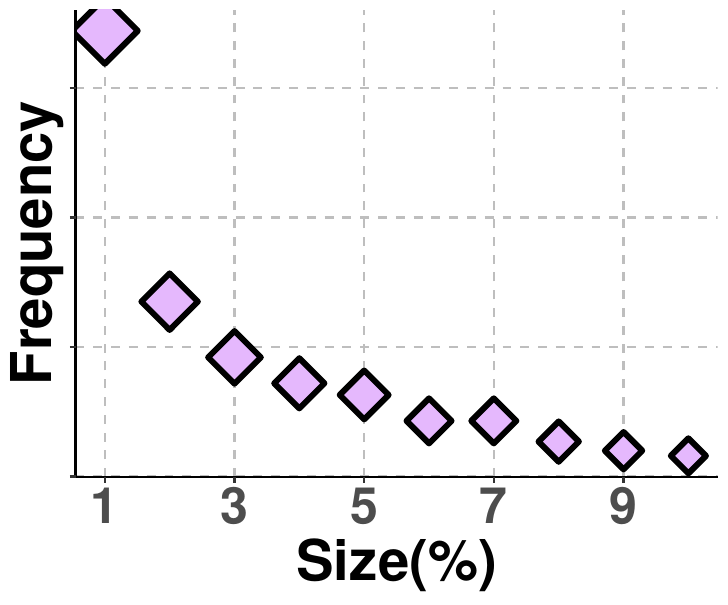}  
\label{fig:msod_area}
\end{minipage}
}
\subfigure[Average number of objects]{   
\begin{minipage}{0.47\linewidth}
\includegraphics[width=\linewidth]{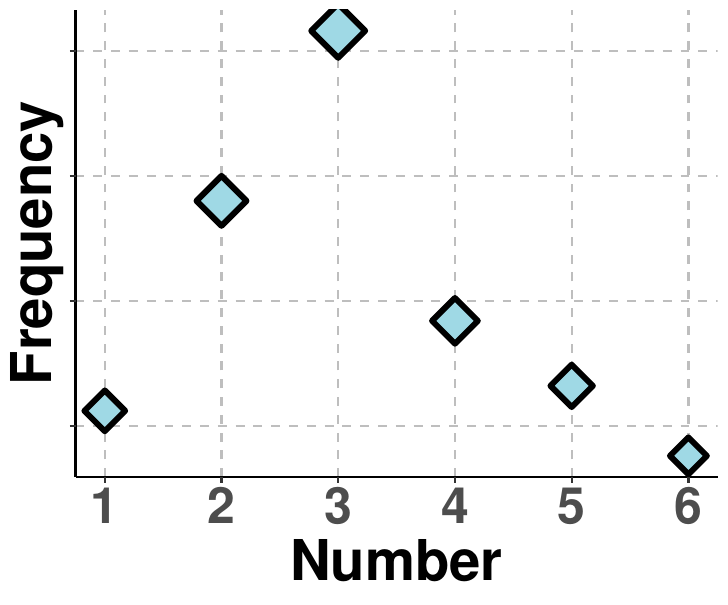}  
\label{fig:msod_num}
\end{minipage}
}
\caption{Statistics on dataset MSOD. \cref{fig:msod_area}: The $x$-axis indicates object size ranges, where Size(\%) refers to the proportion of an object's area (i.e., each connected region as defined in Sec.~\cref{principles of SI_eval}) relative to the total image area. Each positive integer $k$ on the $x$-axis represents the number of objects whose size ratio falls within the interval $[(k-1)\%, k\%]$. This figure highlights small objects ($1 \leq k \leq 10$) to clearly illustrate the motivation behind our study.
\cref{fig:msod_num}: Depicts the average number of objects frequently present in each image, revealing that real-world SOD scenarios typically involve multiple salient objects. }
\label{fig:msod_num_area}
\end{figure}

\begin{qbox}
\begin{center}
\textbf{\textit{How to develop an effective criterion to accommodate the size-variant challenge in SOD?}}
\end{center}
\end{qbox}
\noindent \textbf{Contributions.} In search of an answer, this paper begins with a systematic analysis of current SOD performance measures and presents a novel unified framework revealing the underlying causes of their size-sensitive limitations. To be specific, for a commonly used SOD criterion, we derive that its evaluation concerning an image $\boldsymbol{X}$ could be expressed as a weighted (denoted by $\color{blue}{\lambda(\boldsymbol{X}_k})$) sum of several independent components, where each weighted term $\color{blue}{\lambda(\boldsymbol{X}_k})$ is \textbf{positively correlated with the size} of the corresponding part. This analysis explains why current SOD evaluation protocols produce an inductive bias towards larger objects, and also provides a "golden standard" to ameliorate the reliability and effectiveness of these metrics. Motivated by this, we then propose a simple but effective paradigm for Size-Invariant SOD evaluation (SIEva). The critical idea is to \ul{modify the size-related term $\color{blue}{\lambda(\boldsymbol{X}_k})$ into a size-invariant constant}, ensuring equal consideration for each salient object regardless of size. As demonstrated in \cref{fig:eg}, our proposed metric $\SMAE$ successfully prioritizes predictions that preserve both large and small salient regions, offering a more perceptually aligned and reliable evaluation.

\begin{figure}[!t]
    \centering
    \includegraphics[width=0.98\linewidth]{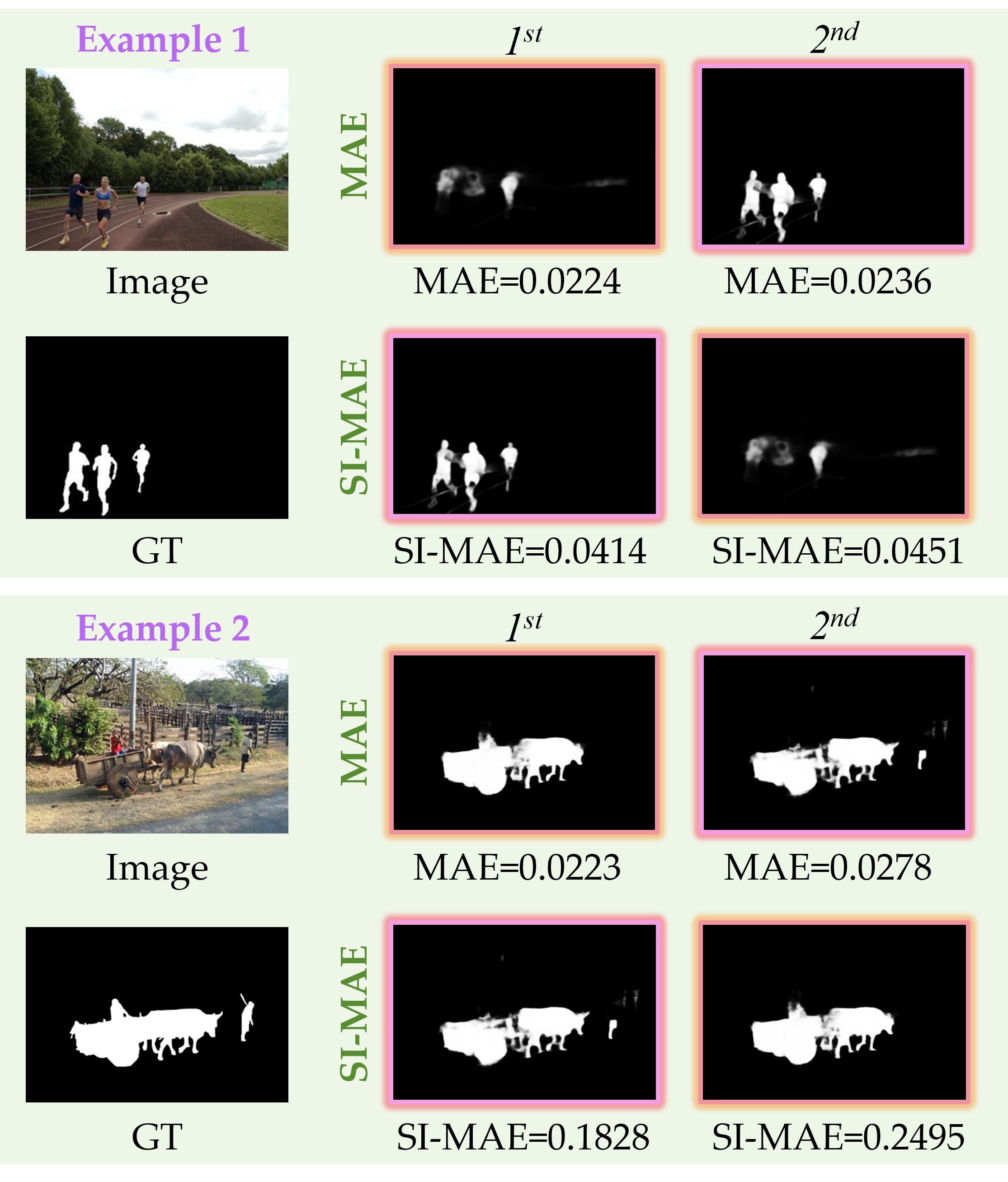}  
    \caption{Illustration of the strengths of our proposed size-invariant metrics over conventional size-sensitive evaluations. Two representative examples are presented with $\MAE$ and $\SMAE$. Apparently, $\SMAE$ offers a more proper assessment of model performance in multiple SOD scenarios. More examples are deferred in \cref{app_fig:eg} in the Appendix. 
    }
    \label{fig:eg}
\end{figure}

Beyond evaluation, we further propose a general Size-Invariant SOD Optimization (SIOpt) framework to enable models with a size-invariant property to be practically applied. Following our proposed SIEva framework, we realize that the most widely adopted loss functions for SOD optimization, including binary cross entropy ($\BCE$), square-based AUC, $\mathsf{Dice}$ and $\mathsf{IOU}$ losses, are also size-sensitive. To deal with this, it is natural to calculate the losses of each part independently first, and then merge them for backward updates, staying the same as SIEva. However, different from other instance-wise pixel-level losses (such as $\BCE$ and $\mathsf{Dice}$), the ranking-aware AUC loss requires each salient pixel paired with all non-salient pixels, suffering from unaffordable $\mathcal{O}(S^2)$ per image time complexity where $S$ represents the number of pixels within an image. To ensure the efficiency of SIOpt towards AUC-style measures, a Pixel-level Bipartite Acceleration (PBAcc) strategy is developed, which can reduce the complexity from $\mathcal{O}(S^2)$ to $\mathcal{O}(S)$ per image significantly.

Going a step further, we proceed to investigate the generalization performance of the SOD algorithm to show the effectiveness of our proposed paradigm. To the best of our knowledge, such a problem \textbf{remains barely explored} in the SOD community. As a result, we find that for composite losses (defined in \cref{Revisting}), the size-invariant loss function leads to a sharper bound than its size-sensitive counterparts.

Finally, extensive experiments over a range of benchmark datasets speak to the efficacy of our proposed SIOpt method.


This work extends from our ICML Spotlight paper \cite{SISOD}. In this version, we have included systematic improvements in methodology and experiments for imbalanced multi-object SOD tasks. To sum up, the new contributions include:

\noindent(1) \textbf{More Comprehensive Size-invariant Framework.} The original paper primarily focuses on $\MAE$ and $\F$-measure to highlight the size-sensitive limitations of current SOD metrics. This version expands the discussion by systematically analyzing another widely used metric—$\AUC$, which reflects the model's rank ability in distinguishing salient versus non-salient pixels. Building on this, we propose a size-invariant AUC metric ($\SI\text{-}\mathsf{AUC}$) and further develop an efficient optimization strategy named PBAcc, to mitigate the heavy computational optimization overhead without compromising performance (see \cref{20250404Sec3.3} and \cref{20250404Sec4.2}). Together, these contributions establish a more complete and effective framework for size-invariant SOD.

\noindent(2) \textbf{A Series of New Experiments.} We conduct a broad range of new empirical studies to evaluate the effectiveness of the proposed size-invariant protocol. These include benchmarking against newly introduced SOD competitors, extending SIOpt to more complex scenarios (e.g., RGB-D and RGB-T SOD tasks), validating its scalability across different backbones and foundation models (e.g., SAM-based models), and performing finer-grained analyses to understand better its practical benefits (see \cref{Experiments}).

\noindent(3) \textbf{Miscellaneous Contents.} We also improve some existing contents to make our work more complete, including the abstract, introduction, review of prior arts (\cref{related_work}), methodology and theory (\cref{Size-invariant Metrics} and \cref{SI-SOD}), and experiments (\cref{Experiments}).

\section{Prior Arts} \label{related_work}
Over the past few decades, the remarkable success of Salient Object Detection (SOD) can be attributed to two key areas of effort. One crucial avenue for achieving promising performance is exploring diverse deep neural network (DNN) architectures to extract discriminative features effectively \cite{DBLP:journals/tcsv/WangTLZL24,DBLP:journals/pami/ChengGBTLW22,DBLP:journals/tnn/CongHLZHK24,DBLP:journals/pami/WuWWLLXCHL25,sun2025conditional}. Another essential factor is the development of appropriate and consistent metrics (loss functions) to evaluate and optimize the model \cite{DBLP:journals/tip/LiSXMT21,DBLP:conf/aaai/XuLL021,DBLP:journals/tnn/ZhangLT23,DBLP:conf/iclr/XiaL0GY0S22,DBLP:conf/pkdd/HanTC16}. In the following, we briefly elaborate on related studies to this work.
\subsection{Deep Models of SOD}
With the rise of deep learning \cite{honovo,reniclshield}, significant advancements in Salient Object Detection (SOD) have been primarily driven by approaches based on encoder-decoder architectures and hierarchical feature aggregation mechanisms (such as convolutional neural networks (CNNs) and U-Net~\cite{U-Net}) \cite{DBLP:conf/iclr/XiaL0S0L23,DBLP:conf/igarss/WangKCWZGL21,DBLP:journals/eswa/WangZZGYCLYG25,DBLP:journals/tim/HeSXCH25,DBLP:journals/pr/ZhuQE25}. To name a few, UCF \cite{UCF} introduces a reformulated dropout mechanism after specific convolutional layers to learn deep uncertain salient features. DCL \cite{DCL} proposes a multi-stream framework that includes a pixel-level fully convolutional stream and a segment-wise spatial pooling stream to improve SOD performance. PiCANet \cite{PiCANet} employs a pixel-wise contextual attention network that selectively captures informative context locations for each pixel, integrating both global and local networks within a U-Net architecture to leverage multi-level features. RDCPN \cite{RDCPN} introduces a novel multi-level ROIAlign-based decoder that adaptively aggregates multi-level features, resulting in better mask predictions. LDF \cite{LDF} designs a two-branch decoder to predict saliency maps by utilizing both the body complement and detailed information of objects. Similar architectures have been explored in recent works, such as EDN \cite{EDN}, ICON \cite{ICON}, Bi-Directional \cite{Bi-Directional}, and CANet \cite{CANet}, addressing the challenges posed by objects of varying sizes and complex relationships. Additionally, to fully extract and fuse multi-scale features, \cite{piao2019depth} designs a recurrent attention network for feature refinement, while \cite{ji2022dmra} adopts a cascaded hierarchical feature fusion strategy to facilitate better information interaction between features. ADMNet \cite{ADMNet} introduces a multi-scale perception encoder module and a dual attention decoder module, providing promising performance with a lightweight computational burden. Recently, the power of Vision Transformers (ViTs) and large foundation models has been leveraged in the SOD community \cite{VST2, TriTransNet, SwinNet}. For example, \cite{VST, VST2} develop a novel unified Visual Saliency Transformer (VST), which takes image patches as inputs and leverages the transformer architecture to propagate global contexts among patches. \cite{SAMAd} highlights the limitations of the current Segment Anything Model (SAM) \cite{SAM} in certain segmentation tasks, such as Salient Object Detection (SOD), and introduces the SAM-Adapter, which integrates domain-specific information or visual prompts into SAM to enhance its performance \cite{SAM1,SAM3}. Last but not least, enriching supervision signals—such as object boundaries, depth, and temperature information—has become a crucial element for improving SOD performance \cite{DBLP:journals/tmm/WuHLX24, ji2023multispectral}. For instance, PoolNet \cite{PoolNet} and MENet \cite{MENet} leverage joint supervision of the salient object and its boundary in the final SOD predictions. \cite{ji2020accurate,zhang2023c,li2023dvsod,li2023delving} leverage depth map to improve RGB-D SOD performance. Besides, \cite{zhang2019memory,zhang2020lfnet} utilize light field data as an auxiliary input for SOD tasks, while \cite{zhang2021dynamic, ji2023multispectral, li2023dvsod} effectively incorporate inter-frame information for video SOD tasks. Unlike previous works that focus on enhancing network architecture, this paper aims to improve SOD performance from an optimization perspective, proposing \textbf{a general protocol} applicable to current SOD models to address the overlooked size-sensitive challenges. 

\subsection{Evaluation and Optimization in SOD \label{20250325sec2.2}}
Beyond architectural advancements, the selection of an appropriate evaluation metric is pivotal in driving progress in SOD. Unlike conventional classification tasks that primarily assess image-level accuracy, SOD requires a pixel-wise evaluation, where each predicted value represents the probability of a pixel being salient \cite{DBLP:journals/ijon/JiangYHLN25, DBLP:journals/tip/HaoYLXYY25}. Over the years, SOD metrics have evolved from early pixel-wise measures to more advanced region-level and structure-aware metrics, each designed to capture distinct aspects of prediction quality and facilitate targeted model optimization \cite{SOD,IDSurvey_2022, DBLP:journals/cgf/ZhangHCWHL24}. A fundamental yet widely adopted metric is \textit{Mean Absolute Error} ($\MAE$), which computes the average absolute deviation between predicted and ground-truth saliency maps, offering a straightforward assessment of overall prediction accuracy \cite{MAE}. Another widely used metric is the $\F$\textit{-measure} \cite{F_measure}, including its variants \textit{$\F_{\beta}$-measure} and \textit{$\F_{\beta}^w$-measure} \cite{F-beta}, which assess the different balance between precision and recall to enhance the attention to structural consistency in saliency predictions. The PASCAL score, introduced in the context of the PASCAL Visual Object Classes challenge \cite{everingham2015pascal}, evaluates how well the detected salient regions match the true salient regions regarding precision, recall, and $\F$-measure. Additionally, \textit{Area Under the ROC Curve} ($\AUC$) offers a comprehensive assessment of the model's ability to distinguish salient pixels from background pixels across various thresholds, making it a valuable metric for evaluating pixel-level classification performance \cite{DBLP:journals/tip/BorjiCJL15,DBLP:journals/tip/LiZWYWZLW16}. Notably, AUC has also been widely acknowledged in the machine learning community for its insensitivity to underlying data distribution \cite{MAUC,pAUC,SFCML,DBLP:journals/csur/YangY23,DBLP:conf/aistats/YaoLY23,DBLP:conf/nips/Luo0Y00024}, which makes it particularly effective for addressing class imbalance issues between salient and non-salient pixels. In terms of structural metrics, the \textit{Structural Similarity Index} (SSIM) \cite{SSIM} evaluates the structural similarity between two images by considering luminance, contrast, and structural components. In contrast, \cite{S_measure1, S_measure2} introduces a more comprehensive measure, incorporating both region consistency and edge consistency, to assess the quality of a saliency map. Furthermore, the \textit{Enhanced-alignment measure} ($\E$-measure) \cite{Emeasure} combines local pixel-level values with the image-level mean to capture fine-grained boundary details better, enhancing edge-preserving accuracy in saliency detection.  

Beyond evaluation, designing loss functions that consistently align with the evaluation objectives is also essential for effective model training. Current SOD methods often adopt a hybrid loss paradigm, combining multiple loss functions to achieve balanced performance across various metrics \cite{LDF,EDN,Object-Detection-in-Aerial-Images}. Typically, to prioritize pixel-wise accuracy, the most commonly used pixel-level loss functions are binary cross-entropy ($\BCE$) \cite{ADMNet, DBLP:journals/tip/BaoZZCYZY25} and mean squared error ($\MSE$) \cite{GateNet, RDSN}. At the region level, widely adopted loss functions, such as $\mathsf{Dice}$ Loss \cite{DiceLoss} and $\mathsf{IOU}$ Loss \cite{IOULoss}, are used to improve overall prediction accuracy by emphasizing the holistic representation of the target region. Furthermore, maximizing $\AUC$-ware surrogate has emerged as a promising approach to directly enhance the ranking performance of SOD models \cite{DBLP:conf/aistats/RosenfeldMTG14, DBLP:journals/pami/WuWWLLXCHL25}. Despite significant success, this paper highlights the size-sensitive limitation of current SOD metrics and its corresponding optimizations, particularly when multiple salient objects of varying sizes co-exist, leading to inadequate evaluation of smaller objects. While prior works \cite{Borji2013Quantitative, Bylinskii2019Measurement} provide comprehensive analyses of various evaluation metrics and model ranking consistency, limited attention has been devoted to scenarios involving the coexistence of salient objects with diverse sizes, which is common in real-world applications. Furthermore, we observe that other pixel-level tasks related to SOD, such as semantic segmentation, adopt metrics like $\mathsf{mIOU}$ \cite{mIOU} and AUCSeg \cite{AUCSeg}. However, since all salient objects are typically labeled with a binary value of 1 without additional class labels, these metrics may not effectively capture the fine details of small and intricate salient objects when directly applied to SOD. Due to space limitations, a detailed introduction to the universally agreed-upon evaluation metrics is deferred to the Appendix \cref{protocol_appendix}.




\section{Size-invariant SOD Evaluation}
\label{Size-invariant Metrics}

In this section, we first provide an in-depth analysis to reveal the fundamental limitation of the currently used metrics in the field of SOD, such as Mean Absolute Error ($\MAE$), $\F$-measure, and Area Under the ROC Curve ($\AUC$). Then, we propose a general and effective approach to serve the goal of size-invariant SOD evaluations. 

\subsection{Revisiting Current SOD Evaluation Metrics} \label{Revisting}

\textbf{Notations.} Let \(\mathcal{D} = \{(\boldsymbol{X}^{(i)}, \boldsymbol{Y}^{(i)}) \mid \boldsymbol{X}^{(i)} \in \mathbb{R}^{3 \times H \times W }, \boldsymbol{Y}^{(i)} \in \mathbb{R}^{H \times W}\}_{i=1}^N\) denote the training dataset, where \(\boldsymbol{X}^{(i)}\) is the $i$-th input image and \(\boldsymbol{Y}^{(i)}\) is its corresponding \textbf{pixel-level} ground truth. Here, \(N\) represents the total number of samples, and \(W\) and \(H\) denote the width and height of the image, respectively. In the context of salient object detection (SOD), the objective is to train a well-performing model \(f_{\theta}: \mathbb{R}^{3 \times H \times W} \to (0, 1)^{H \times W}\) (with \(\theta\) representing the model parameters) that takes an image \(\boldsymbol{X}^{(i)}\) as input and classifies the saliency degree at each pixel \cite{SISOD, AUCSeg,DBLP:journals/tip/BaoZZCYZY25,DBLP:journals/tip/HaoYLXYY25}. In other words, SOD is essentially a pixel-level classification problem, where the output \(f_{\theta} (\boldsymbol{X}^{(i)})_{h,w}\), for \(h \in [H]\) and \(w \in [W]\), represents the probability belonging to a salient pixel. A summary of the key notations involved in this paper is presented in \cref{notation}.

In the following, for the convenience of our analysis, we first unify the current SOD metrics using standard functions, which can be broadly categorized into two groups: \textit{separable} and \textit{composite} functions, defined as follows:


\begin{definetitle}(\textbf{Separable Function}).{\label{def3.1}} Given a predictor $f$, a function $v$ applied to $f$ is separable if the following equation formally holds:
    \begin{equation} \label{eq:separable function}
        v(f(\boldsymbol{X}), \boldsymbol{Y})=\sum_{k=1}^K \lambda({\boldsymbol{X}_k})\cdot v(f(\boldsymbol{X}_k), \boldsymbol{Y}_k),
    \end{equation}
    with 
    \begin{equation}
    \begin{aligned}
    \bigcup_{k=1}^K \boldsymbol{X}_k = \boldsymbol{X}, & \; \bigcap_{k=1}^K \boldsymbol{X}_k = \varnothing, \\
    \bigcup_{k=1}^K \boldsymbol{Y}_k = \boldsymbol{Y}, & \; \bigcap_{k=1}^K \boldsymbol{Y}_k = \varnothing,
    \end{aligned}
    \end{equation}
    where $\boldsymbol{X}$ is an input image and $\boldsymbol{Y}$ is the pixel-level ground truth; $(\boldsymbol{X}_1, \boldsymbol{X}_2, \cdots, \boldsymbol{X}_K)$ and $(\boldsymbol{Y}_1, \boldsymbol{Y}_2, \cdots, \boldsymbol{Y}_K)$ are $K$ \textbf{non-intersect} parts of $\boldsymbol{X}$ and $\boldsymbol{Y}$, respectively; and $\lambda({\boldsymbol{X}_k})$ is an ${\boldsymbol{X}_k}$-related weight for the term $v(f(\boldsymbol{X}_k), \boldsymbol{Y}_k)$. 
\end{definetitle}

\noindent \textbf{Remark.} Note that, with a slight abuse of notation, we abbreviate \( f_{\theta} \) and \( \boldsymbol{X}^{(i)} \) ($\boldsymbol{Y}^{(i)}$) as \( f \) and \( \boldsymbol{X} \) ($\boldsymbol{Y}$), respectively, when the context is clear. Def.\ref{def3.1} indicates that the evaluation of the model $f$ toward a sample $(\boldsymbol{X}, \boldsymbol{Y})$ could be expressed as a series of independent parts and then merged in a weighting fashion. According to the definition above, it is easy to verify that current point-wise evaluation metrics in the SOD community are \textit{separable}, such as Mean Absolute Error ($\MAE$)~\cite{MAE} and Mean Square Error ($\MSE$). 


    \begin{definetitle}(\textbf{Composite Function.}){\label{def3.2}} In contrast to the separable function in Def.\ref{def3.1}, a \textit{composite function} $V$ applied to $f$ could be expressed as a series of compositions of separable functions \cref{eq:separable function}, denoted by
 \begin{align*}
    V(f(\boldsymbol{X}), \boldsymbol{Y})=(v_1\circ v_2 \circ \dots \circ v_T)\left(f(\boldsymbol{X}),\boldsymbol{Y}\right),
 \end{align*}
 where $T$ is the number of compositions.
\end{definetitle}

\noindent\textbf{Remark.} According to the definition of Def.\ref{def3.2}, complicated evaluation metrics such as $\F$-score \cite{F_measure}, $\IOU$ \cite{RCNN}, $\AUC$ \cite{AUC_eval, pAUC} and S-measure \cite{S_measure1, S_measure2} are composite. 

\noindent\textbf{Why Size-invariant Property Matters?} Before our presentations, we aim to re-emphasize the importance of size-invariant evaluations. As discussed in Sec. \ref{pami:intro}, \uline{size-invariance is crucial for the SOD task} for several key reasons. Firstly, real-world salient objects appear in various sizes, as illustrated in \cref{fig:msod_num_area}, and a size-invariant evaluation and optimization scheme ensures that the SOD model \textbf{maintains consistent performance across different scales}. Without such a property, the model may prefer larger objects, negatively degrading detection accuracy for smaller objects—often the most significant and critical ones, as illustrated in \cref{fig:eg} (see \textcolor[rgb]{0.72, 0.42, 0.94}{\textbf{Example 1}} and \textcolor[rgb]{0.72, 0.42, 0.94}{\textbf{Example 2}}). Second, size-invariance enhances the generalization ability of models, making them more adaptable and robust to diverse real-world scenarios. Unfortunately, as we will demonstrate in the next section, most widely adopted SOD metrics are size-sensitive.

\subsubsection{Current Separable Metrics are NOT Size-Invariant} \label{sec3.1.3}
Taking $\MAE$ as an example, it is one of the most typical SOD measures, which evaluates the average absolute pixel-level error between the predicted saliency map and the ground truth. According to Def.\ref{def3.1}, the following result holds: 
\begin{proposition}(\textbf{$\MAE$ is NOT size-invariant}). \label{pami:propMAE}
\begin{equation}
\label{eq:MAE_sensitive}
    \begin{aligned}
        \MAE (f) & = \frac{1}{S} \sum_{h=1}^H \sum_{w=1}^W  |f(\boldsymbol{X})_{h,w}-Y_{h,w}| \\
        &= \sum_{k=1}^{|\mathcal{C}({\boldsymbol{X}})|}  \frac{\Vert f(\boldsymbol{X}_k)-\boldsymbol{Y}_k \Vert_{1,1}}{S} \\
        & = \sum_{i=1}^{|\mathcal{C}({\boldsymbol{X}})|} \frac{S_k}{S} \cdot \frac{\Vert f(\boldsymbol{X}_k)-\boldsymbol{Y}_k \Vert_{1,1} }{S_k} \\
        & = \sum_{k=1}^{|\mathcal{C}({\boldsymbol{X}})|} \frac{S_k}{S} \cdot \MAE(f_k) \\
        & = \sum_{k=1}^{|\mathcal{C}({\boldsymbol{X}})|} {\color{blue}{\boldsymbol{\mathbb{P}_{X_k}}}} \cdot \MAE(f_k).
    \end{aligned}
\end{equation}
Note that, we merely consider the $\MAE$ value for one image here to simplify the discussion, i.e., $\boldsymbol{X}:= \boldsymbol{X}^{(i)}$ for short and $\MAE(f):= \MAE(f(\boldsymbol{X}^{(i)}), \boldsymbol{Y}^{(i)}))$; $\mathcal{C}({\boldsymbol{X}})$ is a pre-defined split function to obtain non-intersect part of $\boldsymbol{X}$ and corresponding $\{\boldsymbol{Y}_k\}_{k=1}^{|\mathcal{C}(\boldsymbol{X})|}$; $S=H \times W$ is the number of pixels in an image and $S_k= H_k \times W_k$ is the size of the $k$-th part of $\boldsymbol{X}$, where $H_k, W_k$ correspond to the height and width of $\boldsymbol{X}_k$. 
\end{proposition}

Compared to Def.\ref{def3.1}, we have \(\lambda(\boldsymbol{X}_k) := \color{blue}{\boldsymbol{\mathbb{P}_{X_k}} = S_k/S}\) in the case of \(\MAE\). In this context, \(\color{blue}{\boldsymbol{\mathbb{P}_{X_k}}}\) is a \textbf{size-sensitive} term when there are multiple salient objects of varying sizes in the image. This means that \textbf{larger objects contribute more} during evaluation, while \uline{smaller but more significant targets may be overshadowed}. Consequently, using the traditional \(\MAE\) metric introduces an accumulated inductive bias across all images, resulting in models prioritizing the larger salient objects. Similar effects are observed for other point-wise metrics in the SOD community, such as $\MSE$.


\subsubsection{Current Composite Metrics are NOT Size-Invariant} For simplicity, we abbreviate $v(f(\boldsymbol{X}), \boldsymbol{Y})$ and $V(f(\boldsymbol{X}), \boldsymbol{Y})$ as $v(f)$ and $V(f)$ for a clear presentation. Similar to Sec.\ref{sec3.1.3}, we could formally rewrite the composite metric $V(f)$ as follows:
\begin{equation}
\label{eq:composite}
    V(f)=\frac{\sum_{k=1}^{|\mathcal{C}({\boldsymbol{X}})|} {\color{blue}{\mathbb{P}'_{\boldsymbol{X}_k}}}(a_1v_1(f_k)+\cdots+a_{T}v_T(f_k))}{\sum_{k=1}^{|\mathcal{C}({\boldsymbol{X}})|} {\color{blue}{\mathbb{P}'_{\boldsymbol{X_k}}}}(b_1v_1(f_k)+\cdots+b_Tv_T(f_k))},
\end{equation}
where again $\boldsymbol{X}:= \boldsymbol{X}^{(i)}$ and $v_t(f_k):=v_t(f(\boldsymbol{X}_k),\boldsymbol{Y}_k), t\in [T]$ is a certain separable metric value over the $k$-th part $\boldsymbol{X}_k$; $a_t$ and $b_t$ represent different coefficients for separable functions, and here ${\color{blue}{\mathbb{P}'_{\boldsymbol{X_k}}}}$ is also a \textbf{size-sensitive} weight for each separable part of $\boldsymbol{X}_k$. Besides, it is noteworthy that we allow the number of $T$ included in the numerator and denominator to be different in general. In what follows, we will elaborate on two typical instantiations for our arguments, including \textit{F-measure} and $\AUC$. 

\noindent\textbf{\textit{F-measure}}. $\F$-measure \cite{F_measure}, a.k.a. $\F$-score, is defined as the combination of precision and recall: 
\begin{equation}\label{pami:eq5}
    \F(f) = \frac{2 \times \mathsf{Precision}(f) \times \mathsf{Recall}(f)}{\mathsf{Precision}(f) + \mathsf{Recall}(f)},
\end{equation}
where 
\begin{equation} \label{pami:eq6}
    \mathsf{Precision} = \frac{\TP(f)}{\TP(f) + \FP(f)}, \; \mathsf{Recall} = \frac{\TP(f)}{\TP(f)+\FN(f)},
\end{equation}
where $\TP,\TN,\FP,\FN$ are \textbf{T}rue \textbf{P}ositive, \textbf{T}rue \textbf{N}egative, \textbf{F}alse \textbf{P}ositive and \textbf{F}alse \textbf{N}egative. Since the outputs of $f$ are usually continuous probability, existing SOD literature often adopts a set of thresholds to binarize the result as $\{0, 1\}^{H \times W}$ and then calculate the average performance \cite{EDN, ADMNet, PoolNet+}. However, by carefully revising the calculation of F-measure, we have the following result:

\begin{proposition}(\textbf{$\F$-measure is NOT size-invariant}).\label{pami:prop3.1}
     By substituting Eq.(\ref{pami:eq6}) to Eq.(\ref{pami:eq5}), we have the following conclusion:
\begin{equation}
\label{eq:f-sensitive}
    \begin{aligned}
    & \F(f) = \frac{2 \times \TP(f)}{2 \times \TP(f) + \FP(f) + \FN(f)} 
    \\ & =  \frac{\sum_{k=1}^{|\mathcal{C}(\boldsymbol{X})|} 2 \times \TP(f_k)}{\sum_{k=1}^{|\mathcal{C}(\boldsymbol{X})|} [ 2 \times \TP(f_k) + \FP(f_k)+ \FN(f_k)]} \\
    & = \frac{2 \sum_{k=1}^{|\mathcal{C}(\boldsymbol{X})|} \frac{S^+_k}{S^+}\cdot \frac{\TP(f_k)}{S^+_k}  }{\sum_{k=1}^{|\mathcal{C}(\boldsymbol{X})|} \left[ \frac{S^+_k}{S^+} \cdot (\frac{2\TP(f_k)}{S^+_k} + \frac{\FP(f_k)}{S^+_k}+ \frac{\FN(f_k)}{S^+_k})\right]} \\
    & = \frac{\sum_{k=1}^{|\mathcal{C}(\boldsymbol{X}|} {\color{blue}{\mathbb{P}_{\boldsymbol{X}^+_k}}}\cdot 2  \TPR(f_k)}{\sum_{k=1}^{|\mathcal{C}(\boldsymbol{X}|} {\color{blue}{\mathbb{P}_{\boldsymbol{X}^+_k}}} \cdot \left(2\TPR(f_k) + \FNR(f_k)\right) + \FP(f_k)/S^+},
    \end{aligned}
\end{equation}
where $\TP(f_k)$, $\FP(f_k)$, $\FN(f_k)$ represent the number of \textbf{T}rue \textbf{P}ositives, \textbf{F}alse \textbf{P}ositives and \textbf{F}alse \textbf{N}egatives at $\boldsymbol{X}_k$'s predictions, and $\TPR(f_k), \FNR(f_k)$ represent the corresponding \textbf{T}rue \textbf{P}ositive \textbf{R}ate and \textbf{F}alse \textbf{N}egative \textbf{R}ate; $S^+$ ($S^+_k$) is the number of all salient pixels in the part $\boldsymbol{X}$ ($\boldsymbol{X}_k$).
\end{proposition}
 Intuitively, we realize that the last term $\FP(f_k)/S^+$ is independent of the concrete splits of $k$ for a well-trained model. However, we still encounter a notorious term $\color{blue}{\mathbb{P}'_{X_k} := \mathbb{P}_{X^+_k}=S^+_k/S^+}$ concerning other terms in Eq.(\ref{eq:f-sensitive}), suggesting that \uline{$\F$-measure is also \textbf{sensitive} to the size of salient objects} like $\MAE$. 


\noindent\textbf{\textit{Area Under the ROC Curve ($\AUC$)}}. As discussed in previous literature \cite{DBLP:journals/csur/YangY23, DBLP:conf/aaai/YangKVY23, MAUC, SFCML}, $\AUC$ is expressed as the entire Area under the ROC curve by plotting the $\TPR$ against the $\FPR$ of a given predictor with all possible thresholds:
\begin{equation}
    \mathsf{AUC}(f) =\int_0^1 \mathsf{TPR}_f(\mathsf{FPR^{-1}}_f(\tau))d\tau,
\end{equation} 
where $f$ is the SOD model, $\tau$ represents the decision threshold. 

Mathematically, $\AUC$ indicates the likelihood that a positive sample \textbf{scores higher} than a negative one \cite{DBLP:conf/nips/CortesM03,DBLP:journals/pami/0001L00Z24}, which could be approximated by the following empirical version in the SOD task:
\begin{equation}\label{pami:eq9}
    \mathsf{\hat{AUC}}(f) = \frac{1}{N}\sum_{i=1}^N \sum_{p=1}^{S^{(i), +}} \sum_{q=1}^{S^{(i), -}} \frac{\mathbb{I}\{f^{(i), +}_{ p}>f^{(i), -}_{q}\}}{S^{(i), +} S^{(i), -}},
\end{equation}
where $f^{(i), +}_{ p}:= f(\boldsymbol{X}^{(i)})_{p, +}$ represents the prediction of the $p$-th salient pixel in image $\boldsymbol{X}^{(i)}$, i.e., $$f(\boldsymbol{X}^{(i)})_{p, +} \in \{f(\boldsymbol{X}^{(i)})_{h, w}|Y^{(i)}_{h, w}=1,h \in [H], w \in [W]\};$$
$S^{(i), +}$ is the cardinality of the positive set; similarly, $f^{(i), -}_{q}:= f(\boldsymbol{X}^{(i)})_{q, -}$ corresponds to the prediction of the $q$-th non-salient pixel in image $\boldsymbol{X}^{(i)}$, i.e., $$f(\boldsymbol{X}^{(i)})_{q, -} \in \{f(\boldsymbol{X}^{(i)})_{h, w}|Y^{(i)}_{h, w}=0,h \in [H], w \in [W]\};$$ $S^{(i), -}$ is the cardinality of the negative set; $\mathbb{I}\{x\}=1$ is the indication function returning $1$ when the condition $x$ holds, otherwise 0. 

It is interesting to note that Eq.(\ref{pami:eq9}) is calculated independently for each image. Therefore, we could still consider one image to analyze the limitation of $\AUC$, i.e., let $\boldsymbol{X}:= \boldsymbol{X}^{(i)}$ in Eq.(\ref{pami:eq9}) for short. Then, we have
\begin{equation}\label{pami:eq10}
    \mathsf{\tilde{AUC}}(f, \boldsymbol{X}^+, \boldsymbol{X}^-) = \sum_{p=1}^{S^{+}} \sum_{q=1}^{S^{-}} \frac{\mathbb{I}\{f^{+}_{ p}>f^{-}_{q}\}}{S^{+} S^{-}},
\end{equation}
where we omit the upper corner marker $(i)$ and employ $\boldsymbol{X}^+$ ($\boldsymbol{X}^-$) representing the positive (negative) part in the image $\boldsymbol{X}^{(i)}$ for the sake of convenience.

Subsequently, similar to Eq.(\ref{eq:composite}), we have the following conclusion:
\begin{proposition}(\textbf{$\AUC$ is NOT size-invariant}).\label{pami:prop3.2}
\begin{equation}\label{pami:eq11}
\begin{aligned}
    \mathsf{\tilde{AUC}}(f, \boldsymbol{X}^+, \boldsymbol{X}^-) & = \sum_{p=1}^{S^{+}} \sum_{q=1}^{S^{-}} \frac{\mathbb{I}\{f^{+}_{ p}>f^{-}_{q}\}}{S^{+} S^{-}} \\
    &= \sum_{k=1}^{|\mathcal{C}(\boldsymbol{X})|} \frac{S_k^+}{S^+}\sum_{p=1}^{S_k^+} \sum_{q=1}^{S^-} \frac{\mathbb{I}\{f^{+}_{ k,p}>f^{-}_{q}\}}{S_k^+ S^-} \\
    &= \sum_{k=1}^{|\mathcal{C}(\boldsymbol{X})|} {\color{blue}{\mathbb{P}_{\boldsymbol{X}^+_k}}}\sum_{p=1}^{S_k^+} \sum_{q=1}^{S^-} \frac{\mathbb{I}\{f^{+}_{ k,p}>f^{-}_{q}\}}{S_k^+ S^-} \\
    &= \sum_{k=1}^{|\mathcal{C}(\boldsymbol{X})|} {\color{blue}{\mathbb{P}_{\boldsymbol{X}^+_k}}} \cdot \mathsf{\tilde{AUC}}(f, \boldsymbol{X}_k^+, \boldsymbol{X}^-),
\end{aligned}
\end{equation}
where $f^{+}_{k, p}:= f(\boldsymbol{X}_k)_{p, +}$ denotes the prediction for the $p$-th salient pixel in the separable image part $\boldsymbol{X}_k$, i.e., $$f(\boldsymbol{X}_k)_{p, +} \in \{f(\boldsymbol{X}_k)_{h, w}|Y_{h, w}=1, h \in [H_k], w \in [W_k]\};$$
$S_k^{+}$ represents the number of salient pixels involved in $\boldsymbol{X}_k$, and $\boldsymbol{X}_k^+$ denotes the salient region within $\boldsymbol{X}_k$.
\end{proposition}

 Specifically, the third equation in Prop.\ref{pami:prop3.2} holds because $\AUC$ evaluates the ranking performance of all salient pixels versus non-salient pixels in pairs, which can be decomposed into a sum of AUC values for each salient pixel compared to all negative pixels. Based on Eq.(\ref{pami:eq11}), we observe that $\AUC$ remains \textbf{size-sensitive} to variations in object sizes, as it includes a same weight term ${\color{blue}{\mathbb{P}_{\boldsymbol{X}^+_k}} = S_k^+/S^+}$ as $\F$-measure. This introduces a similar bias as the previously discussed metrics. Furthermore, by comparing \({\color{blue}{\mathbb{P}_{\boldsymbol{X}_k}}}\) with \({\color{blue}{\mathbb{P}_{\boldsymbol{X}^+_k}}}\), our above derivations highlight the effectiveness of composite metrics over separable ones to some extent, since ${\color{blue}{\mathbb{P}_{\boldsymbol{X}^+_k}}}$ places much emphasis on the salient parts of the image.

\subsection{A General Principle for Size-Invariant Evaluations} \label{principles of SI_eval}

According to our above analysis, we argue that the critical factor stems from the size-sensitive term $\lambda(\boldsymbol{X}_k)$ (i.e., $\color{blue}{\mathbb{P}_{X_i}}$ or $\color{blue}{\mathbb{P}'_{X_i}}$) introduced in current metrics. To address this, a principal way to achieve size-invariant evaluation includes two key steps: \textbf{(S1)} balancing the effect of the weighting term $\lambda(\boldsymbol{X}_k)$ across various objects and \textbf{(S2)} developing an appropriate function $\mathcal{C}(\boldsymbol{X})$ to divide the image $\boldsymbol{X}$ into $|\mathcal{C}(\boldsymbol{X})|$ parts. 

In terms of \textbf{(S1)}, we propose a simple yet effective size-invariant protocol to get rid of the size-sensitive weight, i.e., independently evaluate the performance of each part by fixing {\color{orange} $\lambda(\boldsymbol{X}_k) \equiv 1$}. To this end, we arrive at the following size-invariant functions:
\begin{equation}
\label{eq:separable-insensitive}
    v_{\SI}(f)=\frac{1}{|\mathcal{C}(\boldsymbol{X})|}\sum_{k=1}^{|\mathcal{C}(\boldsymbol{X})|} {\color{orange}{1}} \cdot v(f_k),
\end{equation}
\begin{equation}
\label{eq:SI-composite}
    V_{\SI}(f)=\frac{1}{|\mathcal{C}(\boldsymbol{X})|}\sum_{k=1}^{|\mathcal{C}(\boldsymbol{X})|} {\color{orange}{1}} \cdot \frac{(a_1v_1(f_1)+\cdots+a_Tv_T(f_T))}{(b_1v_1(f_1)+\cdots+b_Tv_T(f_T))}, 
\end{equation}
where $\frac{1}{|\mathcal{C}(\boldsymbol{X})|}$ is introduced to ensure its value belonging to $[0,1]$. 

In terms of \textbf{(S2)}, it is important to consider that the primary goal of SOD tasks is to accurately and comprehensively detect salient objects. Therefore, borrowing the idea of other popular computer vision tasks, such as object detection and segmentation \cite{FasterRCNN, YOLO,zhang2020causal}, we propose separating targets of varying sizes within a single image into a series of foreground frames based on their minimum bounding boxes. Additionally, pixels not included in any bounding box are treated as part of the background. 

Formally, assume that there are $M$ salient objects in an image $\boldsymbol{X}$ and let $C_k=\{(a_j,b_j)\}_{j=1}^{M_k}, k\in [M]$ be the coordinate set for the object $k$, where $M_k$ is the number of pixels of the object $k$. Then, the minimum rectangle bounding box for object $k$ could be determined clockwise by the following vertex coordinates: 
\begin{equation} \label{eq:bounding_box}
\begin{aligned}
    \boldsymbol{X}_k^{fore}= & \{(a_k^{min}, b_k^{max}), (a_k^{max}, b_k^{max}), \\ 
    & (a_k^{max}, b_k^{min}), (a_k^{min},b_k^{min})
    \},
\end{aligned}
\end{equation}
where $a_k^{min}, a_k^{max}, b_k^{min}, b_k^{max}$ are the minimum and maximum coordinates in $C_k$, respectively.

\begin{figure}[!t]
\centering
\subfigure[Single-object scenario]{
\begin{minipage}{0.465\linewidth}
\includegraphics[width=\linewidth]{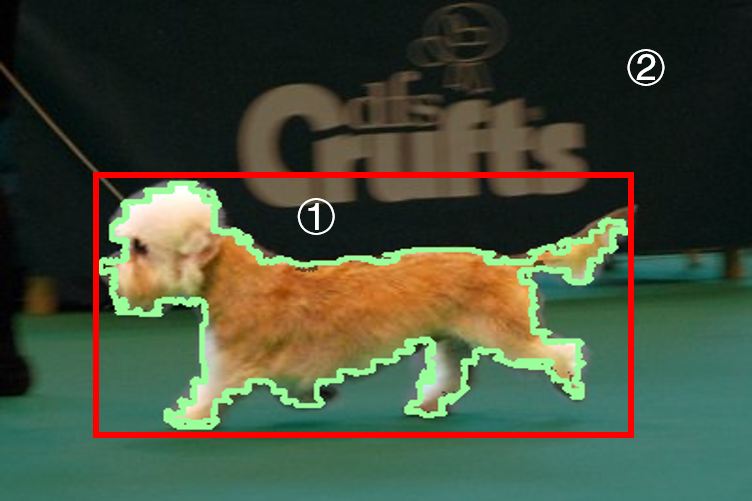}  
\label{fig:object_frame1}
\end{minipage}
}
\subfigure[Multi-object scenario]{
\begin{minipage}{0.47\linewidth}
\includegraphics[width=\linewidth]{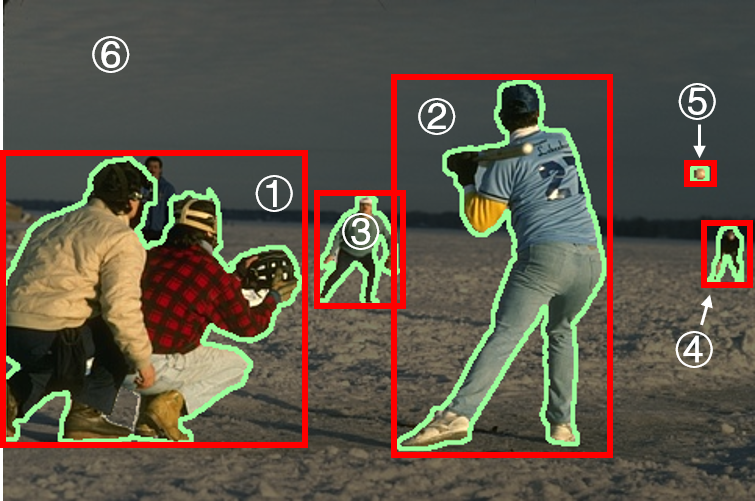}  
\label{fig:object_frame2}
\end{minipage}
}
\caption{Examples of partitions. In \cref{fig:object_frame1}, there is a foreground frame \ding{192} and a background frame \ding{193}. In \cref{fig:object_frame2}, there are five foreground frames from \ding{192} to \ding{196}, and a background frame \ding{197}.} 
\label{fig:eg_frame} 
\end{figure}

\begin{figure}[t]
\centering
\includegraphics[width=0.95\linewidth]{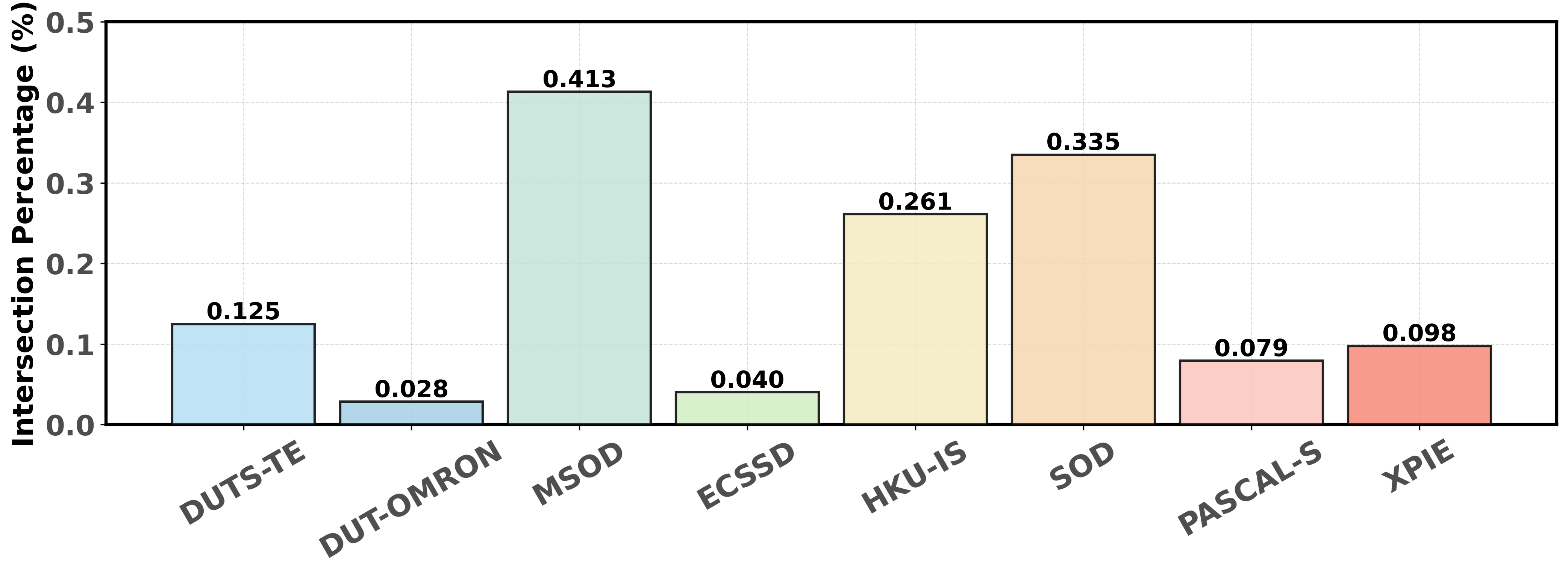}  
\caption{Statistics of the average intersection percentage across real-world SOD datasets, where the intersection percentage for each image is defined as the ratio of overlapping pixels to the total image size ($384 \times 384$). Detailed descriptions of these datasets can be found in \cref{dataset_appendix}.} 
\label{fig:intersect} 
\end{figure}

Correspondingly, the background frame is defined as follows:
\begin{equation}
    \boldsymbol{X}^{back}_{M+1} = \boldsymbol{X} \setminus \boldsymbol{F},
\end{equation}
where 
\begin{equation}
    \boldsymbol{F} = \bigcup_{k=1}^M \boldsymbol{X}_k^{fore}
\end{equation}
is the union of all minimum bounding boxes for salient objects. In this context, $\boldsymbol{Y}_k$ can be obtained by extracting pixels from the same position as $\boldsymbol{X}_k$, and we thus have $|\mathcal{C}(\boldsymbol{X})| = M + 1$.

However, as a fundamental task within the object detection domain, SOD primarily focuses on identifying the presence and location of salient objects, rather than delineating the precise boundaries of each individual instance. Consequently, most practical SOD datasets lack instance-level labels and typically provide only binary annotations (i.e., foreground or background, labeled as 0 or 1), making it infeasible to explicitly supervise individual object instances during training. To address this limitation, we treat \textbf{each connected region formed by salient foreground pixels} in the saliency map as an independent object proxy, denoted as $C_k$. Specifically, for a given input image $\boldsymbol{X}$, we perform connected component analysis over its ground-truth saliency map, among which each contiguous foreground blob is then enclosed within a bounding box and considered as a proxy object. Note that, all proxy sets $\mathcal{C}(\boldsymbol{X})$s are \ul{carried out offline as a preprocessing step prior to training}, thereby introducing no additional computational overhead during model optimization. A toy example of partitions is presented in \cref{fig:eg_frame}, where two squatting people are considered as a whole $C_k$ to obtain a single bounding box $\textcircled{1}$ in \cref{fig:eg_frame}-(b). Moreover, as observed in regions \textcircled{2} and \textcircled{3} of \cref{fig:eg_frame}-(b), this connected-component strategy may occasionally suffer from overlaps between different \( C_k \), due to the complexity of real-world SOD scenarios. Fortunately, as evidenced in \cref{fig:intersect}, such overlaps are extremely rare in practice, with the average intersection ratio per image consistently remaining below $1\%$. Further implementation details regarding the connected component extraction can be found in \cref{details_appendix}, and alternative instantiations of \( \mathcal{C}(\boldsymbol{X}) \) are discussed in \cref{20250408sec6.3}.

\noindent\textbf{Clarifications.} The bounding boxes introduced in this work are similar to those commonly employed in object detection tasks~\cite{Few-Shot-Object-Detection, Object-Detection-in-Aerial-Images}. However, the underlying motivation differs significantly. In object detection, ground-truth bounding boxes are used as supervision signals to guide the regression of predicted boxes. In contrast, our approach utilizes bounding boxes solely to partition images into regions of varying sizes. This is an intermediate step ensuring all objects, regardless of size, are considered equally during evaluations. Moreover, our proposed framework is flexible and can accommodate diverse segmentation strategies, as demonstrated in \cref{20250408sec6.3} of the experiments.

\subsection{Instantiations of Size-invariant Evaluations \label{20250404Sec3.3}}

In this section, we will adopt widely used metrics, i.e., $\MAE$, $\F$-score and $\AUC$, to instantiate our size-invariant principles. Note that our proposed strategy could also be applied to other metrics as mentioned in \cref{Revisting}.

\subsubsection{Size-Invariant Separable Metrics\label{SI-MAE}} 
\noindent\textit{\textbf{Size-Invariant $\MAE$ ($\SMAE$). }}Following the principle of Sec.\ref{principles of SI_eval}, the formulation of $\SMAE$ is defined as follows:
\begin{equation}
\label{naive Size-invariant}
    \SMAE(f)=\frac{1}{M + 1} \sum_{k=1}^{M + 1} {\color{orange}{1}} \cdot \MAE(f_k),
\end{equation}
where again we use $\MAE(f_k) := \MAE(f(\boldsymbol{X}_k), \boldsymbol{Y}_k)$ for short and $|\mathcal{C}(\boldsymbol{X})| = M + 1$ in this context. 

Moreover, we recognize that the size of the background, $\boldsymbol{X}_{M+1}^{back}$, may be significantly larger than that of the foreground, $\boldsymbol{X}_{k}^{fore}, k \in [M]$ in some cases, leading to an unnoticeable bias in favor of the background during evaluation. To alleviate this issue, we introduce a reconciliation term $\alpha$ to balance the model attention adaptively, and then the expression for $\SMAE$ becomes:
\begin{equation}
\label{wMAE_def}
\begin{aligned}
    \SMAE(f)&= \\
    \frac{1}{M+\alpha} & \left[\sum_{k=1}^{M} \MAE(f_k^{fore}) + \alpha \MAE(f_{M+1}^{back})\right],
\end{aligned}
\end{equation}
where the parameter $\alpha=S_{M+1}^{back}/\sum_{k=1}^{M}S_k^{fore}$ is the size ratio of the background and the sum of all foreground frames. By doing so, the predictor could pay equal consideration to salient objects of various sizes and impose an appropriate penalty for misclassifications in the background. We demonstrate its effectiveness in reducing the model's $\FP$ in \cref{Ablation Studies}.  

Taking a step further, we can examine the advantages of our proposed $\SMAE$ against $\MAE$ by the following simplified case with proof in \cref{prop1_proof}.

\begin{proposition}(\textbf{Size-Invariant Property of $\SMAE$.})\label{prop1} Without loss of generality, given two well-trained predictors, \( f_A \) and \( f_B \), with different parameters, \( \SMAE \) is always \textbf{more effective} than \( \MAE \) during evaluation, even when object sizes are imbalanced. 

{{\textbf{Case 1: Single Salient Object (\( M = 1 \)).}}} When there is only one salient object in the image \( \boldsymbol{X} \), no size imbalance exists among objects. In this case, \( \SMAE \) is equivalent to \( \MAE \).

{{\textbf{Case 2: Multiple Salient Objects (\( M \geq 2 \)).}}} Suppose the image \( \boldsymbol{X} \) contains multiple salient objects, represented by the ground-truth salient pixel sets \( \{C_1, C_2, \dots, C_M\} \) with \( S_1^{{fore}} \leq S_2^{fore} \leq \dots \leq S_M^{{fore}} \). Assume the two detection models, \( f_A \) and \( f_B \), identify the same total number of salient pixels, given by \( \sum_{i=m+1}^{M} |C_i| \) for some \( m \in \{1, 2, \dots, M-1\} \). Additionally, suppose \( f_A \) perfectly detects only the larger objects \( \{C_{m+1}, \dots, C_M\} \), whereas $f_B$ detects all $\{C_1,\dots C_M\}$ but only partially. In general, \underline{\( f_B \) is considered superior to \( f_A \)} since the latter entirely fails to detect the smaller objects $\{C_1,\cdots,C_m\}$. Under this setting, we observe that \( \SMAE(f_A) > \SMAE(f_B) \) but \( \MAE(f_A) = \MAE(f_B) \).

\end{proposition}

\noindent\textbf{Remark.} \cref{fig:eg} provides a toy illustration of \textbf{Case 2}. Note that for MAE-based metrics (including \( \MAE \) and \( \SMAE \)), lower values indicate better performance. However, the above proposition highlights that the traditional \( \MAE \) measure is highly sensitive to object size, allowing the performance on larger objects to dominate the evaluation process. In contrast, \( \SMAE \) effectively mitigates this issue, providing a more fair assessment.

\subsubsection{Size-Invariant Composite Metrics} \label{SI-F}
Here we instantiate our size-invariant principle with common composite metrics, including $\F$-measure and $\AUC$.

\noindent\textit{\textbf{Size-Invariant $\F$-measure ($\SF$). }} Following the idea of \cref{principles of SI_eval}, we define the $\SF$ as follows:

\begin{equation}
    \label{wF_def}
    \SF(f)=\frac{1}{M}\sum_{k=1}^{M}\F(f_k^{fore}),
\end{equation}
where $\F(f_k^{fore})$ denotes the $F$-score of the prediction $f(\boldsymbol{X}^{fore}_k)$ with its corresponding ground-truth $\boldsymbol{Y}^{fore}_k$. Similarly, we can also make an informal analysis of the size sensitivity of $\SF$ following the same idea as in \cref{prop1}. For details, please refer to \cref{20250421SecB.2}.





\noindent\textit{\textbf{Size-Invariant $\AUC$ ($\SAUC$).}} In terms of another common composite metric $\AUC$, we similarly define $\SAUC$ as follows:
\begin{equation}
    \label{wAUC_def}
    \SAUC(f)=\frac{1}{M}\sum_{k=1}^{M}\AUC(f_k^{fore}),
\end{equation}
where we have
\[
\AUC(f_k^{fore}) := \mathsf{\tilde{AUC}}(f, \boldsymbol{X}_k^{fore, +}, \boldsymbol{X}^-),
\]
and $\boldsymbol{X}_k^{fore, +}$ represents the salient region in the $k$-th part $\boldsymbol{X}_k$, $\boldsymbol{X}^-$ denotes all non-salient region in the whole image $\boldsymbol{X}$. To this end, we can see that each salient region in the image is paired with all other negative pixels to calculate $\AUC$ independently, thus alleviating the imbalance effect of varying sizes.

\noindent\textbf{Brief Discussion of Size Sensitivity in $\E_m$ and $\Sm_m$.}  
Given that $\E_m$ and $\Sm_m$ are neither pixel-wise nor region-wise independent, the framework introduced in \cref{Revisting} may not be directly applicable. Nevertheless, as detailed in \cref{20250420SecA.2}, a closer examination of the computational formulations reveals indicative evidence of inherent size sensitivity therein. Fortunately, compared to other metrics such as $\MAE$ and $\AUC$, the size-sensitivity issues in $\E_m$ and $\Sm_m$ are relatively less pronounced, as both incorporate globally multiple cues to yield a more holistic assessment. Moreover, we emphasize that even when practitioners aim to achieve high performance on $\E_m$ and $\Sm_m$, \textbf{they typically adopt common loss functions (e.g., $\BCE$, $\AUC$, and $\mathsf{Dice}$) introduced in \cref{20250404Sec4.2} to supervise model learning}. In other words, no well-established guidance is currently available for directly optimizing a model toward improved $\E_m$ and $\Sm_m$ scores. Consequently, models may still exhibit a bias toward larger objects during optimization, thereby reinforcing the necessity of our proposed $\mathsf{SIOpt}$. Accordingly, as demonstrated in \cref{Experiments}, $\mathsf{SIOpt}$ consistently improves performance on both $\E_m$ and $\Sm_m$ in most cases. This stems from the model's enhanced ability to adequately learn each object without suppression from others, leading to a more balanced and comprehensive representation of salient regions.

\section{Size-Invariant SOD Optimization} \label{SI-SOD}
In previous sections, we have outlined how to achieve size-invariant evaluation for SOD tasks. Now, in this section, we elaborate on how to optimize these size-invariant metrics to promote practical SOD performance directly.
\subsection{A Generic Principle for Size-Invariant Optimization \label{20250408SIOpt}}
Guided by the analyses in Sec.\ref{Revisting}, it becomes evident that \textbf{existing optimization strategies for the SOD model also exhibit size variance} when multiple objects are present in an image. Similarly, the primary reason is that the pixel-level classification loss function is applied across the entire image, allowing larger objects to dominate the learning process.

Therefore, the core idea of the solution is to ensure that \uline{the model independently considers the pixel-level classification error of each object} and learns to minimize these errors equally. Inspired by the principles outlined in Sec.\ref{principles of SI_eval}, we propose the following optimization paradigm to achieve size-invariant SOD in practice:
\begin{equation}
\label{eq:loss}
\begin{aligned}
\mathcal{L}_{\SI}(f, \boldsymbol{X}, \boldsymbol{Y})= \frac{1}{M+\alpha_{SI}} \left[\sum_{k=1}^{M} L(f^{fore}_k)  + \alpha_{SI} L(f^{back}_{M+1})\right],
\end{aligned}
\end{equation}
where $L(f^{fore}_k)$ and $L(f^{back}_{M+1})$ represent the loss value of the specific $k$-th foreground and background parts, respectively; $L(\cdot)$ could be any popular loss used in the SOD community and $\alpha_{SI}$ is a self-induced parameter to determine the contribution of the remaining background. For simplicity, we focus here on a single image pair \((\boldsymbol{X}, \boldsymbol{Y})\) as an example, while the value of \(\mathcal{L}_{\SI}(f, \boldsymbol{X}, \boldsymbol{Y})\) is computed over the entire dataset \(\mathcal{D}\) during training.

 Based on Eq.(\ref{eq:loss}), in what follows, we will present how to instantiate size-invariant optimizations (SIOpt) with different SOD goals. 

\subsection{Instantiations of Size-invariant Optimizations \label{20250404Sec4.2}}
\subsubsection{SIOpt with Classification-aware Losses} \label{20250411:Sec.4.2.1}
In conventional SOD tasks, the most commonly used loss function is \(\BCE\) \cite{EDN, ICON,DBLP:journals/pami/TangSQLWYJ17}, which formulates SOD as a \textbf{pixel-level binary classification} problem aimed at maximizing the posterior probability. In this context, we introduce the following size-invariant version:
\begin{equation}\label{319eq:22}
\begin{aligned}
\mathcal{L}_{\SI\BCE}(f)= \mathop{\hat{\mathbb{E}}}_{(\boldsymbol{X},\boldsymbol{Y}) \sim \mathcal{D}}&\bigg[\sum_{k=1}^{M} \frac{\ell_{\BCE}(f(\boldsymbol{X}^{fore}_k), \boldsymbol{Y}^{fore}_k)}{M+\alpha_{SI}} \\ +&  \frac{\alpha_{SI}\ell_{\BCE}(f(\boldsymbol{X}^{back}_{M+1}), \boldsymbol{Y}^{back}_{M+1})}{M+\alpha_{SI}}\bigg],
\end{aligned}
\end{equation}
where we have
\begin{equation}\nonumber
\begin{aligned}
    & L(f^{fore}_k) := \ell_{\BCE}(f(\boldsymbol{X}^{fore}_k), \boldsymbol{Y}^{fore}_k) =  \\
    & \frac{-\sum_{h=1}^{H_k^{fore}} \sum_{w=1}^{W_k^{fore}} Y_{h,w} \log \hat{P}_{h,w} + (1 - Y_{h,w}) \log (1 - \hat{P}_{h,w})}{S_k^{fore}},
\end{aligned}    
\end{equation}
$\hat{P}_{h,w} := f(\boldsymbol{X})_{h, w}$ represents the predicted value for the pixel located at the $h$-th row and $w$-th column, and $Y_{h,w}$ denotes the corresponding ground truth. $S_k^{fore} = H_k^{fore} \times W_k^{fore}$ represents the total number of pixels in the $k$-th foreground region, where $H_k^{fore}$ and $W_k^{fore}$ are its height and width, respectively. The weight factor $\alpha_{SI}=\frac{S^{back}_{M+1}}{\sum_{k=1}^{M}S_k^{fore}}$ follows a similar formulation as in Eq.(\ref{wMAE_def}). Additionally, $L(f^{back}_{M+1}):=\ell_{\BCE}(f(\boldsymbol{X}^{back}_{M+1}), \boldsymbol{Y}^{back}_{M+1})$ is defined analogously to $\ell_{\BCE}(f(\boldsymbol{X}^{fore}_k), \boldsymbol{Y}^{fore}_k)$. 

Similar to Eq.(\ref{319eq:22}), another pixel-level loss, $\MSE$ \cite{GateNet, LDF, PoolNet} that emphasizes the smoothness of overall prediction, can also be directly reformulated into a size-invariant form, and thus is omitted here. To show the advantage of our proposed SIOpt, we have the following simple proposition:

\begin{proposition}(\textbf{Mechanism of SIOpt}).\label{re-attention_prop} Without loss of generality, we merely consider one sample \((\boldsymbol{X}, \boldsymbol{Y})\) here. Given a separable loss function \( L(f) \), such as $L(f):= \ell_{\BCE}(f(\boldsymbol{X}), \boldsymbol{Y})$ and its corresponding size-invariant loss \( \mathcal{L}_{\SI}(f, \boldsymbol{X}, \boldsymbol{Y}) \) as defined in Eq.(\ref{eq:loss}), the following properties hold:  

\begin{enumerate} 
    \item[(\color{blue}{\textbf{P1}})] The weight assigned by the size-invariant loss satisfies \( \lambda_{\mathcal{L}_{\SI}}(\boldsymbol{X}_k^{fore}) > \lambda_{L}(\boldsymbol{X}_k^{fore}) \) if \( S_k^{fore} < \frac{S}{M + \alpha_{SI}} \), for all \( k \in [M] \).  
    \item[(\color{blue}{\textbf{P2}})] The weight assigned by the size-invariant loss follows an inverse relationship with region size, i.e., \( \lambda_{\mathcal{L}_{\SI}}(\boldsymbol{X}_{k_1}^{fore}) > \lambda_{\mathcal{L}_{\SI}}(\boldsymbol{X}_{k_2}^{fore}) \) if \( S_{k_1}^{fore} < S_{k_2}^{fore} \), for all \( k_1, k_2 \in [M] \), with \( k_1 \neq k_2 \).  
\end{enumerate}  

For simplicity, \( \lambda_{\mathcal{L}_{\SI}}(\boldsymbol{X}_k^{fore}) \) here represents the \textbf{weight} assigned to each pixel in \( \boldsymbol{X}_k^{fore} \) by the size-invariant loss \( \mathcal{L}_{\SI}(f, \boldsymbol{X}, \boldsymbol{Y}) \), whereas \( \lambda_{L}(\boldsymbol{X}_k^{fore}) \) denotes the \textbf{weight} assigned by the original size-variant loss \( L(f) \). Note that pixels within the same region \( \boldsymbol{X}_k^{fore} \) are treated equally in both \( \mathcal{L}_{\SI}(f, \boldsymbol{X}, \boldsymbol{Y}) \) and \( L(f) \).
\end{proposition}

\noindent\textbf{Remark.} \label{re-attention_remark} The proof of this proposition is straightforward and is provided in Appendix.\ref{Proof_of_4.2.1}. Specifically, (\textcolor{blue}{\textbf{P1}}) suggests that if the object is small below the average size (i.e., \( S_{k}^{fore} < S/(M + \alpha_{SI}) \le S/M \) with \( \alpha_{SI} \ge 0 \)) the model optimized by SIOpt will place greater emphasis on those objects to balance the learning process. In contrast, the pixel's weight assigned by the original size-variant loss \( L(f) \) remains a constant, i.e., \( \lambda_{L}(\boldsymbol{X}^{fore}_k) = 1/S \). Furthermore, (\textcolor{blue}{\textbf{P2}}) demonstrates that SIOpt increases the weight of pixels in smaller salient objects, ultimately mitigating size sensitivity.

\subsubsection{SIOpt with Region-aware Losses}
Unlike the pixel-wise losses discussed earlier (e.g., \(\BCE\) and \(\MSE\)), region-aware (RegA) losses operate at the object level, promoting structural coherence, regional consistency, and a clearer distinction between salient objects and background regions. As investigated in Sec.\ref{20250325sec2.2}, representative examples include Dice Loss~\cite{DiceLoss}, and IoU Loss~\cite{IOULoss}, defined as follows:
\begin{equation}
    \ell_{\mathsf{Dice}}(f(\boldsymbol{X}), \boldsymbol{Y})=1-\frac{2\cdot \sum_{h=1}^{H} \sum_{w=1}^{W} \hat{P}_{h,w} \cdot Y_{h,w}}{\sum_{h=1}^{H} \sum_{w=1}^{W} \hat{P}_{h,w}+ Y_{h,w}},
\end{equation}
and 
\begin{equation}
\begin{aligned}
\ell_{\mathsf{IOU}}(f(\boldsymbol{X}),& \boldsymbol{Y})= \\
1-&\frac{\sum_{h=1}^{H} \sum_{w=1}^{W} (\hat{P}_{h,w} \cdot Y_{h,w})}{\sum_{h=1}^{H} \sum_{w=1}^{W} \hat{P}_{h,w}+Y_{h,w} - \hat{P}_{h,w} \cdot Y_{h,w}}. 
\end{aligned}
\end{equation}

Accordingly, we propose the following size-invariant framework to address the size-variant challenge of RegA:
\begin{equation}\label{20250323eq25}
    \mathcal{L}_{\SI\mathsf{Dice}}(f)= \mathop{\hat{\mathbb{E}}}_{(\boldsymbol{X},\boldsymbol{Y}) \sim \mathcal{D}}\bigg[\frac{1}{M}\sum_{k=1}^{M}\ell_{\mathsf{Dice}}(f(\boldsymbol{X}^{fore}_k), \boldsymbol{Y}^{fore}_k)\bigg],
\end{equation}
and 
\begin{equation}\label{20250323eq25}
    \mathcal{L}_{\SI\mathsf{IOU}}(f)= \mathop{\hat{\mathbb{E}}}_{(\boldsymbol{X},\boldsymbol{Y}) \sim \mathcal{D}}\bigg[\frac{1}{M}\sum_{k=1}^{M}\ell_{\mathsf{IOU}}(f(\boldsymbol{X}^{fore}_k), \boldsymbol{Y}^{fore}_k)\bigg],
\end{equation}
where we set $\alpha_{SI} \equiv0$ in Eq.(\ref{eq:loss}) since $Y_{h,w}$ is always be $0$ within $(\boldsymbol{X}_{M + 1}^{back}, \boldsymbol{Y}_{M + 1}^{back})$. At first glance, this size-invariant RegA framework may appear to neglect the remaining background $(\boldsymbol{X}_{M+1}^{back}, \boldsymbol{Y}_{M+1}^{back})$ performance. However, it is crucial to emphasize that RegA primarily targets the performance of salient pixels, rather than focusing on global background accuracy. By enforcing size invariance, this optimization strategy enhances object boundary refinement and preserves the global structure of salient objects, ultimately improving model performance concerning region-level metrics (such as $\F$-measure). Furthermore, as detailed in subsequent sections (Sec.\ref{20250324Sec4.3}), RegA losses are often combined with other loss (such as $\BCE$) functions during training to facilitate region-level alignment between predictions and ground truth, where the background $(\boldsymbol{X}_{M+1}^{back}, \boldsymbol{Y}_{M+1}^{back})$ is explicitly used for BCE optimization. 

\subsubsection{SIOpt with Ranking-aware Losses}\label{20250419Sec4.2.3}
Another promising direction is to incorporate ranking-aware (RanA) loss functions \cite{DBLP:journals/tmm/WuHLX24}, such as AUC-based losses \cite{DBLP:journals/tnn/GultekinSRP20,DBLP:journals/pami/0001L00Z24}, which offer \textbf{several key advantages} over traditional pixel-wise losses (e.g., $\BCE$) in SOD tasks. First, AUC optimization prioritizes the relative ranking between salient objects (i.e., positive pixels) and background regions (i.e., negative ones) rather than optimizing absolute pixel values. This encourages the model to establish a well-structured global ranking, \uline{improving object boundary delineation}. Furthermore, given that non-salient pixels typically outnumber salient ones by a large margin, conventional methods often struggle to detect smaller or less frequent objects. In contrast, AUC optimization is \uline{insensitive to data distribution} \cite{DBLP:journals/csur/YangY23,DBLP:conf/nips/CortesM03} and has demonstrated remarkable success in instance-level long-tail classification. Thus, applying AUC-based ranking optimization to pixel-level SOD tasks can therefore significantly enhance the model’s robustness in handling complex scenes with severe class imbalance, improving both detection performance and generalization \cite{AUC_rank, AUC_eval}.

Following existing AUC-based approaches \cite{DBLP:journals/csur/YangY23}, we optimize the following surrogate loss to maximize the pixel-wise AUC value in Eq.~(\ref{pami:eq10}):  
\begin{equation}
 \ell_{\AUC}(f, \boldsymbol{X}^+, \boldsymbol{X}^-) = \sum_{p=1}^{S^{+}} \sum_{q=1}^{S^{-}} \frac{\ell_{sq}(f^{+}_{p}-f^{-}_{q})}{S^{+} S^{-}},    
\end{equation}  
where $f^{+}_{p}$ and $f^{-}_{q}$ stay the same as Eq.(\ref{pami:eq9}) and 
$\ell_{sq}(z) = (1-z)^2$ is the square surrogate loss designed to address the non-differentiability of the original indicator function $\mathbb{I}(\cdot)$. We refer interested readers to \cite{DBLP:journals/csur/YangY23, DBLP:journals/tnn/GultekinSRP20} for a more comprehensive discussion on AUC optimization and its theoretical foundations.  

Subsequently, to tackle the size-variant issue of traditional AUC optimizations, we propose the following size-invariant approach to pursue $\SAUC$:
 \begin{equation}\label{20250323eq28}
    \mathcal{L}_{\SI\AUC}(f)= \mathop{\hat{\mathbb{E}}}_{(\boldsymbol{X},\boldsymbol{Y}) \sim \mathcal{D}}\bigg[\underbrace{\frac{1}{M}\sum_{k=1}^{M} \ell_{\SI\AUC}(f, \boldsymbol{X}_k^{fore, +}, \boldsymbol{X}^-)}_{\textcolor{orange}{(\textbf{OP}_0)}}\bigg],
\end{equation}
where 
\begin{equation}
     \ell_{\SI\AUC}(f, \boldsymbol{X}_k^{fore, +}, \boldsymbol{X}^-) = \sum_{p=1}^{S^{fore,+}_k} \sum_{q=1}^{S^{-}} \frac{\ell_{sq}(f^{fore,+}_{k,p}-f^{-}_{q})}{S^{fore,+}_k S^{-}}. 
\end{equation}
Notably, since each \(\boldsymbol{X}_{k}^{fore}\) is paired with all negative pixels in the image to compute \(\SAUC\), only \(M\) foreground parts are required in this context compared to Eq.(\ref{319eq:22}).

 However, directly optimizing Eq.(\ref{20250323eq28}) for SIOpt poses a significant challenge. Specifically, each positive pixel must be paired with all negative pixels to compute \( \ell_{\SI\AUC}(f, \boldsymbol{X}_k^{fore, +}, \boldsymbol{X}^-) \), resulting in an almost \(\mathcal{O}(S^2)\) complexity for each image \( \boldsymbol{X} \). For instance, in the widely used DUTS dataset, where \( \boldsymbol{X} \in \mathbb{R}^{3\times 384 \times 384 } \) and \( \boldsymbol{Y} \in \mathbb{R}^{384 \times 384} \), Eq.(\ref{20250323eq28}) would entail an overwhelming computational complexity of approximately \(\mathcal{O}(10^{10})\) per image—making it impractical for stochastic optimization. 

\noindent\textbf{Pixel-level Bipartite Acceleration (PBAcc).} To address the computational challenge, we further propose the PBAcc strategy, which significantly reduces the complexity from \(\mathcal{O}(S^2)\) to \(\mathcal{O}(S)\) per image. The core idea of PBAcc is to modify the pairing process between positive and negative pixels through a more efficient bipartite formulation, effectively decoupling exhaustive pairwise comparisons while preserving the size-invariant ranking-aware properties of \(\ell_{\SI\AUC}(f, \boldsymbol{X}_k^{fore, +}, \boldsymbol{X}^-)\).   

Specifically, let $\mathcal{G}^{(\boldsymbol{X})}=(\mathcal{N}^{(\boldsymbol{X})},\mathcal{E}^{(\boldsymbol{X})},\mathcal{A}^{(\boldsymbol{X})})$ be a bipartite graph constructed from image ${\boldsymbol{X}}$, where the vertex set $\mathcal{N}^{(\boldsymbol{X})}=\{(x_p,x_q)|x_p, x_q \in \boldsymbol{X}\}$ is the set of all pixel-pixel pairs in image $\boldsymbol{X}$, $\mathcal{E}^{(\boldsymbol{X})}=\{(p,q)|y_p \neq y_q, y_p, y_q \in \boldsymbol{Y}\}$ is the edge set and $\mathcal{A}^{(\boldsymbol{X})} \in \mathbb{R}^{S \times S}$ denotes the the adjacent matrix. For any pixel pair $(p, q) \in \mathcal{E}^{(\boldsymbol{X})}$, if there exists an index $k\in[M]$ such that either $x_p \in \boldsymbol{X}_k^{fore,+}$ or $x_q \in \boldsymbol{X}_k^{fore,+}$, we define
\[ 
 \mathcal{A}^{(\boldsymbol{X})}_{p,q}=
\frac{1}{S_k^{fore, +}S^-},
\]
otherwise $\mathcal{A}^{(\boldsymbol{X})}_{p,q}= 0$. Intuitively, $\mathcal{G}^{(\boldsymbol{X})}$ encodes the relationship that if a pixel pair $(x_p, x_q)$ belongs to different classes (i.e., one is salient while the other is not), their edge is assigned a weight of $\frac{1}{S_k^{fore, +}S^-}$. Notably, each positive pixel is exclusively attributed to a unique $\boldsymbol{X}_k^{fore,+}$.  

More generally, let $\bar{\boldsymbol{y}} \in \mathbb{R}^S$ be the flattened vector of $\boldsymbol{Y}$, then the adjacent matrix $\mathcal{A}^{(\boldsymbol{X})}$ could be expressed as
\begin{equation}
    \mathcal{A}^{(\boldsymbol{X})} = \frac{1}{S^-} \left[\tilde{\boldsymbol{y}}(\boldsymbol{1} - \bar{\boldsymbol{y}})^\top + (\boldsymbol{1} - \bar{\boldsymbol{y}})\tilde{\boldsymbol{y}}^\top \right].
\end{equation}
where $\tilde{\boldsymbol{y}} := \bar{\boldsymbol{y}} \odot \boldsymbol{c} \in \mathbb{R}^S$ denotes an element-wise weighted version of $\bar{\boldsymbol{y}}$, and $\odot$ represents the Hadamard (element-wise) product. The coefficient vector $\boldsymbol{c} = [c_1, c_2, \dots, c_S] \in \mathbb{R}^{S}$ is defined such that $c_p=1/S^{fore,+}_k$ if $x_p \in \boldsymbol{X}_k^{fore,+}$, otherwise $c_p=0$. 

Correspondingly, let $\mathcal{P}^{(\boldsymbol{X})}$ be the Laplacian matrix of $\mathcal{G}^{(\boldsymbol{X})}$, the following equation holds:
\begin{equation}
\begin{aligned}
\mathcal{P}^{(\boldsymbol{X})} &= \text{diag}(\mathcal{A}^{(\boldsymbol{X})}\boldsymbol{1}) - \mathcal{A}^{(\boldsymbol{X})} \\
&= \text{diag}(\tilde{\boldsymbol{y}} + c^+(\boldsymbol{1} - \bar{\boldsymbol{y}})) - \mathcal{A}^{(\boldsymbol{X})},
\end{aligned}
\end{equation}
where $\boldsymbol{1} \in \mathbb{R}^S$ is an all-ones vector; $\text{diag}(\mathcal{A}^{(\boldsymbol{X})}\boldsymbol{1}) \in \mathbb{R}^{S \times S}$ represents the degree matrix of $\mathcal{G}^{(\boldsymbol{X})}$, where $\text{diag}(\cdot)_{i,i}$ corresponds to the degree of vertex $i$ in the graph; the scalar $c^+ := \tilde{\boldsymbol{y}}^\top\boldsymbol{1}$ denotes the sum of nonzero entries in $\tilde{\boldsymbol{y}}$.  

Through further derivation, it is easy to verify the following critical equivalence:
\begin{equation}\label{20250419eq32}
    \textcolor{orange}{(\textbf{OP}_0)} = \underbrace{(\bar{\boldsymbol{f}}^{(\boldsymbol{X})} - \bar{\boldsymbol{y}})^\top\mathcal{P}^{(\boldsymbol{X})}(\bar{\boldsymbol{f}}^{(\boldsymbol{X})} - \bar{\boldsymbol{y}})}_{\textcolor{orange}{(\textbf{OP}_1)}},
\end{equation}
where $\bar{\boldsymbol{f}}^{(\boldsymbol{X})} \in \mathbb{R}^{S}$ denotes the flattened vector of $f(\boldsymbol{X}) \in \mathbb{R}^{H \times W}$ with $S=H \times W$. Consequently, we present the following proposition, which guarantees the computational efficiency of $\textcolor{orange}{(\textbf{OP}_1)}$ over $\textcolor{orange}{(\textbf{OP}_0)}$:

\begin{proposition}(\textbf{Efficiency of $\textcolor{orange}{(\textbf{OP}_1)}$}). \label{pppp222} Given $\boldsymbol{Q}:=(\bar{\boldsymbol{f}}^{(\boldsymbol{X})} - \bar{\boldsymbol{y}}) \in \mathbb{R}^{S}$ and $\mathcal{P}^{(\boldsymbol{X})}$, the calculation of $\textcolor{orange}{(\textbf{OP}_1)}$ can be finished in nearly $\mathcal{O}(S)$ complexity due to
\begin{equation}
\begin{aligned}\nonumber
    \boldsymbol{Q}^\top\mathcal{P}^{(\boldsymbol{X})}\boldsymbol{Q}=&\boldsymbol{Q}^\top\left(\underbrace{\text{diag}(\tilde{\boldsymbol{y}} + c^+(1 - \bar{\boldsymbol{y}}))}_{\textcolor{blue}{\textcircled{1}}} -\mathcal{A}^{(\boldsymbol{X})}\right)\boldsymbol{Q} \\
\end{aligned}
\end{equation}
and
\begin{equation}
\begin{aligned}\label{20250827eq33}
    \boldsymbol{Q}^\top&\mathcal{A}^{(\boldsymbol{X})} \boldsymbol{Q} =
    \boldsymbol{Q}^\top \left(\frac{\tilde{\boldsymbol{y}}(1 - \bar{\boldsymbol{y}})^\top + (1 - \bar{\boldsymbol{y}})\tilde{\boldsymbol{y}}^\top}{S^-}\right)\boldsymbol{Q} \\
    =& \boldsymbol{Q}^\top \underbrace{\left(\frac{\tilde{\boldsymbol{y}}(1 - \bar{\boldsymbol{y}})^\top}{S^-}\right)}_{\textcolor{blue}{\textcircled{2}}} \boldsymbol{Q} + \boldsymbol{Q}^\top \underbrace{\left(\frac{(1 - \bar{\boldsymbol{y}})\tilde{\boldsymbol{y}}^\top}{S^-}\right)}_{\textcolor{blue}{\textcircled{3}}} \boldsymbol{Q}.
\end{aligned}
\end{equation}
\end{proposition}
   
\noindent\textbf{Remark.} Based on Eq.~(\ref{20250827eq33}), we observe that both $\textcolor{blue}{\textcircled{1}}$, $\textcolor{blue}{\textcircled{2}}$ and $\textcolor{blue}{\textcircled{3}}$ are flattened vectors of length $S$, and that the key component of $\textcolor{orange}{(\textbf{OP}_1)}$, denoted as $\boldsymbol{Q} := (\bar{\boldsymbol{f}}^{(\boldsymbol{X})} - \bar{\boldsymbol{y}})$, also has $S$ elements. Consequently, each term in the expression $\boldsymbol{Q}^\top\mathcal{P}^{(\boldsymbol{X})}\boldsymbol{Q}$ can be computed with a time complexity of approximately $\mathcal{O}(S)$ per image. This efficient computation indicates that the proposed PBAcc method significantly reduces the computational cost of the naive approach $\textcolor{orange}{(\textbf{OP}_0)}$ from $\mathcal{O}(S^2)$ to $\mathcal{O}(S)$, thereby making \textbf{SIOpt-RanA computationally feasible} for large-scale datasets without sacrificing performance. Furthermore, this formulation also implies that PBAcc achieves an almost $\mathcal{O}(S)$ space complexity per image, in contrast to the $\mathcal{O}(S^2)$ space requirement of $\textcolor{orange}{(\textbf{OP}_0)}$.

\noindent\textbf{Clarifications.} We acknowledge that previous studies \cite{AUCSeg, DBLP:conf/aistats/RosenfeldMTG14} have explored the application of AUC optimization in pixel-level tasks, such as semantic segmentation and structured prediction. However, our approach differs in two fundamental aspects. First, existing pixel-level AUC methods typically rely on cross-image comparisons between positive and negative pixels from different images, whereas ours does not require such contrasts due to the class-agnostic binary nature of SOD. More importantly, prior approaches primarily focus on addressing the imbalance between salient and non-salient pixels, while our optimization of Eq.~(\ref{20250323eq28}) explicitly accounts for two levels of imbalance: (i) between salient and non-salient pixels and (ii) among salient pixels of varying object sizes. To the best of our knowledge, this is the first work to tackle the size-variance challenge in pixel-level AUC optimization.

\vspace{-0.2cm}
\subsection{Final Size-invariant Optimization Framework \label{20250324Sec4.3}}
Given that different loss functions typically correspond to distinct evaluation metrics, existing methods often adopt a hybrid loss paradigm—integrating multiple losses to optimize practical SOD models simultaneously. Inspired by this, we follow the hybrid SIOpt objective:
\begin{equation}
    \mathcal{L}_{\mathsf{SIOpt}}(f) = \sum_{b=1}^B \gamma_b\cdot\mathcal{L}_{b}(f),
\end{equation}
where $B$ denotes the number of hybrid losses, $\mathcal{L}_{b}(f)$ represents each ($\mathsf{SIOpt}$) loss function, and $\gamma_b$ is the corresponding trade-off coefficient for balanced optimization.

The choice of loss functions plays a crucial role in determining downstream performance. In this paper, we explore two promising variants of \(\mathsf{SIOpt}\). The first variant, \(\mathsf{SIOpt1}\), directly modifies the original loss functions used for training SOD models into their size-invariant counterparts. For instance, given that \cite{EDN} utilizes BCE and Dice losses to train the EDN backbone, \(\mathsf{SIOpt1}\) instead leverages \(\mathcal{L}_{\SI\BCE}(f)\) and \(\mathcal{L}_{\SI\mathsf{Dice}}(f)\) to ensure size invariance, with a similar adaptation for other state-of-the-art SOD frameworks \cite{GateNet, PoolNet+}. The second variant, \(\mathsf{SIOpt2}\), incorporates our proposed \(\mathcal{L}_{\SI\AUC}(f)\) alongside other commonly used loss functions, such as classification- and region-aware losses, to further enhance performance. Detailed $\mathsf{SIOpt}$ implementations tailored to different backbones are provided in Appendix.\ref{details_appendix}.

\vspace{-0.2cm}

\section{Generalization Bound}




In this section, we proceed to derive a generalization bound that characterizes the performance of \(\mathsf{SIOpt}\) on unseen data. However, we note that the theoretical analysis in the context of SOD is non-trivial due to several challenges inherent challenges. First, SOD is typically a structured prediction problem~\cite{Carlo_2020_nips, li_2021_nips}, where the dependencies among output substructures complicate the direct application of Rademacher complexity-based theoretical tools. In particular, classical results for bounding empirical Rademacher complexity~\cite{Probability1991} typically assume real-valued prediction functions. To address this limitation, we leverage the vector contraction inequality~\cite{Maurer_vector_2016}, which enables the extension of these results to vector-valued settings. By incorporating the Lipschitz properties of structured outputs~\cite{vector_Contraction}, we derive a \textbf{sharper} generalization bound. Second, the diversity of loss functions employed in SOD models presents an additional challenge, impeding the derivation of generalization guarantees within a unified analytical framework. To tackle this, we conduct a categorical analysis based on separable and composite losses, and establish their corresponding Lipschitz continuity properties~\cite{p_Lipschitz}. This allows us to obtain a \textbf{unified} generalization result applicable across different loss paradigms.

We summarize our theoretical findings below, while the detailed proofs are provided in \cref{generalization_bound_proof}.

\begin{table*}[!ht]
    \centering
    \caption{Performance comparison on the DUTS-TE dataset. The best results are highlighted in bold, and the second-best are marked with underline. Here, $\uparrow$ indicates that higher values denote better performance, while $\downarrow$ indicates that lower values are preferable.}
    \scalebox{0.8}{
      \begin{tabular}{c|c|cccccccccc}
      \toprule
      \multicolumn{1}{c}{Dataset} & \multicolumn{1}{|c|}{Methods} & $\MAE \downarrow$ & $\SMAE \downarrow$ & $\AUC \uparrow $ & $\SAUC \uparrow $ & $\F_m^{\beta} \uparrow$ & $\SF_m^{\beta} \uparrow $ & $\F_{max}^{\beta} \uparrow$ & $\SF_{max}^{\beta} \uparrow$ & $\E_m \uparrow$ & $\Sm_m \uparrow$ \\
      \midrule
      \multirow{21}[12]{*}{DUTS-TE} & PoolNet \cite{PoolNet+} & \cellcolor[rgb]{ .973,  .973,  .988}0.0656  & \cellcolor[rgb]{ .973,  .969,  .988}0.0609  & \cellcolor[rgb]{ .894,  .89,  .961}0.9607  & \cellcolor[rgb]{ .89,  .886,  .957}0.9550  & \cellcolor[rgb]{ .992,  .992,  1}0.7200  & \cellcolor[rgb]{ .992,  .992,  1}0.7569  & \cellcolor[rgb]{ .973,  .973,  .992}0.8245  & \cellcolor[rgb]{ .941,  .937,  .976}0.8715  & 0.8103  & \cellcolor[rgb]{ .988,  .984,  .996}0.9112  \\
            & + $\mathsf{SIOpt1}$ & \cellcolor[rgb]{ .957,  .957,  .984}\underline{0.0621}  & \cellcolor[rgb]{ .949,  .945,  .98}\underline{0.0562}  & \cellcolor[rgb]{ .871,  .863,  .949}\textbf{0.9706} & \cellcolor[rgb]{ .867,  .859,  .949}\textbf{0.9647} & \cellcolor[rgb]{ .961,  .957,  .984}\underline{0.7479}  & \cellcolor[rgb]{ .922,  .918,  .973}\textbf{0.8172} & \cellcolor[rgb]{ .941,  .941,  .98}\underline{0.8438}  & \cellcolor[rgb]{ .886,  .878,  .957}\textbf{0.9029} & \cellcolor[rgb]{ .949,  .945,  .98}\underline{0.8478}  & \cellcolor[rgb]{ .953,  .949,  .984}\underline{0.9188}  \\
            & + $\mathsf{SIOpt2}$ & \cellcolor[rgb]{ .941,  .937,  .976}\textbf{0.0576} & \cellcolor[rgb]{ .941,  .937,  .976}\textbf{0.0543} & \cellcolor[rgb]{ .878,  .875,  .953}\underline{0.9666}  & \cellcolor[rgb]{ .878,  .871,  .953}\underline{0.9609}  & \cellcolor[rgb]{ .953,  .949,  .984}\textbf{0.7543} & \cellcolor[rgb]{ .953,  .953,  .984}\underline{0.7895}  & \cellcolor[rgb]{ .929,  .925,  .973}\textbf{0.8534} & \cellcolor[rgb]{ .906,  .898,  .965}\underline{0.8919}  & \cellcolor[rgb]{ .945,  .941,  .98}\textbf{0.8508} & \cellcolor[rgb]{ .918,  .91,  .969}\textbf{0.9269} \\
  \cmidrule{2-12}          & ADMNet \cite{ADMNet} & 0.0721  & 0.0666  & 0.9188  & 0.9098  & 0.7127  & 0.7491  & 0.8068  & 0.8355  & \cellcolor[rgb]{ .996,  .996,  1}0.8139  & 0.9079  \\
            & + $\mathsf{SIOpt1}$ & \cellcolor[rgb]{ .988,  .988,  .996}\underline{0.0700}  & \cellcolor[rgb]{ .98,  .98,  .992}\textbf{0.0627} & \cellcolor[rgb]{ .957,  .957,  .984}\underline{0.9360}  & \cellcolor[rgb]{ .953,  .953,  .984}\underline{0.9293}  & \cellcolor[rgb]{ .988,  .988,  .996}\textbf{0.7240} & \cellcolor[rgb]{ .961,  .961,  .984}\textbf{0.7837} & \cellcolor[rgb]{ .992,  .988,  .996}\underline{0.8140}  & \cellcolor[rgb]{ .961,  .961,  .988}\underline{0.8584}  & \cellcolor[rgb]{ .969,  .969,  .988}\textbf{0.8335} & \cellcolor[rgb]{ .953,  .949,  .984}\textbf{0.9190} \\
            & + $\mathsf{SIOpt2}$ & \cellcolor[rgb]{ .988,  .984,  .992}\textbf{0.0692} & \cellcolor[rgb]{ .98,  .98,  .992}\underline{0.0631}  & \cellcolor[rgb]{ .922,  .918,  .973}\textbf{0.9496} & \cellcolor[rgb]{ .922,  .914,  .969}\textbf{0.9433} & \cellcolor[rgb]{ .996,  .996,  1}\underline{0.7183}  & \underline{0.7515}  & \cellcolor[rgb]{ .973,  .973,  .992}\textbf{0.8249} & \cellcolor[rgb]{ .945,  .945,  .98}\textbf{0.8676} & \cellcolor[rgb]{ .992,  .992,  1}\underline{0.8170}  & \cellcolor[rgb]{ .953,  .949,  .984}\underline{0.9188}  \\
  \cmidrule{2-12}          & LDF \cite{LDF} & \cellcolor[rgb]{ .878,  .871,  .953}\underline{0.0419}  & \cellcolor[rgb]{ .89,  .886,  .957}0.0440  & \cellcolor[rgb]{ .965,  .961,  .988}0.9337  & \cellcolor[rgb]{ .957,  .953,  .984}0.9282  & \cellcolor[rgb]{ .871,  .863,  .953}\textbf{0.8203} & \cellcolor[rgb]{ .918,  .914,  .969}0.8201  & \cellcolor[rgb]{ .894,  .89,  .961}0.8735  & \cellcolor[rgb]{ .925,  .922,  .973}0.8802  & \cellcolor[rgb]{ .898,  .894,  .961}\underline{0.8821}  & \cellcolor[rgb]{ .902,  .898,  .965}0.9296  \\
            & + $\mathsf{SIOpt1}$ & \cellcolor[rgb]{ .886,  .878,  .953}0.0440  & \cellcolor[rgb]{ .882,  .875,  .953}\underline{0.0422}  & \cellcolor[rgb]{ .875,  .867,  .953}\underline{0.9690}  & \cellcolor[rgb]{ .867,  .859,  .949}\textbf{0.9645} & \cellcolor[rgb]{ .886,  .882,  .957}0.8076  & \cellcolor[rgb]{ .898,  .89,  .961}\underline{0.8388}  & \cellcolor[rgb]{ .894,  .89,  .961}\underline{0.8736}  & \cellcolor[rgb]{ .871,  .863,  .949}\underline{0.9117}  & \cellcolor[rgb]{ .89,  .882,  .957}\textbf{0.8895} & \cellcolor[rgb]{ .867,  .859,  .949}\textbf{0.9374} \\
            & + $\mathsf{SIOpt2}$ & \cellcolor[rgb]{ .871,  .863,  .949}\textbf{0.0405} & \cellcolor[rgb]{ .871,  .863,  .949}\textbf{0.0397} & \cellcolor[rgb]{ .871,  .867,  .953}\textbf{0.9696} & \cellcolor[rgb]{ .871,  .863,  .953}\underline{0.9628}  & \cellcolor[rgb]{ .871,  .867,  .953}\underline{0.8199}  & \cellcolor[rgb]{ .882,  .875,  .957}\textbf{0.8511} & \cellcolor[rgb]{ .875,  .867,  .953}\textbf{0.8865} & \cellcolor[rgb]{ .867,  .859,  .949}\textbf{0.9140} & \cellcolor[rgb]{ .898,  .894,  .961}0.8820  & \cellcolor[rgb]{ .882,  .875,  .957}\underline{0.9344}  \\
  \cmidrule{2-12}          & ICON \cite{ICON} & \cellcolor[rgb]{ .894,  .89,  .957}0.0461  & \cellcolor[rgb]{ .898,  .894,  .961}0.0454  & \cellcolor[rgb]{ .929,  .925,  .973}0.9469  & \cellcolor[rgb]{ .929,  .925,  .973}0.9398  & \cellcolor[rgb]{ .878,  .875,  .953}\underline{0.8131}  & \cellcolor[rgb]{ .91,  .906,  .965}0.8270  & \cellcolor[rgb]{ .91,  .906,  .965}\underline{0.8648}  & \cellcolor[rgb]{ .922,  .918,  .973}0.8815  & \cellcolor[rgb]{ .894,  .886,  .961}0.8858  & \cellcolor[rgb]{ .886,  .882,  .957}\underline{0.9330}  \\
            & + $\mathsf{SIOpt1}$ & \cellcolor[rgb]{ .89,  .886,  .957}\underline{0.0454}  & \cellcolor[rgb]{ .89,  .882,  .957}\underline{0.0435}  & \cellcolor[rgb]{ .886,  .878,  .957}\textbf{0.9640} & \cellcolor[rgb]{ .882,  .875,  .953}\textbf{0.9593} & \cellcolor[rgb]{ .894,  .886,  .961}0.8031  & \cellcolor[rgb]{ .894,  .89,  .961}\underline{0.8395}  & \cellcolor[rgb]{ .914,  .906,  .969}0.8629  & \cellcolor[rgb]{ .898,  .894,  .961}\textbf{0.8958} & \cellcolor[rgb]{ .886,  .878,  .957}\underline{0.8921}  & \cellcolor[rgb]{ .871,  .863,  .949}\textbf{0.9371} \\
            & + $\mathsf{SIOpt2}$ & \cellcolor[rgb]{ .886,  .882,  .957}\textbf{0.0445} & \cellcolor[rgb]{ .886,  .882,  .957}\textbf{0.0432} & \cellcolor[rgb]{ .906,  .902,  .965}\underline{0.9557}  & \cellcolor[rgb]{ .91,  .902,  .965}\underline{0.9478}  & \cellcolor[rgb]{ .878,  .871,  .953}\textbf{0.8158} & \cellcolor[rgb]{ .886,  .882,  .957}\textbf{0.8455} & \cellcolor[rgb]{ .898,  .894,  .961}\textbf{0.8713} & \cellcolor[rgb]{ .91,  .906,  .965}\underline{0.8886}  & \cellcolor[rgb]{ .875,  .867,  .953}\textbf{0.9007} & \cellcolor[rgb]{ .914,  .91,  .969}0.9275  \\
  \cmidrule{2-12}          & GateNet \cite{GateNet} & \cellcolor[rgb]{ .863,  .855,  .945}\underline{0.0383}  & \cellcolor[rgb]{ .863,  .855,  .945}0.0380  & \cellcolor[rgb]{ .89,  .882,  .957}0.9629  & \cellcolor[rgb]{ .886,  .882,  .957}0.9565  & \cellcolor[rgb]{ .859,  .851,  .945}\textbf{0.8292} & \cellcolor[rgb]{ .882,  .875,  .953}0.8519  & \cellcolor[rgb]{ .878,  .875,  .953}\underline{0.8835}  & \cellcolor[rgb]{ .882,  .878,  .957}0.9041  & \cellcolor[rgb]{ .867,  .859,  .949}\underline{0.9053}  & \cellcolor[rgb]{ .859,  .851,  .945}\textbf{0.9390} \\
            & + $\mathsf{SIOpt1}$ & \cellcolor[rgb]{ .871,  .863,  .949}0.0399  & \cellcolor[rgb]{ .859,  .855,  .945}\underline{0.0375}  & \cellcolor[rgb]{ .882,  .875,  .953}\underline{0.9663}  & \cellcolor[rgb]{ .878,  .875,  .953}\underline{0.9594}  & \cellcolor[rgb]{ .875,  .867,  .953}0.8185  & \cellcolor[rgb]{ .859,  .851,  .945}\textbf{0.8687} & \cellcolor[rgb]{ .894,  .89,  .961}0.8743  & \cellcolor[rgb]{ .871,  .863,  .949}\underline{0.9116}  & \cellcolor[rgb]{ .867,  .863,  .949}0.9038  & \cellcolor[rgb]{ .863,  .855,  .949}\underline{0.9387}  \\
            & + $\mathsf{SIOpt2}$ & \cellcolor[rgb]{ .859,  .851,  .945}\textbf{0.0368} & \cellcolor[rgb]{ .859,  .851,  .945}\textbf{0.0367} & \cellcolor[rgb]{ .859,  .851,  .945}\textbf{0.9740} & \cellcolor[rgb]{ .859,  .851,  .945}\textbf{0.9673} & \cellcolor[rgb]{ .863,  .855,  .949}\underline{0.8282}  & \cellcolor[rgb]{ .875,  .867,  .953}\underline{0.8571}  & \cellcolor[rgb]{ .859,  .851,  .945}\textbf{0.8956} & \cellcolor[rgb]{ .859,  .851,  .945}\textbf{0.9174} & \cellcolor[rgb]{ .863,  .855,  .949}\textbf{0.9087} & \cellcolor[rgb]{ .863,  .855,  .949}\underline{0.9387}  \\
  \cmidrule{2-12}          & EDN \cite{EDN} & \cellcolor[rgb]{ .867,  .859,  .945}\underline{0.0389}  & \cellcolor[rgb]{ .867,  .859,  .945}0.0388  & \cellcolor[rgb]{ .898,  .89,  .961}0.9600  & \cellcolor[rgb]{ .902,  .894,  .961}0.9513  & \cellcolor[rgb]{ .863,  .855,  .949}\textbf{0.8288} & \cellcolor[rgb]{ .875,  .867,  .953}\underline{0.8565}  & \cellcolor[rgb]{ .894,  .886,  .961}0.8752  & \cellcolor[rgb]{ .886,  .882,  .957}\underline{0.9017}  & \cellcolor[rgb]{ .871,  .863,  .949}0.9033  & \cellcolor[rgb]{ .863,  .855,  .949}\underline{0.9385}  \\
            & + $\mathsf{SIOpt1}$ & \cellcolor[rgb]{ .867,  .859,  .945}0.0392  & \cellcolor[rgb]{ .863,  .855,  .945}\underline{0.0381}  & \cellcolor[rgb]{ .882,  .875,  .957}\textbf{0.9658} & \cellcolor[rgb]{ .878,  .875,  .953}\textbf{0.9596} & \cellcolor[rgb]{ .863,  .859,  .949}\underline{0.8260}  & \cellcolor[rgb]{ .863,  .855,  .949}\textbf{0.8672} & \cellcolor[rgb]{ .89,  .886,  .961}\underline{0.8765}  & \cellcolor[rgb]{ .871,  .863,  .949}\textbf{0.9119} & \cellcolor[rgb]{ .863,  .855,  .949}\underline{0.9072}  & \cellcolor[rgb]{ .863,  .855,  .949}\textbf{0.9388} \\
            & + $\mathsf{SIOpt2}$ & \cellcolor[rgb]{ .863,  .855,  .945}\textbf{0.0386} & \cellcolor[rgb]{ .863,  .855,  .945}\textbf{0.0380} & \cellcolor[rgb]{ .89,  .886,  .957}\underline{0.9622}  & \cellcolor[rgb]{ .894,  .886,  .961}\underline{0.9545}  & \cellcolor[rgb]{ .867,  .859,  .949}0.8237  & \cellcolor[rgb]{ .886,  .878,  .957}0.8485  & \cellcolor[rgb]{ .886,  .882,  .957}\textbf{0.8784} & \cellcolor[rgb]{ .894,  .886,  .961}0.8991  & \cellcolor[rgb]{ .859,  .851,  .945}\textbf{0.9091} & \cellcolor[rgb]{ .89,  .882,  .957}0.9325  \\
      \cmidrule{2-12} &  VST \cite{VST2} & \cellcolor[rgb]{ .863,  .855,  .945}0.0367  & \cellcolor[rgb]{ .863,  .855,  .945}0.0362  & \cellcolor[rgb]{ .882,  .878,  .957}\underline{0.9650}  & \cellcolor[rgb]{ .886,  .878,  .957}0.9575  & \cellcolor[rgb]{ .863,  .855,  .949}0.8326  & \cellcolor[rgb]{ .878,  .871,  .953}\underline{0.8548}  & \cellcolor[rgb]{ .886,  .882,  .957}0.8785  & \cellcolor[rgb]{ .89,  .886,  .961}0.8993  & \cellcolor[rgb]{ .863,  .859,  .949}0.9154  & \cellcolor[rgb]{ .871,  .863,  .949}0.9368  \\
      & + $\mathsf{SIOpt1}$ & \cellcolor[rgb]{ .859,  .851,  .945}\textbf{0.0353} & \cellcolor[rgb]{ .859,  .851,  .945}\textbf{0.0347} & \cellcolor[rgb]{ .882,  .878,  .957}0.9649  & \cellcolor[rgb]{ .886,  .878,  .957}\underline{0.9577}  & \cellcolor[rgb]{ .859,  .851,  .945}\textbf{0.8354} & \cellcolor[rgb]{ .875,  .871,  .953}\textbf{0.8556} & \cellcolor[rgb]{ .886,  .882,  .957}\underline{0.8788}  & \cellcolor[rgb]{ .89,  .886,  .961}\underline{0.8995}  & \cellcolor[rgb]{ .859,  .851,  .945}\textbf{0.9183} & \cellcolor[rgb]{ .867,  .859,  .949}\underline{0.9377}  \\
      & + $\mathsf{SIOpt2}$ & \cellcolor[rgb]{ .859,  .851,  .945}\underline{0.0359}  & \cellcolor[rgb]{ .863,  .855,  .945}\underline{0.0356}  & \cellcolor[rgb]{ .882,  .875,  .957}\textbf{0.9659} & \cellcolor[rgb]{ .882,  .875,  .957}\textbf{0.9586} & \cellcolor[rgb]{ .863,  .855,  .949}\underline{0.8335}  & \cellcolor[rgb]{ .878,  .875,  .953}0.8527  & \cellcolor[rgb]{ .886,  .878,  .957}\textbf{0.8800} & \cellcolor[rgb]{ .89,  .886,  .957}\textbf{0.9000} & \cellcolor[rgb]{ .863,  .855,  .949}\underline{0.9161}  & \cellcolor[rgb]{ .867,  .859,  .949}\textbf{0.9379} \\
      \bottomrule
      \end{tabular}%
      }
    \label{tab:exp_result}%
  \end{table*}%

\begin{theorem}[\textbf{Generalization Bound for SI-SOD}]\label{generalization_bound}
Assume $\mathcal{F} \subseteq \{ f:\mathcal{X} \to \mathbb{R}^S \}$, where $S=H \times W$ is the number of pixels in an image,  $g^{(i)}$ is the risk over $i$-th sample, and is $K$-Lipschitz with respect to the $l_{\infty}$ norm, (i.e. $\Vert g(x)-g(\tilde{x})\Vert_\infty \le K\cdot \Vert x-\tilde{x} \Vert_{\infty}$). When there are $N$ $i.i.d.$ samples, there exists a constant $C>0$ for any $\epsilon > 0$, the following generalization bound holds with probability at least $1-\delta$:
\begin{equation}
\begin{aligned}
    &~ \sup_{f \in \mathcal{F}}(\mathbb{E}[g(f)]-\hat{\mathbb{E}}[g(f)]) \\
    \le&~  \frac{CK\sqrt{S}}{N} \cdot \max_i \mathfrak{R}_N(\mathcal{F}|_i)\cdot\log^{\frac{3}{2}+\epsilon}\left(\frac{N}{\max_i \mathfrak{R}_N(\mathcal{F}|_i)}\right) \\
    &+3\sqrt{\frac{\log \frac{2}{\delta}}{2N}},
\end{aligned}
\end{equation}
where $g(f)$ could be any loss functions introduced in \cref{20250404Sec4.2}, $\mathbb{E}[g(f)]$ and $\hat{\mathbb{E}}[g(f)]:=\frac{1}{N}\sum_{i=1}^{N}g^{(i)}(f)$ represent the expected risk and empirical risk. $\mathfrak{R}_N(\mathcal{F}|_i)=\max_{x_{1:N}} \mathfrak{R}(\mathcal{F};x_{1:N})$ denotes the worst-case Rademacher complexity, and $\mathfrak{R}_N(\mathcal{F} | i)$ denotes its restriction to output coordinate $i$. 
\end{theorem}

\noindent\textbf{Remark.}
We derive a generalization bound of \(\mathcal{O}(\frac{\sqrt{S} \log N}{N})\) for SOD tasks, which implies reliable generalization as the training size \(N\) increases. To the best of our knowledge, this problem remains largely unexplored within the SOD community. Concretely, if the loss function \(g(f)\) is separable and instantiated by a \(\mu\)-Lipschitz continuous function \(\ell(\cdot)\) (cf. \cref{20250411:Sec.4.2.1}), then the Lipschitz constant \(K = \mu\). In the case where \(g(f)\) is a composite function and the Dice loss \cite{DiceLoss} is used, we obtain \(K = \frac{4}{\rho}\), where \(\rho = \min \frac{S_k^{+,i}}{S_k^i}\) denotes the minimum foreground-to-box ratio across all frames \(k\) and samples \(i\). By contrast, conventional size-sensitive losses compute the ratio \(\rho' = \frac{S^{+,i}_k}{S^i}\), based on the proportion of foreground pixels in the entire image. Since \(\mathsf{SIOpt}\) localizes learning to per-object bounding boxes, thereby reducing the denominator from the entire image (\(S^i\)) to localized regions (\(S_k^i\)), it ensures \(\rho > \rho'\), which results in a \textbf{smaller} Lipschitz constant \(K\) and thus leads to a \textbf{tighter} generalization bound.

\section{Experiments} \label{Experiments}
\textbf{Please see \cref{Experiments_appendix} for a longer version.}

\begin{figure*}[!th]
    \centering
    \includegraphics[width=0.95\linewidth]{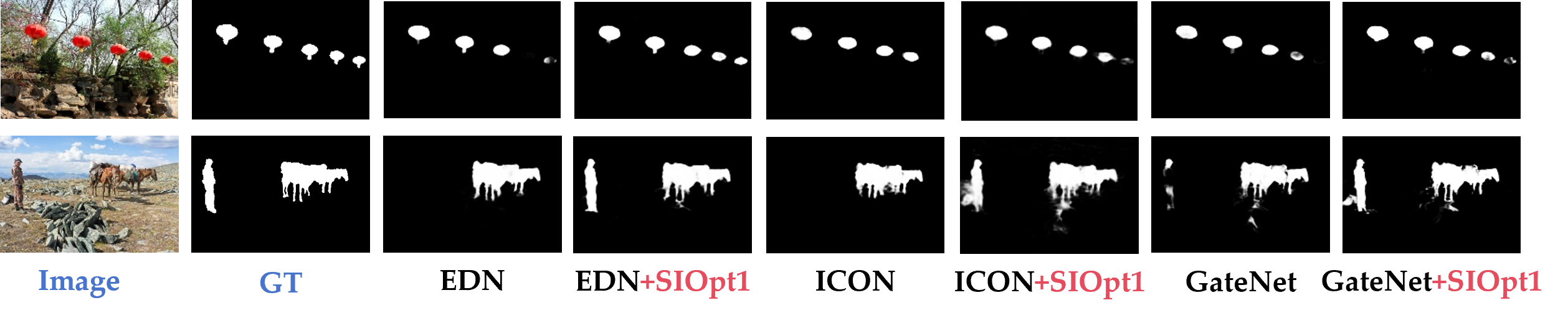}
    \vspace{-0.3cm}
    \caption{Qualitative visualizations on different backbones before and after using our proposed $\mathsf{SIOpt1}$.}
    \label{fig:vis}
    \vspace{-0.3cm}
\end{figure*}
\vspace{-2mm}
\subsection{Overall Performance}\label{20250419Sec6.2}
Partial results on the DUTS-TE and DUT-OMRON datasets are presented in \cref{tab:exp_result}, with additional results provided in \cref{app:overall}. Several key observations can be drawn from these results. First, our proposed $\mathsf{SIOpt}$ approach consistently improves SOD performance across various backbones, with either $\mathsf{SIOpt1}$ or $\mathsf{SIOpt2}$ achieving the best results in most cases. More importantly, as shown in \cref{fig:leida-all}, the performance gains are particularly significant in challenging scenarios involving multiple salient objects, such as the MSOD and HKU-IS datasets. For instance, using ADMNet as the backbone, the relative improvements of $\mathsf{SIOpt2}$ over the baseline reach up to $6.7\%$, $7.5\%$, $7.1\%$, $9.0\%$, $4.0\%$, $6.3\%$, $4.6\%$, $13.7\%$, $2.9\%$ and $2.1\%$ on $\MAE$, $\SMAE$, $\AUC$, $\SI\text{-}\AUC$, $\F_m^\beta$, $\SF_m^\beta$, $\F_{max}^\beta$, $\SF_{max}^\beta$, $\E_m$, and $\Sm_m$, respectively. These results clearly validate the effectiveness and scalable capability of our size-invariant framework. Meanwhile, we observe that, compared to CNN-based models, the improvements over the transformer-based VST baseline are relatively modest and sometimes comparable. One possible reason lies in that the inputs of VST are a sequence of equal-sized patches, naturally reducing the dominance of large objects over smaller ones more or less. However, as further discussed in \cref{20250408sec6.3}, using fixed patch-wise partitions may incidentally fragment the small targets, thereby limiting overall performance gains. Additionally, regarding the comparison between $\mathsf{SIOpt1}$ and $\mathsf{SIOpt2}$, each exhibits its own advantages, making it difficult to declare one categorically superior to the other. The selection of the appropriate $\mathsf{SIOpt}$ objective should therefore be guided by the evaluation criteria and practical requirements of the downstream task \cite{Borji_2019_survey, IDSurvey_2022}.

To further demonstrate the advantages of our size-invariant framework, we present qualitative comparisons in \cref{fig:vis}, \cref{fig:vis_appendix}, and \cref{fig:vis_SIOpt} across different backbone models. While the original models often struggle with identifying smaller salient objects in complex scenes, our method significantly enhances multi-object detection. For example, in the first image of \cref{fig:vis}, EDN only detects the largest sailboat on the right, missing two smaller ones on the left, whereas our method successfully identifies all three. In the fourth image, EDN yields fewer false positives, while ICON detects one additional object at the cost of increased false positives. In the third, fifth, and sixth examples, our framework consistently improves the detection of salient objects on the right. More qualitative comparisons can be found in \cref{Qualitative_appendix}.

\begin{figure}[!t]
    \centering
    \includegraphics[width=1.0\linewidth]{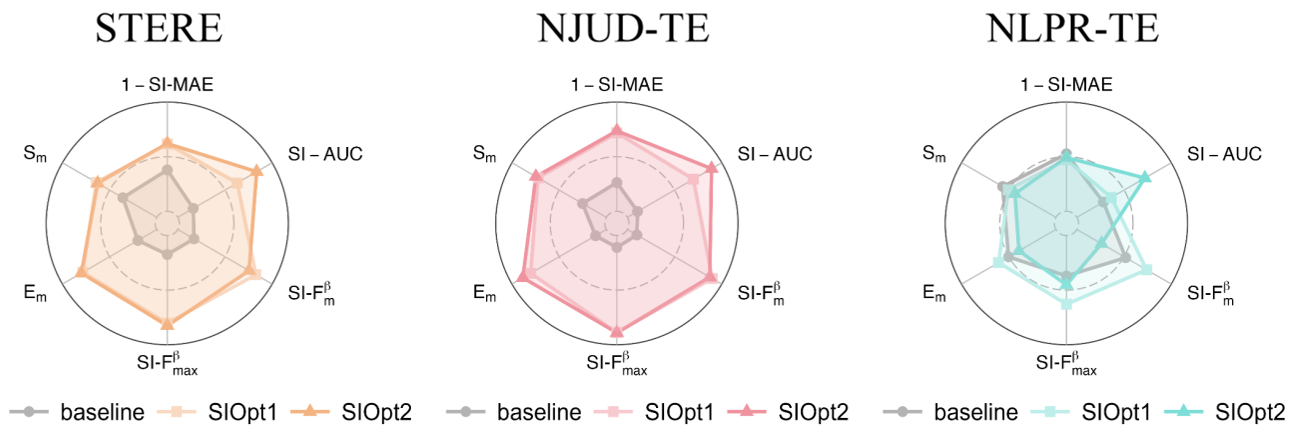} 
    \caption{Radar plots of CRNet-based methods on RGB-D benchmarks.}  
    \vspace{-0.3cm}  
    \label{fig:RGBD}    
\end{figure}
\begin{figure}[!t]
\centering
\subfigure[DUT-OMRON]{   
\includegraphics[width=0.305\linewidth]{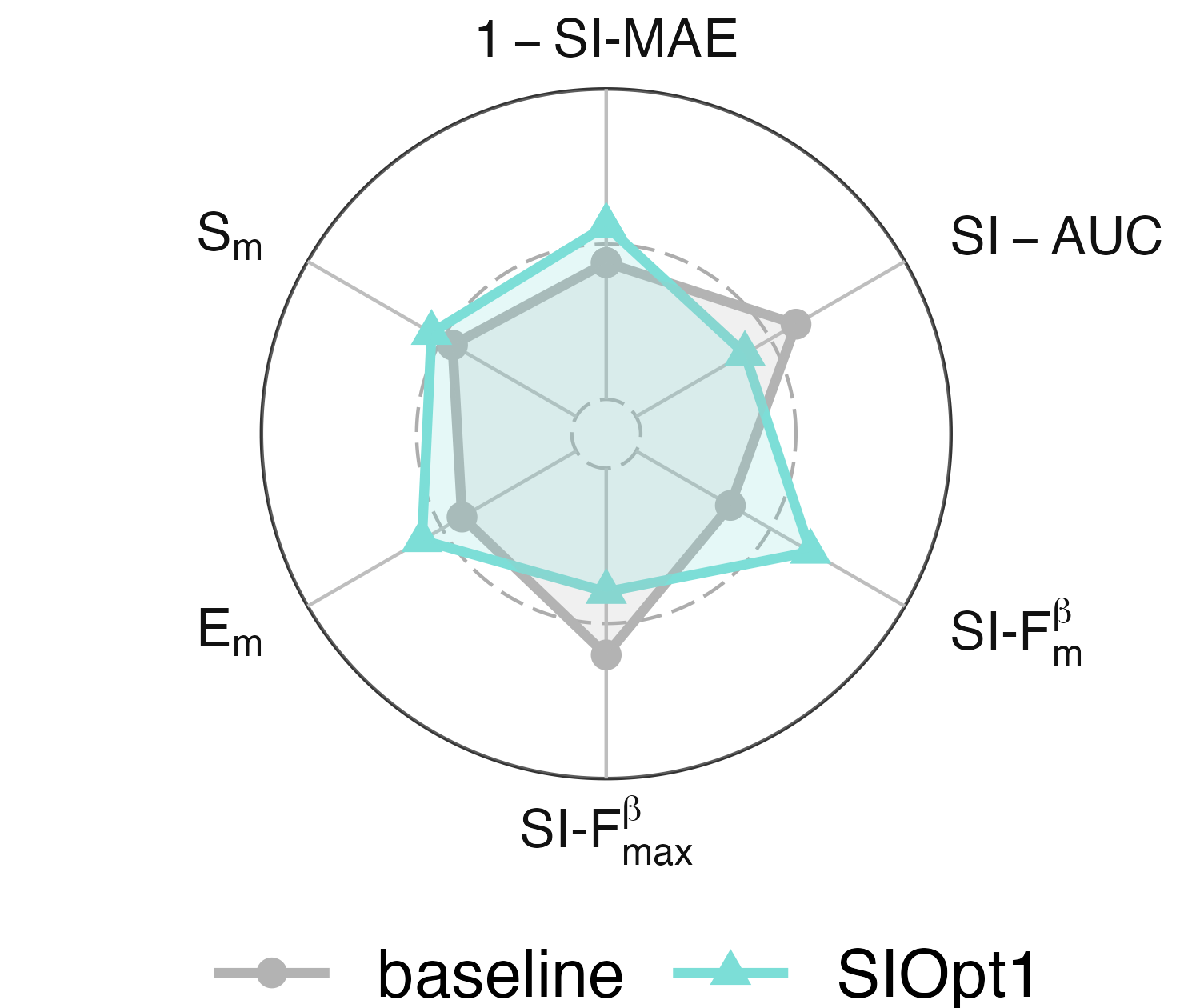}  
}
\subfigure[ECSSD]{
\includegraphics[width=0.305\linewidth]{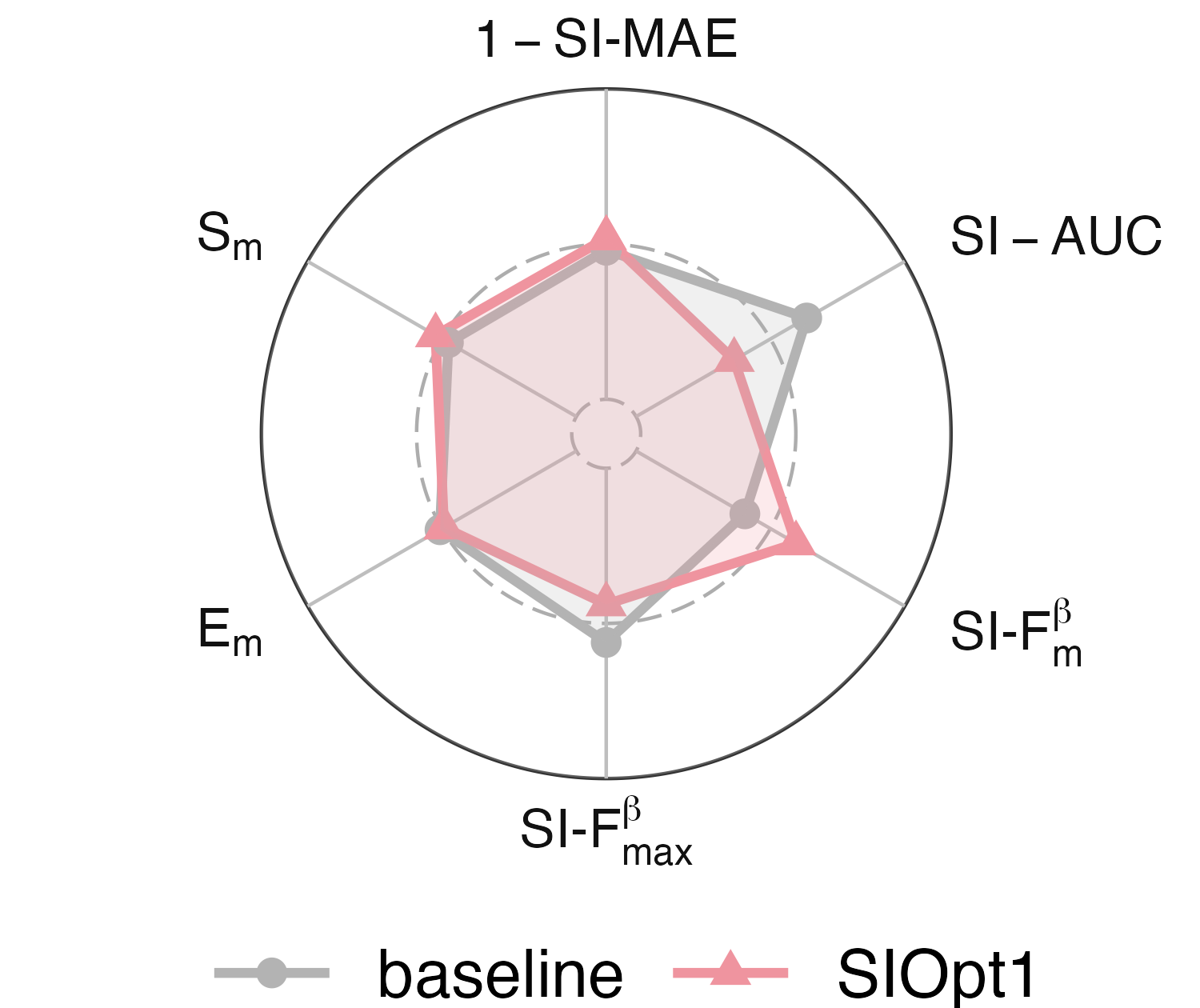} 
}
\subfigure[HKU-IS]{
\includegraphics[width=0.305\linewidth]{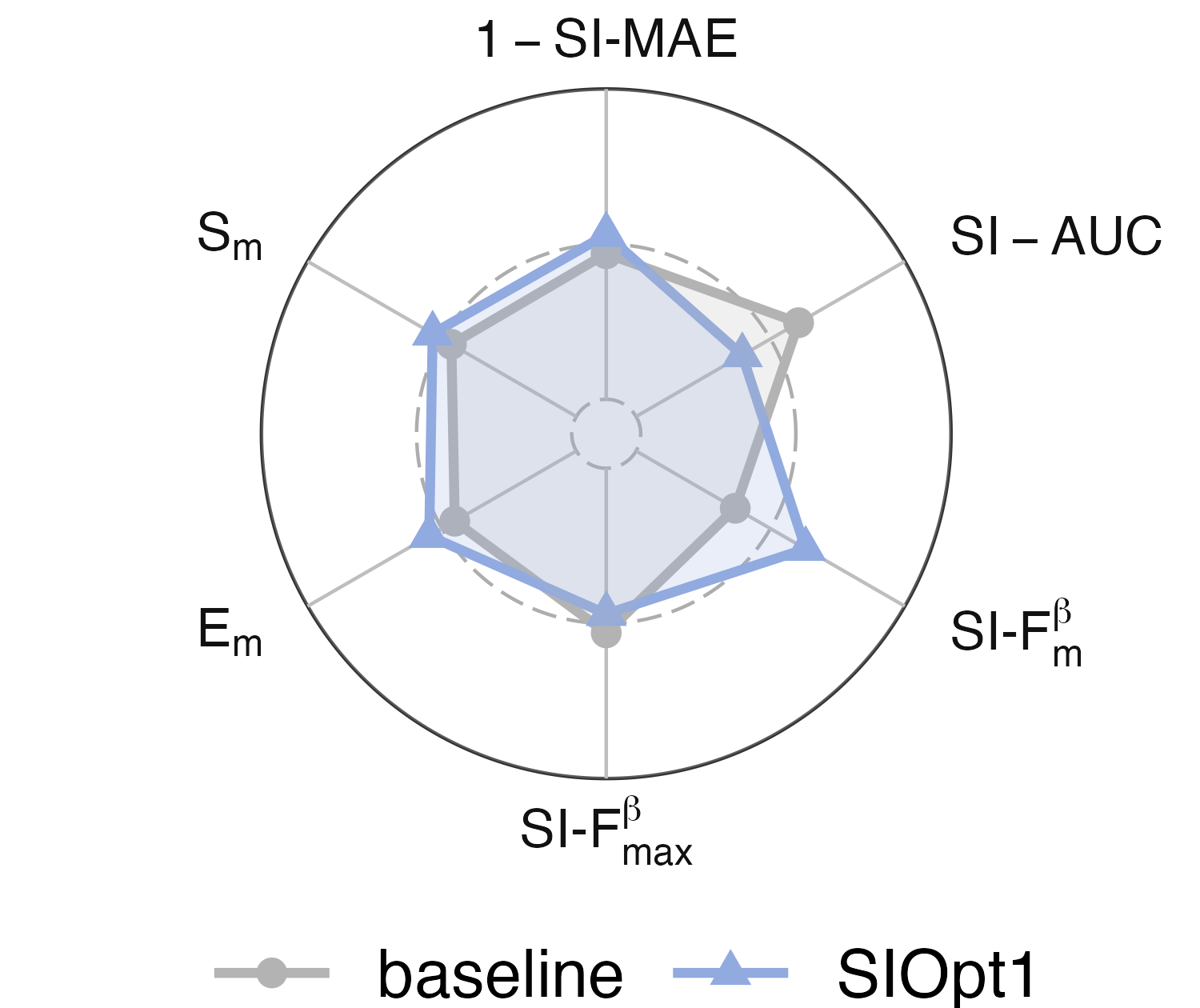} 
}
\caption{Radar plots of TS-SAM-based backbones. }    
\vspace{-0.3mm}
\label{fig:SAM}    
\end{figure}

\vspace{-2mm}
\subsection{Further Scalability Verifications for SIOpt}
\vspace{-1.1mm}
It is important to highlight that both \(\mathsf{SIEva}\) and \(\mathsf{SIOpt}\) are solely conditioned on the ground-truth masks \(\boldsymbol{Y}\), \textbf{making them readily adaptable to various SOD variants, such as RGB-D and RGB-T tasks. }

\noindent\textbf{Extension to Other Tasks.} To validate its generality, we extend our analysis to two widely studied SOD settings—RGB-D and RGB-T \cite{zhang2023c, li2023dvsod}—where depth cues and thermal infrared images are respectively utilized to enhance detection performance. In each setting, we apply \(\mathsf{SIOpt}\) to two competitive baselines \cite{CRNet,hu2024cross,TNet,DCNet} and evaluate their performance on commonly used benchmark datasets. More details can be found in \cref{app:RGBD_details} and \cref{RGBT_detail} of the Appendix. 

\noindent\textbf{Extension to Foundation Models.} Following the idea of Two-Stream SAM (TS-SAM) \cite{SAM3}, we also attempt to fine-tune a size-invariant SAM-adapter using $\mathsf{SIOpt}$ for RGB SOD tasks. More details can be found in \cref{SAM_detail} of the Appendix. 

\noindent\textbf{Performance Comparisons.} Empirical results on various SOD tasks and TS-SAM algorithms \cite{SAM3} are illustrated in \cref{fig:RGBD}, \cref{fig:SAM}, \cref{tab:RGBD}, \cref{tab:RGBT}, and \cref{tab:SAM}, respectively. For consistency with other metrics, we report $1 - \SMAE$ in those figures to align all values under the "higher is better" convention. Due to space constraints, further analysis and discussion are provided in \cref{app:more_RGB-SOD} and \cref{app:out_TS-SAM}. As shown, our proposed approach consistently delivers competitive and, in many cases, superior performance across both scenarios. 


\begin{figure}[!t]
    \centering
    \includegraphics[width=0.85\linewidth]{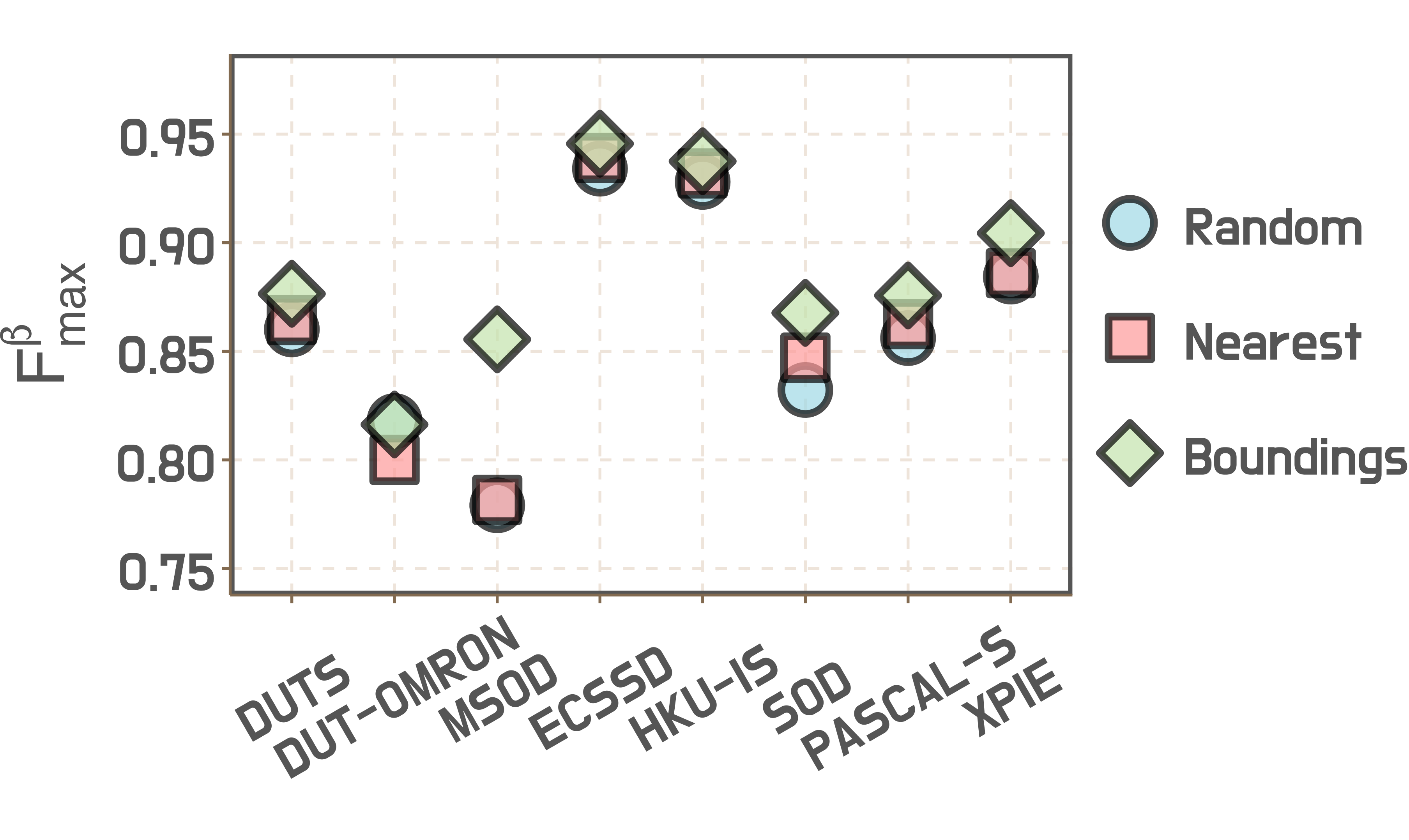}  
   \vspace{-0.35cm}
    \caption{Ablation studies of different approximation $\mathcal{C}(\boldsymbol{X})$ on $8$ benchmarks.}
    \label{fig:diff_CK}
    \vspace{-0.25cm}
\end{figure}

\vspace{-2mm}

\subsection{Different approximation functions of $\mathcal{C}(\boldsymbol{X})$ \label{20250408sec6.3}}
In this section, we investigate the impact of different implementations of $\mathcal{C}(\boldsymbol{X})$ on the final SOD performance. Specifically, we compare the default strategy \textbf{Boundings} (introduced in \cref{principles of SI_eval}) with two alternative splitting schemes, \textbf{Random} and \textbf{Nearest}. The quantitative results are presented in \cref{fig:diff_CK}, with additional introductions available in \cref{20250411:E.4} of the Appendix. Empirically, we observe that $\mathsf{SIOpt1}$ combined with the default \textbf{Boundings} method consistently demonstrates superior performance. A key reason is that bounding box-based splitting encourages the model to focus more on object boundaries, which in turn enhances detection quality. In contrast, the \textbf{Random} scheme may arbitrarily fragment coherent object regions, disrupting the model’s ability to capture holistic object structures and resulting in degraded performance. The \textbf{Nearest} method also introduces ambiguity in the spatial grouping of non-salient pixels, potentially impacting feature consistency. It is also important to note that the $\mathsf{SIEva}$ scores obtained under different $\mathcal{C}(\boldsymbol{X})$ definitions are not strictly comparable, due to variations in partition granularity and object coverage. Despite this, our experiments demonstrate a generally consistent correlation between $\mathsf{SIEva}$ and and those metrics that are less sensitive to object size, such as $\E_m$ and $\Sm_m$, i.e., higher $\mathsf{SIEva}$ scores tend to reflect stronger real-world performance across varying object sizes. Addressing the comparability of evaluation results under different $\mathcal{C}(\boldsymbol{X})$ configurations remains an open and important direction for future research.

\begin{figure}[!t]
    \centering
    \subfigure[MSOD]{   
    \begin{minipage}{0.43\linewidth}
    \includegraphics[width=\linewidth]{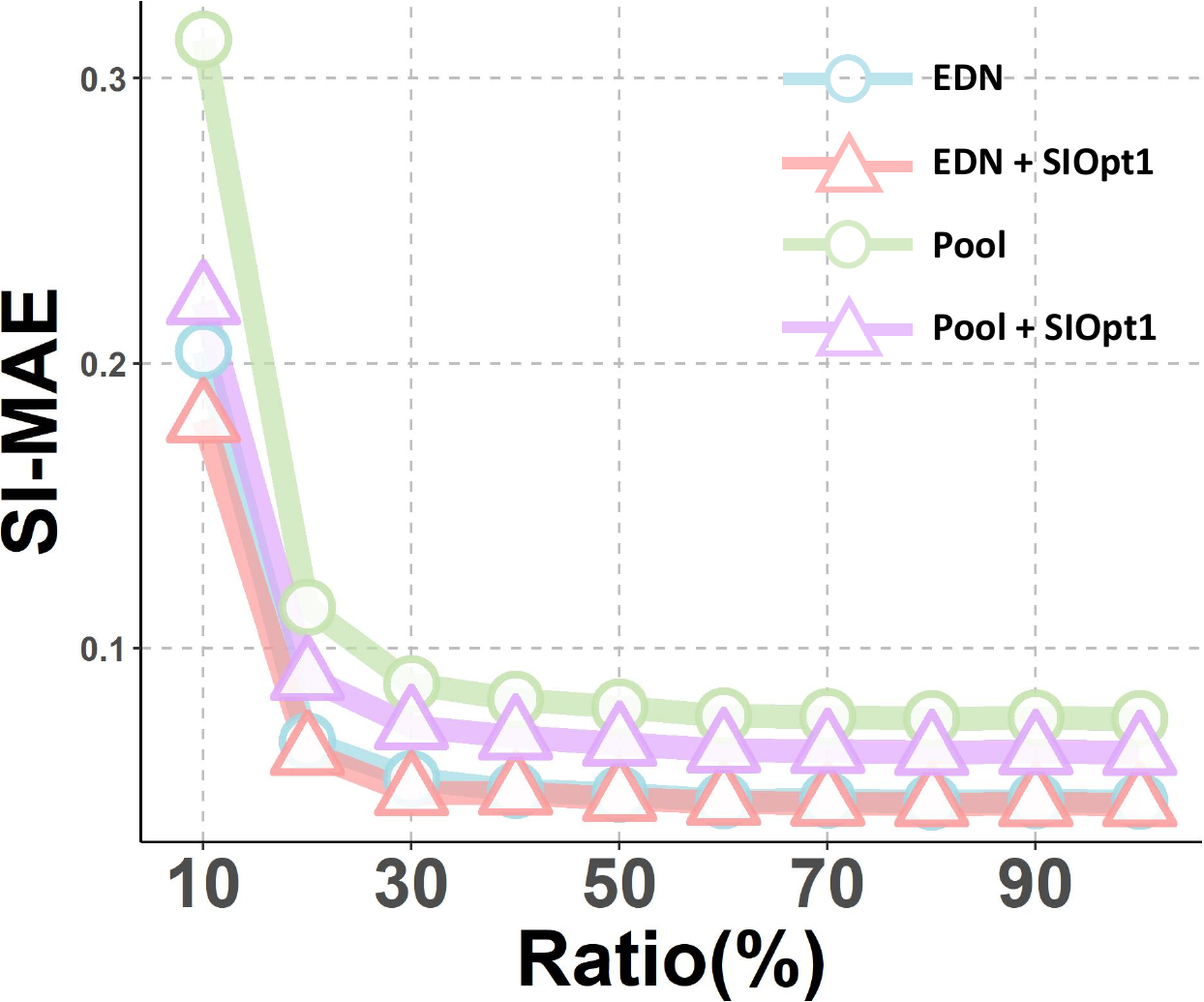}  
    \label{fig:EDN_Pool_msod_ratio_line}
    \end{minipage}
    }
    \subfigure[DUTS-TE]{   
    \begin{minipage}{0.43\linewidth}
    \includegraphics[width=\linewidth]{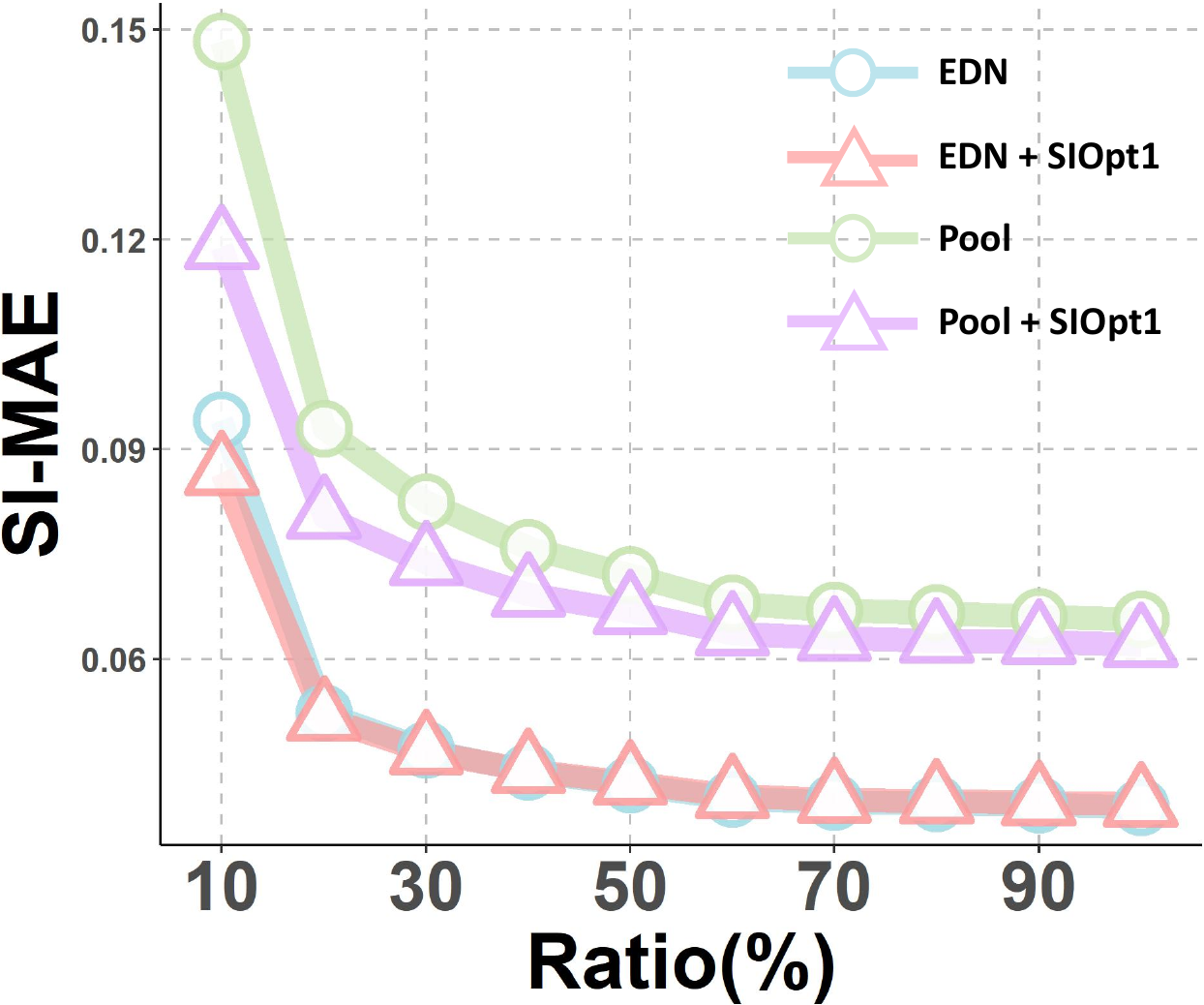}  
    \label{fig:EDN_Pool_DUTS_ratio_line}
    \end{minipage}
    }
    \vspace{-0.3mm}
    \caption{Fine-grained performance comparisons of varying object sizes.}
    \label{fig:fine-analysis-EDN-Pool-ratio} 
    \vspace{-1.3mm}
    \end{figure}

\vspace{-5mm} 

\subsection{Fine-grained Performance Analysis}

\begin{figure}[!t]
    \centering
    \subfigure[MSOD]{   
    \begin{minipage}{0.43\linewidth}
    \includegraphics[width=\linewidth]{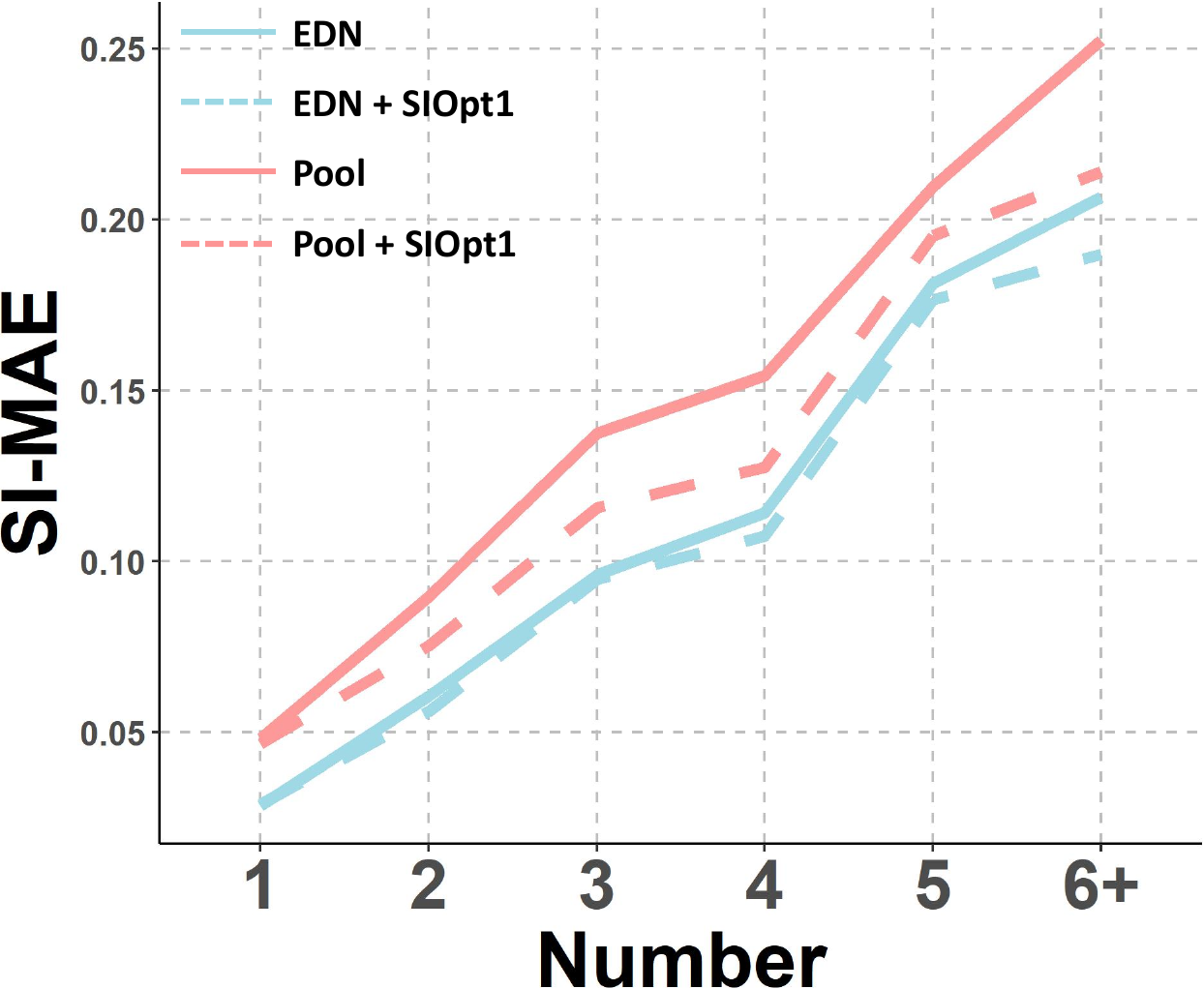}  
    \label{fig:EDN_Pool_msod_num_line}
    \end{minipage}
    }
    \subfigure[DUTS]{   
    \begin{minipage}{0.43\linewidth}
   \includegraphics[width=\linewidth]{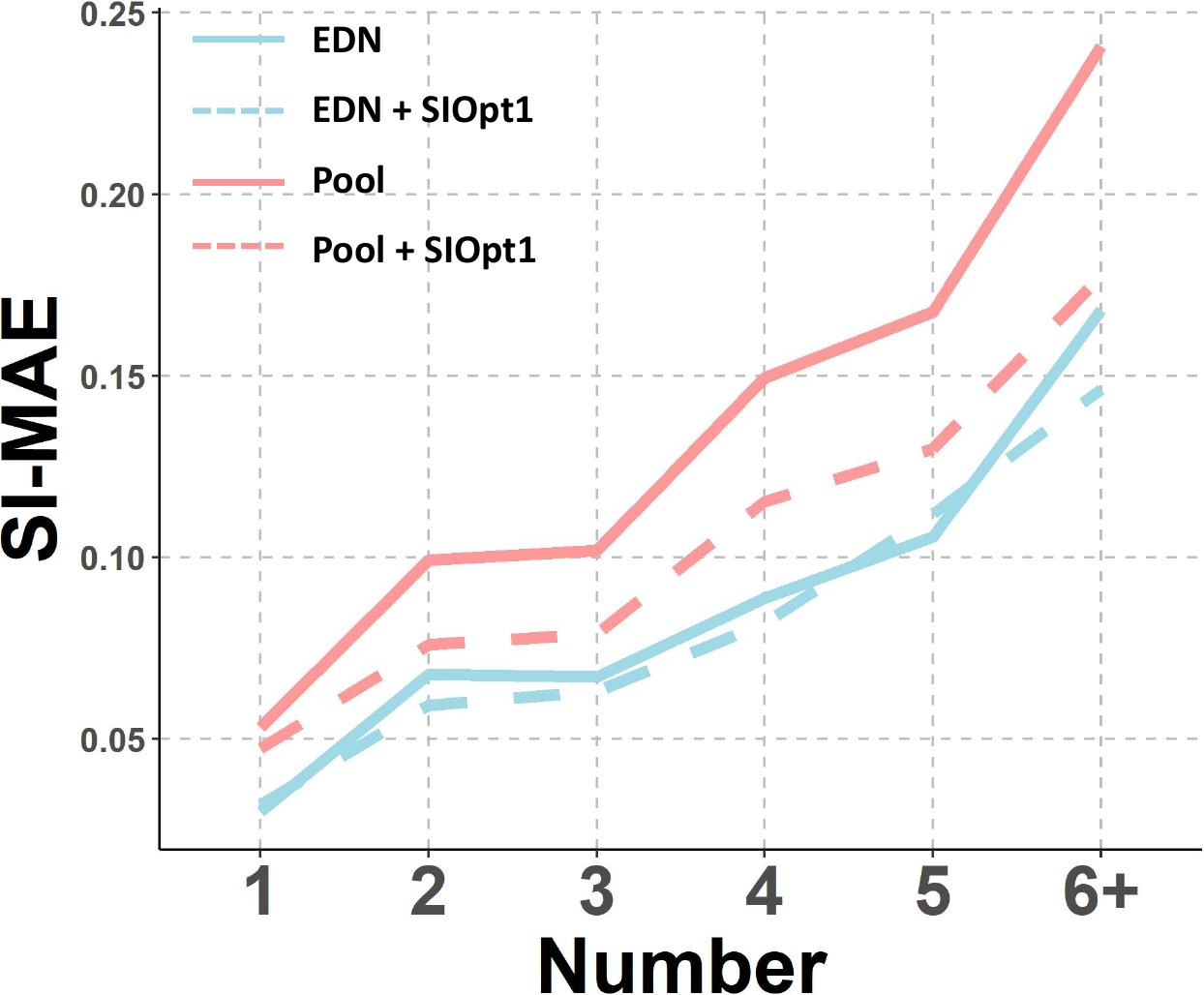}
    \label{fig:EDN_Pool_DUTS_num_line}
    \end{minipage}
    }
    \vspace{-0.3cm}
    \caption{Fine-grained performance under different object numbers.}
    \label{fig:fine-analysis-EDN-Pool-num}
    \vspace{-0.5cm}
\end{figure}

\subsubsection{Performance Comparisons across Varying Sizes}
\cref{fig:fine-analysis-EDN-Pool-ratio} presents the size-relevant performance of salient objects between $\mathsf{SIOpt}$ and traditional SOD approaches. More detailed configurations and performance comparisons are provided in \cref{size-fine-grained_appendix}. Accordingly, the proposed $\mathsf{SIOpt}$ framework consistently outperforms the original methods across all object-size groups. More importantly, the performance improvements are most pronounced for the small-object regime, particularly those occupying less than \(10\%\) of the image area. For instance, on the MSOD dataset, $\mathsf{SIOpt1}$ achieves a performance gain of approximately 0.024 in $\SMAE$ over the baseline EDN model in the \([0\%, 10\%]\) size group. This provides strong empirical evidence that $\mathsf{SIOpt}$ effectively mitigates the limitations of existing methods when dealing with small-scale salient objects. In addition, the performance gap narrows as object size increases, which aligns with expectations, since larger objects are inherently less affected by the size-invariance issue.

\vspace{-2mm}
\subsubsection{Performance Comparisons across Object Quantities}
\cref{fig:fine-analysis-EDN-Pool-num} presents the performance of $\mathsf{SIOpt}$ and traditional SOD methods across varying numbers of salient objects. Detailed configurations and performance comparisons provided in \cref{number-fine-grained_appendix}. We first observe that the performance of our method and the baselines is largely comparable when an image contains only a single salient object. This observation is consistent with our discussions in \cref{Revisting}, where single-object scenarios tend to be less ambiguous and current algorithms could perform very well. However, as the number of salient objects increases, $\mathsf{SIOpt}$ framework could yield better performance than baselines in most cases, particularly on challenging datasets such as MSOD and HKU-IS. (e.g., MSOD and HKU-IS datasets). These trends further underscore the effectiveness of our framework in handling complex scenes with multiple salient objects.

\vspace{-1mm}
\subsubsection{Ablation Studies of $\alpha_{SI}$} \label{Ablation Studies}
Partial results of the ablation studies on the $\alpha_{SI}$ are presented in \cref{fig:ablation_SIMAE}. More detailed discussions are deferred to \cref{ablation_appendix}. We can see that the default of $\alpha_{SI}$ consistently delivers balanced and often superior performance in most cases. Although $\alpha_{SI} = 0$ sometimes yields strong results on $\F$-based metrics, it significantly underperforms on the remaining metrics. The reason lies in that ignoring the background entirely (i.e., $\alpha_{SI}=0$) causes the model to overfit to the foreground regions, achieving high precision locally but poor performance globally. Additionally, $\alpha_{SI} = 1$ shows better performance compared to $\alpha_{SI} = 0$, but still falls short of our adaptive strategy. This consistently supports the effectiveness of choosing a suitable $\alpha_{SI}$.

\begin{figure}[!t]
\centering
\subfigure[MSOD]{   
\begin{minipage}{0.44\linewidth}
\includegraphics[width=\linewidth]{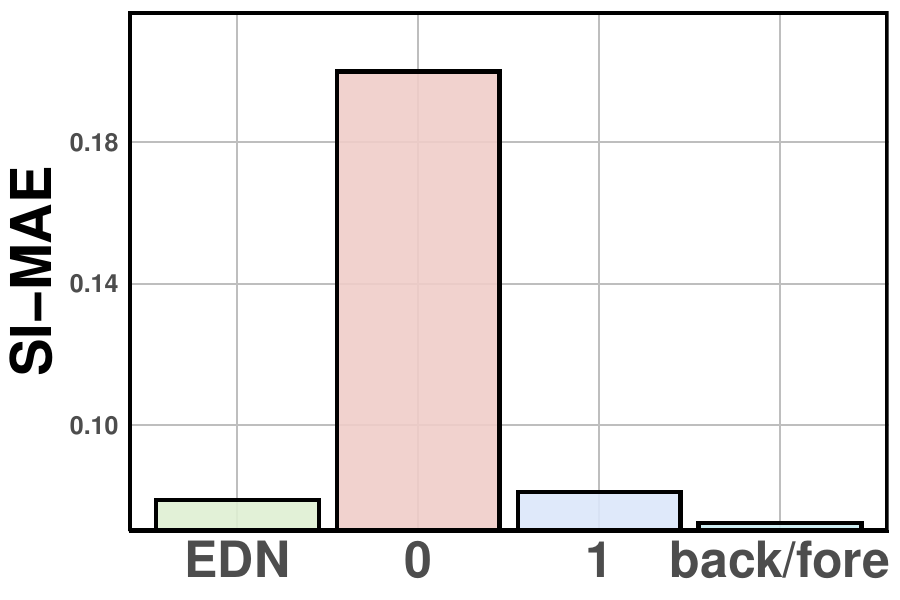}  
\label{fig:EDN_msod_ablation}
\end{minipage}
}
\subfigure[DUTS]{   
\begin{minipage}{0.44\linewidth}
\includegraphics[width=\linewidth]{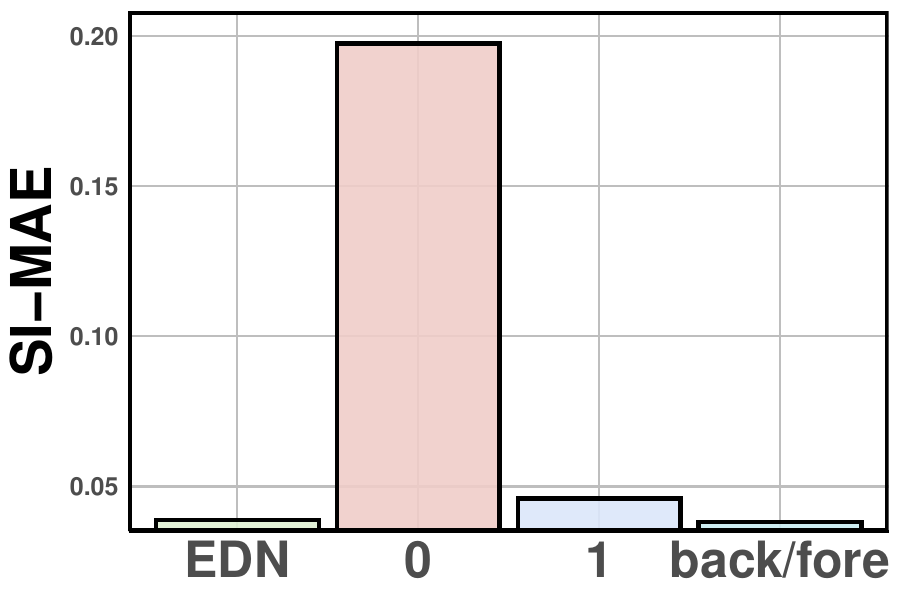}  
\label{fig:EDN_DUTS_ablation}
\end{minipage}
}
\vspace{-0.3cm}
\caption{Ablation Studies of $\alpha_{SI}$ on the DUTS-TE and MSOD datasets.}
\label{fig:ablation_SIMAE}
\vspace{-0.3cm}
\end{figure}

\begin{table}[!t]
    \centering
    \caption{Averaged training cost per epoch (mean $\pm$ std seconds).}
    \vspace{-0.2cm}
    \scalebox{0.8}{
    \begin{tabular}{c|c|c|c}
    \toprule
        Backbone & Original & $\mathsf{SIOpt1}$ & $\mathsf{SIOpt2}$ \\ 
    \midrule
    PoolNet & 523.5 $\pm$ 1.1s & 690.2 $\pm$ 1.5s & 625.4 $\pm$ 9.2s \\ 
    ADMNet &  153.6 $\pm$ 4.8s & 158.8 $\pm$ 3.6s  & 159.9 $\pm$ 4.9s\\ 
    LDF & 109.2 $\pm$ 0.6s & 244.5 $\pm$ 1.4s & 263.4 $\pm$  1.9s \\
    ICON & 162.0 $\pm$ 0.5s & 340.1 $\pm$ 0.1s & 353.1 $\pm$ 4.5s\\ 
    GateNet & 561.2 $\pm$ 0.8s & 1270.2 $\pm$ 35.3s & 505.5 $\pm$ 2.6s  \\ 
    EDN & 340.0 $\pm$ 3.3s & 543.8 $\pm$ 0.7s & 350.3 $\pm$ 3.9s  \\ 
    VST & 260.0 $\pm$ 1.3s  & 381.3 $\pm$ 5.1s & 265.7 $\pm$ 1.2s  \\ 
    \bottomrule
    \end{tabular}
    \label{tab:eff}
    }
    \vspace{-0.35cm}
\end{table}

\vspace{-2mm}
\subsubsection{Time Cost Comparison} \label{main:efficacy}

\cref{tab:eff} summarizes the training overhead comparisons among all evaluated methods. In theory, the time complexity of $\mathsf{SIOpt1}$ is $M$ times that of standard approaches such as $\BCE$ and $\mathsf{Dice}$, which aligns well with the empirical observations reported in \cref{tab:eff} ($M \approx 1.21$). Regarding $\mathsf{SIOpt2}$, although the computation of $\mathcal{L}_{\SI\AUC}$ remains at $\tilde{\mathcal{O}}(S)$ complexity, it is typically paired with auxiliary loss functions to stabilize training and boost performance (e.g., in ICON), which is a common practice in AUC-oriented optimization \cite{MAUC,DBLP:conf/iccv/Yuan0SY21}. This combination introduces a slight increase in computational overhead. Overall, the results demonstrate that the proposed $\mathsf{SIOpt}$ framework achieves promising performance with a reasonable training cost. For completeness, we also include a discussion of preprocessing time and ablation studies of PBAcc in Appendix~\ref{time_cost}.

\begin{table}[!t]
    \centering
    \caption{GPU memory usage of $\mathsf{SIOpt2}$.}
    \vspace{-0.2cm}
    \scalebox{0.8}{
      \begin{tabular}{c|cccc}
      \toprule
      Method & PoolNet & ICON  & GateNet & EDN \\
      \midrule
      $\mathsf{SIOpt2}$ w/ PBAcc & \cellcolor[rgb]{ .796,  .804,  .894}2164 MB & \cellcolor[rgb]{ .796,  .804,  .894}16956 MB & \cellcolor[rgb]{ .796,  .804,  .894}23458 MB & \cellcolor[rgb]{ .796,  .804,  .894}18644 MB \\
      $\mathsf{SIOpt2}$ w/o PBAcc & 32656 MB & /     & /     & / \\
      \bottomrule
      \end{tabular}%
    }
    \label{tab:SIOpt2-comparability}%
    \vspace{-0.2cm}
\end{table}%

\vspace{-2mm}
\subsubsection{Ablation Studies of $\mathsf{SIOpt2}$ with and without PBAcc} 
\cref{main:efficacy} has demonstrated that $\mathsf{SIOpt2}$ incurs a moderate training overhead with PBAcc. \cref{tab:SIOpt2-comparability} reports a controlled experiment result comparing the GPU memory usage of $\mathsf{SIOpt2}$ with PBAcc (w/) and without PBAcc (w/o) across different backbones. Note that, the \textbf{symbol "/"} indicates that the required memory \textbf{exceeds our available limit} (82 GB), even when using a batch size of 1. We can  see that $\mathsf{SIOpt2}$ \textbf{w/ PBAcc} is substantially more memory-efficient than its naive counterpart (\textbf{w/o PBAcc}). These findings consistently highlight the efficiency and practicality of our proposed acceleration scheme. Please refer to \cref{PBAcc_appendix} for detailed discussions.

\vspace{-1mm}

\subsubsection{Empirical Comparisons with Pixel-level AUC Methods}
\cref{fig:AUCSeg} presents a subset of qualitative results comparing our proposed 
\(\mathsf{SIOpt2}\) framework with a state-of-the-art AUC-oriented method, AUCSeg \cite{AUCSeg}. Due to space constraints, detailed discussions are provided in \cref{app:E.10}. Overall, the empirical results consistently highlight the superiority of \(\mathsf{SIOpt2}\), particularly with respect to size-invariant metrics.

\begin{figure}[!t]
    \centering
    \includegraphics[width=1.0\linewidth]{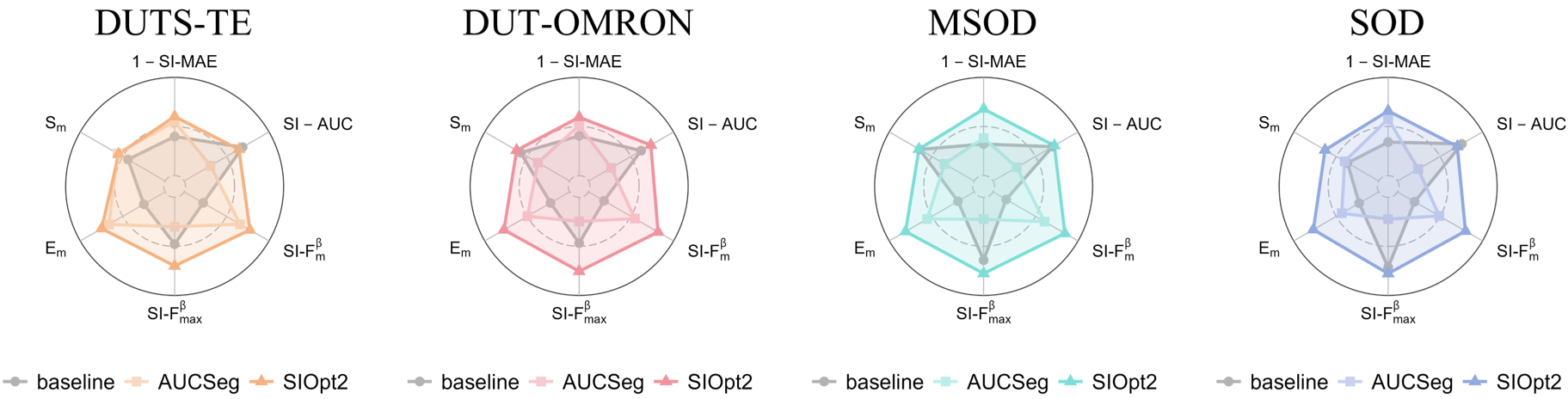} 
    \caption{Performance comparisons taking GateNet as the baseline.}    
    \label{fig:AUCSeg}    
    \vspace{-3mm}
\end{figure}

\vspace{-2mm}
\section{Conclusion}
In this paper, we address the challenge of size-invariance SOD tasks. Through precise derivations, we demonstrate that existing SOD evaluation metrics are inherently size-sensitive—tending to favor larger objects while overlooking smaller ones, especially when multiple salient objects of varying sizes appear within a single image. To tackle this issue, we propose a generic framework for size-invariant evaluation ($\mathsf{SIEva}$) and optimization ($\mathsf{SIOpt}$). Specifically, we introduce three metrics, including $\SMAE$, $\SF$, and $\SI\text{-}\AUC$, that evaluate each salient object individually before aggregating the results. In parallel, we design an end-to-end optimization framework that directly targets these metrics by adaptively balancing object-wise contributions to ensure equitable treatment across different sizes. Theoretically, we provide a generalization analysis of current SOD formulations to support the validity and effectiveness of $\mathsf{SIOpt}$. Extensive experiments on multiple benchmarks confirm the superior performance of our approach.

\vspace{-2mm}
\section*{Acknowledgements}
This work was supported in part by National Natural Science Foundation of China: 62525212, 62236008, 62025604, 62441232, 62441619, U21B2038, U23B2051, 62502496, 62206264 and 92370102, in part by Youth Innovation Promotion Association CAS, in part by the Strategic Priority Research Program of the Chinese Academy of Sciences, Grant No. XDB0680201, in part by the Postdoctoral Fellowship Program of CPSF under Grant GZB20240729, and in part by General Program of the China Postdoctoral Science Foundation under Grant No. 2025M771492.

\vspace{-2mm}
\bibliographystyle{IEEEtran}
\bibliography{main}

\begin{thebibliography}{100}
\providecommand{\url}[1]{#1}
\csname url@samestyle\endcsname
\providecommand{\newblock}{\relax}
\providecommand{\bibinfo}[2]{#2}
\providecommand{\BIBentrySTDinterwordspacing}{\spaceskip=0pt\relax}
\providecommand{\BIBentryALTinterwordstretchfactor}{4}
\providecommand{\BIBentryALTinterwordspacing}{\spaceskip=\fontdimen2\font plus
\BIBentryALTinterwordstretchfactor\fontdimen3\font minus
  \fontdimen4\font\relax}
\providecommand{\BIBforeignlanguage}[2]{{%
\expandafter\ifx\csname l@#1\endcsname\relax
\typeout{** WARNING: IEEEtran.bst: No hyphenation pattern has been}%
\typeout{** loaded for the language `#1'. Using the pattern for}%
\typeout{** the default language instead.}%
\else
\language=\csname l@#1\endcsname
\fi
#2}}
\providecommand{\BIBdecl}{\relax}
\BIBdecl

\bibitem{IDSurvey_2022}
W.~Wang, Q.~Lai, H.~Fu, J.~Shen, H.~Ling, and R.~Yang, ``Salient object
  detection in the deep learning era: An in-depth survey,'' \emph{IEEE TPAMI},
  pp. 3239--3259, 2022.

\bibitem{Borji_2019_survey}
A.~Borji, M.-M. Cheng, Q.~Hou, H.~Jiang, and J.~Li, ``Salient object detection:
  A survey,'' \emph{CVM}, pp. 117--150, 2019.

\bibitem{AUCSeg}
B.~Han, Q.~Xu, Z.~Yang, S.~Bao, P.~Wen, Y.~Jiang, and Q.~Huang, ``Aucseg:
  Auc-oriented pixel-level long-tail semantic segmentation,'' in
  \emph{NeurIPS}, 2024.

\bibitem{ji2023multispectral}
W.~Ji, J.~Li, C.~Bian, Z.~Zhou, J.~Zhao, A.~L. Yuille, and L.~Cheng,
  ``Multispectral video semantic segmentation: A benchmark dataset and
  baseline,'' in \emph{CVPR}, 2023, pp. 1094--1104.

\bibitem{DBLP:journals/pami/TanLLYYHO23}
J.~Tan, B.~Li, X.~Lu, Y.~Yao, F.~Yu, T.~He, and W.~Ouyang, ``The equalization
  losses: Gradient-driven training for long-tailed object recognition,''
  \emph{{IEEE} TPAMI}, vol.~45, no.~11, pp. 13\,876--13\,892, 2023.

\bibitem{mIOU}
Z.~Wang, M.~Berman, A.~Rannen{-}Triki, P.~H.~S. Torr, D.~Tuia, T.~Tuytelaars,
  L.~V. Gool, J.~Yu, and M.~B. Blaschko, ``Revisiting evaluation metrics for
  semantic segmentation: Optimization and evaluation of fine-grained
  intersection over union,'' in \emph{NeurIPS}, 2023.

\bibitem{DBLP:conf/iccv/Yuan0SY21}
Z.~Yuan, Y.~Yan, M.~Sonka, and T.~Yang, ``Large-scale robust deep {AUC}
  maximization: {A} new surrogate loss and empirical studies on medical image
  classification,'' in \emph{{ICCV}}, 2021, pp. 3020--3029.

\bibitem{DBLP:journals/sensors/LiHWYYL23}
J.~Li, D.~Han, X.~Wang, P.~Yi, L.~Yan, and X.~Li, ``Multi-sensor medical-image
  fusion technique based on embedding bilateral filter in least squares and
  salient detection,'' \emph{Sensors}, vol.~23, no.~7, p. 3490, 2023.

\bibitem{DBLP:journals/corr/abs-2301-13862}
B.~B. May, N.~J. Tatro, P.~Kumar, and N.~Shnidman, ``Salient conditional
  diffusion for defending against backdoor attacks,'' \emph{ArXiv}, 2023.

\bibitem{DBLP:journals/tits/ZhengLGZSWCZW24}
S.~Zheng, W.~Liu, Y.~Guo, Y.~Zang, S.~Shen, C.~Wen, M.~Cheng, P.~Zhong, and
  C.~Wang, ``Sr-adv: Salient region adversarial attacks on 3d point clouds for
  autonomous driving,'' \emph{{IEEE} TITS}, vol.~25, no.~10, pp.
  14\,019--14\,030, 2024.

\bibitem{tracking_2009}
V.~Mahadevan and N.~Vasconcelos, ``Saliency-based discriminant tracking,'' in
  \emph{CVPR}, 2009.

\bibitem{scene_2014}
Z.~Ren, S.~Gao, L.-T. Chia, and I.~W.-H. Tsang, ``Region-based saliency
  detection and its application in object recognition,'' \emph{IEEE TCSVT}, pp.
  769--779, 2014.

\bibitem{tang2017Tri}
J.~Tang, X.~Shu, G.-J. Qi, Z.~Li, M.~Wang, S.~Yan, and R.~Jain, ``Tri-clustered
  tensor completion for social-aware image tag refinement,'' \emph{IEEE TPAMI},
  vol.~39, no.~8, pp. 1662--1674, 2017.

\bibitem{li2019deep}
Z.~Li, J.~Tang, and T.~Mei, ``Deep collaborative embedding for social image
  understanding,'' \emph{IEEE TPAMI}, vol.~41, no.~9, pp. 2070--2083, 2019.

\bibitem{zhang2020causal}
D.~Zhang, H.~Zhang, J.~Tang, X.-S. Hua, and Q.~Sun, ``Causal intervention for
  weakly-supervised semantic segmentation,'' in \emph{NeurIPS}, vol.~33, 2020,
  pp. 655--666.

\bibitem{jiang2023Hierarchical}
Y.~Jiang, C.~Hua, Y.~Feng, and Y.~Gao, ``Hierarchical set-to-set representation
  for 3-d cross-modal retrieval,'' \emph{IEEE TNNLS}, pp. 1--13, 2023.

\bibitem{tang2024remote}
S.~Gui, S.~Song, R.~Qin, and Y.~Tang, ``Remote sensing object detection in the
  deep learning era—a review,'' \emph{Remote Sensing}, vol.~16, no.~2, 2024.

\bibitem{EDN}
Y.-H. Wu, Y.~Liu, L.~Zhang, M.-M. Cheng, and B.~Ren, ``Edn: Salient object
  detection via extremely-downsampled network,'' \emph{IEEE TIP}, pp.
  3125--3136, 2022.

\bibitem{Luo_2017_CVPR}
Z.~Luo, A.~Mishra, A.~Achkar, J.~Eichel, S.~Li, and P.-M. Jodoin, ``Non-local
  deep features for salient object detection,'' in \emph{CVPR}, 2017.

\bibitem{MENet}
Y.~Wang, R.~Wang, X.~Fan, T.~Wang, and X.~He, ``Pixels, regions, and objects:
  Multiple enhancement for salient object detection,'' in \emph{CVPR}, 2023,
  pp. 10\,031--10\,040.

\bibitem{Ma_2021_AAAI}
M.~Ma, C.~Xia, and J.~Li, ``Pyramidal feature shrinking for salient object
  detection,'' in \emph{AAAI}, 2021, pp. 2311--2318.

\bibitem{U-Net}
O.~Ronneberger, P.~Fischer, and T.~Brox, ``U-net: Convolutional networks for
  biomedical image segmentation,'' \emph{MICCAI}, 2015.

\bibitem{zhang_2021_nips}
J.~Zhang, J.~Xie, N.~Barnes, and P.~Li, ``Learning generative vision
  transformer with energy-based latent space for saliency prediction,''
  \emph{NeurIPS}, vol.~34, pp. 15\,448--15\,463, 2021.

\bibitem{SwinNet}
Z.~Liu, Y.~Tan, Q.~He, and Y.~Xiao, ``Swinnet: Swin transformer drives
  edge-aware {RGB-D} and {RGB-T} salient object detection,'' \emph{{IEEE}
  TCSVT}, vol.~32, no.~7, pp. 4486--4497, 2022.

\bibitem{DBLP:journals/tmm/WuHLX24}
J.~Wu, F.~Hao, W.~Liang, and J.~Xu, ``Transformer fusion and pixel-level
  contrastive learning for {RGB-D} salient object detection,'' \emph{{IEEE}
  TMM}, vol.~26, pp. 1011--1026, 2024.

\bibitem{VST2}
N.~Liu, Z.~Luo, N.~Zhang, and J.~Han, ``{VST++:} efficient and stronger visual
  saliency transformer,'' \emph{{IEEE} TPAMI}, vol.~46, no.~11, pp. 7300--7316,
  2024.

\bibitem{SAM}
A.~Kirillov, E.~Mintun, N.~Ravi, H.~Mao, C.~Rolland, L.~Gustafson, T.~Xiao,
  S.~Whitehead, A.~C. Berg, W.-Y. Lo, P.~Dollár, and R.~Girshick, ``Segment
  anything,'' 2023.

\bibitem{SAM1}
R.~Cui, S.~He, and S.~Qiu, ``Adaptive low rank adaptation of segment anything
  to salient object detection,'' \emph{Arxiv}, 2023.

\bibitem{SAMAd}
T.~Chen, L.~Zhu, C.~Ding, R.~Cao, Y.~Wang, S.~Zhang, Z.~Li, L.~Sun, Y.~Zang,
  and P.~Mao, ``Sam-adapter: Adapting segment anything in underperformed
  scenes,'' in \emph{ICCV}, 2023, pp. 3359--3367.

\bibitem{SAM3}
Y.~Yu, C.~Xu, and K.~Wang, ``{TS-SAM:} fine-tuning segment-anything model for
  downstream tasks,'' in \emph{ICME}, 2024, pp. 1--6.

\bibitem{pAUC}
Z.~Yang, Q.~Xu, S.~Bao, Y.~He, X.~Cao, and Q.~Huang, ``Optimizing two-way
  partial auc with an end-to-end framework,'' \emph{IEEE TPAMI}, vol.~45,
  no.~8, pp. 10\,228--10\,246, 2023.

\bibitem{chen_2021_ijcv}
H.~Chen, Y.~Li, Y.~Deng, and G.~Lin, ``Cnn-based rgb-d salient object
  detection: Learn, select, and fuse,'' \emph{IJCV}, vol. 129, no.~7, pp.
  2076--2096, 2021.

\bibitem{sun_2022_ijcv}
P.~Sun, W.~Zhang, S.~Li, Y.~Guo, C.~Song, and X.~Li, ``Learnable
  depth-sensitive attention for deep rgb-d saliency detection with multi-modal
  fusion architecture search,'' \emph{IJCV}, vol. 130, no.~11, pp. 2822--2841,
  2022.

\bibitem{MAUC}
Z.~Yang, Q.~Xu, S.~Bao, X.~Cao, and Q.~Huang, ``Learning with multiclass {AUC:}
  theory and algorithms,'' \emph{{IEEE} TPAMI}, vol.~44, no.~11, pp.
  7747--7763, 2022.

\bibitem{AUC_eval}
A.~Borji, H.~R. Tavakoli, D.~N. Sihite, and L.~Itti, ``Analysis of scores,
  datasets, and models in visual saliency prediction,'' in \emph{ICCV}, 2013,
  pp. 921--928.

\bibitem{AUC_rank}
S.~J. Mason and N.~E. Graham, ``Areas beneath the relative operating
  characteristics (roc) and relative operating levels (rol) curves: Statistical
  significance and interpretation,'' \emph{Quarterly Journal of the Royal
  Meteorological Society}, vol. 128, no. 584, pp. 2145--2166, 2002.

\bibitem{F_measure}
R.~Achanta, S.~Hemami, F.~Estrada, and S.~Susstrunk, ``Frequency-tuned salient
  region detection,'' in \emph{CVPR}, 2009.

\bibitem{Emeasure}
D.-P. Fan, C.~Gong, Y.~Cao, B.~Ren, M.-M. Cheng, and A.~Borji,
  ``Enhanced-alignment measure for binary foreground map evaluation,'' in
  \emph{IJCAI}, 2018, pp. 698--704.

\bibitem{S_measure1}
D.~Fan, M.~Cheng, Y.~Liu, T.~Li, and A.~Borji, ``Structure-measure: {A} new way
  to evaluate foreground maps,'' in \emph{ICCV}, 2017, pp. 4558--4567.

\bibitem{S_measure2}
M.~Cheng and D.~Fan, ``Structure-measure: {A} new way to evaluate foreground
  maps,'' \emph{IJCV}, vol. 129, no.~9, pp. 2622--2638, 2021.

\bibitem{SISOD}
F.~Li, Q.~Xu, S.~Bao, Z.~Yang, R.~Cong, X.~Cao, and Q.~Huang, ``Size-invariance
  matters: Rethinking metrics and losses for imbalanced multi-object salient
  object detection,'' in \emph{{ICML}}, 2024.

\bibitem{DBLP:journals/tcsv/WangTLZL24}
K.~Wang, Z.~Tu, C.~Li, C.~Zhang, and B.~Luo, ``Learning adaptive fusion bank
  for multi-modal salient object detection,'' \emph{{IEEE} TCSVT}, vol.~34,
  no.~8, pp. 7344--7358, 2024.

\bibitem{DBLP:journals/pami/ChengGBTLW22}
M.~Cheng, S.~Gao, A.~Borji, Y.~Tan, Z.~Lin, and M.~Wang, ``A highly efficient
  model to study the semantics of salient object detection,'' \emph{IEEE
  TPAMI}, vol.~44, no.~11, pp. 8006--8021, 2022.

\bibitem{DBLP:journals/tnn/CongHLZHK24}
R.~Cong, K.~Huang, J.~Lei, Y.~Zhao, Q.~Huang, and S.~Kwong, ``Multi-projection
  fusion and refinement network for salient object detection in
  360{\textdegree} omnidirectional image,'' \emph{IEEE TNNLS}, vol.~35, no.~7,
  pp. 9495--9507, 2024.

\bibitem{DBLP:journals/pami/WuWWLLXCHL25}
Z.~Wu, W.~Wang, L.~Wang, Y.~Li, F.~Lv, Q.~Xia, C.~Chen, A.~Hao, and S.~Li,
  ``Pixel is all you need: Adversarial spatio-temporal ensemble active learning
  for salient object detection,'' \emph{IEEE TPAMI}, vol.~47, no.~2, pp.
  858--877, 2025.

\bibitem{sun2025conditional}
K.~Sun, Z.~Chen, X.~Lin, X.~Sun, H.~Liu, and R.~Ji, ``Conditional diffusion
  models for camouflaged and salient object detection,'' \emph{IEEE TPAMI},
  2025.

\bibitem{DBLP:journals/tip/LiSXMT21}
J.~Li, J.~Su, C.~Xia, M.~Ma, and Y.~Tian, ``Salient object detection with
  purificatory mechanism and structural similarity loss,'' \emph{IEEE TIP},
  vol.~30, pp. 6855--6868, 2021.

\bibitem{DBLP:conf/aaai/XuLL021}
B.~Xu, H.~Liang, R.~Liang, and P.~Chen, ``Locate globally, segment locally: {A}
  progressive architecture with knowledge review network for salient object
  detection,'' in \emph{AAAI}, 2021, pp. 3004--3012.

\bibitem{DBLP:journals/tnn/ZhangLT23}
J.~Zhang, T.~Liu, and D.~Tao, ``An optimal transport analysis on generalization
  in deep learning,'' \emph{IEEE TNNLS}, vol.~34, no.~6, pp. 2842--2853, 2023.

\bibitem{DBLP:conf/iclr/XiaL0GY0S22}
X.~Xia, T.~Liu, B.~Han, M.~Gong, J.~Yu, G.~Niu, and M.~Sugiyama, ``Sample
  selection with uncertainty of losses for learning with noisy labels,'' in
  \emph{{ICLR}}, 2022.

\bibitem{DBLP:conf/pkdd/HanTC16}
B.~Han, I.~W. Tsang, and L.~Chen, ``On the convergence of a family of robust
  losses for stochastic gradient descent,'' in \emph{{ECML} {PKDD}}, 2016, pp.
  665--680.

\bibitem{honovo}
Z.~Y. Ho, S.~Liang, S.~Zhang, Y.~Zhan, and D.~Tao, ``Novo: Norm voting off
  hallucinations with attention heads in large language models,'' in
  \emph{ICLR}, 2025.

\bibitem{reniclshield}
Z.~Ren, S.~Liang, A.~Liu, and D.~Tao, ``Iclshield: Exploring and mitigating
  in-context learning backdoor attacks,'' in \emph{ICML}, 2025.

\bibitem{DBLP:conf/iclr/XiaL0S0L23}
X.~Xia, J.~Liu, J.~Yu, X.~Shen, B.~Han, and T.~Liu, ``Moderate coreset: {A}
  universal method of data selection for real-world data-efficient deep
  learning,'' in \emph{{ICLR}}, 2023.

\bibitem{DBLP:conf/igarss/WangKCWZGL21}
Z.~Wang, X.~Kong, Z.~Cui, M.~Wu, C.~Zhang, M.~Gong, and T.~Liu, ``Vecnet: {A}
  spectral and multi-scale spatial fusion deep network for pixel-level cloud
  type classification in himawari-8 imagery,'' in \emph{{IEEE} {IGARSS}}, 2021,
  pp. 4083--4086.

\bibitem{DBLP:journals/eswa/WangZZGYCLYG25}
H.~Wang, Y.~Zhao, F.~Zhang, G.~Gui, L.~Yu, B.~Chen, M.~Liao, C.~Yang, and
  W.~Gui, ``R-net: Recursive decoder with edge refinement network for salient
  object detection,'' \emph{Expert Syst. Appl.}, vol. 261, p. 125562, 2025.

\bibitem{DBLP:journals/tim/HeSXCH25}
Z.~He, F.~Shao, Z.~Xie, X.~Chai, and Y.~Ho, ``Sihenet: Semantic interaction and
  hierarchical embedding network for 360{\textdegree} salient object
  detection,'' \emph{{IEEE} TIM}, vol.~74, pp. 1--15, 2025.

\bibitem{DBLP:journals/pr/ZhuQE25}
J.~Zhu, X.~Qin, and A.~Elsaddik, ``Dc-net: Divide-and-conquer for salient
  object detection,'' \emph{PR}, vol. 157, p. 110903, 2025.

\bibitem{UCF}
P.~Zhang, D.~Wang, H.~Lu, H.~Wang, and B.~Yin, ``Learning uncertain
  convolutional features for accurate saliency detection,'' in \emph{ICCV},
  2017, pp. 212--221.

\bibitem{DCL}
G.~Li and Y.~Yu, ``Deep contrast learning for salient object detection,'' in
  \emph{CVPR}, 2016, pp. 478--487.

\bibitem{PiCANet}
N.~Liu, J.~Han, and M.-H. Yang, ``Picanet: Learning pixel-wise contextual
  attention for saliency detection,'' in \emph{CVPR}, 2018.

\bibitem{RDCPN}
Y.-H. Wu, Y.~Liu, L.~Zhang, W.~Gao, and M.-M. Cheng, ``Regularized
  densely-connected pyramid network for salient instance segmentation,''
  \emph{IEEE TIP}, vol.~30, pp. 3897--3907, 2021.

\bibitem{LDF}
J.~Wei, S.~Wang, Z.~Wu, C.~Su, Q.~Huang, and Q.~Tian, ``Label decoupling
  framework for salient object detection,'' in \emph{CVPR}, 2020.

\bibitem{ICON}
M.~Zhuge, D.-P. Fan, N.~Liu, D.~Zhang, D.~Xu, and L.~Shao, ``Salient object
  detection via integrity learning,'' \emph{IEEE TPAMI}, 2022.

\bibitem{Bi-Directional}
L.~Zhang, J.~Dai, H.~Lu, Y.~He, and G.~Wang, ``A bi-directional message passing
  model for salient object detection,'' in \emph{CVPR}, 2018.

\bibitem{CANet}
Q.~Ren, S.~Lu, J.~Zhang, and R.~Hu, ``Salient object detection by fusing local
  and global contexts,'' \emph{IEEE TMM}, vol.~23, pp. 1442--1453, 2021.

\bibitem{piao2019depth}
Y.~Piao, W.~Ji, J.~Li, M.~Zhang, and H.~Lu, ``Depth-induced multi-scale
  recurrent attention network for saliency detection,'' in \emph{ICCV}, 2019,
  pp. 7254--7263.

\bibitem{ji2022dmra}
W.~Ji, G.~Yan, J.~Li, Y.~Piao, S.~Yao, M.~Zhang, L.~Cheng, and H.~Lu, ``Dmra:
  Depth-induced multi-scale recurrent attention network for rgb-d saliency
  detection,'' \emph{IEEE TIP}, vol.~31, pp. 2321--2336, 2022.

\bibitem{ADMNet}
X.~Zhou, K.~Shen, and Z.~Liu, ``Admnet: Attention-guided densely multi-scale
  network for lightweight salient object detection,'' \emph{{IEEE} TMM},
  vol.~26, pp. 10\,828--10\,841, 2024.

\bibitem{TriTransNet}
Z.~Liu, Y.~Wang, Z.~Tu, Y.~Xiao, and B.~Tang, ``Tritransnet: {RGB-D} salient
  object detection with a triplet transformer embedding network,'' in \emph{ACM
  MM}, 2021, pp. 4481--4490.

\bibitem{VST}
N.~Liu, N.~Zhang, K.~Wan, L.~Shao, and J.~Han, ``Visual saliency transformer,''
  in \emph{ICCV}, 2021, pp. 4702--4712.

\bibitem{PoolNet}
J.-J. Liu, Q.~Hou, M.-M. Cheng, J.~Feng, and J.~Jiang, ``A simple pooling-based
  design for real-time salient object detection,'' in \emph{CVPR}, 2019.

\bibitem{ji2020accurate}
W.~Ji, J.~Li, M.~Zhang, Y.~Piao, and H.~Lu, ``Accurate rgb-d salient object
  detection via collaborative learning,'' in \emph{ECCV}, 2020, pp. 52--69.

\bibitem{zhang2023c}
M.~Zhang, S.~Yao, B.~Hu, Y.~Piao, and W.~Ji, ``C2dfnet: Criss-cross dynamic
  filter network for rgb-d salient object detection,'' \emph{IEEE TMM},
  vol.~25, pp. 5142--5154, 2023.

\bibitem{li2023dvsod}
L.~Jingjing, J.~Wei, W.~Size, L.~Wenbo, and L.~Cheng, ``Dvsod: Rgb-d video
  salient object detection,'' in \emph{NeurIPS}, 2023, pp. 8774--8787.

\bibitem{li2023delving}
J.~Li, W.~Ji, M.~Zhang, Y.~Piao, H.~Lu, and L.~Cheng, ``Delving into calibrated
  depth for accurate rgb-d salient object detection,'' \emph{IJCV}, vol. 131,
  no.~4, pp. 855--876, 2023.

\bibitem{zhang2019memory}
M.~Zhang, J.~Li, J.~Wei, Y.~Piao, and H.~Lu, ``Memory-oriented decoder for
  light field salient object detection,'' \emph{NeurIPS}, pp. 896--906, 2019.

\bibitem{zhang2020lfnet}
M.~Zhang, W.~Ji, Y.~Piao, J.~Li, Y.~Zhang, S.~Xu, and H.~Lu, ``Lfnet: Light
  field fusion network for salient object detection,'' \emph{IEEE TIP},
  vol.~29, pp. 6276--6287, 2020.

\bibitem{zhang2021dynamic}
M.~Zhang, J.~Liu, Y.~Wang, Y.~Piao, S.~Yao, W.~Ji, J.~Li, H.~Lu, and Z.~Luo,
  ``Dynamic context-sensitive filtering network for video salient object
  detection,'' in \emph{ICCV}, 2021, pp. 1553--1563.

\bibitem{DBLP:journals/ijon/JiangYHLN25}
Z.~Jiang, L.~Yu, Y.~Han, J.~Li, and F.~Niu, ``Global-aware interaction network
  for {RGB-D} salient object detection,'' \emph{Neur.comp.}, vol. 621, p.
  129204, 2025.

\bibitem{DBLP:journals/tip/HaoYLXYY25}
C.~Hao, Z.~Yu, X.~Liu, J.~Xu, H.~Yue, and J.~Yang, ``A simple yet effective
  network based on vision transformer for camouflaged object and salient object
  detection,'' \emph{{IEEE} TIP}, vol.~34, pp. 608--622, 2025.

\bibitem{SOD}
V.~Movahedi and J.~H. Elder, ``Design and perceptual validation of performance
  measures for salient object segmentation,'' in \emph{CVPR workshop}, 2010,
  pp. 49--56.

\bibitem{DBLP:journals/cgf/ZhangHCWHL24}
S.~Zhang, J.~Huang, S.~Chen, Y.~Wu, T.~Hu, and J.~Liu, ``Sod-diffusion: Salient
  object detection via diffusion-based image generators,'' \emph{CGF}, vol.~43,
  no.~7, pp. i--xxii, 2024.

\bibitem{MAE}
F.~Perazzi, P.~Krahenbuhl, Y.~Pritch, and A.~Hornung, ``Saliency filters:
  Contrast based filtering for salient region detection,'' in \emph{CVPR},
  2012.

\bibitem{F-beta}
R.~Margolin, L.~Zelnik-Manor, and A.~Tal, ``How to evaluate foreground maps,''
  in \emph{CVPR}, 2014.

\bibitem{everingham2015pascal}
M.~Everingham, S.~A. Eslami, L.~Van~Gool, C.~K. Williams, J.~Winn, and
  A.~Zisserman, ``The pascal visual object classes challenge: A
  retrospective,'' \emph{IJCV}, vol. 111, pp. 98--136, 2015.

\bibitem{DBLP:journals/tip/BorjiCJL15}
A.~Borji, M.~Cheng, H.~Jiang, and J.~Li, ``Salient object detection: {A}
  benchmark,'' \emph{{IEEE} TIP}, vol.~24, no.~12, pp. 5706--5722, 2015.

\bibitem{DBLP:journals/tip/LiZWYWZLW16}
X.~Li, L.~Zhao, L.~Wei, M.~Yang, F.~Wu, Y.~Zhuang, H.~Ling, and J.~Wang,
  ``Deepsaliency: Multi-task deep neural network model for salient object
  detection,'' \emph{{IEEE} TIP}, vol.~25, no.~8, pp. 3919--3930, 2016.

\bibitem{SFCML}
S.~Bao, Q.~Xu, Z.~Yang, X.~Cao, and Q.~Huang, ``Rethinking collaborative metric
  learning: Toward an efficient alternative without negative sampling,''
  \emph{{IEEE} TPAMI}, vol.~45, no.~1, pp. 1017--1035, 2023.

\bibitem{DBLP:journals/csur/YangY23}
T.~Yang and Y.~Ying, ``{AUC} maximization in the era of big data and {AI:} {A}
  survey,'' \emph{{ACM} Comput. Surv.}, vol.~55, no.~8, pp. 172:1--172:37,
  2023.

\bibitem{DBLP:conf/aistats/YaoLY23}
Y.~Yao, Q.~Lin, and T.~Yang, ``Stochastic methods for {AUC} optimization
  subject to auc-based fairness constraints,'' in \emph{AISTATS}, 2023, pp.
  10\,324--10\,342.

\bibitem{DBLP:conf/nips/Luo0Y00024}
J.~Luo, F.~Hong, J.~Yao, B.~Han, Y.~Zhang, and Y.~Wang, ``Revive re-weighting
  in imbalanced learning by density ratio estimation,'' in \emph{NeurIPS},
  2024.

\bibitem{SSIM}
Z.~Wang, A.~Bovik, H.~Sheikh, and E.~Simoncelli, ``Image quality assessment:
  from error visibility to structural similarity,'' \emph{IEEE TIP}, vol.~13,
  no.~4, pp. 600--612, 2004.

\bibitem{Object-Detection-in-Aerial-Images}
J.~Ding, N.~Xue, G.-S. Xia, X.~Bai, W.~Yang, M.~Y. Yang, S.~Belongie, J.~Luo,
  M.~Datcu, M.~Pelillo, and L.~Zhang, ``Object detection in aerial images: A
  large-scale benchmark and challenges,'' \emph{IEEE TPAMI}, vol.~44, pp.
  7778--7796, 2022.

\bibitem{DBLP:journals/tip/BaoZZCYZY25}
L.~Bao, X.~Zhou, B.~Zheng, R.~Cong, H.~Yin, J.~Zhang, and C.~Yan, ``Ifenet:
  Interaction, fusion, and enhancement network for {V-D-T} salient object
  detection,'' \emph{{IEEE} TIP}, vol.~34, pp. 483--494, 2025.

\bibitem{GateNet}
X.~Zhao, Y.~Pang, L.~Zhang, H.~Lu, and L.~Zhang, ``Suppress and balance: A
  simple gated network for salient object detection,'' in \emph{ECCV}, 2020.

\bibitem{RDSN}
S.~Jia and N.~D.~B. Bruce, ``Richer and deeper supervision network for salient
  object detection,'' \emph{ArXiv}, 2019.

\bibitem{DiceLoss}
F.~Milletari, N.~Navab, and S.-A. Ahmadi, ``V-net: Fully convolutional neural
  networks for volumetric medical image segmentation,'' in \emph{3DV}, 2016,
  pp. 565--571.

\bibitem{IOULoss}
J.~Yu, Y.~Jiang, Z.~Wang, Z.~Cao, and T.~Huang, ``Unitbox: An advanced object
  detection network,'' in \emph{ACM MM}, 2016, pp. 516--520.

\bibitem{DBLP:conf/aistats/RosenfeldMTG14}
N.~Rosenfeld, O.~Meshi, D.~Tarlow, and A.~Globerson, ``Learning structured
  models with the {AUC} loss and its generalizations,'' in \emph{{AISTATS}},
  2014, pp. 841--849.

\bibitem{Borji2013Quantitative}
A.~Borji, D.~N. Sihite, and L.~Itti, ``Quantitative analysis of human-model
  agreement in visual saliency modeling: A comparative study,'' \emph{IEEE
  TIP}, vol.~22, no.~1, pp. 55--69, 2013.

\bibitem{Bylinskii2019Measurement}
Z.~Bylinskii, T.~Judd, A.~Oliva, A.~Torralba, and F.~Durand, ``What do
  different evaluation metrics tell us about saliency models?'' \emph{IEEE
  TPAMI}, vol.~41, no.~3, pp. 740--757, 2019.

\bibitem{RCNN}
R.~Girshick, J.~Donahue, T.~Darrell, and J.~Malik, ``Rich feature hierarchies
  for accurate object detection and semantic segmentation,'' in \emph{CVPR},
  2014, pp. 580--587.

\bibitem{PoolNet+}
J.-J. Liu, Q.~Hou, Z.-A. Liu, and M.-M. Cheng, ``Poolnet+: Exploring the
  potential of pooling for salient object detection,'' \emph{IEEE TPAMI},
  vol.~45, no.~1, pp. 887--904, 2023.

\bibitem{DBLP:conf/aaai/YangKVY23}
Z.~Yang, Y.~L. Ko, K.~R. Varshney, and Y.~Ying, ``Minimax {AUC} fairness:
  Efficient algorithm with provable convergence,'' in \emph{{AAAI}}, 2023, pp.
  11\,909--11\,917.

\bibitem{DBLP:conf/nips/CortesM03}
C.~Cortes and M.~Mohri, ``{AUC} optimization vs. error rate minimization,'' in
  \emph{NeurIPS}, 2003, pp. 313--320.

\bibitem{DBLP:journals/pami/0001L00Z24}
Z.~Xie, Y.~Liu, H.~He, M.~Li, and Z.~Zhou, ``Weakly supervised {AUC}
  optimization: {A} unified partial {AUC} approach,'' \emph{{IEEE} TPAMI},
  vol.~46, no.~7, pp. 4780--4795, 2024.

\bibitem{FasterRCNN}
S.~Ren, K.~He, R.~Girshick, and J.~Sun, ``Faster r-cnn: Towards real-time
  object detection with region proposal networks,'' \emph{NeurIPS}, 2015.

\bibitem{YOLO}
J.~Redmon, S.~Divvala, R.~Girshick, and A.~Farhadi, ``You only look once:
  Unified, real-time object detection,'' in \emph{CVPR}, 2016, pp. 779--788.

\bibitem{Few-Shot-Object-Detection}
Y.~Xiao and R.~Marlet, ``Few-shot object detection and viewpoint estimation for
  objects in the wild,'' in \emph{ECCV}, 2020, pp. 192--210.

\bibitem{DBLP:journals/pami/TangSQLWYJ17}
J.~Tang, X.~Shu, G.~Qi, Z.~Li, M.~Wang, S.~Yan, and R.~C. Jain, ``Tri-clustered
  tensor completion for social-aware image tag refinement,'' \emph{TPAMI},
  vol.~39, no.~8, pp. 1662--1674, 2017.

\bibitem{DBLP:journals/tnn/GultekinSRP20}
S.~Gultekin, A.~Saha, A.~Ratnaparkhi, and J.~W. Paisley, ``{MBA:} mini-batch
  {AUC} optimization,'' \emph{{IEEE} TNNLS}, vol.~31, no.~12, pp. 5561--5574,
  2020.

\bibitem{Carlo_2020_nips}
C.~Ciliberto, L.~Rosasco, and A.~Rudi, ``A general framework for consistent
  structured prediction with implicit loss embeddings,'' \emph{JMLR}, vol.~21,
  no.~1, pp. 3852--3918, 2020.

\bibitem{li_2021_nips}
J.~Li, W.~Ji, Q.~Bi, C.~Yan, M.~Zhang, Y.~Piao, H.~Lu \emph{et~al.}, ``Joint
  semantic mining for weakly supervised rgb-d salient object detection,''
  \emph{NeurIPS}, pp. 11\,945--11\,959, 2021.

\bibitem{Probability1991}
M.~T. Michel~Ledoux, \emph{Probability in Banach Spaces}, 1991.

\bibitem{Maurer_vector_2016}
A.~Maurer, ``A vector-contraction inequality for rademacher complexities,'' in
  \emph{Algorithmic Learning Theory}, 2016, pp. 3--17.

\bibitem{vector_Contraction}
D.~J. Foster and A.~Rakhlin, ``$\ell_{\infty}$ vector contraction for
  rademacher complexity,'' 2019.

\bibitem{p_Lipschitz}
K.~Dembczy{\'{n}}ski, W.~Kot{\l}owski, O.~Koyejo, and N.~Natarajan,
  ``Consistency analysis for binary classification revisited,'' in \emph{ICML},
  2017, pp. 961--969.

\bibitem{CRNet}
G.~Zong, L.~Wei, S.~Guo, and Y.~Wang, ``A cascaded refined rgb-d salient object
  detection network based on the attention mechanism,'' \emph{Appl. Intell.},
  vol.~53, no.~11, pp. 13\,527--13\,548, 2023.

\bibitem{hu2024cross}
X.~Hu, F.~Sun, J.~Sun, F.~Wang, and H.~Li, ``Cross-modal fusion and progressive
  decoding network for rgb-d salient object detection,'' \emph{IJCV}, vol. 132,
  no.~8, pp. 3067--3085, 2024.

\bibitem{TNet}
R.~Cong, K.~Zhang, C.~Zhang, F.~Zheng, Y.~Zhao, Q.~Huang, and S.~Kwong, ``Does
  thermal really always matter for {RGB-T} salient object detection?''
  \emph{TMM}, 2022.

\bibitem{DCNet}
Z.~Tu, Z.~Li, C.~Li, and J.~Tang, ``Weakly alignment-free {RGBT} salient object
  detection with deep correlation network,'' \emph{TIP}, vol.~31, pp.
  3752--3764, 2022.

\bibitem{PiCANet_2018}
N.~Liu, J.~Han, and M.-H. Yang, ``Picanet: Pixel-wise contextual attention
  learning for accurate saliency detection,'' \emph{IEEE TIP}, 2018.

\bibitem{Amulet_2017}
P.~Zhang, D.~Wang, H.~Lu, H.~Wang, and X.~Ruan, ``Amulet: Aggregating
  multi-level convolutional features for salient object detection,''
  \emph{arXiv}, 2017.

\bibitem{DUTS}
L.~Wang, H.~Lu, Y.~Wang, M.~Feng, D.~Wang, B.~Yin, and X.~Ruan, ``Learning to
  detect salient objects with image-level supervision,'' in \emph{CVPR}, 2017.

\bibitem{DUT-OMRON}
C.~Yang, L.~Zhang, H.~Lu, s.~Ruan, and M.-H. Yang, ``Saliency detection via
  graph-based manifold ranking,'' in \emph{CVPR}, 2013, pp. 3166--3173.

\bibitem{MSOD}
B.~Deng, A.~P. French, and M.~P. Pound, ``Addressing multiple salient object
  detection via dual-space long-range dependencies,'' \emph{CVIU}, vol. 235, p.
  103776, 2023.

\bibitem{ECSSD}
Q.~Yan, L.~Xu, J.~Shi, and J.~Jia, ``Hierarchical saliency detection,'' in
  \emph{CVPR}, 2013.

\bibitem{HKU-IS}
G.~Li and Z.~Yu, ``Visual saliency based on multiscale deep features,'' in
  \emph{CVPR}, 2015, pp. 5455--5463.

\bibitem{XPIE}
C.~Xia, J.~Li, X.~Chen, A.~Zheng, and Y.~Zhang, ``What is and what is not a
  salient object? learning salient object detector by ensembling linear
  exemplar regressors,'' in \emph{CVPR}, 2017, pp. 4399--4407.

\bibitem{ImageNet}
J.~Deng, W.~Dong, R.~Socher, L.-J. Li, K.~Li, and L.~Fei-Fei, ``Imagenet: A
  large-scale hierarchical image database,'' in \emph{CVPR}, 2009, pp.
  248--255.

\bibitem{SUN}
J.~Xiao, J.~Hays, K.~A. Ehinger, A.~Oliva, and A.~Torralba, ``Sun database:
  Large-scale scene recognition from abbey to zoo,'' in \emph{CVPR}, 2010, pp.
  3485--3492.

\bibitem{SOD_construct}
D.~Martin, C.~Fowlkes, D.~Tal, and J.~Malik, ``A database of human segmented
  natural images and its application to evaluating segmentation algorithms and
  measuring ecological statistics,'' in \emph{ICCV}, 2002.

\bibitem{NLPR}
H.~Peng, B.~Li, W.~Xiong, W.~Hu, and R.~Ji, ``Rgbd salient object detection: A
  benchmark and algorithms,'' in \emph{ECCV}, 2014, pp. 92--109.

\bibitem{STERE}
Y.~Niu, Y.~Geng, X.~Li, and F.~Liu, ``Leveraging stereopsis for saliency
  analysis,'' in \emph{CVPR}, 2012, pp. 454--461.

\bibitem{VT821}
G.~Wang, C.~Li, Y.~Ma, A.~Zheng, J.~Tang, and B.~Luo, ``{RGB-T} saliency
  detection benchmark: Dataset, baselines, analysis and a novel approach,'' in
  \emph{{IGTA}}, 2018, pp. 359--369.

\bibitem{VT1000}
Z.~Tu, T.~Xia, C.~Li, X.~Wang, Y.~Ma, and J.~Tang, ``{RGB-T} image saliency
  detection via collaborative graph learning,'' \emph{TMM}, vol.~22, no.~1, pp.
  160--173, 2020.

\bibitem{VT5000}
Z.~Tu, Y.~Ma, Z.~Li, C.~Li, J.~Xu, and Y.~Liu, ``Rgbt salient object detection:
  A large-scale dataset and benchmark,'' \emph{IEEE Transactions on
  Multimedia}, vol.~25, pp. 4163--4176, 2022.

\end{thebibliography}

\clearpage
\begin{IEEEbiography}[{\includegraphics[width=1in,height=1.25in,clip,keepaspectratio]{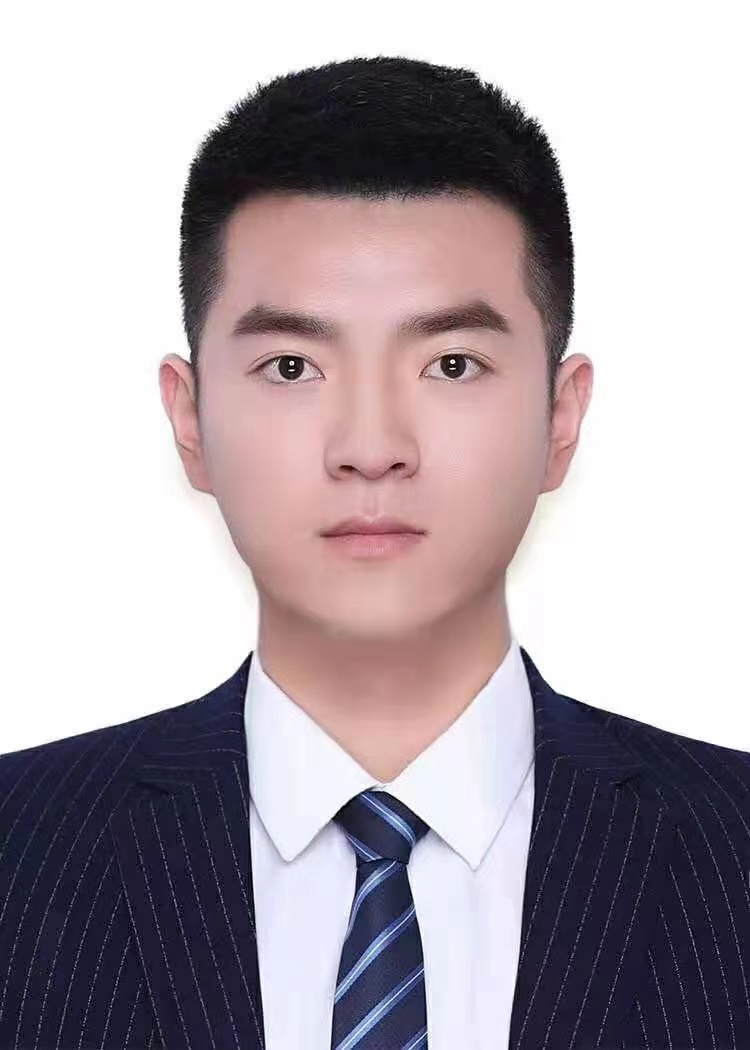}}]{\textbf{Shilong Bao}} 
	received the B.S. degree from the College of Computer Science and Technology at Qingdao University in 2019 and the Ph.D. degree from the Institute of Information Engineering, Chinese Academy of Sciences (IIE, CAS) in 2024. He is currently a postdoctoral research fellow with the University of Chinese Academy of Sciences (UCAS). His research interests are machine learning and data mining. He has authored or co-authored several academic papers in top-tier international conferences and journals including T-PAMI/ICML/NeurIPS. He served as a reviewer for several top-tier journals and conferences such as ICML, NeurIPS, ICLR, CVPR, ICCV, AAAI, IEEE Transactions on Circuits and Systems for Video Technology, and IEEE Transactions on Multimedia. 
\end{IEEEbiography}

\begin{IEEEbiography}
	[{\includegraphics[width=1in,height=1.25in,clip,keepaspectratio]{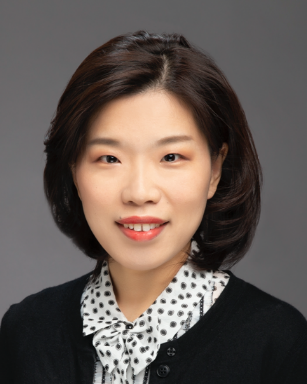}}]{Qianqian Xu} received the B.S. degree in computer science from China University of Mining and Technology in 2007 and the Ph.D. degree in computer science from University of Chinese Academy of Sciences in 2013. She is currently a Professor at the Institute of Computing Technology, Chinese Academy of Sciences, Beijing, China. Her research interests include statistical machine learning, with applications in multimedia and computer vision. She has authored or coauthored 100+ academic papers in prestigious international journals and conferences (including T-PAMI, IJCV, T-IP, NeurIPS, ICML, CVPR, AAAI, etc). Moreover, she serves as an associate editor of IEEE Transactions on Circuits and Systems for Video Technology, IEEE Transactions on Multimedia, and ACM Transactions on Multimedia Computing, Communications, and Applications.
\end{IEEEbiography}

\begin{IEEEbiography}[{\includegraphics[width=1in,height=1.25in,clip,keepaspectratio]{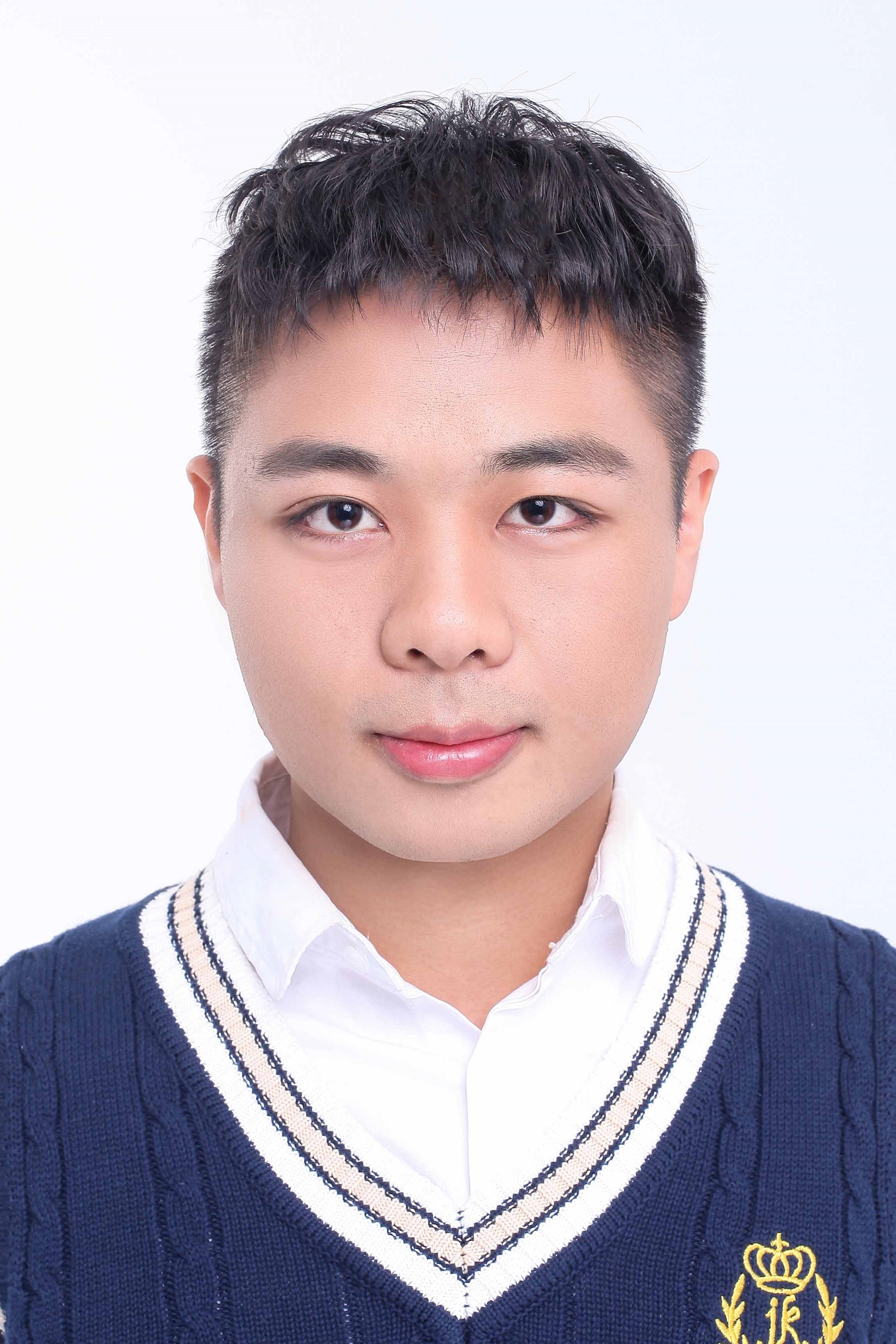}}]{Feiran Li} received the B.S. degree from the China University of Mining and Technology (CUMT), in 2023. He is currently working toward the Ph.D. degree with the Institute of Information Engineering, Chinese Academy of Sciences (IIE, CAS). His research interests are machine learning and computer vision. He has authored or co-authored several academic papers in top-tier international conferences, including ICML. He served as a reviewer for several top-tier conferences and journals such as ICLR and IEEE Transactions on Multimedia.
\end{IEEEbiography}

\begin{IEEEbiography}[{\includegraphics[width=1in,height=1.25in,clip,keepaspectratio]{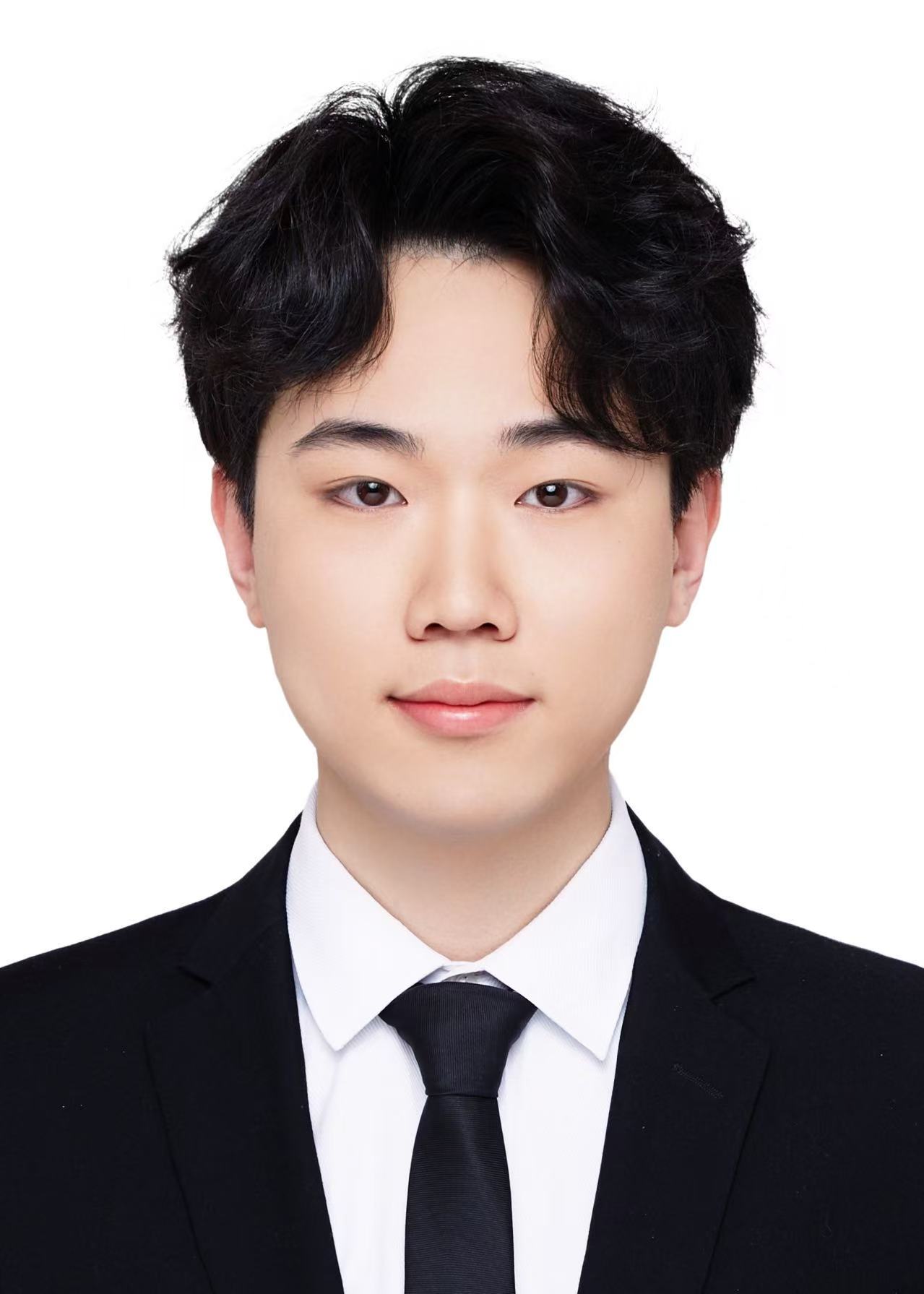}}]{Boyu Han} received the B.S. degree from the University of Electronic Science and Technology of China (UESTC), in 2023. He is currently working toward the Ph.D. degree with the Institute of Computing Technology, Chinese Academy of Sciences (ICT, CAS). His research interests are machine learning and computer vision. He has authored or coauthored several academic papers in top-tier international conferences including NeurIPS/ICML. He served as a reviewer for several top-tier journals and conferences such as ICML, NeurlPS, ICLR and IEEE Transactions on Multimedia.
\end{IEEEbiography}

\begin{IEEEbiography}[{\includegraphics[width=1in,height=1.25in,clip,keepaspectratio]{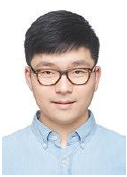}}]{Zhiyong Yang} received his M.Sc. degree in computer science and technology from the University of Science and Technology Beijing (USTB) in 2017, and Ph.D. degree from the University of Chinese Academy of Sciences (UCAS) in 2021. He is currently an Associate Professor at the University of Chinese Academy of Sciences. His research interests include trustworthy machine learning, long-tail learning, and optimization frameworks for complex metrics. He is one of the key developers of the X-curve learning framework (https://xcurveopt.github.io/), designed to address decision biases between model trainers and users. His work has been recognized with various awards, including Top 100 Baidu AI Chinese Rising Stars Around the World, Top-20 Nomination for the Baidu Fellowship,  Asian Trustworthy Machine Learning (ATML) Fellowship, and the China Computer Federation (CCF) Doctoral Dissertation Award. He has authored or co-authored over 60 papers in top-tier international conferences and journals, including more than 30 papers in T-PAMI, ICML, and NeurIPS. He has also served as an Area Chair (AC) for NeurIPS 2024/ICLR 2025, a Senior Program Committee (SPC) member for IJCAI 2021, and as a reviewer for several prestigious journals and conferences, such as T-PAMI, IJCV, TMLR, ICML, NeurIPS, and ICLR.
\end{IEEEbiography}

\begin{IEEEbiography}
	[{\includegraphics[width=1in,height=1.25in,clip,keepaspectratio]{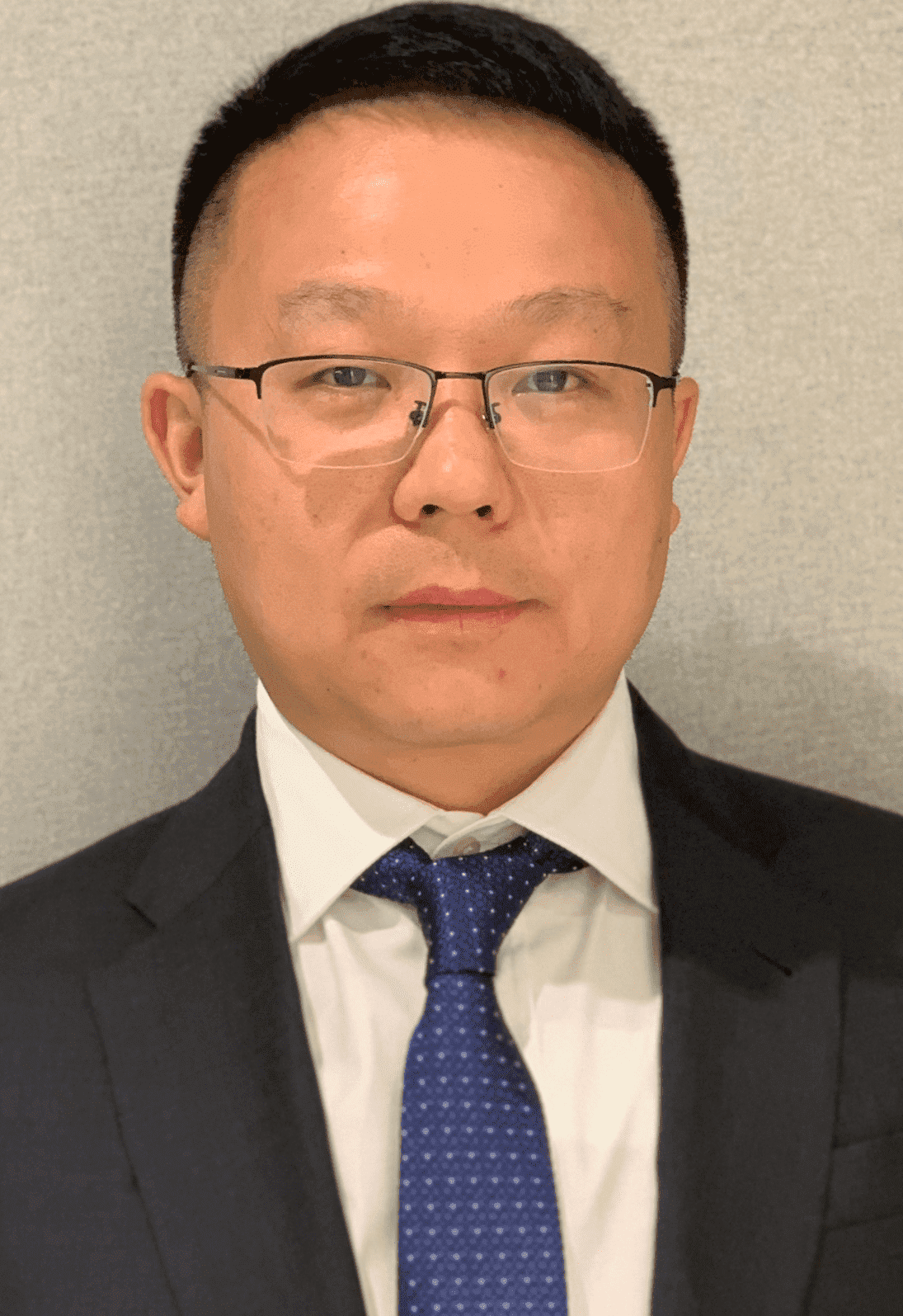}}]{Xiaochun Cao} is a Professor of the School of Cyber Science and Technology, Shenzhen Campus of Sun Yat-sen University. He received the B.E. and M.E. degrees both in computer science from Beihang University (BUAA), China, and the Ph.D. degree in computer science from the University of Central Florida, USA, with his dissertation nominated for the university-level Outstanding Dissertation Award. After graduation, he spent about three years at ObjectVideo Inc. as a Research Scientist. From 2008 to 2012, he was a professor at Tianjin University. Before joining SYSU, he was a professor at the Institute of Information Engineering, Chinese Academy of Sciences. He has authored and co-authored over 200 journal and conference papers. In 2004 and 2010, he was the recipient of the Piero Zamperoni Best Student Paper Award at the International Conference on Pattern Recognition. He is on the editorial boards of IEEE Transactions on Image Processing and IEEE Transactions on Multimedia and was on the editorial board of IEEE Transactions on Circuits and Systems for Video Technology.
\end{IEEEbiography}

\begin{IEEEbiography}
	[{\includegraphics[width=1in,height=1.25in,clip,keepaspectratio]{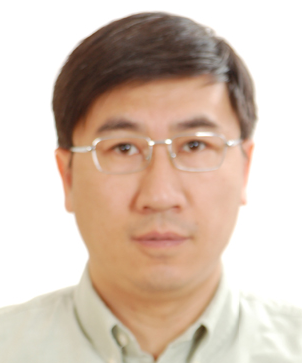}}] {Qingming Huang} is a chair professor in the University of Chinese Academy of Sciences and an adjunct research professor in the Institute of Computing Technology, Chinese Academy of Sciences. He graduated with a Bachelor degree in Computer Science in 1988 and Ph.D. degree in Computer Engineering in 1994, both from Harbin Institute of Technology, China. His research areas include multimedia computing, image processing, computer vision and pattern recognition. He has authored or coauthored more than 400 academic papers in prestigious international journals and top-level international conferences. He was the associate editor of IEEE Trans. on CSVT and Acta Automatica Sinica, and the reviewer of various international journals including IEEE Trans. on PAMI, IEEE Trans. on Image Processing, IEEE Trans. on Multimedia, etc. He is a Fellow of IEEE and has served as general chair, program chair, area chair and TPC member for various conferences, including ACM Multimedia, CVPR, ICCV, ICML, AAAI, ICMR, PCM, BigMM, PSIVT, etc.
\end{IEEEbiography}

\clearpage
\newpage
\onecolumn
\appendices
\definecolor{app_blue}{RGB}{0,20,115}
\textcolor{white}{dasdsa}

\section*{\textcolor{app_blue}{\Large{Contents}}}

\titlecontents{section}[0em]{\color{app_blue}\bfseries}{\thecontentslabel. }{}{\hfill\contentspage}

\titlecontents{subsection}[1.5em]{\color{app_blue}}{\thecontentslabel. }{}{\titlerule*[0.75em]{.}\contentspage} 

\startcontents[sections]

\printcontents[sections]{1}{1}{\setcounter{tocdepth}{3}}
\newpage

\begin{table}
	\centering
	\setlength{\abovecaptionskip}{5pt}    
	\setlength{\belowcaptionskip}{15pt}    
	\setlength{\tabcolsep}{8pt}
	\caption{A summary of key notations and corresponding descriptions in this paper.}
    \label{notation}
	\scalebox{0.95}{
	\begin{tabular}{ll}
		\toprule
		Notations & Descriptions \\
		\midrule
        $v(\cdot)$, $V(\cdot)$ & a general reference symbol to denote the separable and composite function \\
        $T$ & the number of compositions \\
		$N$ & the number of samples \\
		$S$ & the number of pixels in an image \\
        $H, W$ & height and width of an image \\
        $f_\theta$ & the model with learnable parameters $\theta$, usually ommiting $\theta$ for conveniences \\
        $\boldsymbol{X}^{(i)}$ & the $i$-th input image \\
        $\boldsymbol{Y}^{(i)}$ & the $i$-th ground-truth saliency map \\
        $\boldsymbol{X}_k, k \in [K]$ & the $k$-th splitted region in $\boldsymbol{X}$ \\
        $\boldsymbol{Y}_k, k \in [K]$ & the $k$-th corresponding ground-truth region in $\boldsymbol{Y}$ \\
        $\mathcal{C}({\boldsymbol{X}})$ & a pre-defined split function to obtain non-intersect part of $\boldsymbol{X}$ \\
        $C_k$ & the coordinates of the $k$-th region in $\boldsymbol{X}$ \\
        $\lambda(\boldsymbol{X_k})$ & the weight of the $k$-th region involved in measurements \\
        $H_k, W_k$ & height and width of the $k$-th region \\
        $S_k$ & the number of pixels in the $k$-th region \\
        $S^+ (S_k^+)$ & the number of all salient pixels in the part $\boldsymbol{X} (\boldsymbol{X}_k)$ \\
        $S^-$ & the number of all non-salient pixels in the part $\boldsymbol{X}$ \\
        $\tau$ & a threshold to binarize the prediction \\ 
        $f_p^{(i), +}$, $f_q^{(i), -}$ & the prediction scores of $p$/$q$-th positive (salient) and negative (non-salient) pixels in the $i$-th image \\
        $f_p^{+}$, $f_q^{-}$ & short for $f_p^{(i), +}$ and $f_q^{(i), -}$, respectively \\
        $f_{k,p}^+$ & the prediction score of $p$-th positive pixel in the $k$-th region \\
        $\boldsymbol{X}_k^{fore} (\boldsymbol{X}^{back}_{M+1})$ & the $k$-th foreground ($M+1$-th background) region in $\boldsymbol{X}$ generated by $\mathcal{C}(\boldsymbol{X})$ \\
        $\boldsymbol{Y}_k^{fore}$ & the corresponding ground-truth region of $\boldsymbol{X}_k^{fore}$ \\
        $\boldsymbol{X}_k^{fore,+}$ & the all salient pixels in the $k$-th foreground part of $\boldsymbol{X}$ \\
        $M$ & the number of salient foregrounds in an image, varying in the  image \\
        $\alpha$ & the weight to balance foregrounds and backgrounds in size-invariant learning \\
        $S_k^{fore} (S_k^{fore,+})$ & the number of all pixels (salient pixels) in the $k$-th foreground region \\
        $\mathcal{G}^{(\boldsymbol{X})}=(\mathcal{N}^{(\boldsymbol{X})},\mathcal{E}^{(\boldsymbol{X})},\mathcal{A}^{(\boldsymbol{X})})$ & the graph constructed on top of the image $\boldsymbol{X}$\\
        $\mathcal{N}^{(\boldsymbol{X})},\mathcal{E}^{(\boldsymbol{X})},\mathcal{A}^{(\boldsymbol{X})}$ & the vertex set, edge set, and adjacency matrix of $\mathcal{G}^{(\boldsymbol{X})}$ \\
        $\bar{\boldsymbol{y}}$ & the flattened version of $\boldsymbol{Y}$  \\
        $\boldsymbol{c}$ & a pre-defined vector \\
        $\tilde{\boldsymbol{y}} := \bar{\boldsymbol{y}} \odot \boldsymbol{c}$ & the element-wise weighted version of $\bar{\boldsymbol{y}}$ \\
        $c^+ := \tilde{\boldsymbol{y}}^\top\boldsymbol{1}$ & the sum of nonzero entries in $\tilde{\boldsymbol{y}}$ \\
        $\mathcal{P}^{(\boldsymbol{X})}$ & the Laplacian matrix of $\mathcal{G}^{(\boldsymbol{X})}$ \\
        $\bar{\boldsymbol{f}}^{(\boldsymbol{X})} \in \mathbb{R}^{S}$ & the flattened vector of $f(\boldsymbol{X})$\\
        $B$ & the number of used hybrid losses \\
        $\mathcal{L}_b (\gamma_b)$ & the $b$-th loss (its weight) \\
        $g^{(i)}$ & the risk for the $i$-th sample \\
        $\mathfrak{R}_N(\mathcal{F}|_i)=\max_{x_{1:N}} \mathfrak{R}(\mathcal{F};x_{1:N})$ & the worst-case Rademacher complexity\\
        $\hat{\boldsymbol{Y}}$ & the binarized prediction obtained from $f(\boldsymbol{X})$ using a threshold \\
        $\hat{P}_{h,w} := f(\boldsymbol{X})_{h, w}$ & the predicted value for the pixel located at the $h$-th row and $w$-th column in the model output \\
        \bottomrule
	\end{tabular}
    }
\end{table}

\section{Detailed Evaluations and Optimizations Protocols of SOD} \label{protocol_appendix}

\subsection{Introductions of widely used SOD metrics}\label{20250420SecA.1}
Unlike conventional classification tasks, where accuracy is typically computed at the image level, Salient Object Detection (SOD) is evaluated at the pixel level. While other pixel-wise tasks such as semantic segmentation often employ the mean Intersection-over-Union ($\mathsf{mIOU}$)—which averages the $\IOU$ across multiple semantic classes—SOD differs in that all salient regions are uniformly annotated as 1, without class-specific labels. As a result, there is no universally representative metric for SOD. Instead, the following evaluation metrics are commonly adopted:

\textbf{$\MAE$~\cite{MAE}.}  
The Mean Absolute Error (MAE) computes the average pixel-wise absolute difference between the predicted and ground-truth saliency maps. Prior to evaluation, the prediction is normalized to the range $[0,1]$. It is defined as:
\begin{equation}
    \MAE(f) = \frac{1}{H \times W} \sum_{h=1}^H \sum_{w=1}^W  \left|f(\boldsymbol{X})_{h,w} - Y_{h,w}\right|,
\end{equation}
where $f(\boldsymbol{X}) \in [0,1]^{H \times W}$ and $\boldsymbol{Y} \in \{0,1\}^{H \times W}$ denote the normalized prediction and ground-truth saliency maps, respectively.

\vspace{0.5em}
\textbf{$\F$-measure~\cite{F_measure}.}  
The $\F$-measure is designed to address class imbalance by jointly considering precision and recall. It is defined as:
\begin{equation}
    \F(f) = \frac{2 \cdot \mathsf{Precision}(f) \cdot \mathsf{Recall}(f)}{\mathsf{Precision}(f) + \mathsf{Recall}(f)},
\end{equation}
where
\begin{equation}
    \mathsf{Precision}(f) = \frac{\sum \hat{\boldsymbol{Y}} \cdot \boldsymbol{Y}}{\sum \hat{\boldsymbol{Y}} + \varepsilon}, \quad
    \mathsf{Recall}(f) = \frac{\sum \hat{\boldsymbol{Y}} \cdot \boldsymbol{Y}}{\sum \boldsymbol{Y} + \varepsilon}.
\end{equation}
Here, $\hat{\boldsymbol{Y}}$ denotes the binarized prediction obtained from $f(\boldsymbol{X})$ using a threshold, and $\varepsilon$ is a small constant to avoid division by zero. The notation $\hat{\boldsymbol{Y}} \cdot \boldsymbol{Y}$ represents pixel-wise multiplication between the binarized prediction and the ground truth. In practice, as $f(\boldsymbol{X})$ is typically a continuous-valued map, a set of thresholds $\tau_1, \ldots, \tau_n$ is used to binarize the predictions. All metrics are then computed over these thresholds following the standard implementations\footnote{\url{https://github.com/backseason/PoolNet}}.

Following widely adopted settings~\cite{EDN, PiCANet_2018, Amulet_2017, PoolNet}, we employ the $\F_\beta$-measure~\cite{F-beta}, which emphasizes precision by assigning a larger weight to it:
\begin{equation}
    \F_{\beta}(f) = \frac{(1 + \beta^2) \cdot \mathsf{Precision}(f) \cdot \mathsf{Recall}(f)}{\beta^2 \cdot \mathsf{Precision}(f) + \mathsf{Recall}(f)},
\end{equation}
with $\beta^2 = 0.3$. Accordingly, we report both the mean $\F$-measure ($\F_m^{\beta}$), which averages $\F_{\beta}(f)$ over all thresholds, and the maximum $\F$-measure ($\F_{max}^{\beta}$), which selects the best score among all thresholds.

\textbf{$\AUC$~\cite{AUC_eval}.}  
As salient object detection (SOD) can be viewed as a binary classification task (foreground vs. background), the Area Under the Curve (AUC) is naturally suitable for evaluating performance. It considers both the True Positive Rate (TPR) and False Positive Rate (FPR), and is known for its robustness to class imbalance~\cite{pAUC}.

At a given threshold $\tau$, we define a binarized saliency map as:
\begin{equation}
    \hat{Y}_\tau(h,w) =
    \begin{cases}
        1 & \text{if } f(\boldsymbol{X})_{h,w} \geq \tau \\
        0 & \text{otherwise}
    \end{cases}
\end{equation}

The corresponding true positive rate (TPR) and false positive rate (FPR) are computed as:
\begin{equation}
\begin{split}
    \TPR_{f}(\tau) = \frac{\sum_{h,w} \hat{Y}_\tau(h,w) \cdot Y_{h,w}}{\sum_{h,w} Y_{h,w}}, \quad
    \FPR_{f}(\tau) = \frac{\sum_{h,w} \hat{Y}_\tau(h,w) \cdot (1 - Y_{h,w})}{\sum_{h,w} (1 - Y_{h,w})}
\end{split}
\end{equation}

Let
\[
\TPR(f) = [\TPR_f(\tau_1), \ldots, \TPR_f(\tau_n)], \quad \FPR(f) = [\FPR_f(\tau_1), \ldots, \FPR_f(\tau_n)]
\]
and denote by $\pi$ the permutation that sorts $\FPR(f)$ in ascending order. Then the AUC is approximated by:
\begin{equation}
    \AUC(f) \approx \sum_{i=1}^{n-1} \frac{\TPR_f(\pi(i)) + \TPR_f(\pi(i+1))}{2} \cdot \left( \FPR_f(\pi(i+1)) - \FPR_f(\pi(i)) \right)
\end{equation}

While the formulation above relies on binarized predictions, one can also compute AUC in a ranking-based, non-binary fashion by comparing the relative orderings of predicted saliency scores for positive (salient) and negative (non-salient) pixels, as suggested in~\cite{AUC_rank, DBLP:journals/csur/YangY23}.

\textbf{$\E_m$ (E-measure)~\cite{Emeasure}.}  
E-measure simultaneously captures both pixel-level (local) and image-level (global) similarity, and is specifically designed for evaluating binary maps. It has gained widespread adoption in recent SOD literature. Formally, $\E_m$ is defined as:
\begin{equation}\label{20250420eq43}
    \E_m(f) = \frac{1}{H \times W} \sum_{h=1}^H \sum_{w=1}^W \phi_{FM}(\boldsymbol{X}, \boldsymbol{Y})_{(h,w)},
\end{equation}
where $H$ and $W$ denote the height and width of image $\boldsymbol{X}$, respectively, and $\phi_{FM}$ is formulated as:
\begin{equation}
    \phi_{FM} = J(\xi_{FM}),
\end{equation}
where $f(\cdot)$ is a convex mapping function. Following the original implementation~\cite{Emeasure}, we set $f(x) = \frac{1}{2}(1 + x)^2$. The core similarity measure $\xi_{FM}$ is defined as:
\begin{equation}
    \xi_{FM}(\boldsymbol{X}, \boldsymbol{Y}) = \frac{2 \cdot (\phi_{\boldsymbol{X}} \circ \phi_{\boldsymbol{Y}})}{\phi_{\boldsymbol{X}} \circ \phi_{\boldsymbol{X}} + \phi_{\boldsymbol{Y}} \circ \phi_{\boldsymbol{Y}}},
\end{equation}
where $\circ$ denotes element-wise (Hadamard) product, and $\phi_I = I - \mu_I \cdot \mathbb{A}$, with $\mu_I$ being the global mean of $I$ and $\mathbb{A}$ an all-one matrix with the same shape as $I$.

\textbf{$\Sm_m$ (S-measure)~\cite{S_measure1,S_measure2}.}  
S-measure is designed to assess the structural similarity between a predicted saliency map and the ground truth by integrating object-aware and region-aware perspectives. It is defined as:
\begin{equation}\label{20250420eq46}
    \Sm_m(f) = \alpha \cdot S_o(\boldsymbol{X}, \boldsymbol{Y}) + (1 - \alpha) \cdot S_r(\boldsymbol{X}, \boldsymbol{Y}),
\end{equation}
where $\alpha \in [0,1]$ is a weighting factor (typically set to $0.5$), $S_o$ denotes object-aware similarity, and $S_r$ denotes region-aware similarity.

The object-aware term $S_o$ evaluates foreground and background predictions independently:
\begin{equation}
    S_o(\boldsymbol{X}, \boldsymbol{Y}) = w \cdot \mathrm{Obj}(\boldsymbol{X}_F) + (1 - w) \cdot \mathrm{Obj}(1 - \boldsymbol{X}_B),
\end{equation}
where $\boldsymbol{X}_F$ and $\boldsymbol{X}_B$ denote the predicted values within the foreground and background regions of the ground truth, respectively, and the object-level similarity function $\mathrm{Obj}(\cdot)$ is defined as:
\begin{equation}
    \mathrm{Obj}(\boldsymbol{X}) = \frac{2 \mu_{\boldsymbol{X}}}{\mu_{\boldsymbol{X}}^2 + 1 + \sigma_{\boldsymbol{X}}},
\end{equation}
with $\mu_{\boldsymbol{X}}$ and $\sigma_{\boldsymbol{X}}$ being the mean and standard deviation of $\boldsymbol{X}$.

The region-aware term $S_r$ divides the saliency map into four quadrants, $R_1, R_2, R_3, R_4$, using the centroid of the ground truth as the reference point. Within each region $R_k$, a structural similarity index (SSIM) is computed:
\begin{equation}
    \text{SSIM}(\boldsymbol{X}_k, \boldsymbol{Y}_k) = \frac{2 \mu_{\boldsymbol{X}} \mu_{\boldsymbol{Y}} + C_1}{\mu_{\boldsymbol{X}}^2 + \mu_{\boldsymbol{Y}}^2 + C_1} \cdot \frac{2 \sigma_{\boldsymbol{X}\boldsymbol{Y}} + C_2}{\sigma_{\boldsymbol{X}}^2 + \sigma_{\boldsymbol{Y}}^2 + C_2},
\end{equation}
where $C_1$ and $C_2$ are small constants to avoid division by zero, and $\mu$, $\sigma$, and $\sigma_{\boldsymbol{X}\boldsymbol{Y}}$ are local statistics. The final region-aware similarity $S_r$ is aggregated as:
\begin{equation}
    S_r(\boldsymbol{X}, \boldsymbol{Y}) = \sum_{k=1}^4 w_k \cdot \text{SSIM}(\boldsymbol{X}_k, \boldsymbol{Y}_k),
\end{equation}
where $w_k = \frac{|R_k|}{H \times W}$ represents the relative size of each region.


To explicitly pursue the aforementioned objectives, most SOD approaches incorporate diverse supervision signals, including \textbf{pixel-level} loss to ensure local fidelity, \textbf{ranking-aware} loss to enforce a well-structured global ranking, and \textbf{region-level} loss to promote structural consistency. However, as discussed in the main paper~\cref{Revisting}, existing evaluation metrics often fail to accurately reflect performance across objects of varying sizes.

\subsection{Brief Discussion of the size-sensitive problems for $\E_m$ and $\Sm_m$}\label{20250420SecA.2}

Building upon the analyses in \cref{Revisting} and \cref{20250420SecA.1}, it remains initially unclear whether $\E_m$ and $\Sm_m$ are substantially affected by size sensitivity, as they are neither pixel-wise nor region-wise independent. Consequently, the analytical framework adopted in \cref{Revisting} cannot be directly extended to these metrics. Nevertheless, a closer inspection of their computational formulations reveals indicative evidence of inherent size sensitivity therein: 

\noindent\textbf{Size Sensitivity of $\E_m$.} As defined in Eq.(\ref{20250420eq43}), i.e., 

\[
    \E_m(f) = \frac{1}{H \times W} \sum_{h=1}^H \sum_{w=1}^W \phi_{FM}(\boldsymbol{X}, \boldsymbol{Y})_{(h,w)},
\] 
the $\E_m$ metric inherently treats all pixels with equal importance, regardless of their semantic significance. Consequently, large salient objects, due to occupying more spatial area, exert a dominant influence on the final score. In complex scenes involving multiple objects of varying sizes, such uniform aggregation \textit{tends to underrepresent smaller objects}, whose misclassification has only a marginal impact on the overall evaluation. Although $\phi_{FM}$ employs mean-centering to partially alleviate distributional imbalance, the subsequent global averaging step remains fundamentally biased toward object size. As a result, $\E_m$ is less invariant to the precise localization of small objects.

\noindent\textbf{Size Sensitivity of $\Sm_m$.} As introduced in Eq.(\ref{20250420eq46}),
\begin{equation}
    \Sm_m(f) = \alpha \cdot S_o(\boldsymbol{X}, \boldsymbol{Y}) + (1 - \alpha) \cdot S_r(\boldsymbol{X}, \boldsymbol{Y}),
\end{equation}
the S-measure $\Sm_m$, defined as a weighted combination of object-aware and region-aware structural similarities, aims to jointly assess local detail and global structure. However, its object-aware component $S_o$ employs a weighting factor $w = \frac{\sum Y}{W \times H}$, reflecting the proportion of foreground pixels in the ground truth. This design \textit{implicitly favors larger salient regions}, thereby introducing a size bias that diminishes the impact of smaller objects. Meanwhile, the region-aware component $S_r$ divides the prediction map into four quadrants centered at the centroid of the ground truth. In cases where small objects are confined to peripheral or sparsely populated regions, their contribution may be \textit{diluted by dominant background content} within the same quadrant. Consequently, both components of $\Sm_m$ exhibit an inherent preference for large, centrally located salient objects, which may result in inaccurate assessments in scenes with multiple or small-scale targets.  

It is worth noting, however, that the size-sensitivity issues in $\E_m$ and $\Sm_m$ are relatively less severe than those in pixel-wise or region-wise independent metrics (e.g., $\MAE$ and $\AUC$), as both $\E_m$ and $\Sm_m$ integrate multiple structural cues (such as global-level image signals) to provide a more holistic evaluation. That is precisely why we still leverage both $\E_m$ and $\Sm_m$ as standard metrics in our experiments. Last but not least, we emphasize that, although someone may focus on the performance of $\E_m$ and $\Sm_m$ for model evaluation, \textbf{existing approaches still select to use the hybrid losses introduced in \cref{20250404Sec4.2} to guide model learning}. This ensures that our proposed $\mathsf{SIOpt}$ remains effective during optimization.

\section{Proof for Propositions of Size-Invariant Metrics}
\subsection{Proof of Proposition \ref{prop1}}\label{prop1_proof}

\textbf{Restate of Proposition \ref{prop1}.} (\textit{Size-Invariant Property of $\SMAE$}) Without loss of generality, given two well-trained predictors $f_{A}$ and $f_{B}$ with different parameters, $\SMAE$ is always \textbf{more effective} than $\MAE$ even if the size of these objects is imbalanced. 

\noindent{\color{blue}{\textbf{Case 1: Single Salient Object (\( M = 1 \)).}}} When there is only one salient object in the image \( \boldsymbol{X} \), no size imbalance exists among objects. In this case, \( \SMAE \) is equivalent to \( \MAE \). 

\noindent{\color{blue}{\textbf{Case 2: Multiple Salient Objects (\( M \geq 2 \)).}}} Suppose the image \( \boldsymbol{X} \) contains multiple salient objects, represented by the ground-truth salient pixel sets \( \{C_1, C_2, \dots, C_M\} \) with \( S_1^{fore} \leq S_2^{fore} \leq \dots \leq S_M^{fore} \). Assume the two detection models, \( f_A \) and \( f_B \), identify the same total number of salient pixels, given by \( \sum_{i=m+1}^{M} |C_i| \) for some \( m \in \{1, 2, \dots, M-1\} \). Additionally, suppose \( f_A \) perfectly detects only the larger objects \( \{C_{m+1}, \dots, C_M\} \), whereas $f_B$ detects all $\{C_1,\dots C_M\}$ but only partially. In general, \underline{\( f_B \) is considered superior to \( f_A \)} since the latter entirely fails to detect the smaller objects $\{C_1,\cdots,C_m\}$. Under this setting, we observe that \( \SMAE(f_A) > \SMAE(f_B) \) but \( \MAE(f_A) = \MAE(f_B) \).

\begin{proof}
    For {\color{blue}{\textbf{Case 1 ($M=1$):}}} As there is only one salient object, we can easily divide the image into $\boldsymbol{X}_1^{fore}$ and $\boldsymbol{X}_2^{back}$, and the weight $\alpha=S_2^{back}/S_1^{fore}=(S-S_1^{fore})/S_1^{fore}$. Therefore, we have the following:
    \begin{equation}
    \begin{aligned}
        \MAE(f)&=\frac{\Vert f(\boldsymbol{X}_1^{fore})-\boldsymbol{Y}_1^{fore}\Vert_{1,1} + \Vert f(\boldsymbol{X}_2^{back})-\boldsymbol{Y}_2^{back}\Vert_{1,1}}{S}, \\
        \SMAE(f)&=\frac{1}{1+\alpha} \cdot  \left[ \frac{\Vert f(\boldsymbol{X}_1^{fore})-\boldsymbol{Y}_1^{fore}\Vert_{1,1}}{S_1^{fore}} + \alpha \cdot \frac{\Vert f(\boldsymbol{X}_2^{back})-\boldsymbol{Y}_2^{back}\Vert_{1,1}}{S_2^{back}}  \right] \\
        &=\frac{1}{1+\frac{S-S_1^{fore}}{S_1^{fore}}} \cdot \left[  \frac{\Vert f(\boldsymbol{X}_1^{fore})-\boldsymbol{Y}_1^{fore}\Vert_{1,1}}{S_1^{fore}} + \frac{S_2^{back}}{S_1^{fore}} \cdot \frac{\Vert f(\boldsymbol{X}_2^{back})-\boldsymbol{Y}_2^{back}\Vert_{1,1}}{S_2^{back}}  \right] \\
        &=\frac{S_1^{fore}}{S} \cdot \left[  \frac{\Vert f(\boldsymbol{X}_1^{fore})-\boldsymbol{Y}_1^{fore}\Vert_{1,1} + \Vert f(\boldsymbol{X}_2^{back})-\boldsymbol{Y}_2^{back}\Vert_{1,1}}{S_1^{fore}} \right] \\
        &=\frac{\Vert f(\boldsymbol{X}_1^{fore})-\boldsymbol{Y}_1^{fore}\Vert_{1,1} + \Vert f(\boldsymbol{X}_2^{back})-\boldsymbol{Y}_2^{back}\Vert_{1,1}}{S} \\
        &=\MAE(f).
    \end{aligned}
    \end{equation}

For {\color{blue}{\textbf{Case 2 ($M\ge2$):}}} It is easy to examine that we always have $\MAE(f_A)=\MAE(f_B)$ in this case since both $f_A$ and $f_B$ predict the same number of salient pixels according to the proposition. 

Subsequently, to derive the relationship of $\SMAE$, we first consider the {\color{orange}\textbf{$M=2$ case}}. Suppose $\rho\in [0,1]$ of $C_2$ is correctly detected by $f_B$. Additionally, because there are two salient objects, we have $\alpha=S_3^{back} / (S_1^{fore}+S_2^{fore})$, and $\SMAE$ is calculated as follows:
\begin{equation}
    \begin{aligned}
        \SMAE(f_A) &= \frac{1}{2+\alpha} \cdot \frac{|C_1|}{S_1^{fore}}, \\
        \SMAE(f_B) &= \frac{1}{2+\alpha} \cdot \left[ \frac{|C_2|(1-\rho)}{S_2^{fore}} + \frac{|C_1|-|C_2|(1-\rho)}{S_1^{fore}}  \right],
    \end{aligned}
\end{equation}
then when $S_1^{fore}<S_2^{fore}$, we have:
\begin{equation}
\begin{aligned}
    \SMAE(f_B)-\SMAE(f_A) & = \frac{1}{2+\alpha} \cdot \left[ \frac{|C_2|(1-\rho)}{S_2^{fore}} + \frac{|C_1|-|C_2|(1-\rho)}{S_1^{fore}} - \frac{|C_1|}{S_1^{fore}} \right] \\
    &= \frac{1}{2+\alpha} \cdot \left[  \frac{|C_2|(1-\rho)}{S_2^{fore}} - \frac{|C_2|(1-\rho)}{S_1^{fore}} \right] \\
    &= \frac{|C_2|(1-\rho)}{2+\alpha}\cdot \left( \frac{1}{S_2^{fore}} - \frac{1}{S_1^{fore}} \right) \\
    & <0.
\end{aligned}
\end{equation}


    In terms of {\color{orange}\textbf{$M\ge3$ case}}, we first denote the correct ratio of $f_B$ for each part $\{C_1,C_2,\cdots,C_M\}$ as $[\rho_1,\rho_2,\cdots,\rho_M],\rho_i\in [0,1]$, respectively. Since both $f_A$ and $f_B$ predict the same number of salient pixels, the following equation holds:
    \begin{equation}
        \sum_{i=1}^M \rho_i|C_i|=\sum_{i=m+1}^M|C_i|.
    \end{equation}
Considering that there are $M$ salient objects, we have $\alpha=\frac{S_{M+1}^{back}}{\sum_{i=1}^MS_i^{fore}}$, and thus $\mathsf{SI\text{-}MAE}$ is calculated as follows:
\begin{equation}
    \begin{aligned}
         \mathsf{SI\text{-}MAE}(f_A)&=\frac{1}{M+\alpha}\sum_{i=1}^m\frac{|C_i|}{S_i^{fore}}, \\
         \mathsf{SI\text{-}MAE}(f_B)&=\frac{1}{M+\alpha}\sum_{i=1}^M \frac{(1-\rho_i)|C_i|}{S_i^{fore}}.
    \end{aligned}
\end{equation}

Taking a step further, we have:
\begin{equation}
\begin{aligned}
    \mathsf{SI\text{-}MAE}(f_B)-\mathsf{SI\text{-}MAE}(f_A)&=\frac{1}{M+\alpha}\left[\sum_{i=1}^M\frac{(1-\rho_i)|C_i|}{S_i^{fore}}-\sum_{i=1}^m\frac{|C_i|}{S_i^{fore}} \right] \\
    &=\frac{1}{M+\alpha} \left[\sum_{i=m+1}^M \frac{|C_i|}{S_i^{fore}}-\sum^{M}_{i=1}\frac{\rho_i|C_i|}{S_i^{fore}} \right] \\
    & \le \frac{1}{M+\alpha} \left[\frac{\sum_{i=m+1}^M |C_i|}{S_{m+1}^{fore}}-\sum^{M}_{i=1}\frac{\rho_i|C_i|}{S_i^{fore}}\right] \\
    &=\frac{1}{(M+\alpha)} \left[ \frac{\sum_{i=1}^M\rho_i|C_i|}{S_{m+1}^{fore}}- \sum^{M}_{i=1}\frac{\rho_i|C_i|}{S_i^{fore}}\right] \\
    &\le \frac{1}{(M+\alpha)} \left[\frac{\sum_{i=1}^M\rho_i|C_i|}{S_{m+1}^{fore}}-\frac{\sum^{M}_{i=1}\rho_i|C_i|}{S_M^{fore}}\right] \\
    &<0,
\end{aligned}
\end{equation}
 where the inequalities hold because we assume $S_1^{fore}\le S_2^{fore} \le \cdots \le S_M^{fore}$. 
 
 This completes the proof.
\end{proof}

\subsection{Size-Invariant Property of $\SF$}\label{20250421SecB.2}

    
    

\begin{proposition}(\textbf{Size-Invariant Property of $\SF$.}) \label{prop3} Without loss of generality, given two well-trained predictors, \( f_A \) and \( f_B \), with different parameters, \( \SF \) is \textbf{more effective} than \( \F \) during evaluation in the following case: 

Assume the image \( \boldsymbol{X} \) contains multiple salient objects ($M \ge 2$) with imbalanced sizes, represented by the ground-truth salient pixel sets \( \{C_1, C_2, \dots, C_M\} \) with \( S_1^{fore} \leq S_2^{fore} \leq \dots \leq S_M^{fore} \). Assume the two detection models, \( f_A \) and \( f_B \), identify the same total number of salient pixels, given by \( \sum_{i=m+1}^{M} |C_i| \) for some \( m \in \{1, 2, \dots, M-1\} \). Additionally, suppose \( f_A \) perfectly detects only the larger objects \( \{C_{m+1}, \dots, C_M\} \), whereas $f_B$ detects all $\{C_1,\dots C_M\}$ but only partially. In general, \underline{\( f_B \) is considered superior to \( f_A \)} since the latter entirely fails to detect the smaller objects $\{C_1,\cdots,C_m\}$. Under this setting, we observe that \( \SF(f_A) < \SF(f_B) \) but \( \F(f_A) = \F(f_B) \).
\end{proposition}
\begin{proof} \label{prop3_proof} According to the proposition, since $f_A$ and $f_B$ detect the same amount of salient pixels, we have $\F(f_A)=\F(f_B)$. 

For the proof of $\SF$, we first consider the {\color{orange}\textbf{$M=2$ case}} for a clear presentation. Let $\rho\in [0,1]$ be the correct ratio of the object $C_2$ detected by $f_B$. Then, $\SF$ is calculated as follows:
\begin{equation}
    \begin{split}
        \SF(f_A)&=\frac{1}{2}\cdot(1+0)=\frac{1}{2}, \\
        \SF(f_B)&=\frac{1}{2}\cdot\left[\frac{2\rho}{1+\rho}+\frac{2(1-\rho)|C_2|}{|C_1|+(1-\rho)|C_2|}\right],
    \end{split}
\end{equation}
Apparently, we have $\SF(f_B) > \SF(f_A)$ following
\begin{equation}
    \begin{split}
        \SF(f_B)-\SF(f_A) &= \frac{1}{2}\cdot\left[\frac{2\rho}{1+\rho}+\frac{2(1-\rho) |C_2|}{|C_1|+(1-\rho)|C_2|}\right]-\frac{1}{2} \\
        &=\frac{1}{2}\cdot\left[-\frac{1-\rho}{1+\rho}+\frac{2(1-\rho) |C_2|}{|C_1|+(1-\rho) |C_2|}\right] \\
        &=\frac{1-\rho}{2}\cdot\left[\frac{2 |C_2|}{|C_1|+(1-\rho) |C_2|} - \frac{1}{1+\rho}\right] \\
        &= \frac{1-\rho}{2}\cdot \frac{|C_2|(1 + 3\rho) - |C_1|}{\left(|C_1|+(1-\rho) |C_2|\right)(1 + \rho)} \\
        & >0
    \end{split}
\end{equation}
where the last inequality holds because $\rho \in [0, 1]$ and the size of $C_2$ is larger than $C_1$, i.e., $|C_2| \ge |C_1|$.

In the case of {\color{orange}\textbf{$M \ge 3$}}, a similar induction employed in the $M=2$ scenario can be utilized to derive the relationship for $\SF$, as both $f_A$ and $f_B$ produce saliency maps with an identical number of predicted salient pixels, i.e,:
\begin{equation}
    \sum_{i=1}^M \rho_i|C_i|=\sum_{i=m+1}^M|C_i|.
\end{equation}
This completes the proof.

\end{proof}

\section{Proof for Properties of $\mathsf{SIOpt}$}
\subsection{Proof for Prop.\ref{re-attention_prop}} \label{Proof_of_4.2.1}
\textbf{Restate of Prop.\ref{re-attention_prop}}:  
Without loss of generality, we merely consider one sample \((\boldsymbol{X}, \boldsymbol{Y})\) here. Given a separable loss function \( L(f) \), such as $L(f):= \ell_{\BCE}(f(\boldsymbol{X}), \boldsymbol{Y})$ and its corresponding size-invariant loss \( \mathcal{L}_{\SI}(f, \boldsymbol{X}, \boldsymbol{Y}) \) as defined in Eq.(\ref{eq:loss}), the following properties hold:  

\begin{enumerate} 
    \item[(\textcolor{blue}{\textbf{P1}})] The weight assigned by the size-invariant loss satisfies \( \lambda_{\mathcal{L}_{\SI}}(\boldsymbol{X}_k^{fore}) > \lambda_{L}(\boldsymbol{X}_k^{fore}) \) if \( S_k^{fore} < \frac{S}{M + \alpha_{SI}} \), for all \( k \in [M] \).  
    \item[(\textcolor{blue}{\textbf{P2}})] The weight assigned by the size-invariant loss follows an inverse relationship with region size, i.e., \( \lambda_{\mathcal{L}_{\SI}}(\boldsymbol{X}_{k_1}^{fore}) > \lambda_{\mathcal{L}_{\SI}}(\boldsymbol{X}_{k_2}^{fore}) \) if \( S_{k_1}^{fore} < S_{k_2}^{fore} \), for all \( k_1, k_2 \in [M] \), with \( k_1 \neq k_2 \).  
\end{enumerate}  
For simplicity, \( \lambda_{\mathcal{L}_{\SI}}(\boldsymbol{X}_k^{fore}) \) here represents the \textbf{weight} assigned to each pixel in \( \boldsymbol{X}_k^{fore} \) by the size-invariant loss \( \mathcal{L}_{\SI}(f, \boldsymbol{X}, \boldsymbol{Y}) \), whereas \( \lambda_{L}(\boldsymbol{X}_k^{fore}) \) denotes the \textbf{weight} assigned by the original size-variant loss \( L(f) \). Note that pixels within the same region \( \boldsymbol{X}_k^{fore} \) are treated weighted in both \( \mathcal{L}_{\SI}(f, \boldsymbol{X}, \boldsymbol{Y}) \) and \( L(f) \).  

\begin{proof}
 \label{re-attention_proof}
In what follows, we adopt \(\BCE\) as an illustrative example. Other commonly used separable loss functions, such as Mean Absolute Error (\(\MAE\)) and Mean Squared Error (\(\MSE\)), are also applicable and do not alter the final results.

For (\textcolor{blue}{\textbf{P1}}), we first have:
\begin{equation}
L(f):=\ell_{\BCE}(f(\boldsymbol{X}), \boldsymbol{Y}) = -\frac{1}{S} \sum_{h=1}^{H} \sum_{w=1}^{W} \left(Y_{h,w} \log \hat{P}_{h,w} + (1 - Y_{h,w}) \log (1 - \hat{P}_{h,w})\right)
\end{equation}
where 
$\hat{P}_{h,w} := f(\boldsymbol{X})_{h, w}$ represents the predicted value for the pixel located at the $h$-th row and $w$-th column, and $Y_{h,w}$ denotes the corresponding ground truth. $S = H \times W$ represents the total number of pixels, and $H$ and $W$ are its height and width, respectively. 

In terms of the size-invariant loss $\mathcal{L}_{\SI}$, we have:
\begin{equation}\label{319eq47}
\mathcal{L}_{\SI}(f, \boldsymbol{X}, \boldsymbol{Y})= \frac{1}{M+\alpha_{SI}} \left[\sum_{k=1}^{M} L(f^{fore}_k)  + \alpha_{SI} L(f^{back}_{M+1})\right],
\end{equation}
where we have
\begin{equation}\nonumber
\begin{aligned}
    & L(f^{fore}_k) := \ell_{\BCE}(f(\boldsymbol{X}^{fore}_k), \boldsymbol{Y}^{fore}_k) = 
    & \frac{-\sum_{h=1}^{H_k^{fore}} \sum_{w=1}^{W_k^{fore}} Y_{h,w} \log \hat{P}_{h,w} + (1 - Y_{h,w}) \log (1 - \hat{P}_{h,w})}{S_k^{fore}},
\end{aligned}    
\end{equation}
and 
\begin{equation}\nonumber
\begin{aligned}
    & L(f^{back}_{M+1}) := \ell_{\BCE}(f(\boldsymbol{X}^{back}_{M+1}), \boldsymbol{Y}^{back}_{M+1}) = 
    & \frac{-\sum_{h=1}^{H^{back}_{M+1}} \sum_{w=1}^{W^{back}_{M+1}} Y_{h,w} \log \hat{P}_{h,w} + (1 - Y_{h,w}) \log (1 - \hat{P}_{h,w})}{S^{back}_{M+1}},
\end{aligned}    
\end{equation}

Apparently, \textbf{the weight of each pixel} within $\boldsymbol{X}_k^{fore}$ assigned by the original size-variance loss \( L(f) \) \textbf{is a constant}, i.e., $\lambda_{L}(\boldsymbol{X}_k^{fore}) = 1/S$. Taking the derivate from Eq.(\ref{319eq47}), it is easy to obtain that $\lambda_{\mathcal{L}_{\SI}}(\boldsymbol{X}_k^{fore}) = \frac{1}{S_k^{fore}(M + \alpha_{SI})}$. Therefore, we have $\lambda_{\mathcal{L}_{\SI}}(\boldsymbol{X}_k^{fore})>\lambda_{L}(\boldsymbol{X}_k^{fore})$ if when $S_{k}^{fore}<\frac{S}{M + \alpha_{SI}}$. 

For (\textcolor{blue}{\textbf{P2}}), it is obvious that $\lambda_{\mathcal{L}_{\SI}}(\boldsymbol{X}_{k_1}^{fore})$ is in proportion to $1/S_{k}^{fore}$ according to Eq.(\ref{319eq47}). Therefore, if \( S_{k_1}^{fore} < S_{k_2}^{fore} \), for all \( k_1, k_2 \in [M] \), with \( k_1 \neq k_2 \), we have \( \lambda_{\mathcal{L}_{\SI}}(\boldsymbol{X}_{k_1}^{fore}) > \lambda_{\mathcal{L}_{\SI}}(\boldsymbol{X}_{k_2}^{fore}) \) naturally.

This completed the proof. 
\end{proof} 

\subsection{Proof for \cref{generalization_bound}} \label{generalization_bound_proof}

\subsubsection{Proof for Technical Lemmas}
In this subsection, we present key lemmas that are important for the proof of \cref{generalization_bound}.
\begin{lemma}
\label{lemma1}
The empirical Rademacher complexity of function $g$ with respect to the predictor $f$ is defined as:
\begin{equation}
    \hat{\mathfrak{R}}_{\mathcal F}(g)=\mathbb{E}_{\sigma}[\sup _{f \in \mathcal F}\frac{1}{N} \sum_{i=1}^{N} \sigma_ig(f^{(i)})].
\end{equation}
where $\mathcal{F} \subseteq \{f: \mathcal{X} \to \mathbb{R}^S\}$ is a family of predictors, and $N$ refers to the size of the dataset, and $\sigma_i$s are independent uniform random variables taking values in $\{-1, +1\}$. The random variables $\sigma_i$ are called Rademacher variables.
\end{lemma}

\begin{lemma}
\label{lemma2}
Let $\mathbb{E}[g]$ and $\hat{\mathbb{E}}[g]$ represent the expected risk and empirical risk, and $\mathcal{F} \subseteq \{ f:\mathcal{X} \to \mathbb{R}^S \}$. Then with probability at least $1-\delta$ over the draw of an i.i.d. sample $\mathcal{D}$ of size $N$, the generalization bound holds:
\begin{equation}
    \sup_{f \in \mathcal{F}}\big(\mathbb{E}[g(f)]-\hat{\mathbb{E}}[g(f)] \big) \le 2\hat{\mathfrak{R}}_{\mathcal F}(g)+3\sqrt{\frac{\log \frac{2}{\delta}}{2N}}.
\end{equation}
\end{lemma}

\begin{lemma}
\label{lemma3}
\cite{vector_Contraction} Assume $\mathcal{F} \subseteq \{ f:\mathcal{X} \to \mathbb{R}^S \}$, and $(\sigma_1, \cdots, \sigma_N)$ is a sequence of i.i.d Rademacher random variables. When $\phi_1, \cdots, \phi_N$ are $K$-Lipschitz with respect to the $\ell_{\infty}$ norm, there exists a constant $C>0$ for any $\delta>0$, such that if $|\phi_t(f(x))| \vee \Vert f(x) \Vert_\infty \le \beta$, the following holds:
\begin{equation}
    \mathfrak{R}(\phi \circ \mathcal{F} ; x_{1:N}) \le C \cdot K \sqrt S \cdot \max_i \mathfrak{R}_N(\mathcal{F}|_i)\cdot\log^{\frac{3}{2}+\delta}\left(\frac{\beta N}{\max_i \mathfrak{R}_N(\mathcal{F}|_i)}\right).
\end{equation}
where
\begin{equation}
    \mathfrak{R}(\mathcal{F} ; x_{1:N}) = \mathbb{E}_{\epsilon} \sup_{f\in \mathcal{F}} \sum_{t=1}^N \epsilon_t f(x_t),
\end{equation}
$\circ $ represents the composite function, and $\mathcal{F}$ is a class of functions $f:\mathcal{X} \to \mathbb{R}^S$, and $\epsilon=(\epsilon_1, \cdots, \epsilon_N)$ is a sequence of $i.i.d.$ Rademacher random variables.
Here we set $\beta=1$ because it is obvious that $\Vert f(x) \Vert_\infty \le 1$ and $\phi_t(x) \le 1$. Therefore, 
\begin{equation}
    \mathfrak{R}(\phi \circ \mathcal{F} ; x_{1:N}) \le C \cdot K \sqrt S \cdot \max_i \mathfrak{R}_N(\mathcal{F}|_i)\cdot\log^{\frac{3}{2}+\delta}\left(\frac{N}{\max_i \mathfrak{R}_N(\mathcal{F}|_i)}\right).
\end{equation}
\end{lemma}

\begin{lemma}
\label{lemma4}
When $g(\cdot)$ is Lipschitz continuous, the following holds:
\begin{equation}
    \Vert g(x)-g(\tilde{x})\Vert_{\infty} \le \sup \Vert\nabla_x g\Vert_p  \cdot \Vert x-\tilde x\Vert_q,
\end{equation}
where $\frac{1}{p}+\frac{1}{q}=1$.

\begin{proof}
\begin{equation}
\begin{aligned}
    |g(x)-g(\tilde{x})| &=~ \left|\int_0^1 \left\langle\nabla g(\tau x + (1-\tau)\tilde{x}), x-\tilde{x} \right\rangle d\tau\right| \\
    & \le~ \sup_{ x \in \mathcal{X}} \big[ \Vert \nabla g \Vert_p \big] \cdot  \big\Vert x-\tilde{x}\big\Vert_q 
\end{aligned}
\end{equation}
\end{proof}

Specifically, when $p=1$ and $q=\infty$, we have
\begin{equation}
    \Vert g(x)-g(\tilde{x})\Vert_{\infty} \le \sup \Vert\nabla_x g\Vert_1  \cdot \Vert x-\tilde x\Vert_\infty.
\end{equation}
\end{lemma}

\begin{lemma}
\label{lemma5}
Common composite functions are $p$-Lipschitz, as \cite{p_Lipschitz} stated:
\begin{definetitle}[$p$-Lipschitzness]
    $\Phi(u,v,p)$ is said to be $p$-Lipschitz if for any feasible $u,v,p,u',v',p'$:
\begin{equation}
    | \Phi(u,v,p)-\Phi(u',v',p') | \le U_p |u-u'| + V_p |v-v'| + P_p |p-p'|,
\end{equation}
where $\Phi(u(f), v(f), p)$ is defined as:
\begin{equation}
    u(f) = \text{TP}(f)=\mathbb{P}(f(x)=1, y=1), \qquad v(f)=\mathbb{P}(f(x)=1), \qquad \quad p=\mathbb{P}(y=1).
\end{equation}
\end{definetitle}
Any metric being a function of the confusion matrix can be parameterized in this way. According to \cite{p_Lipschitz}, Accuracy, AM, F-score, Jaccard, G-Mean, and AUC are all $p$-Lipschitz.
\end{lemma}

\subsubsection{Proof for the Generalization Bound}
\textbf{Restate of \cref{generalization_bound}} Assume $\mathcal{F} \subseteq \{ f:\mathcal{X} \to \mathbb{R}^S \}$, where $S=H \times W$ is the number of pixels in an image,  $g^{(i)}$ is the risk over $i$-th sample, and is $K$-Lipschitz with respect to the $l_{\infty}$ norm, (i.e. $\Vert g(x)-g(\tilde{x})\Vert_\infty \le L\cdot \Vert x-\tilde{x} \Vert_{\infty}$). When there are $N$ $i.i.d.$ samples, there exists a constant $C>0$ for any $\epsilon > 0$, the following generalization bound holds with probability at least $1-\delta$:
\begin{equation}
\begin{aligned}
    &~ \sup_{f \in \mathcal{F}}(\mathbb{E}[g(f)]-\hat{\mathbb{E}}[g(f)]) \\
    \le&~  C\cdot \frac{K\sqrt{S}}{N} \cdot \max_i \mathfrak{R}_N(\mathcal{F}|_i)\cdot\log^{\frac{3}{2}+\epsilon}\left(\frac{N}{\max_i \mathfrak{R}_N(\mathcal{F}|_i)}\right) \\
    &+3\sqrt{\frac{\log \frac{2}{\delta}}{2N}},
\end{aligned}
\end{equation}
where again $g(f)$ could be any loss functions introduced in \cref{20250404Sec4.2}, $\mathbb{E}[g(f)]$ and $\hat{\mathbb{E}}[g(f)]$ represent the expected risk and empirical risk, and $\mathfrak{R}_N(\mathcal{F}|_i)=\max_{x_{1:N}\in \mathcal{X}} \mathfrak{R}(\mathcal{F};x_{1:N})$ denotes the worst-case Rademacher complexity. Specifically, 
\begin{enumerate}[label=\textbf{\textit{ Case \arabic*:}},leftmargin=4em]
    \item For separable loss functions $\ell(\cdot)$, if it is $\mu$-Lipschitz, we have $K=\mu$.
    \item For composite loss functions, when $\ell(\cdot)$ is DiceLoss \cite{DiceLoss}, we have $K=\frac{4}{\rho}$, where $\rho=\min \frac{S_{k}^{+,i}}{S_{k}^i}$, which represents the minimum proportion of the salient object in the $k$-th frame within the $i$-th sample.
\end{enumerate}
\begin{proof}
In this subsection, we give the proof combining lemmas above. 

Firstly, the empirical risk over the dataset is:
\begin{equation}
    \hat{\mathbb E}[g(f)]=\frac{1}{N}\sum_{i=1}^N g^{(i)}(f(\boldsymbol{X}),\boldsymbol{Y}),
\end{equation}
where $\boldsymbol{X}, \boldsymbol{Y}$ are the prediction and ground truth.

Combing \cref{lemma1}, \cref{lemma2} and \cref{lemma3}, with probability at least $1-\delta$, we have:
\begin{equation}
\begin{aligned}
    \sup_{f \in \mathcal{F}}(\mathbb{E}[g(f)]-\hat{\mathbb{E}}[g(f)]) & \le 2\hat{\mathfrak{R}}_{\mathcal F}(g)+3\sqrt{\frac{\log \frac{2}{\delta}}{2N}} \\
    & = 2\mathbb{E}_{\sigma}\left[\sup _{f \in \mathcal F}\frac{1}{N} \sum_{i=1}^{N} \sigma_ig^{(i)}(f)\right]+3\sqrt{\frac{\log \frac{2}{\delta}}{2N}} \\
    & \le C\cdot \frac{K\sqrt{S}}{N} \cdot \max_i \mathfrak{R}_N(\mathcal{F}|_i)\cdot\log^{\frac{3}{2}+\epsilon}\left(\frac{N}{\max_i \mathfrak{R}_N(\mathcal{F}|_i)}\right)+3\sqrt{\frac{\log \frac{2}{\delta}}{2N}}.
\end{aligned}
\end{equation}

\noindent For \textbf{\textit{Case 1:}} separable loss functions $\ell(\cdot)$, we have the following equation:
\begin{equation}
    g^{(i)}(f)=\frac{1}{M^{(i)}}\sum_{k \in [M^i]} \frac{1}{S_k^i} \sum_{(h,w) \in S_{k}^i} \ell (f_{h,w}^{(i)}, Y_{h,w}^{(i)}),
\end{equation}
where $M^{(i)}$ is the number of frames in the $i$-th sample, and $S_k^i$ is the size of the $k$-th frame in the $i$-th sample. $\ell(\cdot)$ is a matrix element function, $f_{h,w}^{(i)}:= f(\boldsymbol{X}^{(i)})_{h,w}$ and $Y_{h,w}^{(i)}$ is the prediction and ground truth of the pixel $(h,w)$ in the $k$-th frame.

Assume $\ell(\cdot)$ is $\mu$-Lipschitz, then for Lipschitz continuous of $g^{(i)}$, we have:
\begin{equation}
    \begin{aligned}
g^{(i)}(f)-g^{(i)}(\tilde{f})
&=\frac{1}{M^{(i)}}\sum_{k \in [M^i]} \frac{1}{S_k^i} \sum_{(h,w) \in S_k^i}  \left(\ell(f_{h,w}^{(i)}, Y_{h,w}^{(i)}) - \ell(\tilde{f}_{h,k}^{(w)}, Y_{h,w}^{(i)})\right) \\
& \le \frac{1}{M^{(i)}}\sum_{k \in [M^i]} \frac{1}{S_k^i} \sum_{(h,w) \in S_k^i}  \max_{(h,w)} \left|\left(\ell (f_{h,w}^{(i)}, Y_{h,w}^{(i)}) - \ell (\tilde{f}_{h,w}^{(i)}, Y_{h,w}^{(i)})\right)\right| \\
& \le \frac{1}{M^{(i)}}\sum_{k \in [M^i]} \frac{1}{S_k^i} \sum_{(h,w) \in S_k^i} \mu \max_{(h, w)}\left|f_{h,w}^{(i)}-\tilde{f}_{h,w}^{(i)}\right| \\
& = \mu \Vert f^{(i)}-\tilde{f}^{(i)}\Vert_{\infty},
\end{aligned}
\end{equation}
Apparently, we can always bound $\ell (f_{h,w}^{(i)}, Y_{h,w}^{(i)}) - \ell (\tilde{f}_{h,w}^{(i)}, Y_{h,w}^{(i)})$ with the maximum element $\max_{(h,w)} |(\ell(f_{h,w}^{(i)}, Y_{h,w}^{(i)}) - \ell(\tilde{f}_{h,w}^{(i)}, Y_{h,w}^{(i)}))|$ because there are finite pixels in an image. Therefore, for separable loss $g^{(i)}$, we let $K=\mu$, and complete the proof.

For \textbf{\textit{Case 2}} composite loss. Taking DiceLoss \cite{DiceLoss} as an example, we have the following equation:
\begin{equation}
    g^{(i)}(f)=\frac{1}{M^i} \sum_{k=1}^{M^i} \left[1-\frac{2 \sum_{(h, w) \in S_k^i} \boldsymbol{Y}^{(i)}_{h,w}\cdot f^{(i)}_{h,w}}{\sum_{(h, w) \in S_k^i} \boldsymbol{Y}^{(i)}_{h,w}  + \sum_{(h, w) \in S_k^i}f^{(i)}_{h,w}}\right].
\end{equation}
Considering that formally $\text{DiceLoss}=1-\F$, combining \cref{lemma4} and \cref{lemma5}, we turn to solve $\sup \Vert\nabla_x g\Vert_1  \cdot \Vert$ instead of directly pursuing the Lipschitz constant with respect to $\ell_\infty$ norm. Therefore, we can find the Lipschitz continuous of $g^{(i)}$:
\begin{equation}
    \begin{aligned}
\left\Vert\frac{\partial g^{(i)}}{\partial f^{(i)}_{h,w}}\right\Vert_1 
&= 2\cdot\left|\frac{Y_{h,w}^{(i)}\cdot\left(\sum_{(h,w) \in S_k^i} Y_{h,w}^{(i)}  + \sum_{(h,w) \in S_k^i}f^{(i)}_{j,k}\right)-\sum_{(h,w) \in S_k^i} Y_{h,w}^{(i)}\cdot f^{(i)}_{j,k}}{\left(\sum_{(h,w) \in S_k^i} Y_{h,w}^{(i)}  + \sum_{(h,w) \in S_k^i}f^{(i)}_{j,k}\right)^2}\right| \\
& \le 2\cdot \left(\left|\frac{Y_{h,w}^{(i)}}{\sum_{(h,w) \in S_k^i} Y_{h,w}^{(i)}  + \sum_{(h,w) \in S_k^i}f^{(i)}_{j,k}}\right| + \left|\frac{\sum_{(h,w) \in S_k^i} Y_{h,w}^{(i)}\cdot f^{(i)}_{j,k}}{(\sum_{(h,w) \in S_k^i} Y_{h,w}^{(i)}  + \sum_{(h,w) \in S_k^i}f^{(i)}_{j,k})^2}\right|\right) \\
& \le 2\cdot\left(\frac{1}{\sum_{(h,w) \in S_k^i} Y_{h,w}^{(i)}  + \sum_{(h,w) \in S_k^i}f^{(i)}_{j,k}} + \frac{\sum_{(h,w) \in S_k^i} Y_{h,w}^{(i)}}{\left(\sum_{(h,w) \in S_k^i} Y_{h,w}^{(i)}  + \sum_{(h,w) \in S_k^i}f^{(i)}_{j,k}\right)^2}\right) \\
& \le 2\cdot \left(\frac{1}{\sum_{(h,w) \in S_k^i} Y_{h,w}^{(i)}} + \frac{\sum_{(h,w) \in S_k^i} Y_{h,w}^{(i)}}{\left(\sum_{(h,w) \in S_k^i} Y_{h,w}^{(i)} \right)^2}\right)\\
&= \frac{4}{\sum_{(h,w) \in S_k^i} Y_{h,w}^{(i)}} \\
&=\frac{4}{S_{k}^{+,i}},
\end{aligned}
\end{equation}
where $S_{k}^{+,i}$ stands for the number of salient pixels within the $k$-th part in $i$-th sample. Let $S_{k}^{i}$ be the sizes of the whole $k$-th frame. Therefore, we have,
\begin{equation}
    \begin{aligned}
\Vert\nabla g^{(i)}\Vert_1 &=\frac{1}{M^i}\sum_{k=1}^{M^i}
\sum_{(h,w) \in S_k^{i}}|\nabla g^{(i)}_{h,w}| \\
&= \frac{1}{M^i}\sum_{k=1}^{M^i} \left\Vert\frac{\partial g^{(i)}}{\partial f^{(i)}_{h,w}}\right\Vert_1  \cdot S_k^i \\
& \le \frac{1}{M^i}\sum_{k=1}^{M^i} 4\cdot \frac{S_k^i}{S_{k}^{+,i}} \\
& \le \frac{4}{\rho},
\end{aligned}
\end{equation}   
where $0<\rho \le \frac{S_{k}^{+,i}}{S_{k}^i}$, which depicts the threshold, how much proportion the object occupies in the corresponding frame. Therefore, taking DiceLoss as an example of composite loss, we have $K=\frac{4}{\rho}$. 

In addition, in terms of the original size-variant DiceLoss, we apparently have $M^i=1$ such that $S_k^i=S_1^i=S^i$ and $S_{k}^{+,i}=S_{1}^{+,i}=S^{+,i}$. Therefore, we have $K=\frac{4}{\rho'}, \rho'=\frac{S_k^{+,i}}{S_k^i}$ for the original size-variant DiceLoss as well.

This completes the proof.
\end{proof}

\clearpage
\section{Additional Experiment Settings} \label{Experiments_appendix}
In this section, we provide a longer version of \cref{Experiments}. 

\begin{table*}[!h]
    \centering
    \caption{Statistics on RGB SOD Datasets.}
    \label{tab1:dataset_stat}
    \scalebox{1.0}{
    \begin{tabular}{c|c|c}
    \toprule
    Dataset &  Scale & Characteristics \\
    \midrule
    DUTS~\cite{DUTS}    & 10,553 + 5,019 & Training set (10,553), as well as test set (5,019), is provided. \\
    DUT-OMRON~\cite{DUT-OMRON} & 5,168 & It is characterized by a complex background and diverse contents. \\
    MSOD~\cite{MSOD} & 300 & It consists of the most challenging \textbf{multi-object scenarios} with 1342 objects in total. \\
    ECSSD~\cite{ECSSD} & 1,000 & Semantically meaningful but structurally complex contents are included.  \\
    HKU-IS~\cite{HKU-IS} & 4,447 & Far more \textbf{multiple disconnected objects} are included. \\ 
    SOD~\cite{SOD} & 300 & Many images have \textbf{more than one salient object} that is similar to the background. \\
    PASCAL-S~\cite{ECSSD} & 850 & Images are from the PASCAL VOC 2010 validation set with \textbf{multiple salient objects}. \\
    XPIE~\cite{XPIE} & 10,000 & It covers many complex scenes with \textbf{different} numbers, sizes, and positions of salient objects. \\
    \bottomrule
    \end{tabular}
    }
\end{table*}

\subsection{Datasets} \label{dataset_appendix}
To show the effectiveness of our proposed approaches for traditional RGB SOD tasks, we adopt eight widely used benchmark datasets:
\begin{itemize}
    \item \textbf{DUTS} \cite{DUTS} is a widely used large-scale dataset, consisting of $10,553$ images in the training set (DUTS-TR), and $5,019$ images in the test set (DUTS-TE). All the images are sampled from the ImageNet DET training and test set~\cite{ImageNet}, and some test images are also collected from the SUN data set~\cite{SUN}. It is common practice that leveraging the DUTS-TR subset trains  SOD models and then test its performance on other datasets. 
    \item \textbf{DUT-OMRON} \cite{DUT-OMRON} contains $5,168$ images, where each image includes \textbf{one or more salient objects} embedded within relatively intricate backgrounds and diverse content.
     \item  \textbf{MSOD} \cite{MSOD} is some of the most challenging \textbf{multi-object} scenes among existing SOD datasets. It consists of $300$ test images containing a total of $1,342$ salient objects. The dataset features a wide range of object categories and varying object counts per image, making it well-suited for evaluating models under complex multi-object conditions.
    \item \textbf{ECSSD} \cite{ECSSD}, namely \textbf{E}xtended \textbf{C}omplex \textbf{S}cene \textbf{S}aliency \textbf{D}ataset, contains $1,000$ images featuring complex scenes with rich textures and intricate structures, where each image is collected from the internet, and ground truth masks are independently annotated by five different annotators. A key characteristic of this dataset is the inclusion of semantically meaningful yet structurally challenging images. 
    \item \textbf{HKU-IS} \cite{HKU-IS} has $4,447$ images with a relatively higher proportion of \textbf{multi-object} scenarios. Notably, approximately $50\%$ of the images in this dataset contain multiple disconnected salient objects, making it significantly more challenging.
    \item \textbf{SOD} \cite{SOD} contains $300$ images constructed based on ~\cite{SOD_construct}. Many images have \textbf{more than one salient object}, either similar to the background or located near image boundaries, which increases the difficulty of detection.
    \item \textbf{PASCAL-S} \cite{everingham2015pascal} is a salient object detection dataset consisting of 850 images, selected from the PASCAL VOC 2010 validation set. Each image typically contains \textbf{multiple} salient objects embedded in complex scenes.
    \item \textbf{XPIE} \cite{XPIE} contains $10,000$ images with pixel-wise salient object annotations. The dataset covers a wide range of complex scenes, featuring salient objects with \textbf{varying numbers}, sizes, and spatial distributions.
\end{itemize}

The detailed statistical information is summarized in \cref{tab1:dataset_stat}.

\subsection{Competitors} \label{competitor_appendix}
Given that our proposed SIEva and SIOpt are generic, we apply our proposed methods to $7$ typical backbones to demonstrate their effectiveness. Here we present a more detailed summary of the backbones mentioned in the experiments.
\begin{itemize}
    \item \textbf{PoolNet}~\cite{PoolNet} is a widely adopted SOD baseline, which introduces a set of pooling-based modules to enhance SOD performance. Specifically, it builds upon feature pyramid networks, comprising two key components: the Global Guidance Module (GGM) and the Feature Aggregation Module (FAM). GGM enables the transmission of high-level semantic information to all pyramid layers, while FAM is designed to capture local contextual information at multiple scales and integrate them through adaptive weighting.

    \item \textbf{ADMNet} \cite{ADMNet} is a recently emerging approach, which introduces an attention-guided densely multi-scale network designed for lightweight and efficient SOD. Specifically, it integrates multi-scale features through a multi-scale perception (MP) module, enhancing the network's ability to capture fine-grained details and global context. In addition, a dual attention (DA) module is incorporated to prioritize salient regions, further improving detection performance.

    \item \textbf{LDF}~\cite{LDF} introduces a level decoupling procedure alongside a feature interaction network. The saliency labels are decomposed into body and detail maps, serving as supervisory signals for different network branches. To exploit the complementary nature of these branches sufficiently, a feature interaction mechanism is employed, enabling iterative information exchange and resulting in more accurate saliency predictions.
    
    \item \textbf{ICON}~\cite{ICON} proposes the concepts of micro-integrity and macro-integrity to model both the part-whole relationship within a single salient object and the holistic identification of all salient objects in a scene. The framework comprises three main components: diverse feature aggregation, integrity channel enhancement, and part-whole verification.
    
    \item \textbf{GateNet}~\cite{GateNet} proposes a gated architecture to control the amount of information flowing into the decoder adaptively. Multi-level gate units are designed to balance contributions from different encoder stages and suppress irrelevant background noise. An Atrous Spatial Pyramid Pooling (ASPP) module is integrated to capture multi-scale contextual information.  Additionally, the dual-branch residual structure allows complementary feature learning for enhanced performance.
    
    \item \textbf{EDN}~\cite{EDN} proposes an extreme down-sampling strategy to efficiently extract high-level semantic features. The proposed Extremely Downsampled Block (EDB) enables the network to capture global context aggressively while minimizing computational overhead. Additionally, the Scale-aware Pyramid Convolution (SCPC) in the decoder combines multi-level features for improved final detection accuracy.

    \item \textbf{VST} \cite{VST,VST2} is a transformer-based backbone for SOD that effectively captures global context and long-range dependencies, addressing the limitations of traditional CNNs. It includes a multi-level token fusion mechanism and token upsampling to improve high-resolution detection, along with a multi-task transformer decoder for joint saliency and boundary detection. 
\end{itemize}

Note that, as mentioned above, \textbf{the first $6$ competitors} (including PoolNet, ADMNet, LDF, ICON, GateNet and EDN) \textbf{are CNN-based} SOD backbones, with ADMNet being a specifically tailored lightweight encoder-decoder module, and the others instantiated using ResNet50. The final model, \textbf{VST, is transformer-based}. Introducing our proposed method into \textbf{these diverse models highlights the model-agnostic property of SIOpt.}


\subsection{Evaluation Metrics \label{app:evaluation_metrics}} 
Since a single metric may not comprehensively reflect the performance of SOD models, we evaluate our proposed approach against competitive baselines using \textbf{$10$ different SOD metrics}. First of all, we employ $4$ widely used traditional metrics including $\MAE$, $\AUC$, mean F-measure ($\F_m^\beta$), and maximum F-measure ($\F_{max}^\beta$), along with their size-invariant counterparts from our proposed SIEva framework: $\SMAE$, $\SAUC$, $\SF_m^{\beta}$, and $\SF_{max}^{\beta}$. In addition, we include $2$ commonly adopted structure-aware metrics, namely E-measure ($\E_m$) \cite{Emeasure} and S-measure ($\Sm_m$) \cite{S_measure1, S_measure2}, which jointly account for regional consistency and structural similarity beyond pixel-level error. We do not consider size-invariant variants of $\E_m$ and $\Sm_m$, as these metrics inherently incorporate both local and global perceptual factors, making their sensitivity to object size relatively minor in practical evaluation scenarios. Detailed definitions and computational procedures for all metrics are provided in \cref{protocol_appendix}.

\subsection{Implementation Details for RGB tasks} \label{details_appendix}

All models are implemented using the PyTorch \footnote{\url{https://pytorch.org/}} library and the experiments are conducted on an Ubuntu 20.04.6 LTS server equipped with 256GB RAM, AMD EPYC 7763 64-Core CPU, and an A100 GPU (82GB). Following the widely used standard setups of SOD tasks \cite{EDN, ADMNet, VST2}, we train all algorithms on the DUTS training set (DUTS-TR) and test it on the DUTS test set (DUTS-TE) and the other seven datasets, where all images are resized into $3 \times 384 \times 384$ during both training and testing phases. In terms of every dataset, we adopt the \textit{scikit-image} package to identify the minimum bounding box (i.e., Eq.(\ref{eq:bounding_box})) within each image based on its corresponding ground-truth mask $\boldsymbol{Y}$. All procedures are preprocessed before training and testing, thereby introducing no additional burdens. Furthermore, to pursue a fair performance comparison, all backbones are pre-trained on the ImageNet~\cite{ImageNet} dataset, and we strictly align the training settings of our proposed size-invariant-orient optimizations with the original paper. We summarize the specific settings and optimization details for each backbone in \cref{loss_description}. 
All $\gamma_1$, $\gamma_2$ and $\gamma_3$ are searched within $\{1\times 10^{-3}, 5\times 10^{-3}, 1\times 10^{-2}, 5\times 10^{-2}, 0.1, 1.0\}$ for all backbones.
\subsection{Optimization Details for Different Backbones} \label{loss_description}


As introduced in \cref{20250324Sec4.3}, we explore two hybrid variants of SIOpt, namely $\mathsf{SIOpt1}$ and $\mathsf{SIOpt2}$. Each variant incorporates different aspects of SOD performance to guide the optimization process more effectively. Specifically, $\mathsf{SIOpt1}$ directly modifies the original loss functions into a size-invariant version, while $\mathsf{SIOpt2}$ considers the optimization from the size-invariant AUC-oriented perspective, i.e., \(\mathcal{L}_{\SI\AUC}(f)\).

In the following, we provide detailed descriptions of the implementation and optimization strategies for different backbone architectures.

\begin{itemize}
    \item \textbf{Setups for PoolNet}: The original loss function $\mathcal{L}(f)$ used in \cite{PoolNet} is any one among the binary cross entropy ($\BCE$), mean square error ($\MSE$), and mean absolute error ($\MAE$). Here we consider $\BCE$ as an example. 
    Accordingly, in terms of $\mathsf{SIOpt1}$, we modify it into
    \begin{equation} \label{ori:eq79}
        \mathcal{L}_{\mathsf{SIOpt1}}(f) = \mathcal{L}_{\SI\BCE}(f).
    \end{equation} 
    For $\mathsf{SIOpt2}$, the optimization objective is defined as follows: 
    \begin{equation} \label{SIAUC:GOAL}
    \mathcal{L}_{\mathsf{SIOpt2}}(f) = \gamma_1 \mathcal{L}_{\SI\AUC}(f) + \gamma_2 \mathcal{L}_{\BCE}(f), 
    \end{equation} 
    where $\mathcal{L}_{\BCE}(f)$ serves as a regularization term to ensure balanced SOD performance across various evaluation metrics. The rationale behind this design is that $\mathcal{L}_{\SI\AUC}(f)$ primarily focuses on enhancing the ranking consistency between salient and non-salient pixels, whereas other commonly used metrics—such as those based on $\MAE$ and $\F$-measure—are more sensitive to absolute differences between the predicted saliency map and the ground-truth mask. 
    
    During \ul{optimization}, we follow the training protocol in~\cite{PoolNet}, employing the Adam optimizer with an initial learning rate of $5 \times 10^{-5}$ and a weight decay of $5 \times 10^{-4}$. The batch size is set to $1$, and gradients are accumulated over $10$ iterations before each update step. The training lasts 24 epochs in total.

    \item \textbf{Setups for ADMNet:} The original loss function $\mathcal{L}(f)$ used in \cite{ADMNet} is
    \begin{equation}
        \mathcal{L}(f) = \gamma_1\mathcal{L}_{\BCE}(f) + \gamma_2\mathcal{L}_{\mathsf{IOU}}(f) + \gamma_3\mathcal{L}_{\mathsf{SSIM}}(f),
    \end{equation}
    where $\mathcal{L}_{*}(f)$ represents the empirical version of a loss function, such as
    $$\mathcal{L}_{\BCE}(f):=\mathop{\hat{\mathbb{E}}}_{(\boldsymbol{X},\boldsymbol{Y}) \sim \mathcal{D}}[\ell_{\BCE}(f(\boldsymbol{X}), \boldsymbol{Y})],$$ which denotes the empirical $\BCE$ loss computed over all samples in the dataset $\mathcal{D}$.
    
    Correspondingly, $\mathsf{SIOpt1}$ modifies the above loss into
    \begin{equation}
        \mathcal{L}_{\mathsf{SIOpt1}}(f) = \gamma_1\mathcal{L}_{\SI\BCE}(f) + \gamma_2\mathcal{L}_{\mathsf{SIIOU}}(f) + \gamma_3\mathcal{L}_{\mathsf{SISSIM}}(f),
    \end{equation} 
    where each size-invariant loss term is naturally derived according to Eq.~(\ref{eq:loss}).
    
For $\mathsf{SIOpt2}$, we maintain the same optimization objective as in Eq.~(\ref{SIAUC:GOAL}). During \ul{optimization}, we follow the training protocol outlined in \cite{ADMNet}, using the Adam optimizer with parameters $\beta_1=0.9$ and $\beta_2=0.999$, and a weight decay of $0$. The initial learning rate is set to $1.5 \times 10^{-3}$, with a batch size of $12$. The training is conducted for a total of $260$ epochs.
    
    \item \textbf{Setups for LDF:} The original loss function involved in LDF \cite{LDF} is
    \begin{equation} \label{ori:LDFloss}
         \mathcal{L}(f) = \gamma_1\mathcal{L}_{\BCE}(f) + \gamma_2\mathcal{L}_{\mathsf{IOU}}(f).
    \end{equation}
    In the case of $\mathsf{SIOpt1}$, the above loss is modified into
    \begin{equation} \label{ori:eq84}
        \mathcal{L}_{\mathsf{SIOpt1}}(f) = \gamma_1\mathcal{L}_{\SI\BCE}(f) + \gamma_2\mathcal{L}_{\mathsf{SIIOU}}(f).
    \end{equation} 
    $\mathsf{SIOpt2}$ still maintains the same optimization objective as in Eq.~(\ref{SIAUC:GOAL}).

     Note that, considering that LDF is a two-stage framework, we only integrate our method into the second stage during \ul{optimization}, as most previous works do. Specifically, we use the SGD optimizer with the initial learning rate as $5 \times 10^{-2}$, momentum as $0.9$, and the weight decay as $5 \times 10 ^{-4}$. The batch size is set to 32, and the training in the second stage lasts 40 epochs.

    \item  \textbf{Setups for ICON:} The original loss function is the same as Eq.(\ref{ori:LDFloss}) in LDF and thus $\mathsf{SIOpt1}$ follows Eq.(\ref{ori:eq84}). In the case of $\mathsf{SIOpt2}$, we consider 
    \[
        \mathcal{L}_{\mathsf{SIOpt2}}(f) = \gamma_1\mathcal{L}_{\SI\AUC}(f) + \gamma_2\mathcal{L}_{\mathsf{SIDice}}(f).
    \] for ICON. During \ul{optimization}, the SGD optimizer is adopted with the initial learning rate being $10^{-2}$, the weight decay being $10^{-4}$, and the momentum being $0.9$. The batch size is 36, and the training lasts 100 epochs.
    
    \item \textbf{Setups for GateNet:} The original loss function stays the same as PoolNet \cite{GateNet}, and so $\mathsf{SIOpt1}$ and $\mathsf{SIOpt2}$ do. During \ul{optimization}, we strictly follow ~\cite{GateNet} to utilize the SGD optimizer, with the initial learning rate as $10^{-3}$, momentum as $0.9,$ and the weight decay as $5 \times 10^{-4}$. The batch size is set to 12, and the network is iterated within $10^5$ times.
    
    \item \textbf{Setups for EDN:} The original loss function involved in EDN \cite{EDN} is as follows:
    \begin{equation}
        \mathcal{L}(f) = \mathcal{L}_{\mathsf{BCE}}(f) + \mathcal{L}_{\mathsf{Dice}}(f).
    \end{equation}
    Thus, the optimization goal of $\mathsf{SIOpt1}$ is formulated as follows:
    \begin{equation}\label{20250418:eq100}
        \mathcal{L}_{\mathsf{SIOpt1}}(f) = \gamma_1\mathcal{L}_{\SI\BCE}(f) + \gamma_2\mathcal{L}_{\mathsf{SIDice}}(f),
    \end{equation}
    while the definition of $\mathsf{SIOpt2}$ stays the same as Eq.(\ref{SIAUC:GOAL}). During \ul{optimization}, Adam optimizer with $\beta_1=0.9$, and $\beta_2=0.99$ is adopted, with weight decay $10^{-4}$ and batch size $36$. The initial learning rate is set to $5 \times 10^{-5}$ with a poly learning rate strategy, and the training lasts for 100 epochs in total.
    \item \textbf{Setups for VST:} The original paper adopts $\BCE$ loss as the optimization goal, and $\mathsf{SIOpt1}$ and $\mathsf{SIOpt2}$ follows Eq.(\ref{ori:eq79}) and Eq.(\ref{SIAUC:GOAL}), respectively. During \ul{optimization}, we follow the training strategy in~\cite{VST}, where the Adam optimizer is adopted with a learning rate of $1 \times 10^{-4}$ and a weight decay of $5 \times 10^{-4}$. The batch size is set to $12$, and all of the VST-based model is trained for $200$ epochs.
\end{itemize}

\begin{table*}[!t]
    \centering
    \caption{Statistics on RGB-D SOD Datasets.}
    \label{tab1:RGBD_dataset_stat}
    \scalebox{1.0}{
    \begin{tabular}{c|c|c}
    \toprule
    Dataset &  Scale & Characteristics \\
    \midrule
    NLPR-NJU2K-TR    & 700 + 1,485 & Training set: $700$ samples from NLPR and $1, 485$ samples from NJU2K.  \\
    NJUD-TE \cite{NLPR} & $1,985$ & Test set: $1,985$ stereo SOD image with corresponding depth maps.   \\
    NLPR-TE \cite{NLPR} & $1,000$ & Test set: $1,000$ RGB images and corresponding depth maps. \\
    STERE \cite{STERE} & $1,250$ & Test set: $1,250$ stereoscopic image pairs.  \\
    \bottomrule
    \end{tabular}
    }
\end{table*}

\subsection{Implementation Details for RGB-D SOD tasks \label{app:RGBD_details}} 
\subsubsection{Datasets}
To evaluate the effectiveness of the proposed method, we conduct extensive experiments on several challenging RGB-D SOD benchmarks. Here, all models are trained exclusively on the combined NLPR-NJU2K training set (denoted as NLPR-NJU2K-TR), and evaluated on the remaining test sets.

\begin{itemize} \item \textbf{NLPR-NJU2K-TR:} The training set is composed of $700$ samples from NLPR and $1,485$ samples from NJU2K. 
\item \textbf{NJUD-TE \cite{NLPR}:} It contains $1,985$ stereo image pairs collected from the internet, 3D movies, and photographs captured using a Fuji W3 stereo camera. 
\item \textbf{NLPR-TE \cite{NLPR}:} It comprises $1,000$ RGB images and corresponding depth maps captured by a Microsoft Kinect, covering both indoor and outdoor scenes. 
\item \textbf{STERE \cite{STERE}:} It provides $1,250$ stereoscopic image pairs sourced from Flickr, NVIDIA 3D Vision Live, and the Stereoscopic Image Gallery.

\end{itemize}

The summarization of RGB-D datasets is shown in Tab.\ref{tab1:RGBD_dataset_stat}.

\subsubsection{Baselines}
\begin{itemize}
    \item \textbf{CRNet}\footnote{\url{https://github.com/guanyuzong/CR-Net}} \cite{CRNet} is a cascaded RGB-D SOD framework that leverages well-designed attention mechanisms to effectively extract features from both RGB and depth images. Specifically, CRNet consists of two primary modules: a cross-modal attention fusion module that integrates multi-modality and multi-level features extracted by a dual-stream Swin Transformer encoder, and a progressive decoder that gradually fuses low-level features and suppresses noise to generate accurate saliency predictions. The original optimization objective of CRNet combines the B$\BCE$ and $\IOU$ loss functions, similar to the formulation in Eq.~(\ref{ori:LDFloss}). Accordingly, we adopt Eq.~(\ref{ori:eq84}) and Eq.~(\ref{SIAUC:GOAL}) as our size-invariant objectives for $\mathsf{SIOpt1}$ and $\mathsf{SIOpt2}$, respectively.

For implementation, we follow the training and optimization protocols described in the original CRNet paper \cite{CRNet}. During training and testing, the input RGB images are resized to $3 \times 352 \times 352$. Standard data augmentation techniques such as random flipping, rotation, and border cropping are applied. The initial learning rate is set to $1 \times 10^{-4}$, with a batch size of $5$. We use the Adam optimizer with default hyperparameters and train the model for a total of $200$ epochs.
        
    \item \textbf{CPNet}\footnote{\url{https://github.com/hu-xh/CPNet}} \cite{hu2024cross} proposes a cross-modal fusion and progressive decoding network to promote the RGB-D SOD performance, which mainly includes three key parts: a two-stream Swin Transformer \cite{SwinNet} feature encoder module, a cross-modal attention feature fusion module and a progressive feature decoder module. The original optimization objective of CPNet combines $\BCE$ with $\IOU$ loss functions, similar to Eq.(\ref{ori:LDFloss}). Therefore, we still consider Eq.(\ref{ori:eq84}) and Eq.(\ref{SIAUC:GOAL}) as our size-invariant goal for $\mathsf{SIOpt1}$ and $\mathsf{SIOpt2}$, respectively. 

    The implementation details and optimization strategies follow the original paper \cite{hu2024cross}. During the training and testing phase, the input RGB image and depth image are resized to $384 \times 384$ and the depth image is then copied to $3$ channels to ensure consistency with the RGB image. The pre-trained Swin-B model is introduced to initialize the parameters of the backbone network involved in CPNet. Additionally, we leverage the Adam optimizer with an initial learning rate $3 \times 10^{-5}$ which will be divided by $10$ per $100$ epochs. The batch size is set to $8$ and the training lasts $150$ epochs in total.
\end{itemize} 

\subsection{Implementation Details for RGB-T SOD tasks} \label{RGBT_detail}
\subsubsection{Datasets}
To evaluate the effectiveness of the proposed method, we conduct experiments on three publicly available RGB-T salient object detection (SOD) benchmark datasets: \textbf{VT821} \cite{VT821}, \textbf{VT1000} \cite{VT1000}, and \textbf{VT5000} \cite{VT5000}. The VT821 dataset comprises 821 RGB-T image pairs captured across diverse scenes. As the RGB and thermal images in VT821 are manually registered, missing regions may appear in the thermal modality. The VT1000 dataset includes 1,000 RGB-T image pairs with corresponding pixel-wise saliency annotations. VT5000, the largest RGB-T SOD benchmark to date, consists of 5,000 image pairs that span a wide range of complex and challenging scenarios.

In our experiments, we use 2,500 image pairs from VT5000 for training. The remaining 2,500 pairs from VT5000, along with all samples from VT1000 and VT821, are used for testing.

\subsubsection{Baselines}
\begin{itemize}
    \item \textbf{TNet} \footnote{\url{https://github.com/rmcong/TNet_TMM2022}} \cite{TNet} is a novel network designed for RGB-T SOD tasks, incorporating a global illumination estimation module and cross-modality interaction mechanisms to enhance feature fusion and boost detection performance.
    
    The implementation strictly adheres to the settings and optimization strategies outlined in the original paper. The backbone network is a ResNet pretrained on ImageNet, while all other parameters are initialized using PyTorch's default settings. During both training and testing, all RGB and thermal images are resized to $352 \times 352$. The Adam optimizer is employed with an initial learning rate of $10^{-4}$, which is decayed by a factor of $10$ every $45$ epochs, for a total of $100$ epochs. The batch size is set to $16$. Data augmentation is applied following the protocol described in the original work. For $\mathsf{SIOpt1}$, we adopt the $\mathsf{SIBCE}$ loss during the first 30 epochs, followed by the introduction of the $\mathsf{SIIOU}$ loss, consistent with the original setup (i.e., Eq.(\ref{ori:eq84})). For $\mathsf{SIOpt2}$, training begins with the standard $\mathsf{BCE}$ loss for the first $30$ epochs, after which the $\mathsf{SIAUC}$ loss Eq.(\ref{SIAUC:GOAL}) is incorporated. The weight of the $\mathsf{SIAUC}$ loss is set to $0.1$ to ensure it remains on a comparable scale with the $\mathsf{BCE}$ loss.

    \item \textbf{DCNet} \footnote{\url{https://github.com/lz118/Deep-Correlation-Network}} \cite{DCNet} proposes a deep correlation network (DCNet) for weakly aligned RGB-T salient object detection. It introduces a modality alignment module and a bi-directional decoder to enhance feature fusion, along with a modality correlation ConvLSTM for hierarchical feature decoding. 
    
    The implementation strictly adheres to the training settings and optimization strategies outlined in the original paper. During both training and testing, all RGB and thermal images are resized to $352 \times 352$, and data augmentation is applied following the original protocol. The SGD optimizer is employed with an initial learning rate of $10^{-3}$, a weight decay of $5 \times 10^{-4}$, and a momentum of $0.9$. After $50$ epochs, the learning rate is reduced to $1 \times 10^{-4}$, and training proceeds for a total of $100$ epochs with a batch size of $4$. For $\mathsf{SIOpt1}$, the model is initially trained using the standard $\mathsf{BCE}$ loss for the first $20$ epochs, followed by the application of the $\mathsf{SIBCE}$ loss. Similarly, for $\mathsf{SIOpt2}$, training begins with the standard $\mathsf{BCE}$ loss for the first $20$ epochs, after which the $\mathsf{SIAUC}$ loss Eq.(\ref{SIAUC:GOAL}) is introduced. The weight of the $\mathsf{SIAUC}$ loss is set to $0.1$ to ensure its scale remains comparable to that of the $\mathsf{BCE}$ loss.

\end{itemize}


\subsection{Implementation Details for fine-tuning SAM} \label{SAM_detail}
To validate the effectiveness of our proposed framework on large-scale models, we apply $\mathsf{SIOpt}$ to the TS-SAM framework~\cite{SAM3}, an efficient adapter-based approach designed to enhance the performance of the large-scale vision model SAM on downstream SOD tasks. Specifically, TS-SAM introduces a lightweight Parameter-Efficient Fine-Tuning (PEFT) strategy, termed Convolutional Side Adapter (CSA), which facilitates the integration of discriminative features extracted by SAM into a side network for comprehensive feature fusion. To further refine predictions, a Multi-scale Refinement Module (MRM) and a Feature Fusion Decoder (FFD) are subsequently employed. The interested readers are encouraged to read the original paper \cite{SAM3}.

\noindent \textbf{Setups for TS-SAM:} We follow the implementation details and training strategies of the original TS-SAM. Specifically, TS-SAM leverages the pretrained SAM  Vision Transformer (ViT) as the backbone and includes $14$ layers of CSA and $13$ layers of MRM. Subsequently, the $\BCE$ and $\IOU$ loss are used for TS-SAM training, i.e., 
\begin{equation}\label{20250420eq94}
     \mathcal{L}(f) = \gamma_1\mathcal{L}_{\BCE}(f) + \gamma_2\mathcal{L}_{\mathsf{IOU}}(f).
\end{equation}
Thus, the optimization goal of $\mathsf{SIOpt1}$ is formulated as follows:
    \begin{equation}\label{20250420eq95}
        \mathcal{L}_{\mathsf{SIOpt1}}(f) = \gamma_1\mathcal{L}_{\SI\BCE}(f) + \gamma_2\mathcal{L}_{\mathsf{SIIOU}}(f).
    \end{equation}

During \ul{optimization}, we adopt the AdamW optimizer with a learning rate of $8 \times 10^{-4}$. Similar to the protocol in \cref{details_appendix}, all models are trained on the DUTS training set (DUTS-TR) and evaluated on the DUTS test set (DUTS-TE), as well as DUT-OMRON, MSOD, ECSSD, and HKU-IS. All input images are resized to a resolution of $3 \times 1024 \times 1024$. The batch size is set to $2$, and training is conducted for a total of $80$ epochs.

\clearpage
\section{Additional Experiment Results}

\subsection{Empirical results for RGB-based SOD} \label{app:overall}
We show the results on eight benchmarks (including DUTS, DUT-OMRON, MSOD, ECSSD, HKU-IS, SOD, PASCAL-S and XPIE datasets) in \cref{app:exp_result}, \cref{tab:first3} and \cref{tab:latter3}. Additionally, a more intuitive comparison of baseline, $\mathsf{SIOpt1}$ and $\mathsf{SIOpt2}$ is depicted in \cref{fig:leida-all}, where we plot $1-\SMAE$ to make a clearer visualization with other metrics. The performance analysis could be found in \cref{20250419Sec6.2}.


\begin{table*}[!t]
    \centering
    \caption{Performance comparison on the DUTS-TE and DUT-OMRON datasets. The best results are highlighted in bold, and the second-best are marked with underline. Here, $\uparrow$ indicates that higher values denote better performance, while $\downarrow$ indicates that lower values are preferable.}
    \scalebox{0.85}{
\begin{tabular}{c|c|cccccccccc}
      \toprule
      \multicolumn{1}{c}{Dataset} & \multicolumn{1}{|c|}{Methods} & $\MAE \downarrow$ & $\SMAE \downarrow$ & $\AUC \uparrow $ & $\SAUC \uparrow $ & $\F_m^{\beta} \uparrow$ & $\SF_m^{\beta} \uparrow $ & $\F_{max}^{\beta} \uparrow$ & $\SF_{max}^{\beta} \uparrow$ & $\E_m \uparrow$ & $\Sm_m \uparrow$ \\
      \midrule
      \multirow{21}[12]{*}{DUTS-TE} & PoolNet \cite{PoolNet+} & \cellcolor[rgb]{ .973,  .973,  .988}0.0656  & \cellcolor[rgb]{ .973,  .969,  .988}0.0609  & \cellcolor[rgb]{ .894,  .89,  .961}0.9607  & \cellcolor[rgb]{ .89,  .886,  .957}0.9550  & \cellcolor[rgb]{ .992,  .992,  1}0.7200  & \cellcolor[rgb]{ .992,  .992,  1}0.7569  & \cellcolor[rgb]{ .973,  .973,  .992}0.8245  & \cellcolor[rgb]{ .941,  .937,  .976}0.8715  & 0.8103  & \cellcolor[rgb]{ .988,  .984,  .996}0.9112  \\
            & + $\mathsf{SIOpt1}$ & \cellcolor[rgb]{ .957,  .957,  .984}\underline{0.0621}  & \cellcolor[rgb]{ .949,  .945,  .98}\underline{0.0562}  & \cellcolor[rgb]{ .871,  .863,  .949}\textbf{0.9706} & \cellcolor[rgb]{ .867,  .859,  .949}\textbf{0.9647} & \cellcolor[rgb]{ .961,  .957,  .984}\underline{0.7479}  & \cellcolor[rgb]{ .922,  .918,  .973}\textbf{0.8172} & \cellcolor[rgb]{ .941,  .941,  .98}\underline{0.8438}  & \cellcolor[rgb]{ .886,  .878,  .957}\textbf{0.9029} & \cellcolor[rgb]{ .949,  .945,  .98}\underline{0.8478}  & \cellcolor[rgb]{ .953,  .949,  .984}\underline{0.9188}  \\
\bottomrule
      \end{tabular}%
      }
    \label{app:exp_result}%
  \end{table*}%

\begin{table*}[htbp]
  \centering
  \caption{Quantitative comparisons on MSOD, ECSSD, and HKU-IS datasets. The best results are highlighted in bold, and the second-best are marked with underline. Here, $\uparrow$ indicates that higher values denote better performance, while $\downarrow$ indicates that lower values are preferable. Darker colors represent stronger performance.}
  \scalebox{0.91}{
    \begin{tabular}{c|c|cccccccccc}
    \toprule
    \multicolumn{1}{c}{Dataset} & \multicolumn{1}{|c|}{Methods} & $\MAE \downarrow$ & $\SMAE \downarrow$ & $\AUC \uparrow $ & $\SAUC \uparrow $ & $\F_m^{\beta} \uparrow$ & $\SF_m^{\beta} \uparrow $ & $\F_{max}^{\beta} \uparrow$ & $\SF_{max}^{\beta} \uparrow$ & $\E_m \uparrow$ & $\Sm_m \uparrow$ \\
    \midrule
    \multirow{21}[12]{*}{MSOD} & PoolNet \cite{PoolNet+} & \cellcolor[rgb]{ .996,  .988,  .98}0.0752  & 0.1196  & \cellcolor[rgb]{ .992,  .918,  .863}0.9375  & \cellcolor[rgb]{ .992,  .922,  .875}0.9211  & 0.6645  & \cellcolor[rgb]{ 1,  .98,  .969}0.6397  & \cellcolor[rgb]{ 1,  .988,  .98}0.7755  & \cellcolor[rgb]{ .992,  .925,  .878}0.8402  & 0.7529  & 0.8956  \\
          & + $\mathsf{SIOpt1}$ & \cellcolor[rgb]{ .992,  .953,  .922}\textbf{0.0635} & \cellcolor[rgb]{ .992,  .937,  .902}\textbf{0.0924} & \cellcolor[rgb]{ .988,  .894,  .827}\textbf{0.9553} & \cellcolor[rgb]{ .992,  .902,  .835}\textbf{0.9440} & \cellcolor[rgb]{ .996,  .953,  .922}\textbf{0.7314} & \cellcolor[rgb]{ .992,  .925,  .882}\textbf{0.7467} & \cellcolor[rgb]{ .996,  .941,  .902}\textbf{0.8200} & \cellcolor[rgb]{ .992,  .898,  .831}\textbf{0.8867} & \cellcolor[rgb]{ .996,  .945,  .91}\textbf{0.8286} & \cellcolor[rgb]{ .996,  .949,  .918}\underline{0.9162}  \\
          & + $\mathsf{SIOpt2}$ & \cellcolor[rgb]{ .992,  .953,  .925}\underline{0.0646}  & \cellcolor[rgb]{ .992,  .957,  .929}\underline{0.1002}  & \cellcolor[rgb]{ .992,  .91,  .851}\underline{0.9441}  & \cellcolor[rgb]{ .992,  .914,  .859}\underline{0.9306}  & \cellcolor[rgb]{ .996,  .965,  .941}\underline{0.7167}  & \cellcolor[rgb]{ 1,  .969,  .945}\underline{0.6675}  & \cellcolor[rgb]{ .996,  .949,  .914}\underline{0.8130}  & \cellcolor[rgb]{ .992,  .922,  .875}\underline{0.8460}  & \cellcolor[rgb]{ .996,  .953,  .922}\underline{0.8213}  & \cellcolor[rgb]{ .996,  .933,  .89}\textbf{0.9223} \\
\bottomrule
    \end{tabular}%
   }
  \label{tab:first3}%
\end{table*}%

\begin{table*}[htbp]
  \centering
  \caption{Quantitative comparisons on SOD, PASCAL-S, and XPIE. The best results are highlighted in bold, and the second-best are marked with underline. Here, $\uparrow$ indicates that higher values denote better performance, while $\downarrow$ indicates that lower values are preferable. Darker colors represent stronger performance.}
  \label{tab:latter3}
  \scalebox{0.91}{
    \begin{tabular}{c|c|cccccccccc}
    \toprule
    \multicolumn{1}{c}{Dataset} & \multicolumn{1}{c}{Methods} & $\MAE \downarrow$ & $\SMAE \downarrow$ & $\AUC \uparrow $ & $\SAUC \uparrow $ & $\F_m^{\beta} \uparrow$ & $\SF_m^{\beta} \uparrow $ & $\F_{max}^{\beta} \uparrow$ & $\SF_{max}^{\beta} \uparrow$ & $\E_m \uparrow$ & $\Sm_m \uparrow$ \\
    \midrule
    \multirow{21}[12]{*}{SOD} & PoolNet \cite{PoolNet+} & \cellcolor[rgb]{ .996,  .969,  .973}0.1353  & \cellcolor[rgb]{ .996,  .976,  .98}0.1219  & \cellcolor[rgb]{ .988,  .886,  .902}\underline{0.9000}  & \cellcolor[rgb]{ .988,  .89,  .902}\underline{0.8911}  & \cellcolor[rgb]{ 1,  .98,  .984}0.6974  & \cellcolor[rgb]{ 1,  .98,  .98}0.6685  & \cellcolor[rgb]{ .992,  .933,  .941}0.8356  & \cellcolor[rgb]{ .988,  .91,  .922}0.8370  & \cellcolor[rgb]{ 1,  .988,  .988}0.7208  & \cellcolor[rgb]{ .992,  .918,  .925}0.8529  \\
          & + $\mathsf{SIOpt1}$ & \cellcolor[rgb]{ .988,  .937,  .945}\underline{0.1235}  & \cellcolor[rgb]{ .988,  .933,  .941}\textbf{0.1095} & \cellcolor[rgb]{ .98,  .855,  .871}\textbf{0.9146} & \cellcolor[rgb]{ .98,  .855,  .871}\textbf{0.9090} & \cellcolor[rgb]{ .992,  .937,  .945}\textbf{0.7366} & \cellcolor[rgb]{ .992,  .918,  .925}\textbf{0.7246} & \cellcolor[rgb]{ .992,  .918,  .925}\underline{0.8437}  & \cellcolor[rgb]{ .984,  .863,  .878}\textbf{0.8658} & \cellcolor[rgb]{ .992,  .929,  .937}\textbf{0.7662} & \cellcolor[rgb]{ .984,  .882,  .894}\underline{0.8782}  \\
\bottomrule
    \end{tabular}}
\end{table*}%

\begin{figure*}[!t]
    \centering
    \includegraphics[width=\linewidth]{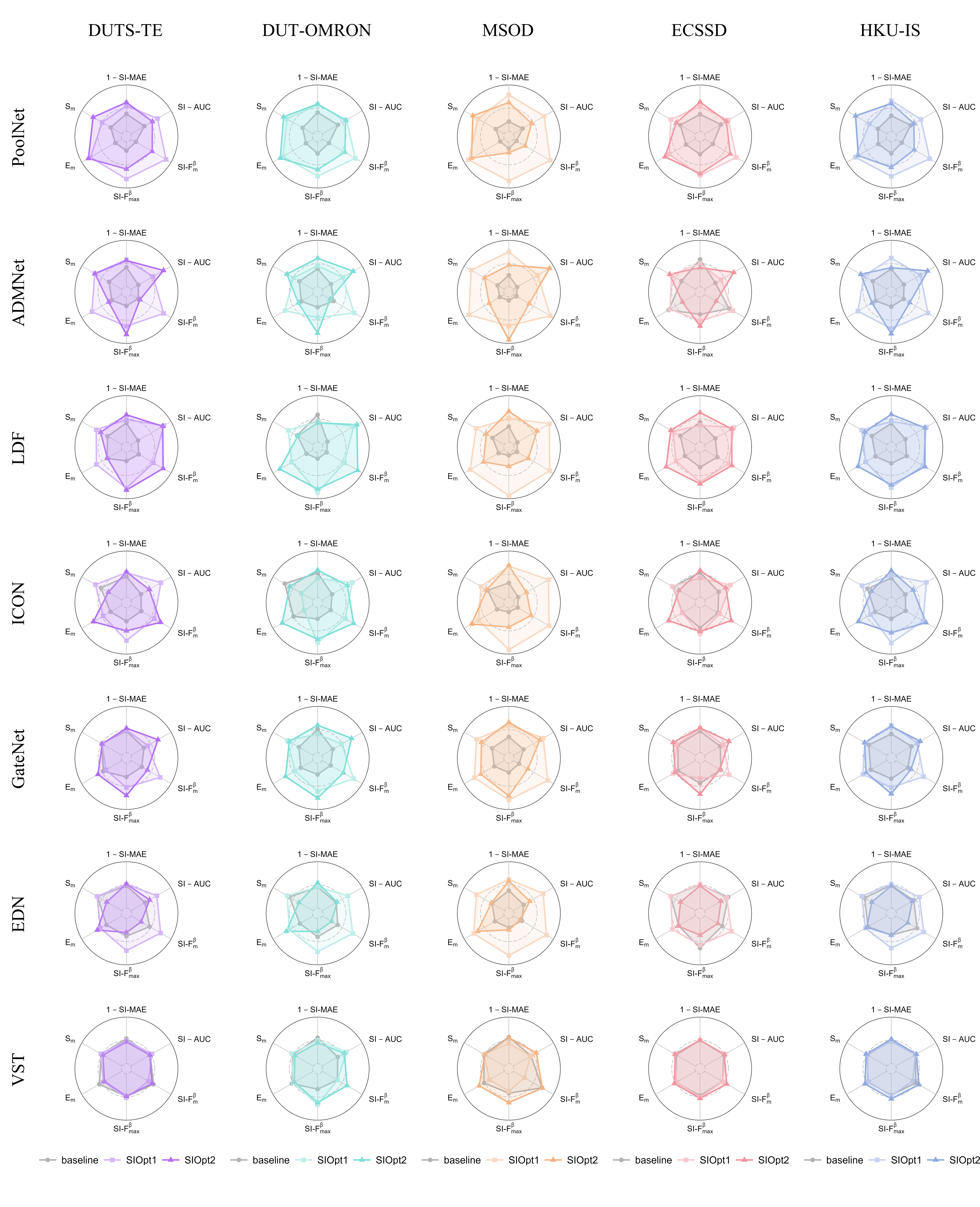}
    \caption{Overall performance comparison on different backbones and datasets. Note that we plot $1-\SMAE$ to make a clearer visualization with other metrics, where values closer to the outer ring indicate better performance.}
    \label{fig:leida-all}
\end{figure*}

\begin{table}[!t]
    \centering
    \caption{Quantitative comparisons on \textcolor{blue}{RGB-D} benchmark datasets. The best results are highlighted in bold, and the second-best are marked with underline. Here, $\uparrow$ indicates that higher values denote better performance, while $\downarrow$ indicates that lower values are preferable. Darker colors represent stronger performance.}
      \begin{tabular}{c|c|cccccccccc}
      \toprule
      \multicolumn{1}{c}{Dataset} & \multicolumn{1}{|c|}{Methods} & $\MAE \downarrow$ & $\SMAE \downarrow$ & $\AUC \uparrow $ & $\SAUC \uparrow $ & $\F_m^{\beta} \uparrow$ & $\SF_m^{\beta} \uparrow $ & $\F_{max}^{\beta} \uparrow$ & $\SF_{max}^{\beta} \uparrow$ & $\E_m \uparrow$ & $\Sm_m \uparrow$ \\
      \midrule
      \multirow{6}[4]{*}{STERE} & CRNet \cite{CRNet} & 0.0663  & 0.0478  & 0.8823  & 0.8804  & 0.8033  & 0.8294  & 0.8321  & 0.8592  & 0.8648  & 0.9053  \\
            & + $\mathsf{SIOpt1}$&  \cellcolor[rgb]{ .988,  .894,  .827}\textbf{0.0571} & \cellcolor[rgb]{ .988,  .894,  .827}\underline{0.0415}  & \cellcolor[rgb]{ .992,  .929,  .882}\underline{0.9103}  & \cellcolor[rgb]{ .992,  .929,  .882}\underline{0.9088}  & \cellcolor[rgb]{ .988,  .894,  .827}\textbf{0.8358} & \cellcolor[rgb]{ .988,  .894,  .827}\textbf{0.8675} & \cellcolor[rgb]{ .992,  .906,  .843}\underline{0.8640}  & \cellcolor[rgb]{ .992,  .898,  .835}\underline{0.8954}  & \cellcolor[rgb]{ .992,  .898,  .835}\underline{0.8939}  & \cellcolor[rgb]{ .988,  .894,  .827}\textbf{0.9131} \\
            & + $\mathsf{SIOpt2}$ & \cellcolor[rgb]{ .988,  .894,  .827}\textbf{0.0571} & \cellcolor[rgb]{ .988,  .894,  .827}\textbf{0.0414} & \cellcolor[rgb]{ .988,  .894,  .827}\textbf{0.9227} & \cellcolor[rgb]{ .988,  .894,  .827}\textbf{0.9214} & \cellcolor[rgb]{ .992,  .918,  .863}\underline{0.8296}  & \cellcolor[rgb]{ .992,  .906,  .847}\underline{0.8633}  & \cellcolor[rgb]{ .988,  .894,  .827}\textbf{0.8671} & \cellcolor[rgb]{ .988,  .894,  .827}\textbf{0.8967} & \cellcolor[rgb]{ .988,  .894,  .827}\textbf{0.8946} & \cellcolor[rgb]{ .992,  .902,  .839}\underline{0.9127}  \\
  \cmidrule{2-12}          & CPNet \cite{hu2024cross} & 0.0311  & 0.0234  & 0.9753  & 0.9741  & \cellcolor[rgb]{ .996,  .961,  .933}\underline{0.9035}  & \cellcolor[rgb]{ .992,  .902,  .835}\underline{0.9368}  & 0.9293  & 0.9541  & \cellcolor[rgb]{ 1,  .992,  .984}\underline{0.9469}  & 0.9360  \\
            & + $\mathsf{SIOpt1}$&  \cellcolor[rgb]{ .988,  .894,  .827}\textbf{0.0303} & \cellcolor[rgb]{ .988,  .894,  .827}\textbf{0.0228} & \cellcolor[rgb]{ .988,  .894,  .827}\textbf{0.9800} & \cellcolor[rgb]{ .988,  .894,  .827}\textbf{0.9790} & \cellcolor[rgb]{ .988,  .894,  .827}\textbf{0.9064} & \cellcolor[rgb]{ .988,  .894,  .827}\textbf{0.9370} & \cellcolor[rgb]{ .988,  .894,  .827}\textbf{0.9346} & \cellcolor[rgb]{ .988,  .894,  .827}\textbf{0.9562} & \cellcolor[rgb]{ .988,  .894,  .827}\textbf{0.9485} & \cellcolor[rgb]{ .988,  .894,  .827}\textbf{0.9384} \\
            & + $\mathsf{SIOpt2}$&  \cellcolor[rgb]{ .992,  .957,  .933}\underline{0.0308}  & \cellcolor[rgb]{ .992,  .961,  .941}\underline{0.0232}  & \cellcolor[rgb]{ .992,  .902,  .839}\underline{0.9797}  & \cellcolor[rgb]{ .992,  .906,  .847}\underline{0.9785}  & 0.9016  & 0.9324  & \cellcolor[rgb]{ .996,  .937,  .898}\underline{0.9325}  & \cellcolor[rgb]{ 1,  .988,  .976}\underline{0.9544}  & 0.9467  & \cellcolor[rgb]{ 1,  .973,  .953}\underline{0.9367}  \\
      \midrule
      \multirow{6}[4]{*}{NJUD-TE} & CRNet \cite{CRNet} & 0.0699  & 0.0550  & 0.8676  & 0.8597  & 0.7900  & 0.7919  & 0.8247  & 0.8293  & 0.8434  & 0.9065  \\
            & + $\mathsf{SIOpt1}$&  \cellcolor[rgb]{ .98,  .867,  .882}\underline{0.0483}  & \cellcolor[rgb]{ .98,  .859,  .875}\underline{0.0377}  & \cellcolor[rgb]{ .988,  .898,  .906}\underline{0.9190}  & \cellcolor[rgb]{ .988,  .894,  .906}\underline{0.9147}  & \cellcolor[rgb]{ .98,  .855,  .871}\textbf{0.8637} & \cellcolor[rgb]{ .98,  .855,  .871}\textbf{0.8711} & \cellcolor[rgb]{ .984,  .867,  .878}\underline{0.8930}  & \cellcolor[rgb]{ .984,  .859,  .875}\underline{0.9009}  & \cellcolor[rgb]{ .984,  .871,  .886}\underline{0.9032}  & \cellcolor[rgb]{ .984,  .863,  .878}\underline{0.9251}  \\
            & + $\mathsf{SIOpt2}$&  \cellcolor[rgb]{ .98,  .855,  .871}\textbf{0.0457} & \cellcolor[rgb]{ .98,  .855,  .871}\textbf{0.0370} & \cellcolor[rgb]{ .98,  .855,  .871}\textbf{0.9381} & \cellcolor[rgb]{ .98,  .855,  .871}\textbf{0.9326} & \cellcolor[rgb]{ .984,  .859,  .875}\underline{0.8630}  & \cellcolor[rgb]{ .984,  .863,  .878}\underline{0.8685}  & \cellcolor[rgb]{ .98,  .855,  .871}\textbf{0.8972} & \cellcolor[rgb]{ .98,  .855,  .871}\textbf{0.9010} & \cellcolor[rgb]{ .98,  .855,  .871}\textbf{0.9104} & \cellcolor[rgb]{ .98,  .855,  .871}\textbf{0.9260} \\
  \cmidrule{2-12}          & CPNet \cite{hu2024cross} & \cellcolor[rgb]{ .992,  .957,  .961}\underline{0.0112}  & 0.0117  & 0.9919  & 0.9892  & \cellcolor[rgb]{ .996,  .953,  .957}\underline{0.9582}  & 0.9526  & 0.9730  & 0.9691  & 0.9773  & 0.9622  \\
            & + $\mathsf{SIOpt1}$&  0.0114  & \cellcolor[rgb]{ .98,  .855,  .871}\textbf{0.0103} & \cellcolor[rgb]{ .984,  .859,  .875}\underline{0.9964}  & \cellcolor[rgb]{ .98,  .855,  .871}\textbf{0.9958} & \cellcolor[rgb]{ .98,  .855,  .871}\textbf{0.9598} & \cellcolor[rgb]{ .98,  .855,  .871}\textbf{0.9581} & \cellcolor[rgb]{ .98,  .855,  .871}\textbf{0.9771} & \cellcolor[rgb]{ .98,  .855,  .871}\textbf{0.9752} & \cellcolor[rgb]{ .992,  .925,  .933}\underline{0.9783}  & \cellcolor[rgb]{ .98,  .855,  .871}\textbf{0.9635} \\
            & + $\mathsf{SIOpt2}$&  \cellcolor[rgb]{ .98,  .855,  .871}\textbf{0.0107} & \cellcolor[rgb]{ .98,  .863,  .878}\underline{0.0104}  & \cellcolor[rgb]{ .98,  .855,  .871}\textbf{0.9965} & \cellcolor[rgb]{ .98,  .855,  .871}\textbf{0.9958} & 0.9574  & \cellcolor[rgb]{ .996,  .969,  .973}\underline{0.9538}  & \cellcolor[rgb]{ .984,  .863,  .878}\underline{0.9769}  & \cellcolor[rgb]{ .984,  .875,  .886}\underline{0.9745}  & \cellcolor[rgb]{ .98,  .855,  .871}\textbf{0.9792} & \cellcolor[rgb]{ .992,  .933,  .941}\underline{0.9628}  \\
      \midrule
      \multirow{6}[4]{*}{NLPR-TE} & CRNet \cite{CRNet} & \cellcolor[rgb]{ .847,  .945,  .929}\textbf{0.0226} & \cellcolor[rgb]{ .847,  .945,  .929}\textbf{0.0204} & 0.9593  & 0.9580  & \cellcolor[rgb]{ .855,  .949,  .933}\underline{0.8844}  & \cellcolor[rgb]{ .922,  .973,  .965}\underline{0.9154}  & \underline{0.9078}  & 0.9381  & \cellcolor[rgb]{ .925,  .973,  .965}\underline{0.9430}  & \cellcolor[rgb]{ .847,  .945,  .929}\textbf{0.9489} \\
            & + $\mathsf{SIOpt1}$&  0.0246  & 0.0215  & \cellcolor[rgb]{ .973,  .992,  .988}\underline{0.9622}  & \cellcolor[rgb]{ .973,  .992,  .988}\underline{0.9610}  & \cellcolor[rgb]{ .847,  .945,  .929}\textbf{0.8850} & \cellcolor[rgb]{ .847,  .945,  .929}\textbf{0.9239} & 0.9076  & \cellcolor[rgb]{ .847,  .945,  .929}\textbf{0.9451} & \cellcolor[rgb]{ .847,  .945,  .929}\textbf{0.9457} & \cellcolor[rgb]{ .922,  .973,  .965}\underline{0.9475}  \\
            & + $\mathsf{SIOpt2}$&  \cellcolor[rgb]{ .945,  .98,  .973}\underline{0.0239}  & \cellcolor[rgb]{ .941,  .976,  .973}\underline{0.0211}  & \cellcolor[rgb]{ .847,  .945,  .929}\textbf{0.9744} & \cellcolor[rgb]{ .847,  .945,  .929}\textbf{0.9737} & 0.8684  & 0.9056  & \cellcolor[rgb]{ .847,  .945,  .929}\textbf{0.9156} & \cellcolor[rgb]{ .953,  .984,  .98}\underline{0.9403}  & 0.9402  & 0.9459  \\
  \cmidrule{2-12}          & CPNet \cite{hu2024cross} & \cellcolor[rgb]{ .847,  .945,  .929}\textbf{0.0169} & \cellcolor[rgb]{ .847,  .945,  .929}\textbf{0.0159} & 0.9797  & 0.9778  & \cellcolor[rgb]{ .847,  .945,  .929}\textbf{0.9144} & \cellcolor[rgb]{ .918,  .973,  .961}\underline{0.9318}  & 0.9399  & 0.9505  & \cellcolor[rgb]{ .847,  .945,  .929}\textbf{0.9617} & \cellcolor[rgb]{ .847,  .945,  .929}\textbf{0.9576} \\
            & + $\mathsf{SIOpt1}$&  0.0174  & \cellcolor[rgb]{ .898,  .961,  .953}\underline{0.0160}  & \cellcolor[rgb]{ .847,  .945,  .929}\textbf{0.9839} & \cellcolor[rgb]{ .847,  .945,  .929}\textbf{0.9830} & \cellcolor[rgb]{ .902,  .965,  .957}\underline{0.9119}  & \cellcolor[rgb]{ .847,  .945,  .929}\textbf{0.9352} & \cellcolor[rgb]{ .847,  .945,  .929}\textbf{0.9412} & \cellcolor[rgb]{ .847,  .945,  .929}\textbf{0.9559} & \cellcolor[rgb]{ .882,  .961,  .945}\underline{0.9615}  & \cellcolor[rgb]{ .847,  .945,  .929}\textbf{0.9576} \\
            & + $\mathsf{SIOpt2}$&  \cellcolor[rgb]{ .969,  .988,  .984}\underline{0.0173}  & 0.0162  & \cellcolor[rgb]{ .867,  .953,  .941}\underline{0.9834}  & \cellcolor[rgb]{ .875,  .957,  .945}\underline{0.9821}  & 0.9071  & 0.9275  & \cellcolor[rgb]{ .898,  .965,  .953}\underline{0.9408}  & \cellcolor[rgb]{ .894,  .965,  .953}\underline{0.9543}  & 0.9608  & 0.9562  \\
      \bottomrule
      \end{tabular}%
    \label{tab:RGBD}%
  \end{table}%

\begin{table}[!t]
  \centering
    \caption{Quantitative comparisons on \textcolor{blue}{RGB-T} benchmark datasets. The best results are highlighted in bold, and the second-best are marked with underline. Here, $\uparrow$ indicates that higher values denote better performance, while $\downarrow$ indicates that lower values are preferable. Darker colors represent stronger performance.}
    \begin{tabular}{c|c|cccccccccc}
    \toprule
    \multicolumn{1}{c}{Dataset} & \multicolumn{1}{|c|}{Methods} & $\MAE \downarrow$ & $\SMAE \downarrow$ & $\AUC \uparrow $ & $\SAUC \uparrow $ & $\F_m^{\beta} \uparrow$ & $\SF_m^{\beta} \uparrow $ & $\F_{max}^{\beta} \uparrow$ & $\SF_{max}^{\beta} \uparrow$ & $\E_m \uparrow$ & $\Sm_m \uparrow$ \\
    \midrule
    \multirow{6}[4]{*}{VT5000} & TNet \cite{TNet} & \cellcolor[rgb]{ .988,  .933,  .937}\underline{0.0452}  & \cellcolor[rgb]{ .98,  .855,  .871}\underline{0.0394}  & \cellcolor[rgb]{ .984,  .859,  .875}\underline{0.9279}  & \cellcolor[rgb]{ .984,  .859,  .875}0.9355  & \cellcolor[rgb]{ .984,  .859,  .875}\underline{0.7942}  & \cellcolor[rgb]{ .984,  .859,  .875}\underline{0.8632}  & \cellcolor[rgb]{ .984,  .859,  .875}0.8389  & \cellcolor[rgb]{ .984,  .859,  .875}\underline{0.9040}  & \cellcolor[rgb]{ .984,  .859,  .875}0.8945  & \cellcolor[rgb]{ .984,  .863,  .878}0.9340  \\
          & + $\mathsf{SIOpt1}$&  \cellcolor[rgb]{ .98,  .855,  .871}\textbf{0.0434} & \cellcolor[rgb]{ .98,  .855,  .871}\textbf{0.0377} & \cellcolor[rgb]{ .984,  .863,  .875}0.9272  & \cellcolor[rgb]{ .984,  .859,  .875}\underline{0.9360}  & \cellcolor[rgb]{ .98,  .855,  .871}\textbf{0.8074} & \cellcolor[rgb]{ .98,  .855,  .871}\textbf{0.8651} & \cellcolor[rgb]{ .98,  .855,  .871}\textbf{0.8491} & \cellcolor[rgb]{ .984,  .859,  .875}0.9029  & \cellcolor[rgb]{ .98,  .855,  .871}\textbf{0.9001} & \cellcolor[rgb]{ .98,  .855,  .871}\textbf{0.9374} \\
          & + $\mathsf{SIOpt2}$&  0.0467  & \cellcolor[rgb]{ .98,  .859,  .871}0.0408  & \cellcolor[rgb]{ .98,  .855,  .871}\textbf{0.9331} & \cellcolor[rgb]{ .98,  .855,  .871}\textbf{0.9410} & \cellcolor[rgb]{ .984,  .863,  .878}0.7858  & \cellcolor[rgb]{ .984,  .859,  .875}0.8571  & \cellcolor[rgb]{ .984,  .859,  .875}\underline{0.8408}  & \cellcolor[rgb]{ .98,  .855,  .871}\textbf{0.9055} & \cellcolor[rgb]{ .984,  .859,  .875}\underline{0.8996}  & \cellcolor[rgb]{ .984,  .863,  .878}\underline{0.9346}  \\
\cmidrule{2-12}          & DCNet \cite{DCNet} & \cellcolor[rgb]{ .996,  .976,  .976}\underline{0.1494}  & \cellcolor[rgb]{ .996,  .973,  .976}\underline{0.1291}  & 0.7280  & 0.7178  & 0.3024  & 0.4435  & 0.3967  & 0.6463  & \cellcolor[rgb]{ 1,  .984,  .984}\textbf{0.6267} & 0.8479  \\
          & + $\mathsf{SIOpt1}$&  \cellcolor[rgb]{ .988,  .929,  .937}\textbf{0.1029} & \cellcolor[rgb]{ .988,  .922,  .929}\textbf{0.0890} & \cellcolor[rgb]{ .988,  .91,  .918}\underline{0.8603}  & \cellcolor[rgb]{ .992,  .918,  .925}\underline{0.8494}  & \cellcolor[rgb]{ .996,  .957,  .961}\textbf{0.4567} & \cellcolor[rgb]{ 1,  .976,  .98}\textbf{0.5183} & \cellcolor[rgb]{ .992,  .929,  .937}\underline{0.6260}  & \cellcolor[rgb]{ .992,  .929,  .937}\underline{0.7752}  & \cellcolor[rgb]{ 1,  .984,  .988}\underline{0.6263}  & \cellcolor[rgb]{ .988,  .91,  .922}\textbf{0.9039} \\
          & + $\mathsf{SIOpt2}$&  0.1715  & 0.1483  & \cellcolor[rgb]{ .988,  .89,  .902}\textbf{0.8840} & \cellcolor[rgb]{ .988,  .898,  .91}\textbf{0.8753} & \cellcolor[rgb]{ 1,  .988,  .988}\underline{0.3510}  & \cellcolor[rgb]{ .999,  .999,  .999} \underline{0.4477}  & \cellcolor[rgb]{ .992,  .918,  .929}\textbf{0.6556} & \cellcolor[rgb]{ .988,  .914,  .925}\textbf{0.8024} & \cellcolor[rgb]{ .999,  .999,  .999}0.5889  & \cellcolor[rgb]{ .992,  .929,  .937}\underline{0.8917}  \\
    \midrule
    \multirow{6}[4]{*}{VT1000} & TNet \cite{TNet} & \cellcolor[rgb]{ .824,  .957,  .949}\textbf{0.0290} & \cellcolor[rgb]{ .824,  .957,  .949}\underline{0.0241} & \cellcolor[rgb]{ .827,  .961,  .953}\underline{0.9634}  & \cellcolor[rgb]{ .827,  .961,  .953}\underline{0.9694}  & \cellcolor[rgb]{ .855,  .965,  .961}\underline{0.8690}  & \cellcolor[rgb]{ .827,  .961,  .953}\underline{0.9156}  & \cellcolor[rgb]{ .827,  .961,  .953}0.9126  & \cellcolor[rgb]{ .894,  .976,  .973}0.8513  & \cellcolor[rgb]{ .827,  .961,  .953}\underline{0.9348}  & \cellcolor[rgb]{ .827,  .961,  .953}0.9485  \\
          & + $\mathsf{SIOpt1}$&  \cellcolor[rgb]{ .824,  .957,  .949}\textbf{0.0290} & \cellcolor[rgb]{ .824,  .957,  .949}\textbf{0.0239} & \cellcolor[rgb]{ .831,  .961,  .953}0.9608  & \cellcolor[rgb]{ .827,  .961,  .953}0.9683  & \cellcolor[rgb]{ .824,  .957,  .949}\textbf{0.9718} & \cellcolor[rgb]{ .824,  .957,  .949}\textbf{0.9168} & \cellcolor[rgb]{ .827,  .961,  .953}\underline{0.9146}  & \cellcolor[rgb]{ .827,  .961,  .953}\underline{0.9501}  & \cellcolor[rgb]{ .827,  .961,  .953}0.9324  & \cellcolor[rgb]{ .824,  .957,  .949}\textbf{0.9502} \\
          & + $\mathsf{SIOpt2}$&  \cellcolor[rgb]{ .827,  .957,  .949}0.0318  & \cellcolor[rgb]{ .827,  .957,  .949}0.0264  & \cellcolor[rgb]{ .824,  .957,  .949}\textbf{0.9666} & \cellcolor[rgb]{ .824,  .957,  .949}\textbf{0.9732} & \cellcolor[rgb]{ .859,  .969,  .961}0.8613  & \cellcolor[rgb]{ .827,  .961,  .953}0.9100  & \cellcolor[rgb]{ .824,  .957,  .949}\textbf{0.9153} & \cellcolor[rgb]{ .824,  .957,  .949}\textbf{0.9524} & \cellcolor[rgb]{ .824,  .957,  .949}\textbf{0.9349} & \cellcolor[rgb]{ .827,  .961,  .953}\underline{0.9494}  \\
\cmidrule{2-12}          & DCNet \cite{DCNet} & \cellcolor[rgb]{ .996,  .996,  .996}\underline{0.1516}  & \cellcolor[rgb]{ .996,  .996,  .996}\underline{0.1234}  & 0.7430  & 0.7403  & 0.3812  & 0.5147  & 0.4810  & 0.6984  & \underline{0.6492}  & 0.8473  \\
          &  + SIOpt1 & \cellcolor[rgb]{ .929,  .98,  .98}\textbf{0.1052} & \cellcolor[rgb]{ .929,  .98,  .976}\textbf{0.0856} & \cellcolor[rgb]{ .867,  .969,  .961}\underline{0.9159}  & \cellcolor[rgb]{ .875,  .969,  .965}\underline{0.9098}  & \cellcolor[rgb]{ .941,  .988,  .984}\textbf{0.5817} & \cellcolor[rgb]{ .945,  .988,  .984}\textbf{0.6439} & \cellcolor[rgb]{ .898,  .976,  .973}\underline{0.7362}  & \cellcolor[rgb]{ .886,  .973,  .969}\underline{0.8658}  & \cellcolor[rgb]{ .973,  .996,  .992}\textbf{0.6931} & \cellcolor[rgb]{ .906,  .98,  .973}\textbf{0.9034} \\
          & + $\mathsf{SIOpt2}$&  0.1521  & 0.1250  & \cellcolor[rgb]{ .851,  .965,  .957}\textbf{0.9329} & \cellcolor[rgb]{ .859,  .969,  .961}\textbf{0.9300} & \cellcolor[rgb]{ .976,  .996,  .996}\underline{0.4687}  & \cellcolor[rgb]{ .992,  1,  1}\underline{0.5355}  & \cellcolor[rgb]{ .878,  .973,  .969}\textbf{0.7803} & \cellcolor[rgb]{ .875,  .973,  .965}\textbf{0.8807} & 0.6472  & \cellcolor[rgb]{ .922,  .98,  .98}\underline{0.8944}  \\
    \midrule
    \multirow{6}[4]{*}{VT821} & TNet \cite{TNet} & \cellcolor[rgb]{ .988,  .922,  .875}0.0915  & \cellcolor[rgb]{ .988,  .925,  .878}0.0803  & \cellcolor[rgb]{ .992,  .902,  .839}0.9283  & \cellcolor[rgb]{ .992,  .902,  .839}0.9254  & \cellcolor[rgb]{ .992,  .906,  .847}\underline{0.7454}  & \cellcolor[rgb]{ .992,  .898,  .831}\underline{0.8481}  & \cellcolor[rgb]{ .992,  .91,  .851}0.7945  & \cellcolor[rgb]{ .992,  .902,  .835}0.8937  & \cellcolor[rgb]{ .992,  .914,  .859}0.8316  & \cellcolor[rgb]{ .992,  .929,  .882}0.8932  \\
          & + $\mathsf{SIOpt1}$&  \cellcolor[rgb]{ .988,  .894,  .827}\textbf{0.0525} & \cellcolor[rgb]{ .988,  .894,  .827}\textbf{0.0461} & \cellcolor[rgb]{ .988,  .894,  .827}\textbf{0.9384} & \cellcolor[rgb]{ .988,  .894,  .827}\textbf{0.9352} & \cellcolor[rgb]{ .988,  .894,  .827}\textbf{0.7973} & \cellcolor[rgb]{ .988,  .894,  .827}\textbf{0.8496} & \cellcolor[rgb]{ .988,  .894,  .827}\textbf{0.8492} & \cellcolor[rgb]{ .992,  .898,  .835}\underline{0.8950}  & \cellcolor[rgb]{ .988,  .894,  .827}\textbf{0.8765} & \cellcolor[rgb]{ .988,  .894,  .827}\textbf{0.9278} \\
          & + $\mathsf{SIOpt2}$&  \cellcolor[rgb]{ .988,  .918,  .867}\underline{0.0855}  & \cellcolor[rgb]{ .988,  .918,  .871}\underline{0.0748}  & \cellcolor[rgb]{ .992,  .898,  .831}\underline{0.9379}  & \cellcolor[rgb]{ .992,  .898,  .831}\underline{0.9350}  & \cellcolor[rgb]{ .992,  .91,  .851}0.7406  & \cellcolor[rgb]{ .992,  .898,  .835}0.8396  & \cellcolor[rgb]{ .992,  .906,  .847}\underline{0.8062}  & \cellcolor[rgb]{ .988,  .894,  .827}\textbf{0.9027} & \cellcolor[rgb]{ .992,  .914,  .859}\underline{0.8337}  & \cellcolor[rgb]{ .992,  .918,  .867}\underline{0.9043}  \\
\cmidrule{2-12}          & DCNet \cite{DCNet} & 0.1881  & 0.1590  & 0.7510  & 0.7491  & 0.3278  & \cellcolor[rgb]{ 1,  .984,  .973}\underline{0.5343}  & 0.4191  & 0.6839  & \cellcolor[rgb]{ 1,  1,  .996}\underline{0.6253}  & 0.8163  \\
          & + $\mathsf{SIOpt1}$&  \cellcolor[rgb]{ .992,  .937,  .902}\textbf{0.1124} & \cellcolor[rgb]{ .992,  .937,  .902}\textbf{0.0950} & \cellcolor[rgb]{ .996,  .937,  .894}\underline{0.8673}  & \cellcolor[rgb]{ .996,  .937,  .894}\underline{0.8652}  & \cellcolor[rgb]{ 1,  .969,  .949}\textbf{0.4701} & \cellcolor[rgb]{ 1,  .98,  .969}\textbf{0.5463} & \cellcolor[rgb]{ .996,  .953,  .918}\underline{0.6252}  & \cellcolor[rgb]{ .996,  .949,  .918}\underline{0.7920}  & \cellcolor[rgb]{ 1,  .988,  .976}\textbf{0.6531} & \cellcolor[rgb]{ .992,  .929,  .882}\textbf{0.8929} \\
          & + $\mathsf{SIOpt2}$&  \cellcolor[rgb]{ .996,  .98,  .969}\underline{0.1660}  & \cellcolor[rgb]{ .996,  .98,  .969}\underline{0.1409}  & \cellcolor[rgb]{ .992,  .925,  .878}\textbf{0.8844} & \cellcolor[rgb]{ .992,  .925,  .875}\textbf{0.8867} & \cellcolor[rgb]{ 1,  .988,  .976}\underline{0.3945}  & 0.4713  & \cellcolor[rgb]{ .996,  .941,  .898}\textbf{0.6735} & \cellcolor[rgb]{ .996,  .937,  .898}\textbf{0.8171} & 0.6174  & \cellcolor[rgb]{ .992,  .929,  .886}\underline{0.8923}  \\
    \bottomrule
    \end{tabular}%
  \label{tab:RGBT}%
\end{table}

  \begin{table}[!t]
    \centering
    \caption{Performance comparisons for TS-SAM-based methods on DUTS, DUT-OMRON, MSOD, ECSSD, and HKU-IS datasets.}
    \scalebox{1.0}{
      \begin{tabular}{c|ccccccccccc}
      \toprule
      \multicolumn{1}{c}{Dataset} & \multicolumn{1}{c}{Methods} & $\MAE \downarrow$ & $\SMAE \downarrow$ & $\AUC \uparrow $ & $\SAUC \uparrow $ & $\F_m^{\beta} \uparrow$ & $\SF_m^{\beta} \uparrow $ & $\F_{max}^{\beta} \uparrow$ & $\SF_{max}^{\beta} \uparrow$ & $\E_m \uparrow$ & $\Sm_m \uparrow$ \\
      \midrule
      \multirow{2}[2]{*}{DUTS-TE} & TS-SAM \cite{SAM3} & \cellcolor[rgb]{ .988,  .914,  .922}0.0335  & \cellcolor[rgb]{ .992,  .941,  .949}0.0368  & \cellcolor[rgb]{ .984,  .886,  .898}\textbf{0.9832} & \cellcolor[rgb]{ .992,  .922,  .929}\textbf{0.9692} & \cellcolor[rgb]{ .988,  .914,  .922}0.8779  & \cellcolor[rgb]{ .98,  .855,  .871}\textbf{0.9371} & \cellcolor[rgb]{ .992,  .929,  .937}0.9237  & \cellcolor[rgb]{ .984,  .867,  .878}\textbf{0.9655} & \cellcolor[rgb]{ .992,  .922,  .929}0.9266  & \cellcolor[rgb]{ .984,  .875,  .886}0.9472  \\
            & +SIOpt1 & \cellcolor[rgb]{ .984,  .906,  .914}\textbf{0.0322} & \cellcolor[rgb]{ .988,  .933,  .941}\textbf{0.0354} & \cellcolor[rgb]{ .988,  .906,  .918}0.9764  & \cellcolor[rgb]{ .992,  .945,  .949}0.9621  & \cellcolor[rgb]{ .988,  .906,  .914}\textbf{0.8861} & \cellcolor[rgb]{ .984,  .859,  .875}0.9359  & \cellcolor[rgb]{ .992,  .918,  .925}\textbf{0.9324} & \cellcolor[rgb]{ .984,  .867,  .882}0.9641  & \cellcolor[rgb]{ .992,  .918,  .925}\textbf{0.9279} & \cellcolor[rgb]{ .98,  .855,  .871}\textbf{0.9490} \\
      \midrule
      \multirow{2}[2]{*}{DUT-OMRON} & TS-SAM \cite{SAM3} & 0.0461  & 0.0452  & \cellcolor[rgb]{ .992,  .945,  .949}\textbf{0.9650} & \cellcolor[rgb]{ .996,  .965,  .969}\textbf{0.9559} & 0.7965  & 0.7967  & \cellcolor[rgb]{ 1,  .996,  .996}\textbf{0.8828} & \cellcolor[rgb]{ 1,  .992,  .996}\textbf{0.8835} & 0.8874  & 0.9335  \\
            & +SIOpt1 & \cellcolor[rgb]{ .996,  .98,  .98}\textbf{0.0434} & \cellcolor[rgb]{ .996,  .98,  .98}\textbf{0.0425} & \cellcolor[rgb]{ 1,  .976,  .98}0.9548  & \cellcolor[rgb]{ 1,  .98,  .98}0.9512  & \cellcolor[rgb]{ 1,  .992,  .992}\textbf{0.8051} & \cellcolor[rgb]{ 1,  .992,  .996}\textbf{0.8052} & 0.8780  & 0.8784  & \cellcolor[rgb]{ 1,  .996,  .996}\textbf{0.8909} & \cellcolor[rgb]{ 1,  .984,  .988}\textbf{0.9352} \\
      \midrule
      \multirow{2}[2]{*}{MSOD} & TS-SAM \cite{SAM3} & \cellcolor[rgb]{ .996,  .973,  .973}0.0420  & \cellcolor[rgb]{ .992,  .965,  .969}0.0405  & \cellcolor[rgb]{ .996,  .953,  .961}\textbf{0.9618} & \cellcolor[rgb]{ .996,  .953,  .957}\textbf{0.9594} & \cellcolor[rgb]{ .996,  .969,  .973}0.8275  & \cellcolor[rgb]{ 1,  .973,  .976}0.8242  & \cellcolor[rgb]{ 1,  .992,  .992}\textbf{0.8854} & \cellcolor[rgb]{ 1,  .992,  .992}\textbf{0.8842} & \cellcolor[rgb]{ 1,  .992,  .992}0.8919  & \cellcolor[rgb]{ .996,  .957,  .961}0.9383  \\
            & +SIOpt1 & \cellcolor[rgb]{ .992,  .961,  .965}\textbf{0.0405} & \cellcolor[rgb]{ .992,  .957,  .961}\textbf{0.0390} & 0.9468  & 0.9441  & \cellcolor[rgb]{ .996,  .961,  .965}\textbf{0.8358} & \cellcolor[rgb]{ .996,  .965,  .969}\textbf{0.8328} & \cellcolor[rgb]{ 1,  .992,  .992}0.8840  & \cellcolor[rgb]{ 1,  .996,  .996}0.8818  & \cellcolor[rgb]{ 1,  .984,  .984}\textbf{0.8966} & \cellcolor[rgb]{ .996,  .949,  .957}\textbf{0.9390} \\
      \midrule
      \multirow{2}[2]{*}{ECSSD} & TS-SAM \cite{SAM3} & \cellcolor[rgb]{ .984,  .886,  .898}0.0292  & \cellcolor[rgb]{ .98,  .878,  .894}0.0276  & \cellcolor[rgb]{ .98,  .855,  .871}\textbf{0.9921} & \cellcolor[rgb]{ .98,  .855,  .871}\textbf{0.9887} & \cellcolor[rgb]{ .984,  .863,  .878}0.9257  & \cellcolor[rgb]{ .984,  .867,  .882}0.9258  & \cellcolor[rgb]{ .98,  .855,  .871}\textbf{0.9705} & \cellcolor[rgb]{ .98,  .855,  .871}\textbf{0.9706} & \cellcolor[rgb]{ .984,  .867,  .882}\textbf{0.9527} & \cellcolor[rgb]{ .984,  .867,  .882}0.9478  \\
            & +SIOpt1 & \cellcolor[rgb]{ .98,  .878,  .89}\textbf{0.0283} & \cellcolor[rgb]{ .98,  .875,  .886}\textbf{0.0268} & \cellcolor[rgb]{ .984,  .871,  .886}0.9874  & \cellcolor[rgb]{ .984,  .882,  .894}0.9813  & \cellcolor[rgb]{ .98,  .855,  .871}\textbf{0.9304} & \cellcolor[rgb]{ .984,  .863,  .878}\textbf{0.9305} & \cellcolor[rgb]{ .984,  .863,  .878}0.9676  & \cellcolor[rgb]{ .984,  .863,  .878}0.9678  & \cellcolor[rgb]{ .984,  .867,  .882}0.9524  & \cellcolor[rgb]{ .984,  .859,  .875}\textbf{0.9488} \\
      \midrule
      \multirow{2}[2]{*}{HKU-IS} & TS-SAM \cite{SAM3} & \cellcolor[rgb]{ .98,  .863,  .875}0.0257  & \cellcolor[rgb]{ .98,  .863,  .875}0.0249  & \cellcolor[rgb]{ .984,  .859,  .875}\textbf{0.9909} & \cellcolor[rgb]{ .984,  .859,  .875}\textbf{0.9878} & \cellcolor[rgb]{ .984,  .875,  .89}0.9129  & \cellcolor[rgb]{ .984,  .882,  .898}0.9112  & \cellcolor[rgb]{ .984,  .875,  .89}\textbf{0.9586} & \cellcolor[rgb]{ .984,  .878,  .89}\textbf{0.9576} & \cellcolor[rgb]{ .984,  .863,  .875}0.9552  & \cellcolor[rgb]{ .984,  .875,  .886}0.9472  \\
            & +SIOpt1 & \cellcolor[rgb]{ .98,  .855,  .871}\textbf{0.0244} & \cellcolor[rgb]{ .98,  .855,  .871}\textbf{0.0236} & \cellcolor[rgb]{ .984,  .875,  .886}0.9867  & \cellcolor[rgb]{ .984,  .878,  .89}0.9825  & \cellcolor[rgb]{ .984,  .867,  .882}\textbf{0.9196} & \cellcolor[rgb]{ .984,  .875,  .89}\textbf{0.9183} & \cellcolor[rgb]{ .984,  .878,  .89}0.9573  & \cellcolor[rgb]{ .984,  .878,  .894}0.9563  & \cellcolor[rgb]{ .98,  .855,  .871}\textbf{0.9573} & \cellcolor[rgb]{ .984,  .859,  .875}\textbf{0.9487} \\
      \bottomrule
      \end{tabular}%
     }
    \label{tab:SAM}%
  \end{table}%

\begin{table}[!t]
  \centering
  \caption{SOD performance comparisons with different approximation functions of $\mathcal{C}(\boldsymbol{X})$ on seven widely used benchmarks. We adopt EDN as the base model, and Boundings is the default version involved in $\mathsf{SIOpt1}$.}
  \scalebox{0.98}{
    \begin{tabular}{c|c|cccccccccc}
    \toprule
    \multicolumn{1}{c}{Dataset} & Split Methods & $\MAE \downarrow$ & $\SMAE \downarrow$ & $\AUC \uparrow $ & $\SAUC \uparrow $ & $\F_m^{\beta} \uparrow$ & $\SF_m^{\beta} \uparrow $ & $\F_{max}^{\beta} \uparrow$ & $\SF_{max}^{\beta} \uparrow$ & $\E_m \uparrow$ & $\Sm_m \uparrow$ \\
    \midrule
    \multirow{3}[2]{*}{DUTS} &  Random & 0.0527  & 0.0497  & 0.9539  & 0.9471  & 0.7837  & 0.8107  & 0.8601  & 0.8899  & 0.8668  & 0.9276  \\
          & Nearest & 0.0394  & 0.0390  & 0.9334  & 0.9320  & \textbf{0.8367} & 0.8621  & 0.8644  & 0.8885  & 0.9054  & 0.9353  \\
          & Boundings & \cellcolor[rgb]{ .867,  .922,  .969}\textbf{0.0392} & \cellcolor[rgb]{ .867,  .922,  .969}\textbf{0.0381} & \cellcolor[rgb]{ .867,  .922,  .969}\textbf{0.9658} & \cellcolor[rgb]{ .867,  .922,  .969}\textbf{0.9596} & \cellcolor[rgb]{ .867,  .922,  .969}0.8260  & \cellcolor[rgb]{ .867,  .922,  .969}\textbf{0.8672} & \cellcolor[rgb]{ .867,  .922,  .969}\textbf{0.8765} & \cellcolor[rgb]{ .867,  .922,  .969}\textbf{0.9119} & \cellcolor[rgb]{ .867,  .922,  .969}\textbf{0.9072} & \cellcolor[rgb]{ .867,  .922,  .969}\textbf{0.9388} \\
    \midrule
    \multirow{3}[2]{*}{DUT-OMRON} &  Random & 0.0602  & 0.0517  & 0.9300  & 0.9282  & 0.7329  & 0.7870  & \textbf{0.8179} & 0.8769  & 0.8383  & 0.9207  \\
          & Nearest & \textbf{0.0536} & \textbf{0.0466} & 0.9080  & 0.9046  & \textbf{0.7648} & 0.8329  & 0.8000  & 0.8692  & \textbf{0.8603} & 0.9212  \\
          & Boundings & \cellcolor[rgb]{ .867,  .922,  .969}0.0557  & \cellcolor[rgb]{ .867,  .922,  .969}0.0483  & \cellcolor[rgb]{ .867,  .922,  .969}\textbf{0.9382} & \cellcolor[rgb]{ .867,  .922,  .969}\textbf{0.9359} & \cellcolor[rgb]{ .867,  .922,  .969}0.7544  & \cellcolor[rgb]{ .867,  .922,  .969}\textbf{0.8381} & \cellcolor[rgb]{ .867,  .922,  .969}0.8163  & \cellcolor[rgb]{ .867,  .922,  .969}\textbf{0.8912} & \cellcolor[rgb]{ .867,  .922,  .969}0.8594  & \cellcolor[rgb]{ .867,  .922,  .969}\textbf{0.9282} \\
    \midrule
    \multirow{3}[2]{*}{MSOD} &  Random & 0.0583  & 0.0956  & 0.9160  & 0.9114  & 0.7034  & 0.6780  & 0.7791  & 0.7941  & 0.8219  & 0.9192  \\
          & Nearest & 0.0479  & 0.0846  & 0.8833  & 0.8715  & 0.7500  & 0.7189  & 0.7816  & 0.7556  & 0.8564  & 0.9250  \\
          & Boundings & \cellcolor[rgb]{ .867,  .922,  .969}\textbf{0.0453} & \cellcolor[rgb]{ .867,  .922,  .969}\textbf{0.0724} & \cellcolor[rgb]{ .867,  .922,  .969}\textbf{0.9401} & \cellcolor[rgb]{ .867,  .922,  .969}\textbf{0.9310} & \cellcolor[rgb]{ .867,  .922,  .969}\textbf{0.8057} & \cellcolor[rgb]{ .867,  .922,  .969}\textbf{0.7990} & \cellcolor[rgb]{ .867,  .922,  .969}\textbf{0.8555} & \cellcolor[rgb]{ .867,  .922,  .969}\textbf{0.8619} & \cellcolor[rgb]{ .867,  .922,  .969}\textbf{0.8936} & \cellcolor[rgb]{ .867,  .922,  .969}\textbf{0.9365} \\
    \midrule
    \multirow{3}[2]{*}{ECSSD} &  Random & 0.0476  & 0.0348  & 0.9704  & 0.9684  & 0.8838  & 0.8966  & 0.9341  & 0.9425  & 0.9142  & 0.9320  \\
          & Nearest & \textbf{0.0355} & \textbf{0.0269} & 0.9586  & 0.9593  & \textbf{0.9150} & \textbf{0.9244} & 0.9390  & 0.9480  & 0.9364  & 0.9394  \\
          & Boundings & \cellcolor[rgb]{ .867,  .922,  .969}0.0358  & \cellcolor[rgb]{ .867,  .922,  .969}\textbf{0.0269} & \cellcolor[rgb]{ .867,  .922,  .969}\textbf{0.9762} & \cellcolor[rgb]{ .867,  .922,  .969}\textbf{0.9743} & \cellcolor[rgb]{ .867,  .922,  .969}0.9084  & \cellcolor[rgb]{ .867,  .922,  .969}0.9216  & \cellcolor[rgb]{ .867,  .922,  .969}\textbf{0.9456} & \cellcolor[rgb]{ .867,  .922,  .969}\textbf{0.9543} & \cellcolor[rgb]{ .867,  .922,  .969}\textbf{0.9375} & \cellcolor[rgb]{ .867,  .922,  .969}\textbf{0.9406} \\
    \midrule
    \multirow{3}[2]{*}{HKU-IS} &  Random & 0.0377  & 0.0381  & 0.9730  & 0.9684  & 0.8756  & 0.8729  & 0.9280  & 0.9311  & 0.9223  & 0.9386  \\
          & Nearest & \textbf{0.0280} & 0.0292  & 0.9549  & 0.9506  & \textbf{0.9088} & \textbf{0.9099} & 0.9319  & 0.9340  & 0.9430  & 0.9449  \\
          & Boundings & \cellcolor[rgb]{ .867,  .922,  .969}0.0287  & \cellcolor[rgb]{ .867,  .922,  .969}\textbf{0.0289} & \cellcolor[rgb]{ .867,  .922,  .969}\textbf{0.9776} & \cellcolor[rgb]{ .867,  .922,  .969}\textbf{0.9753} & \cellcolor[rgb]{ .867,  .922,  .969}0.8986  & \cellcolor[rgb]{ .867,  .922,  .969}0.9072  & \cellcolor[rgb]{ .867,  .922,  .969}\textbf{0.9375} & \cellcolor[rgb]{ .867,  .922,  .969}\textbf{0.9443} & \cellcolor[rgb]{ .867,  .922,  .969}\textbf{0.9442} & \cellcolor[rgb]{ .867,  .922,  .969}\textbf{0.9452} \\
    \midrule
    \multirow{3}[2]{*}{SOD} &  Random & 0.1311  & 0.1165  & 0.8672  & 0.8585  & 0.7260  & 0.6918  & 0.8322  & 0.8164  & 0.7509  & 0.8656  \\
          & Nearest & 0.1016  & 0.0956  & 0.8486  & 0.8362  & 0.7998  & 0.7549  & 0.8471  & 0.8102  & 0.8051  & 0.8889  \\
          & Boundings & \cellcolor[rgb]{ .867,  .922,  .969}\textbf{0.0982} & \cellcolor[rgb]{ .867,  .922,  .969}\textbf{0.0922} & \cellcolor[rgb]{ .867,  .922,  .969}\textbf{0.8892} & \cellcolor[rgb]{ .867,  .922,  .969}\textbf{0.8789} & \cellcolor[rgb]{ .867,  .922,  .969}\textbf{0.8110} & \cellcolor[rgb]{ .867,  .922,  .969}\textbf{0.7728} & \cellcolor[rgb]{ .867,  .922,  .969}\textbf{0.8677} & \cellcolor[rgb]{ .867,  .922,  .969}\textbf{0.8382} & \cellcolor[rgb]{ .867,  .922,  .969}\textbf{0.8207} & \cellcolor[rgb]{ .867,  .922,  .969}\textbf{0.8921} \\
    \midrule
    \multirow{3}[2]{*}{PASCAL-S} &  Random & 0.0839  & 0.0629  & 0.9244  & 0.9266  & 0.7853  & 0.8122  & 0.8561  & 0.8863  & 0.8407  & 0.8968  \\
          & Nearest & 0.0646  & 0.0492  & 0.9232  & 0.9241  & \textbf{0.8330} & \textbf{0.8708} & 0.8625  & 0.8995  & 0.8841  & 0.9072  \\
          & Boundings & \cellcolor[rgb]{ .867,  .922,  .969}\textbf{0.0644} & \cellcolor[rgb]{ .867,  .922,  .969}\textbf{0.0491} & \cellcolor[rgb]{ .867,  .922,  .969}\textbf{0.9456} & \cellcolor[rgb]{ .867,  .922,  .969}\textbf{0.9551} & \cellcolor[rgb]{ .867,  .922,  .969}0.8260  & \cellcolor[rgb]{ .867,  .922,  .969}0.8684  & \cellcolor[rgb]{ .867,  .922,  .969}\textbf{0.8757} & \cellcolor[rgb]{ .867,  .922,  .969}\textbf{0.9114} & \cellcolor[rgb]{ .867,  .922,  .969}\textbf{0.8859} & \cellcolor[rgb]{ .867,  .922,  .969}\textbf{0.9090} \\
    \midrule
    \multirow{3}[2]{*}{XPIE} &  Random & 0.0519  & 0.0422  & 0.9553  & 0.9538  & 0.8281  & 0.8580  & 0.8844  & 0.9161  & 0.8840  & 0.9295  \\
          & Nearest & 0.0439  & 0.0355  & 0.9329  & 0.9305  & \textbf{0.8611} & \textbf{0.8921} & 0.8860  & 0.9169  & \textbf{0.9083} & 0.9341  \\
          & Boundings & \cellcolor[rgb]{ .867,  .922,  .969}\textbf{0.0409} & \cellcolor[rgb]{ .867,  .922,  .969}\textbf{0.0337} & \cellcolor[rgb]{ .867,  .922,  .969}\textbf{0.9598} & \cellcolor[rgb]{ .867,  .922,  .969}\textbf{0.9572} & \cellcolor[rgb]{ .867,  .922,  .969}0.8584  & \cellcolor[rgb]{ .867,  .922,  .969}0.8793  & \cellcolor[rgb]{ .867,  .922,  .969}\textbf{0.9044} & \cellcolor[rgb]{ .867,  .922,  .969}\textbf{0.9256} & \cellcolor[rgb]{ .867,  .922,  .969}0.9043  & \cellcolor[rgb]{ .867,  .922,  .969}\textbf{0.9386} \\
    \bottomrule
    \end{tabular}%
  }
  \label{tab:diff_CK}%
\end{table}%

\begin{figure}[!t]
    \centering
    \includegraphics[width=0.95\linewidth]{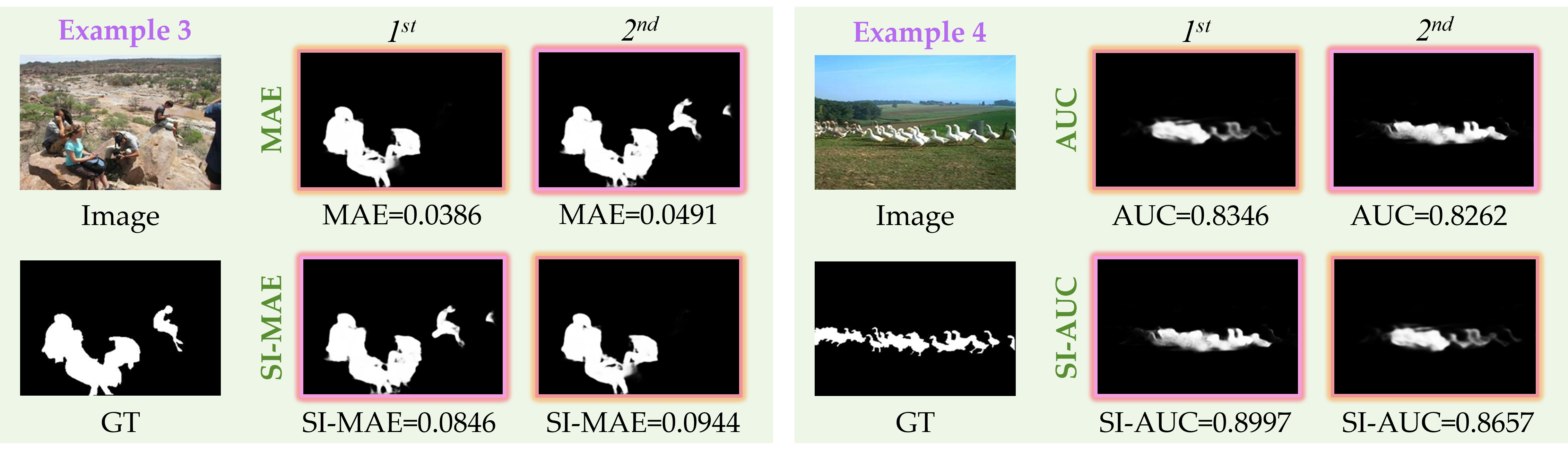}  
    \caption{Additional illustration of the strengths of our proposed size-invariant metrics over conventional size-sensitive evaluations ($\MAE$ v.s $\SMAE$ and $\AUC$ v.s $\SI\text{-}\AUC$). We can see that our proposed $\mathsf{SIEva}$ can offer a more proper assessment of model performance in multiple SOD scenarios. 
    }
    \label{app_fig:eg}
\end{figure}

\begin{figure*}[htbp]
    \centering
    \includegraphics[width=0.95\linewidth]{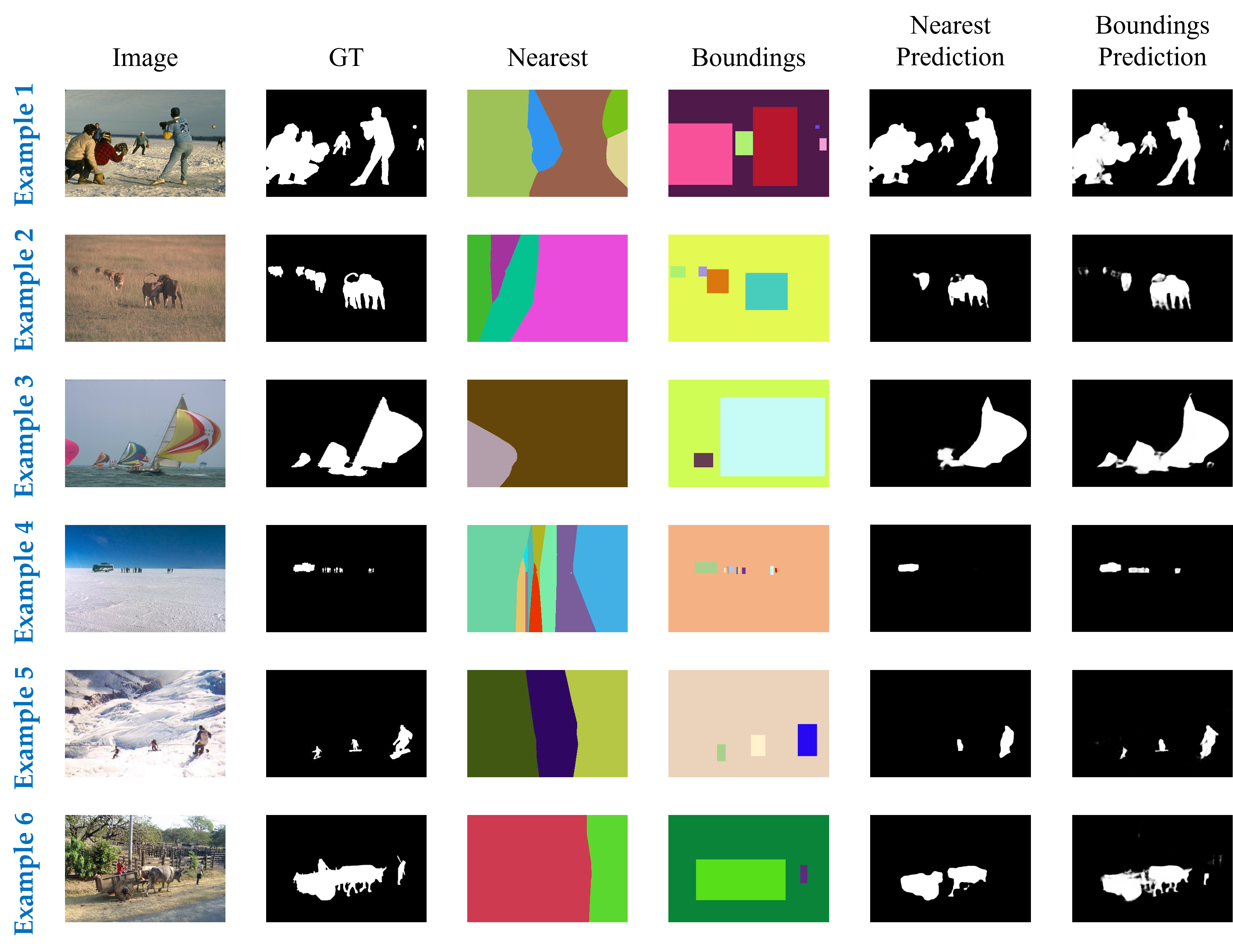}
    \caption{Visual examples of two distinct $\boldsymbol{C}(\boldsymbol{X})$ partition strategies are illustrated, where different colors represent separable regions. Specifically, the \textbf{Nearest} strategy yields $|\boldsymbol{C}(\boldsymbol{X})| = M$, while the \textbf{Boundings} strategy results in $|\boldsymbol{C}(\boldsymbol{X})| = M + 1$, where $M$ denotes the number of connected components in the image $\boldsymbol{X}$. It is also worth noting that the intersection rate introduced by the Boundings strategy (e.g., \textbf{Example 4}) remains minimal—less than $0.5\%$ across all eight benchmarks, as validated in Fig.~\ref{fig:intersect}. Therefore, this slight overlap is unlikely to affect the final evaluation results. Lastly, we do not provide visualizations for \textbf{Random}, as its performance is significantly inferior compared to the other two methods reported in \cref{tab:diff_CK}, rendering its inclusion meaningless.
    }
    \label{fig:vis_diffK}
\end{figure*}

\subsection{More Qualitative Visualizations}  \label{Qualitative_appendix}

To further enhance understanding of our size-invariant framework, we provide additional visual comparisons of different backbone models before and after applying our $\mathsf{SIOpt}$ algorithms. These results are illustrated in \cref{fig:vis}, \cref{fig:vis_appendix}, and \cref{fig:vis_SIOpt}. Consistent with the main text, we observe that the original backbones often fail to detect all salient objects in challenging scenarios, whereas our $\mathsf{SIOpt}$ framework effectively mitigates this issue. Similar improvements are also evident for $\mathsf{SIOpt2}$. For instance, in the ninth image of \cref{fig:vis_SIOpt}, both ADMNet and LDF overlook the two smaller birds at the bottom, while $\mathsf{SIOpt2}$ significantly enhances their detection without compromising the performance on other objects. These results further validate the effectiveness of our approach in promoting balanced learning across objects of varying sizes.

\begin{figure*}[htbp]
    \centering
    \includegraphics[width=\linewidth]{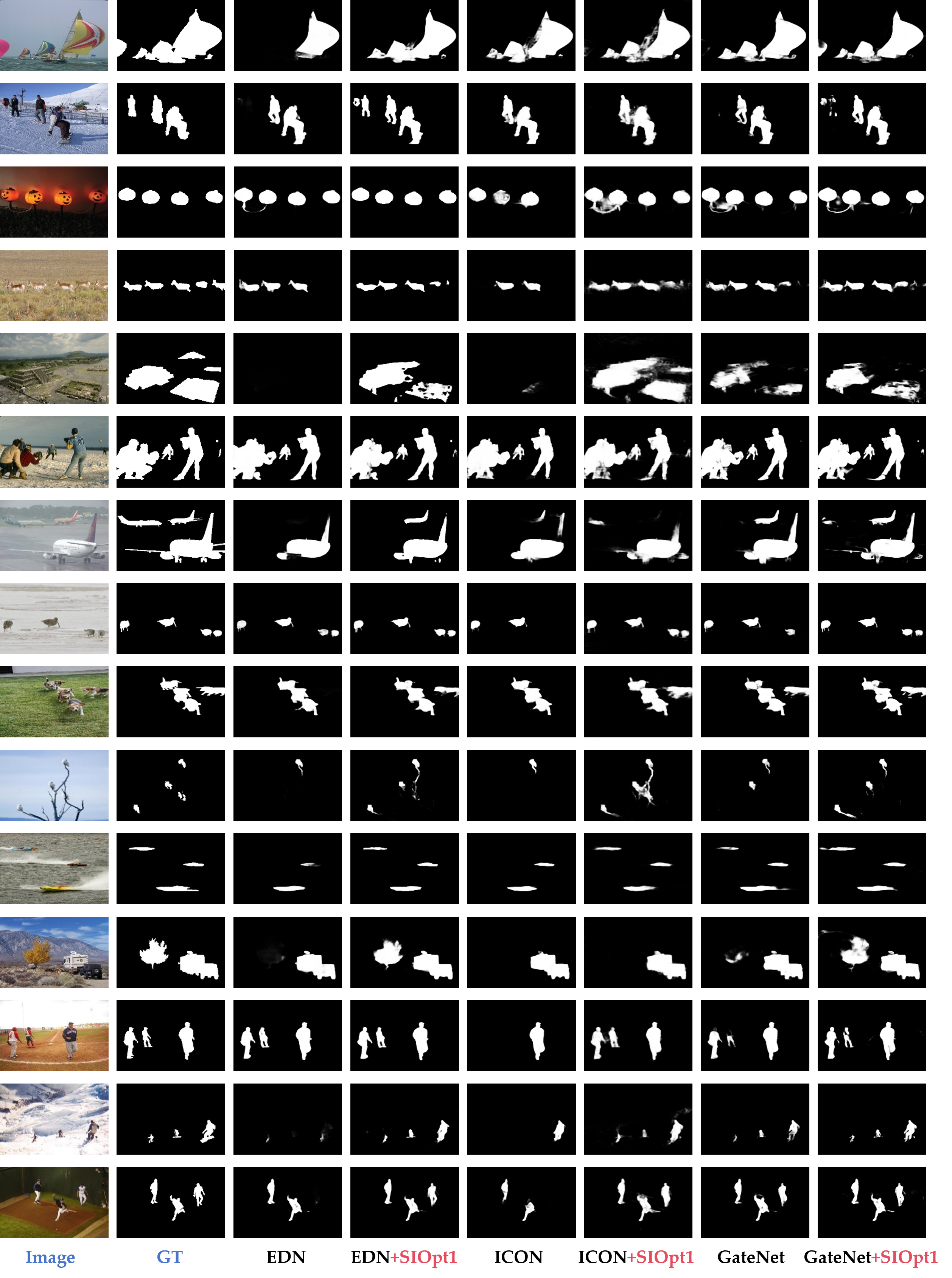}
    \vspace{-0.5cm}
    \caption{Qualitative visualizations on different backbones (including END, ICON and GateNet) before and after using our proposed $\mathsf{SIOpt1}$.}
    \label{fig:vis_appendix}
\end{figure*}

\begin{figure*}[htbp]
    \centering
    \includegraphics[width=\linewidth]{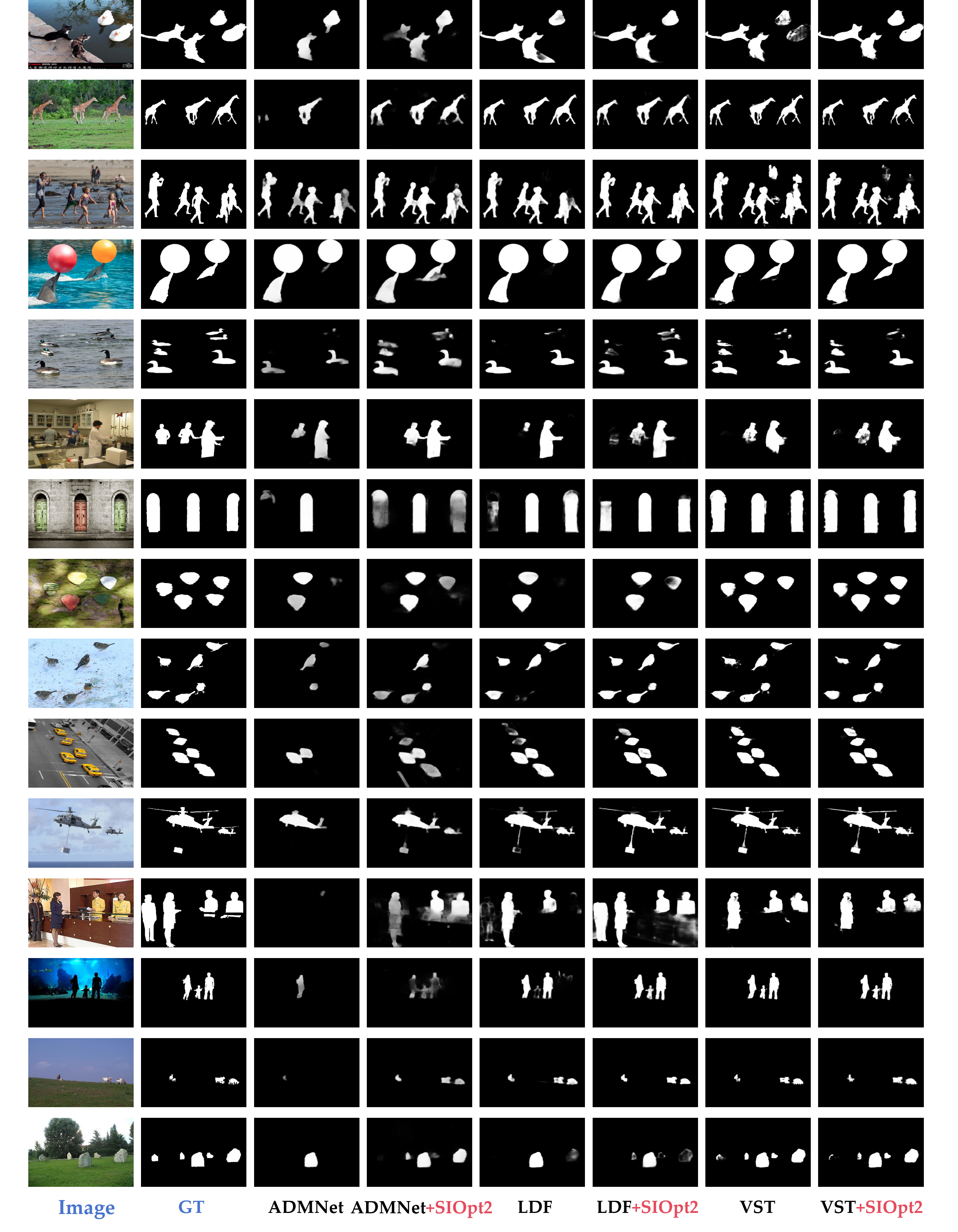}
    \caption{Qualitative visualizations on different backbones (including ADMNet, LDF and VST) before and after using our proposed $\mathsf{SIOpt2}$.}
    \label{fig:vis_SIOpt}
\end{figure*}

\begin{figure}[!t]
    \centering
    \includegraphics[width=0.95\linewidth]{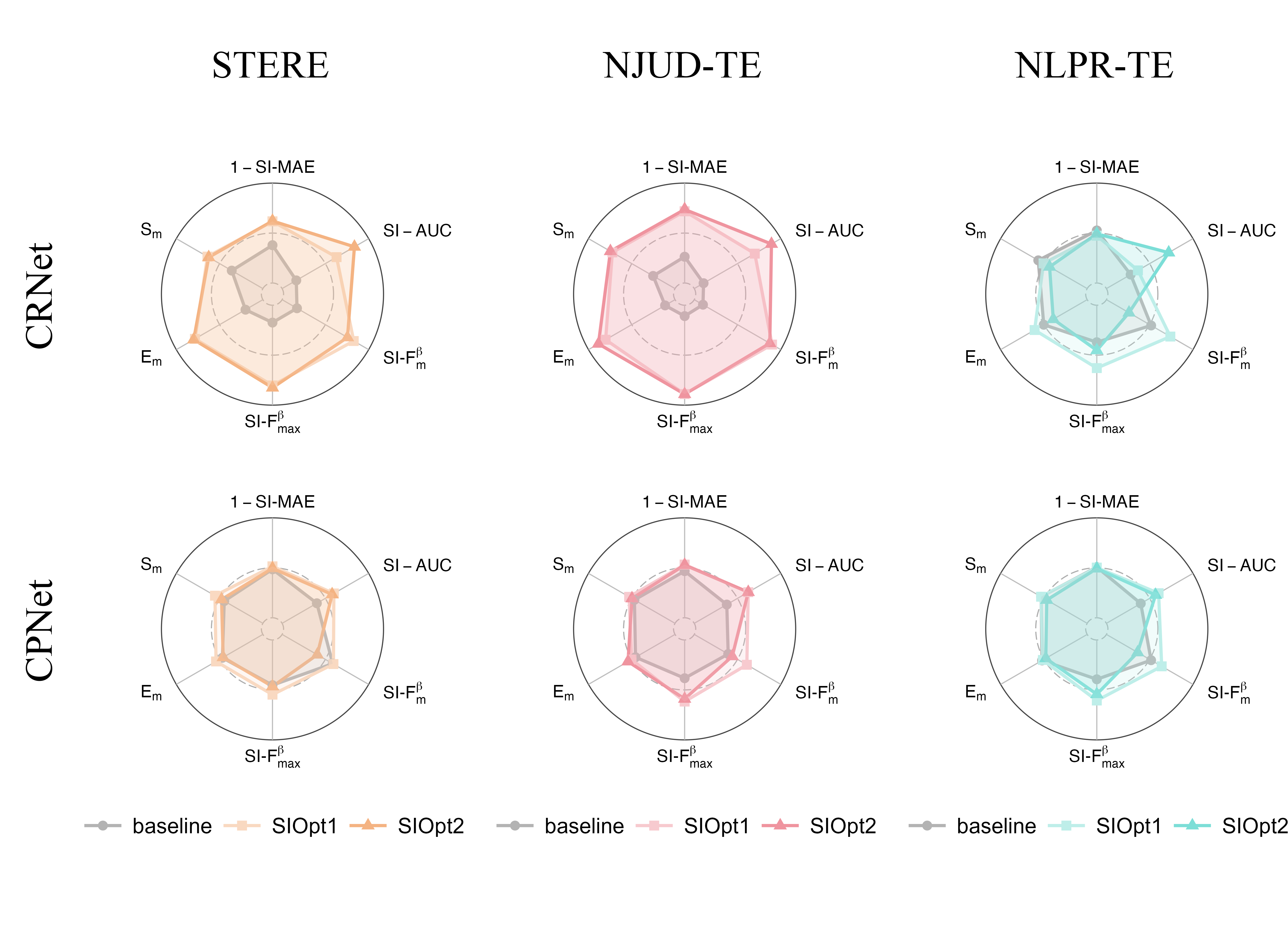}  
    \caption{Radar visualizations offering a comprehensive comparison of RGB-D SOD methods, where we plot $1-\SMAE$ to make a clearer visualization with other metrics. Here, we emphasize size-invariant measurements to better reflect performance across diverse object sizes.}    
    \label{fig:app_RGBD}    
    \end{figure}

\subsection{More Evidences on Various Downstream SOD Tasks}  \label{app:more_RGB-SOD}
\textbf{It is important to highlight that both our proposed evaluation metric (\(\mathsf{SIEva}\)) and optimization strategy (\(\mathsf{SIOpt}\)) are solely conditioned on the ground-truth masks \(\boldsymbol{Y}\), making them readily adaptable to various SOD variants, such as RGB-D and RGB-T tasks. } To demonstrate the scalability of the proposed \(\mathsf{SIOpt}\) framework, this section extends our analysis to the RGB-D and RGB-T SOD tasks \cite{zhang2023c, li2023dvsod}, where depth cues and thermal infrared images are leveraged to enhance detection accuracy, respectively.

\subsubsection{Performance Results on RGB-D SOD tasks} \label{app:out_RGBD}
 In terms of \textbf{RGB-D SOD} tasks, we integrate our \(\mathsf{SIOpt}\) into two competitive baselines—CRNet \cite{CRNet} and CPNet \cite{hu2024cross}—across three widely used datasets: NJUD, NLPR \cite{NLPR}, and STERE \cite{STERE}. The task setup and implementation details are elaborated in \cref{app:RGBD_details}. A comprehensive comparison is presented in \cref{fig:app_RGBD}, with the corresponding numerical results summarized in \cref{tab:RGBD}. Overall, $\mathsf{SIOpt}$ consistently outperforms models trained with traditional size-sensitive losses across a wide range of evaluation metrics. Even in less favorable cases, such as the $\F_m^{\beta}$ and $\E_m$ scores of CPNet on the NLPR dataset, our method remains competitive, underscoring its potential and generalizability.

\subsubsection{Performance Results on RGB-T SOD tasks} \label{app:out_RGBT}
For the task of \textbf{RGB-T SOD}, we integrate our proposed \(\mathsf{SIOpt}\) into the TNet \cite{TNet} and DCNet \cite{DCNet} models, and evaluate their performance on three benchmark datasets: VT821 \cite{VT821}, VT1000 \cite{VT1000}, and VT5000 \cite{VT5000}. The specific experimental settings are detailed in \cref{RGBT_detail}. To facilitate comparison, we present quantitative results in \cref{tab:RGBT}. As shown, the performance of both TNet and DCNet is consistently improved after incorporating our \(\mathsf{SIOpt}\) approach. Additionally, we observe that DCNet-based approaches tend to exhibit inferior performance compared to TNet in our experiments. A possible reason is that DCNet was specifically designed for scenarios involving misalignment between RGB and thermal modalities. Consequently, its architecture may have been overfitted to the particular characteristics of the datasets used during its development, which differ in alignment properties and distribution from the VT datasets used in our evaluation. Nevertheless, $\mathsf{SIOpt}$ is still able to significantly enhance DCNet's performance. Overall, these consistent improvements across different datasets and metrics demonstrate the generalizability and effectiveness of the proposed size-invariant optimization framework in enhancing RGB-T SOD performance.

\subsection{Performance Results on large foundation models} \label{app:out_TS-SAM}

The implementation details of TS-SAM are provided in \cref{SAM_detail}. The performance of TS-SAM-based methods across five benchmark datasets—DUTS, DUT-OMRON, MSOD, ECSSD, and HKU-IS—is illustrated in \cref{fig:app_SAM} and quantitatively summarized in \cref{tab:SAM}. As demonstrated, substituting the original loss functions (\cref{20250420eq94}) in TS-SAM with our proposed \(\mathsf{SIOpt}\) (\cref{20250420eq95}) consistently leads to noticeable performance improvements on most datasets, further confirming the effectiveness of our framework. However, we also observe a slight decline in performance on AUC-based metrics (e.g., \(\AUC\) and \(\SAUC\)) when compared to the default TS-SAM. A plausible explanation is that our approach (\(\mathsf{SIOpt1}\)) evaluates different image regions independently, which may compromise holistic image understanding and thereby affect the quality of global pairwise rankings. This limitation may be mitigated by employing hybrid loss formulations, such as combining \(\mathcal{L}_{\SI\AUC}(f)\) with \(\mathcal{L}_{\mathsf{Dice}}(f)\), and designing more effective region partitioning strategies. We leave this direction for future investigation.

\begin{figure*}[!t]
    \centering
    \subfigure[DUTS-TE]{   
    \includegraphics[width=0.295\linewidth]{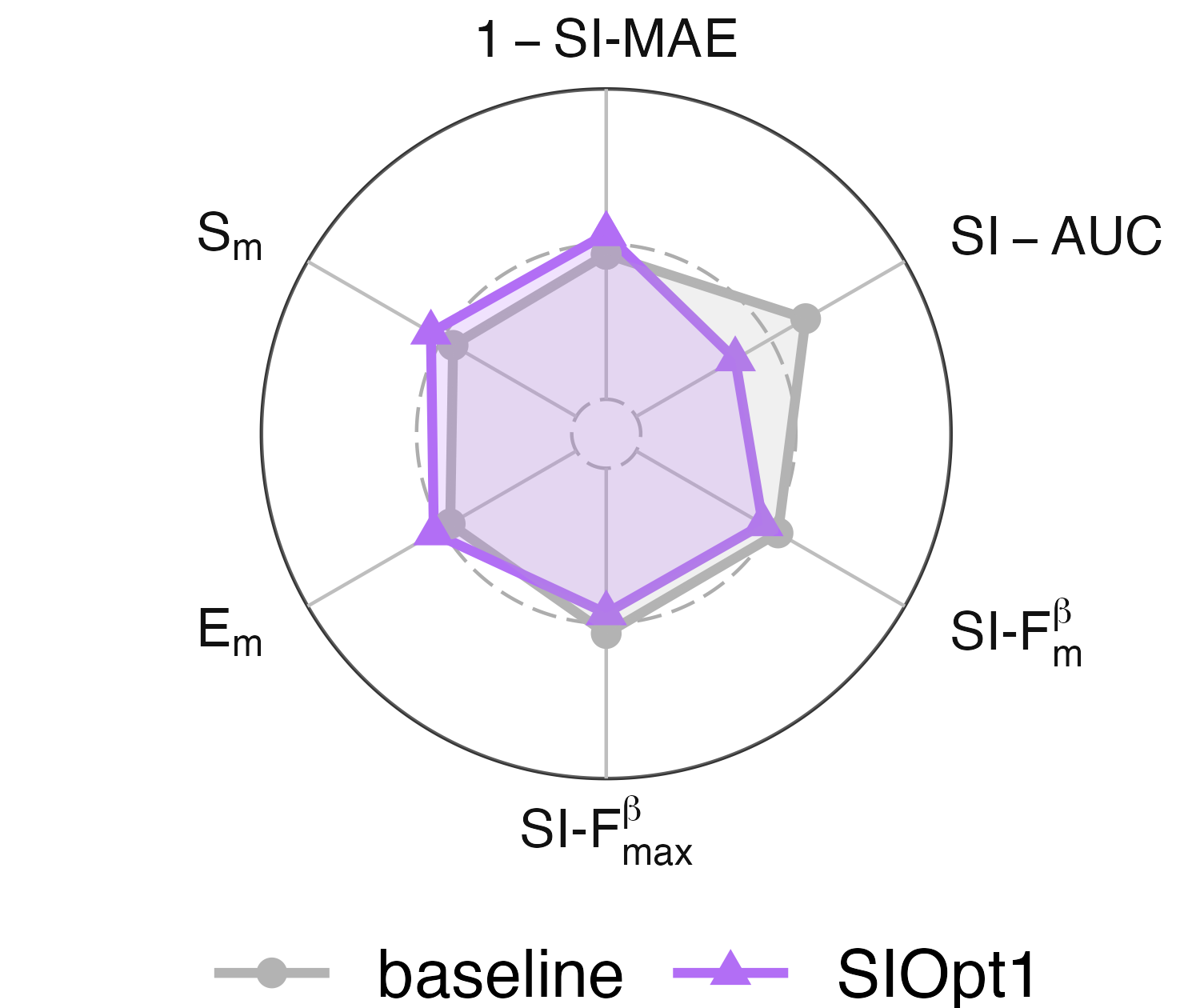}  
    }
    \subfigure[DUT-OMRON]{   
    \includegraphics[width=0.295\linewidth]{fig/exp/SAM_leida/DUT-OMRON_SAM.png}  
    }
    \subfigure[MSOD]{
    \includegraphics[width=0.295\linewidth]{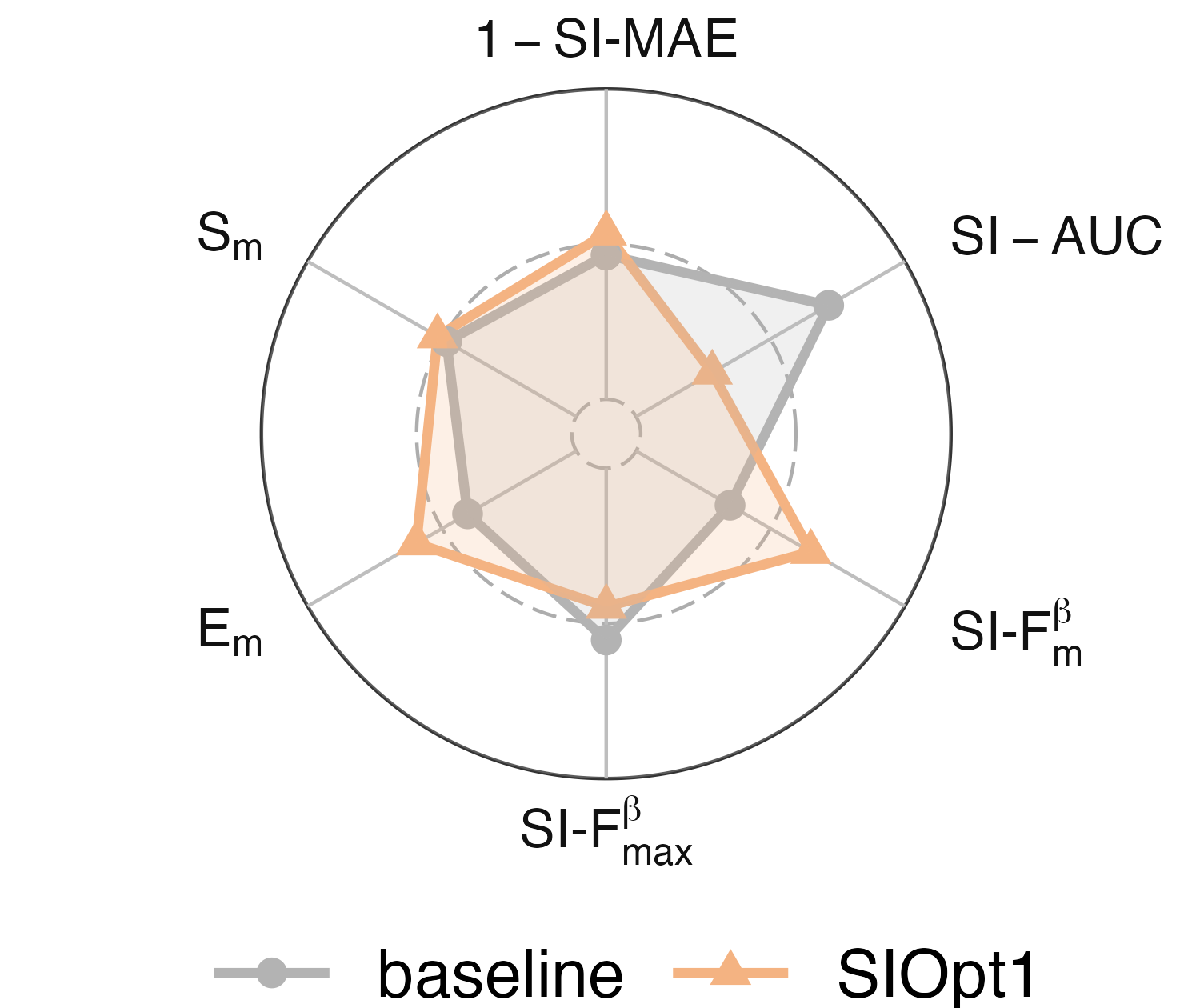}  
    }
    \subfigure[ECSSD]{
    \includegraphics[width=0.295\linewidth]{fig/exp/SAM_leida/ECSSD_SAM.png} 
    }
    \subfigure[HKU-IS]{
    \includegraphics[width=0.295\linewidth]{fig/exp/SAM_leida/HKU-IS_SAM.png} 
    }
    \caption{Radar visualizations offering a comprehensive comparison of TS-SAM-based methods under various metrics, where we plot $1-\SMAE$ to make a clearer visualization with other metrics. Here, we emphasize size-invariant measurements to better reflect performance across diverse object sizes.}    
    \label{fig:app_SAM}    
    \end{figure*}

\begin{figure*}[!t]
    \centering
    \subfigure{   
    \begin{minipage}{0.285\linewidth}
    \includegraphics[width=\linewidth]{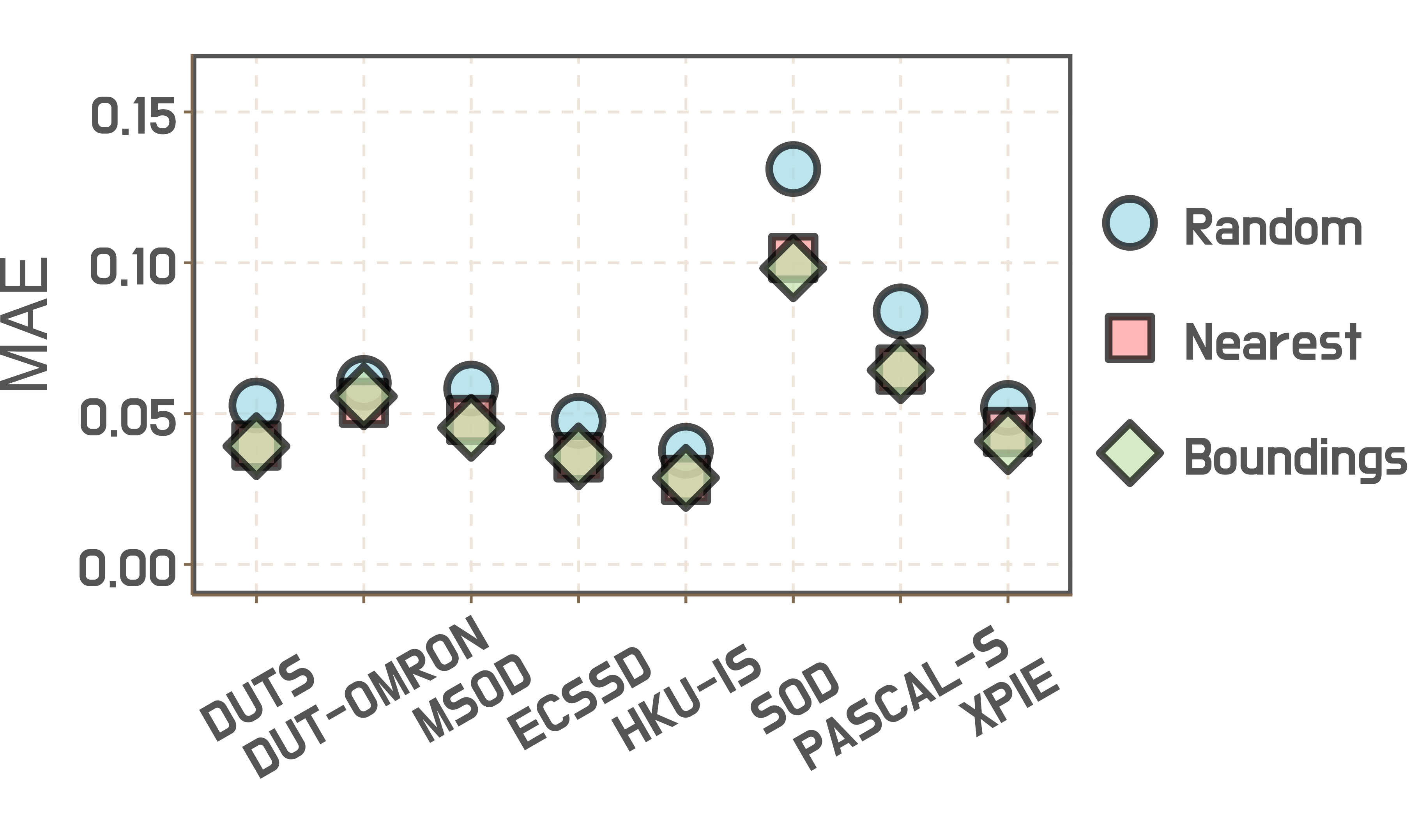}  
    \end{minipage}
    }
    \subfigure{   
    \begin{minipage}{0.285\linewidth}
    \includegraphics[width=\linewidth]{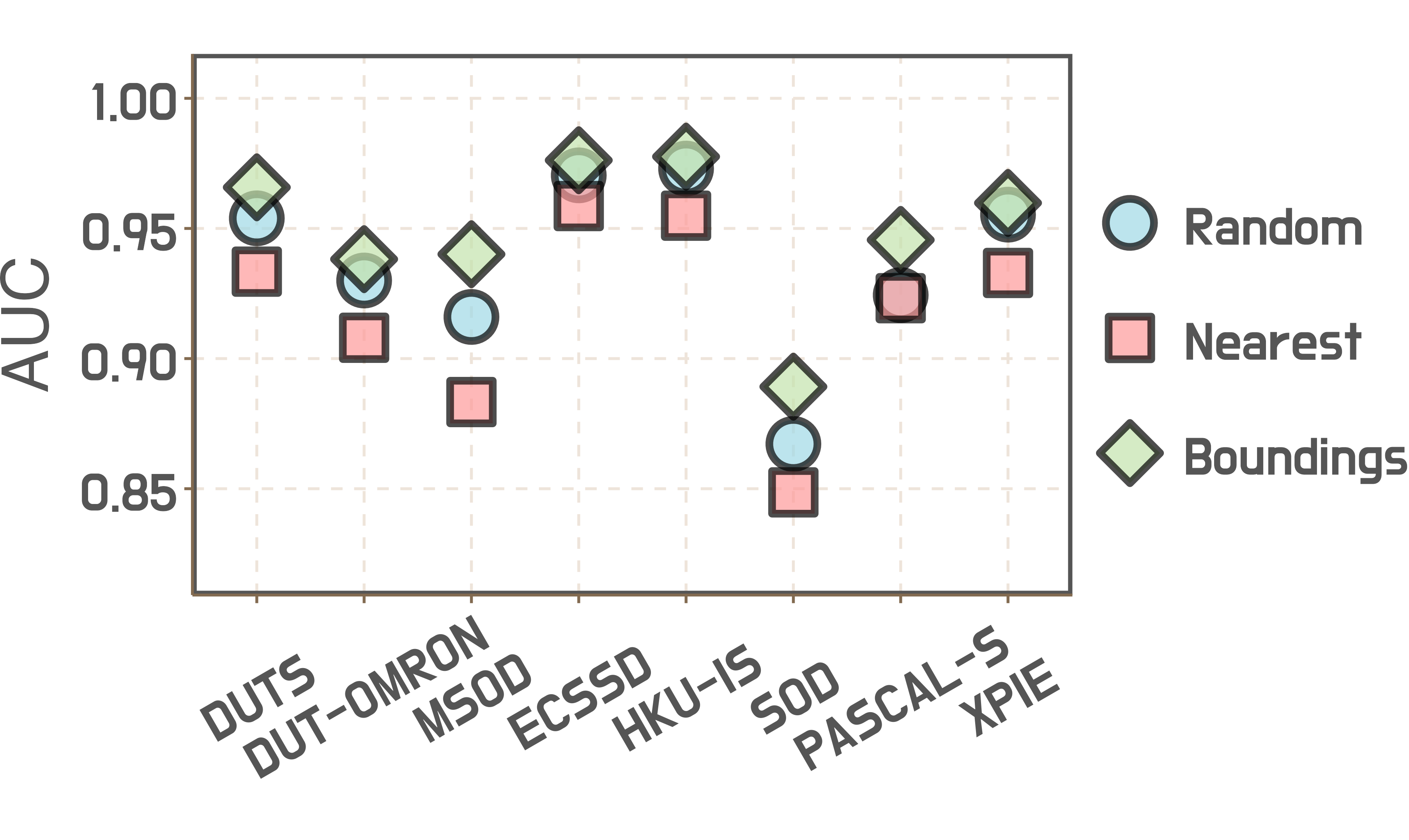}
    \end{minipage}
    }
    \subfigure{   
    \begin{minipage}{0.285\linewidth}
    \includegraphics[width=\linewidth]{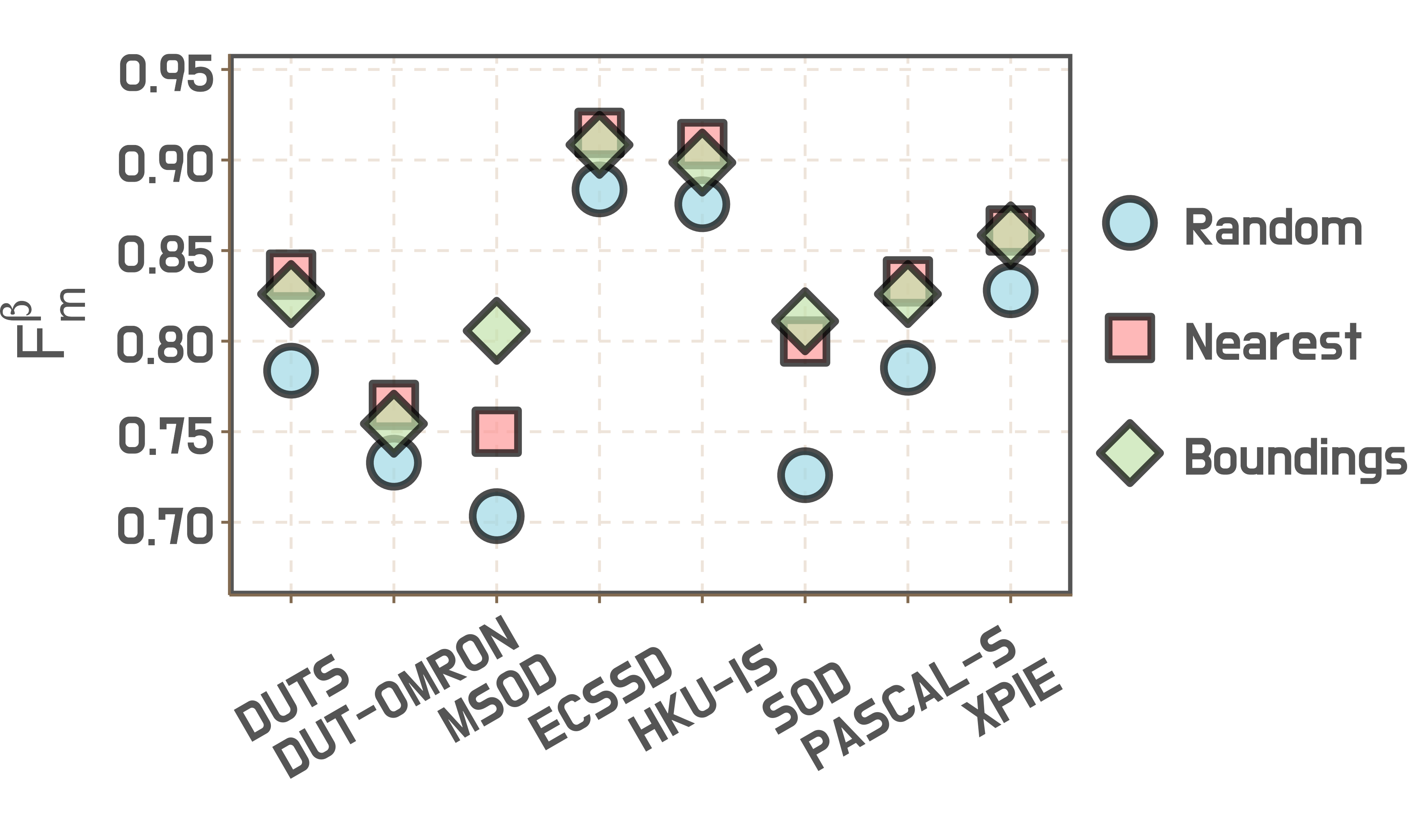} 
    \end{minipage}
    }
    \subfigure{   
    \begin{minipage}{0.285\linewidth}
    \includegraphics[width=\linewidth]{fig/exp/diff_Ck/Fmax.png}
    \end{minipage}
    }
    \subfigure{   
    \begin{minipage}{0.285\linewidth}
    \includegraphics[width=\linewidth]{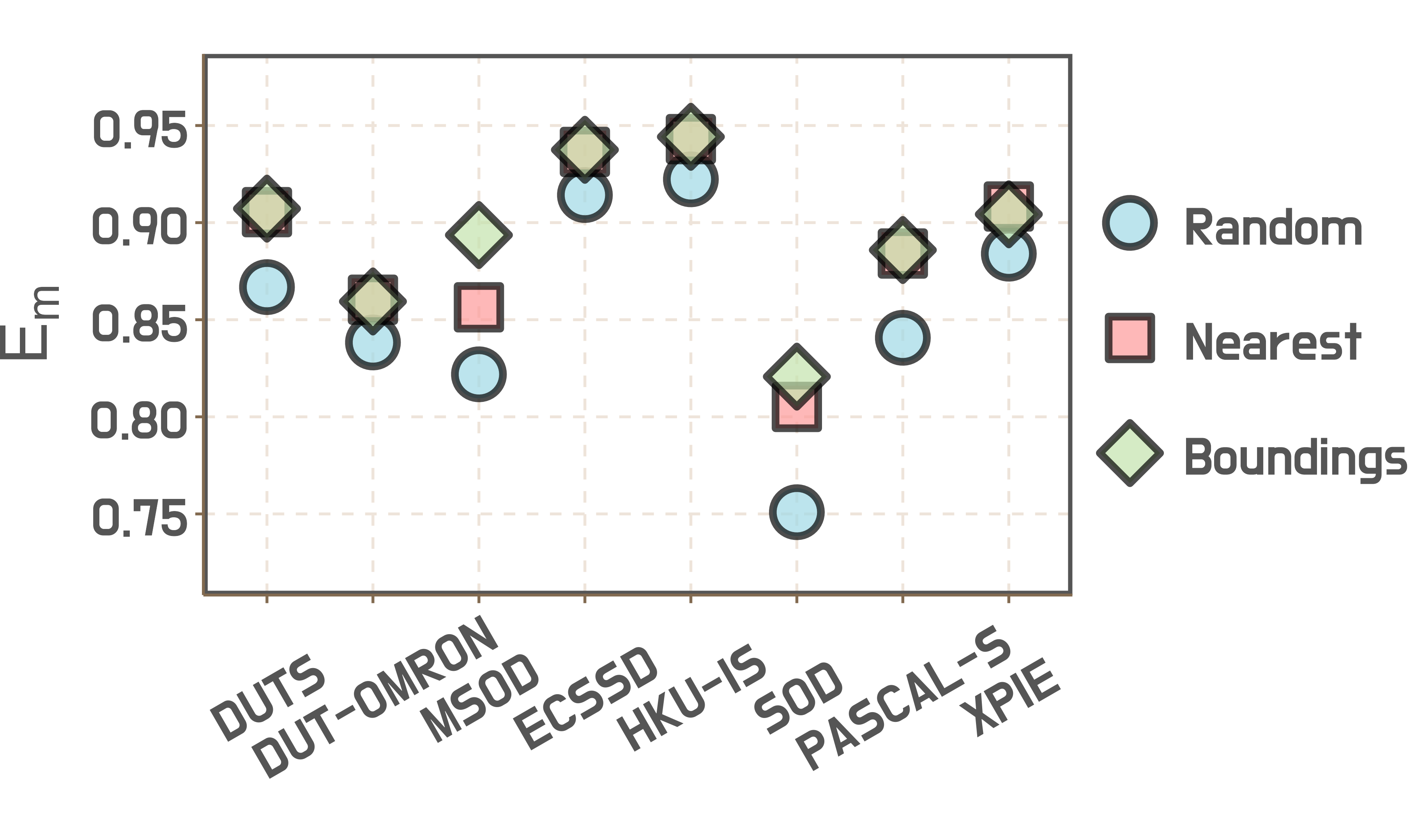} 
    \end{minipage}
    }
    \subfigure{   
    \begin{minipage}{0.285\linewidth}
    \includegraphics[width=\linewidth]{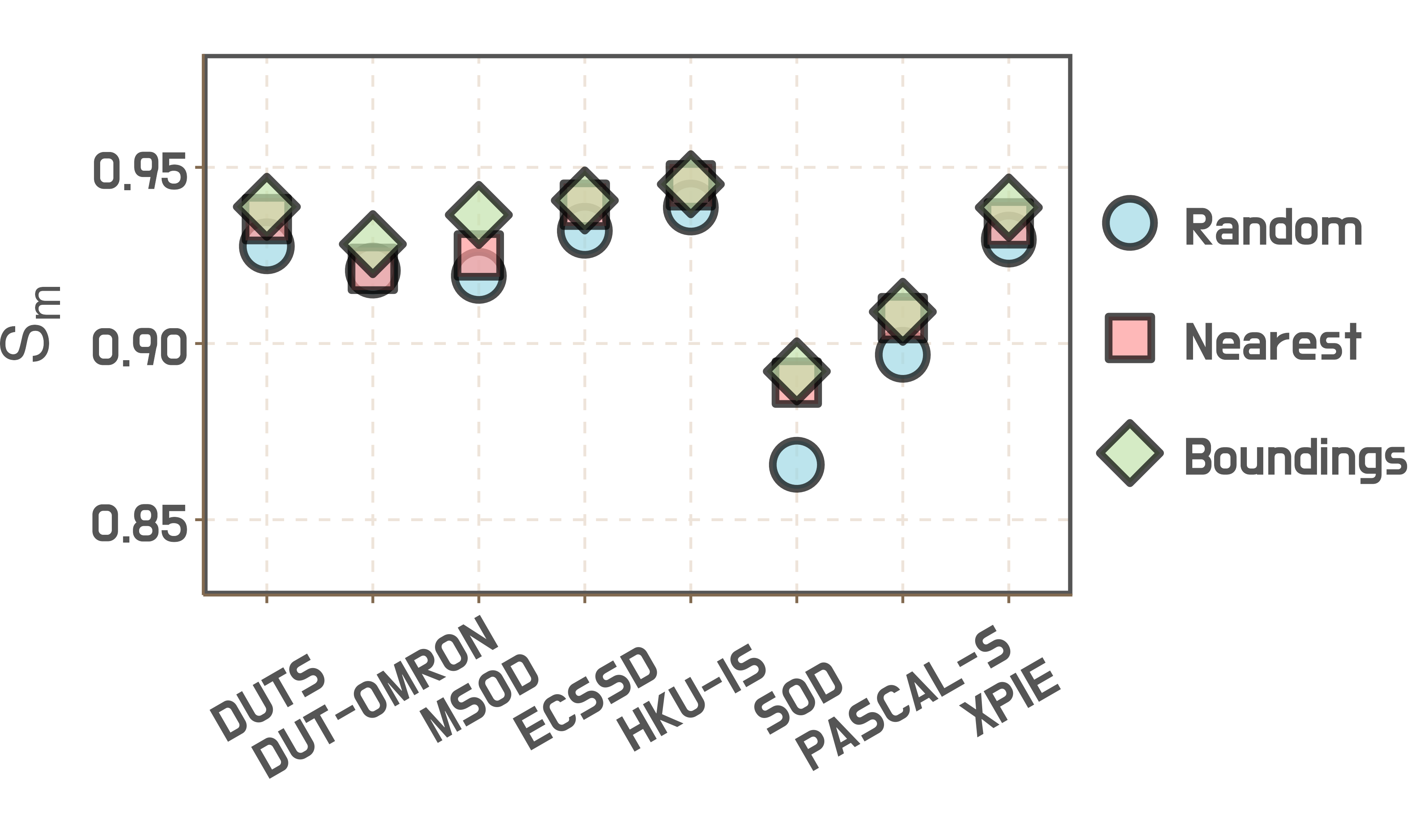} 
    \end{minipage}
    }
    \caption{Ablation performance of different $\boldsymbol{C}(\boldsymbol{X})$ functions across all benchmark datasets. Note that we do not report size-invariant metrics here, as variations in $\boldsymbol{C}(\boldsymbol{X})$ make such metrics less comparable across settings. }    
    \label{fig:app_diff_Ck}    
    \end{figure*}

    \subsection{More Evidences of different $\mathcal{C}(\boldsymbol{X})$ implementations} \label{20250411:E.4}

To investigate the impact of different implementations of $\mathcal{C}(\boldsymbol{X})$ on the final SOD performance, we adopt $\mathsf{SIOpt1}$-based EDN as the base model. Then, we compare the default strategy \textbf{Boundings} (introduced in \cref{principles of SI_eval}) with the following two alternative splitting schemes: 

\noindent (1) \textbf{Random:} In this method, the input image $\boldsymbol{X}$ is partitioned into a series $8 \times 8$ non-overlapping patches.

\noindent (2) \textbf{Nearest:} Under this context, we first extract all connected salient regions in $\boldsymbol{X}$, and then assign each non-salient pixel to its nearest salient component based on Euclidean distance between coordinates, i.e., leading $|\mathcal{C}(X)|=M$. A simple illustration is provided in \cref{fig:vis_diffK} to clarify the difference between \textbf{Nearest} and our default \textbf{Boundings}. 

\noindent\textbf{Overall Performance.} The remaining experimental settings are consistent with those described in \cref{details_appendix}. The results, summarized in \cref{fig:app_diff_Ck} and \cref{tab:diff_CK}, show that the default \textbf{Boundings} strategy achieves the best performance in most scenarios. In contrast, the \textbf{Random} strategy yields the weakest performance in most cases, except for the AUC metric. This anomaly likely arises from the nature of AUC, which focuses on the pairwise ranking of pixel-level predictions and is thus less sensitive to structural inconsistencies. The poor performance of \textbf{Random} on other metrics can be attributed to its pitfall of accidentally partitioning salient objects into disjoint and semantically unrelated regions during optimization, leading to degraded structural integrity. In comparison, the \textbf{Boundings} approach leverages the minimum bounding boxes of individual salient objects, thereby facilitating more accurate boundary delineation between salient and non-salient regions. This advantage is further corroborated by the qualitative comparisons in \cref{Qualitative_appendix}. Meanwhile, the \textbf{Nearest} strategy can introduce boundary ambiguity due to potential misassignments of adjacent pixels belonging to different salient objects, such as \textcolor{blue}{Example 4} shown in \cref{fig:vis_diffK}.

Collectively, these results provide consistent evidence supporting the effectiveness of our proposed framework. In future work, we intend to explore more principled and adaptive strategies for constructing the partition set $\mathcal{C}(\boldsymbol{X})$ to further enhance performance.

\subsection{Fine-grained Performance Comparisons across Varying Object Sizes} \label{size-fine-grained_appendix}

\noindent\textbf{Setups.}  
We conduct a size-oriented analysis on five widely-used RGB SOD benchmark datasets: MSOD, DUTS, ECSSD, DUT-OMRON and HKU-IS, and employ PoolNet \cite{PoolNet,PoolNet+}, GateNet \cite{GateNet} and EDN \cite{EDN} as representative backbones. Since our proposed $\mathsf{SIOpt}$ framework is specifically designed to improve the detection of small objects, we partition all salient objects into \textbf{ten groups} based on their relative area with respect to the entire image, spanning from \([0\%, 10\%]\), \([10\%, 20\%]\), ..., up to \([90\%, 100\%]\). For each group, we evaluate the detection performance using the proposed $\SMAE$ metric, which is tailored to be size-invariant and less biased toward large-object dominance. Importantly, only the foreground regions defined by $\mathcal{C}(\boldsymbol{X})$ (see \cref{principles of SI_eval}) are included in the analysis, as incorporating background pixels can substantially skew the results—particularly when assessing small-object performance.

\noindent\textbf{Results.}  
Detailed results are presented in \cref{fig:fine-analysis-Pool-ratio} (PoolNet), \cref{fig:fine-analysis-GateNet-ratio} (GateNet), and \cref{fig:fine-analysis-EDN-ratio-appendix} (EDN). Across all datasets and backbone architectures, the proposed $\mathsf{SIOpt}$ framework consistently outperforms the original methods across all object-size groups. Notably, the performance improvements are most pronounced for small objects, particularly those occupying less than \(10\%\) of the image area. For instance, on the MSOD dataset, $\mathsf{SIOpt1}$ achieves a performance gain of approximately 0.024 in $\SMAE$ over the baseline EDN model in the \([0\%, 10\%]\) size group. Similar trends are consistently observed across other datasets and backbones, with the most substantial performance gains concentrated in the small-object regime. These findings provide strong empirical evidence in support of our core objective—enhancing the detection of small-scale salient objects, which remains a persistent challenge in SOD. Furthermore, it is not surprising to observe that the performance gap diminishes as object size increases. This aligns with our expectations, as larger objects are inherently more detectable and less susceptible to the limitations of existing SOD methods.

\begin{figure*}[!t]
    \centering
    \subfigure[MSOD]{   
    \begin{minipage}{0.185\linewidth}
    \includegraphics[width=\linewidth]{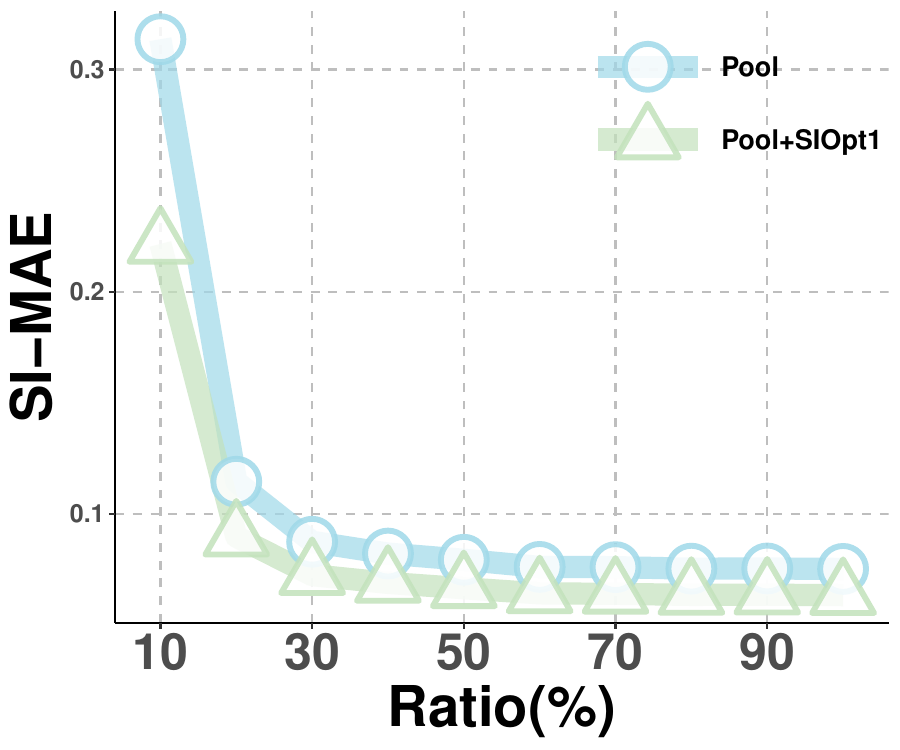}  
    \label{fig:Pool_msod_ratio_line}
    \end{minipage}
    }
    \subfigure[DUTS]{   
    \begin{minipage}{0.185\linewidth}
    \includegraphics[width=\linewidth]{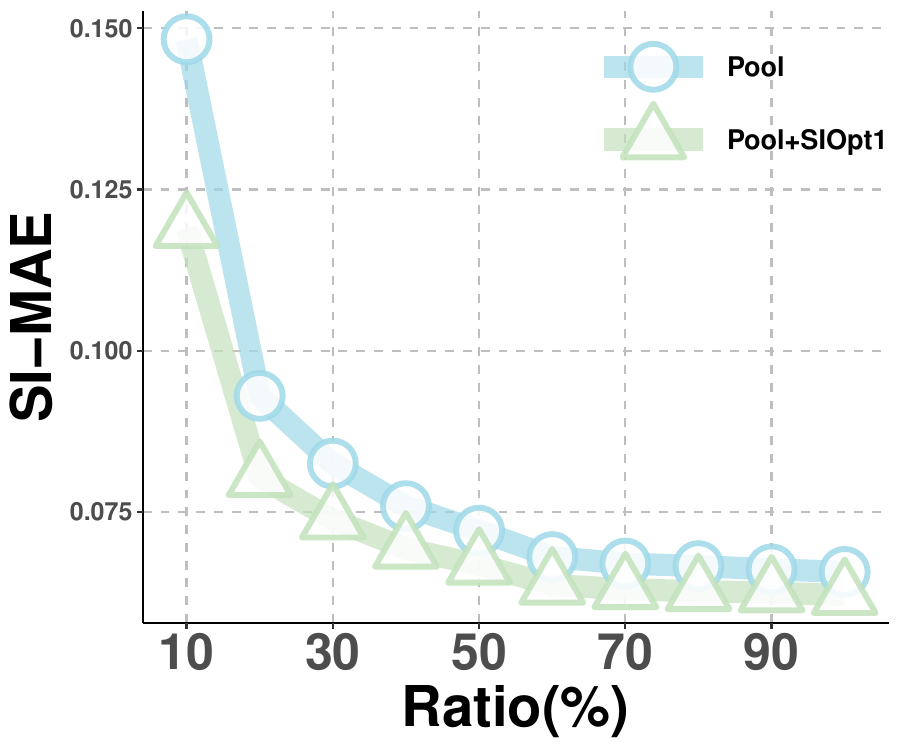}  
    \label{fig:Pool_DUTS_ratio_line}
    \end{minipage}
    }
    \subfigure[ECSSD]{   
    \begin{minipage}{0.185\linewidth}
    \includegraphics[width=\linewidth]{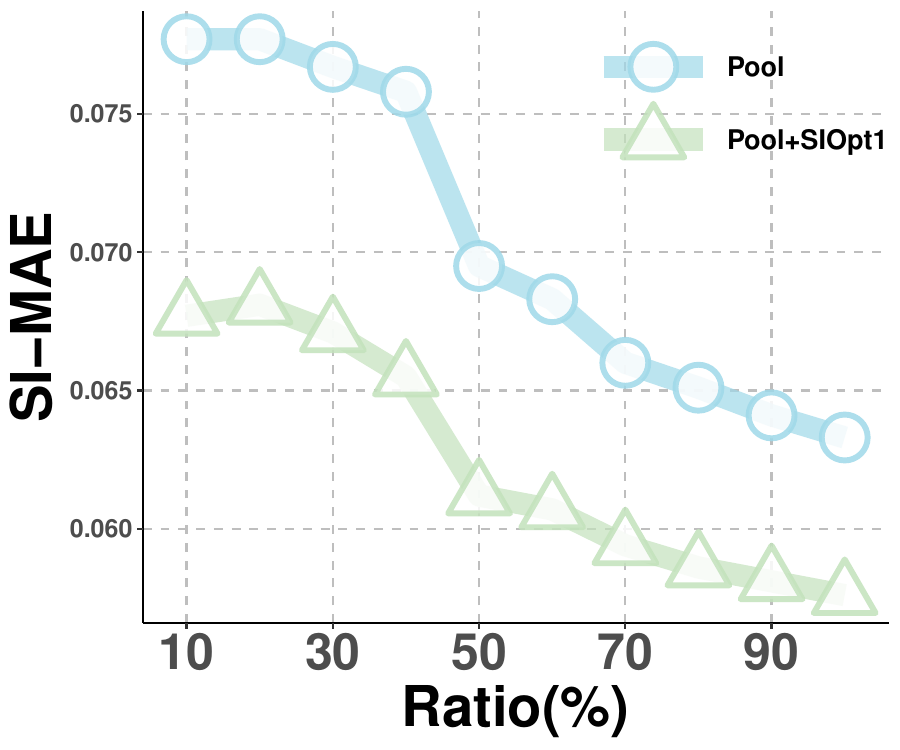}  
    \label{fig:Pool_ECSSD_ratio_line}
    \end{minipage}
    }
    \subfigure[DUT-OMRON]{   
    \begin{minipage}{0.185\linewidth}
    \includegraphics[width=\linewidth]{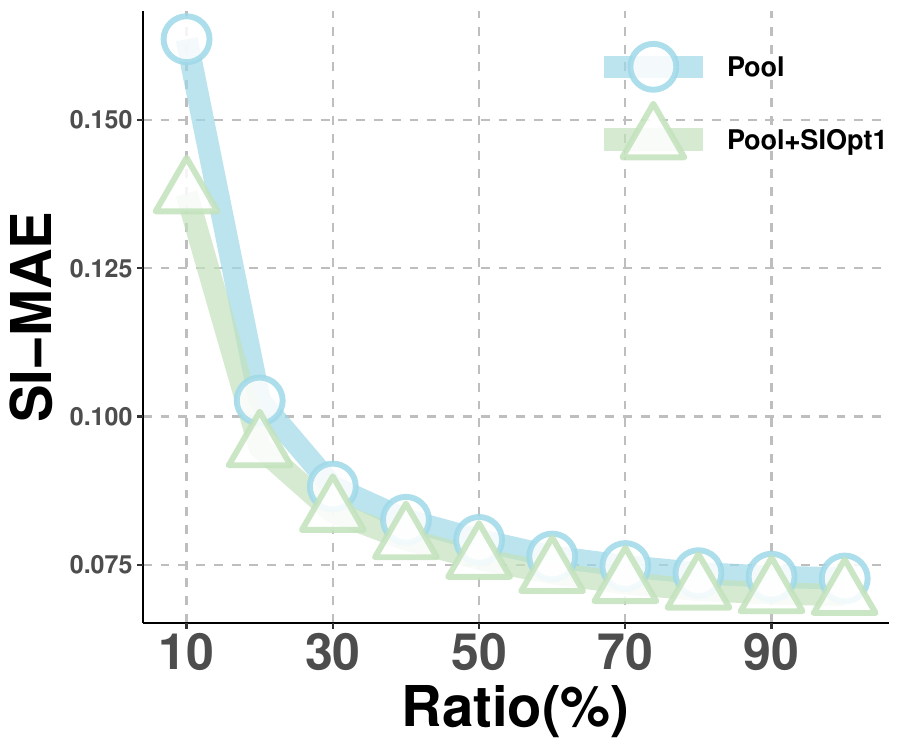}  
    \label{fig:Pool_DUT-OMRON_ratio_line}
    \end{minipage}
    }
    \subfigure[HKU-IS]{   
    \begin{minipage}{0.185\linewidth}
    \includegraphics[width=\linewidth]{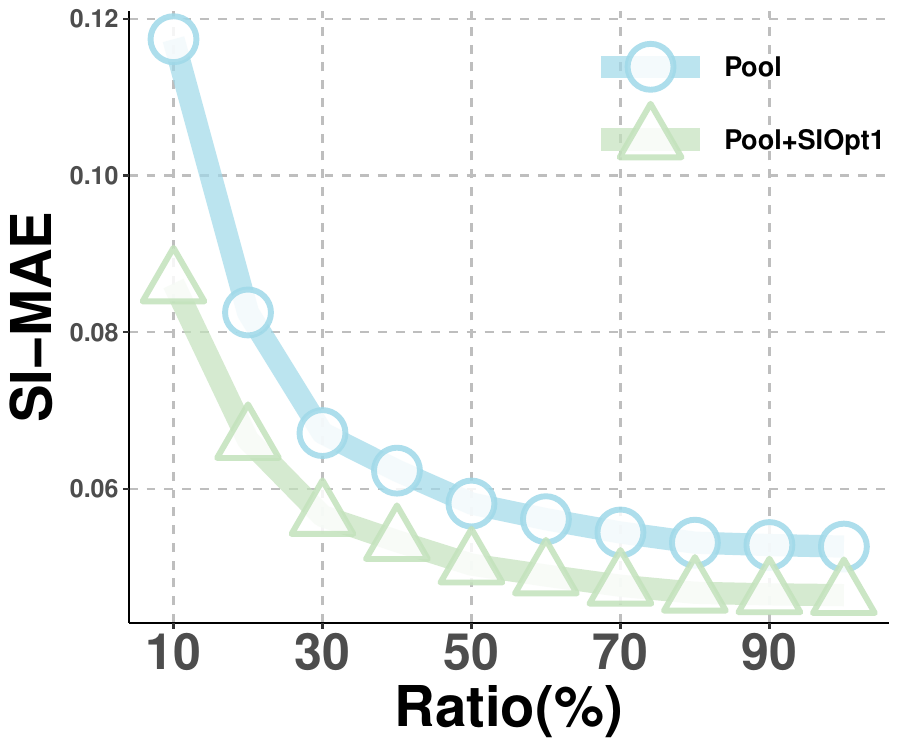}  
    \label{fig:Pool_HKU-IS_ratio_line}
    \end{minipage}
    }
    \caption{Fine-grained performance comparisons of varying object sizes, with PoolNet as the backbone.}    
    \label{fig:fine-analysis-Pool-ratio}    
    \end{figure*}

\begin{figure*}[!t]
    \centering
    \subfigure[MSOD]{   
    \begin{minipage}{0.185\linewidth}
    \includegraphics[width=\linewidth]{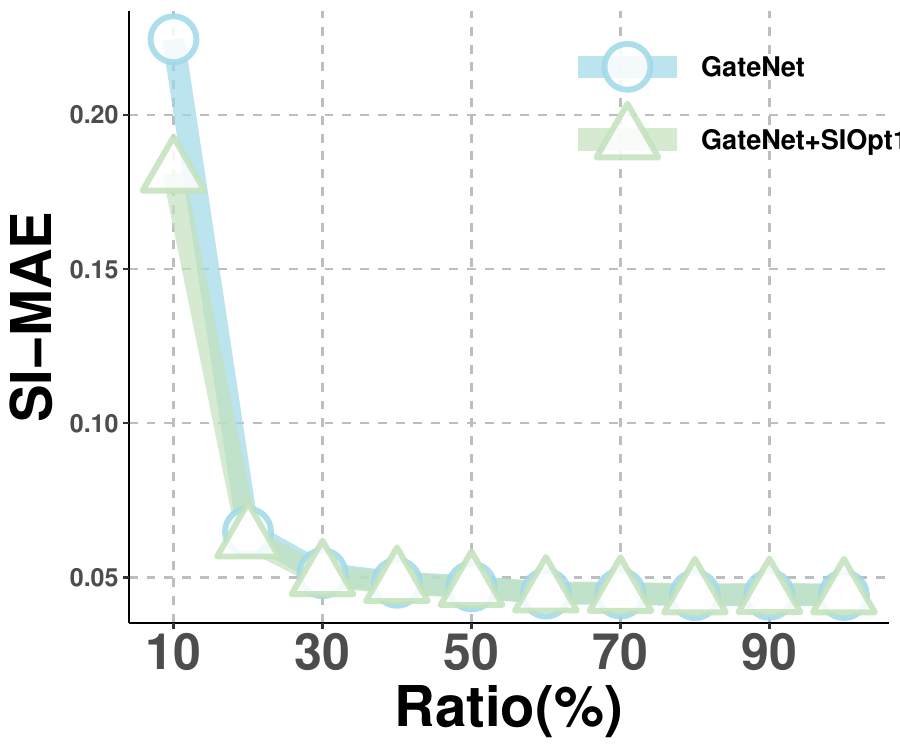}  
    \label{fig:gatenet_msod_ratio_line}
    \end{minipage}
    }
    \subfigure[DUTS]{   
    \begin{minipage}{0.185\linewidth}
    \includegraphics[width=\linewidth]{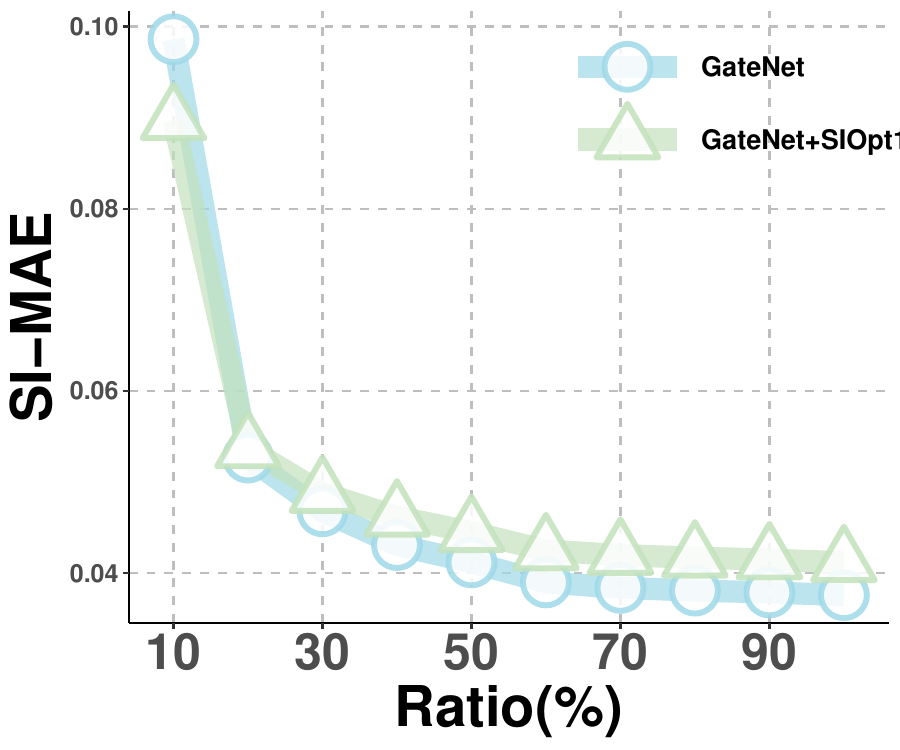}  
    \label{fig:gatenet_DUTS_ratio_line}
    \end{minipage}
    }
    \subfigure[ECSSD]{   
    \begin{minipage}{0.185\linewidth}
    \includegraphics[width=\linewidth]{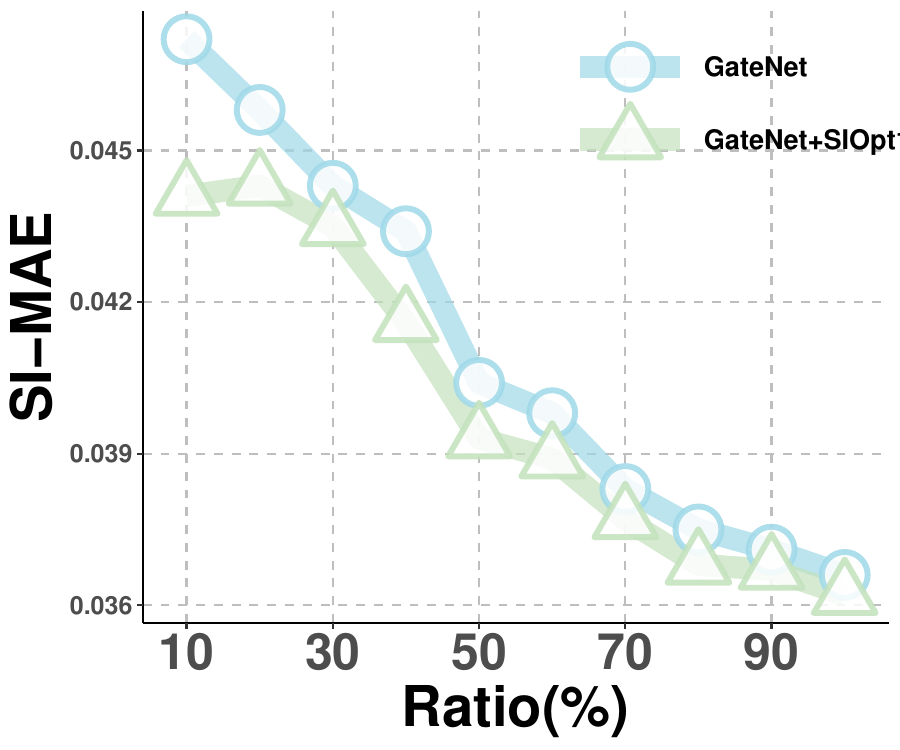}  
    \label{fig:gatenet_ECSSD_ratio_line}
    \end{minipage}
    }
    \subfigure[DUT-OMRON]{   
    \begin{minipage}{0.185\linewidth}
    \includegraphics[width=\linewidth]{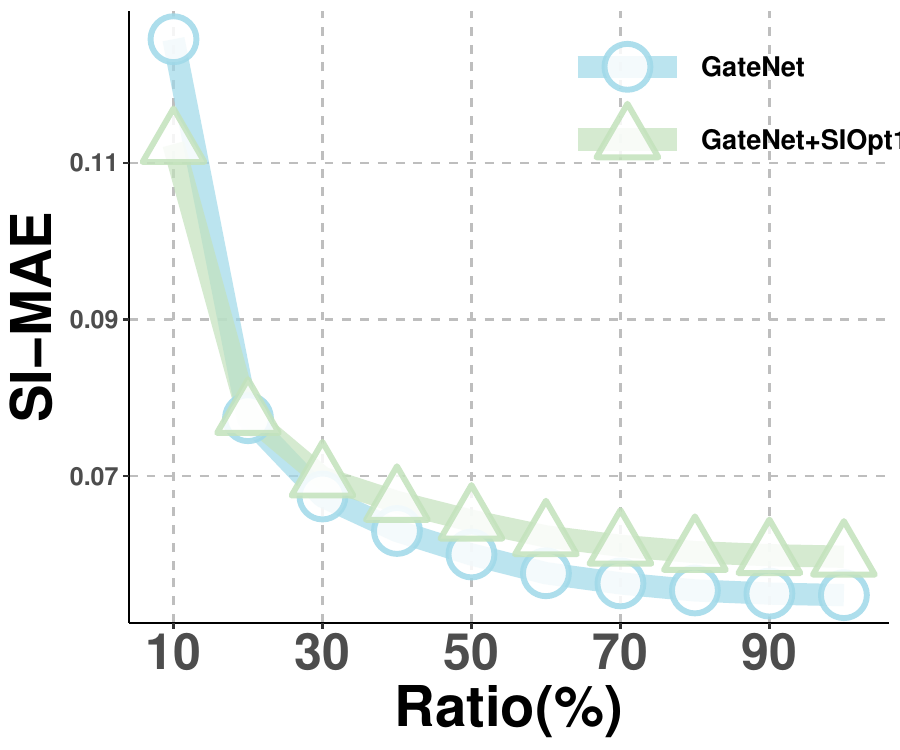}  
    \label{fig:gatenet_DUT-OMRON_ratio_line}
    \end{minipage}
    }
    \subfigure[HKU-IS]{   
    \begin{minipage}{0.185\linewidth}
    \includegraphics[width=\linewidth]{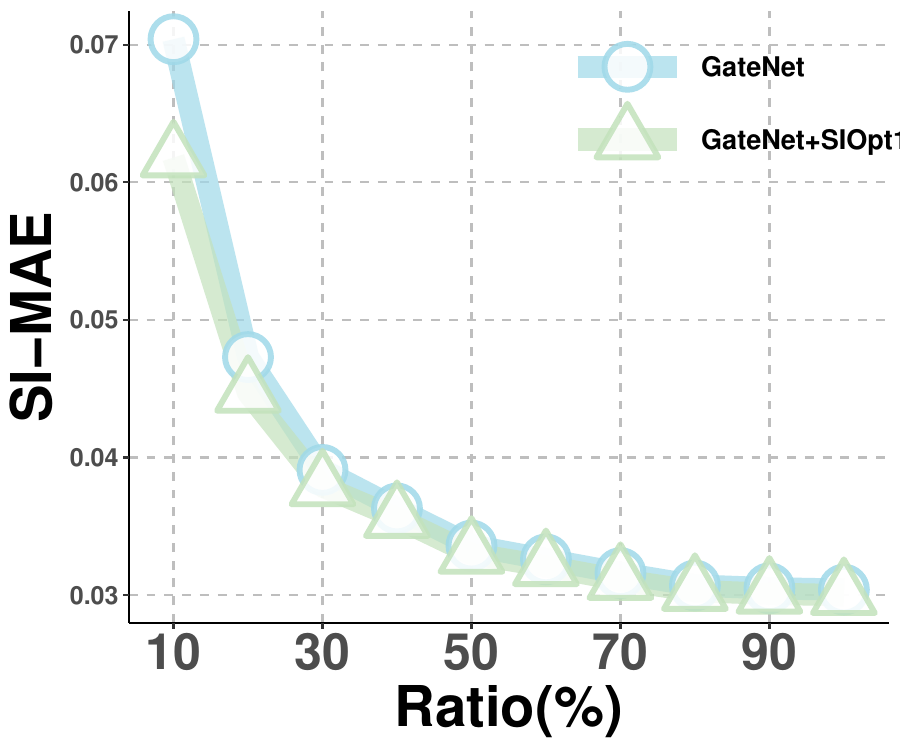}  
    \label{fig:gatenet_HKU-IS_ratio_line}
    \end{minipage}
    }
    \caption{Fine-grained performance comparisons of varying object sizes, with GateNet as the backbone.}    
    \label{fig:fine-analysis-GateNet-ratio}    
    \end{figure*}

\begin{figure*}[!t]
    \centering
    \subfigure[MSOD]{   
    \begin{minipage}{0.185\linewidth}
    \includegraphics[width=\linewidth]{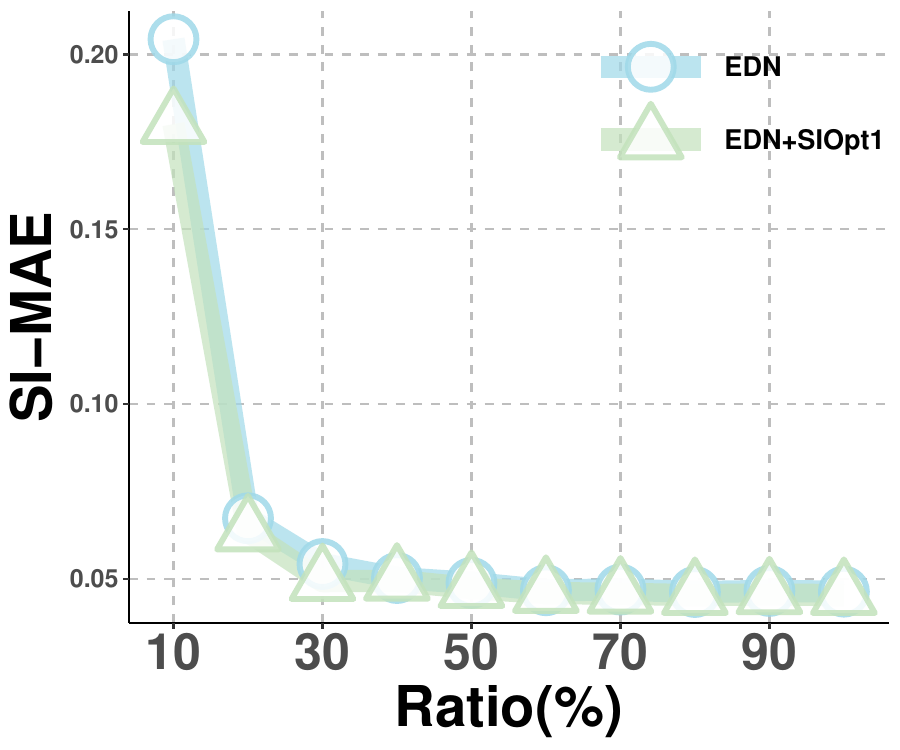}  
    \label{fig:EDN_msod_ratio_line-appendix}
    \end{minipage}
    }
    \subfigure[DUTS]{   
    \begin{minipage}{0.185\linewidth}
    \includegraphics[width=\linewidth]{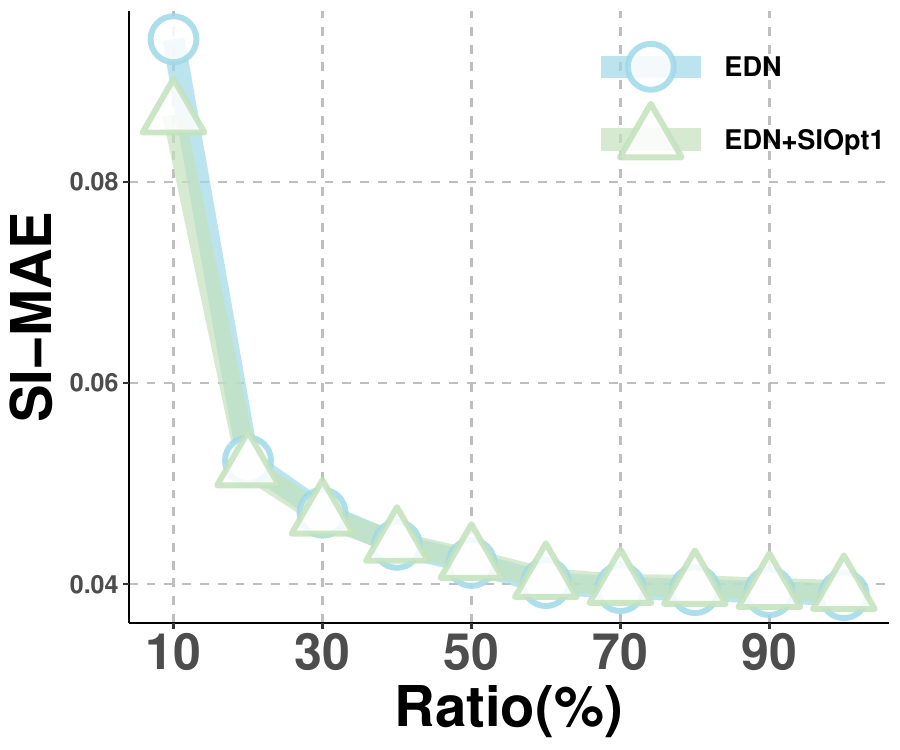}  
    \label{fig:EDN_DUTS_ratio_line-appendix}
    \end{minipage}
    }
    \subfigure[ECSSD]{   
    \begin{minipage}{0.185\linewidth}
    \includegraphics[width=\linewidth]{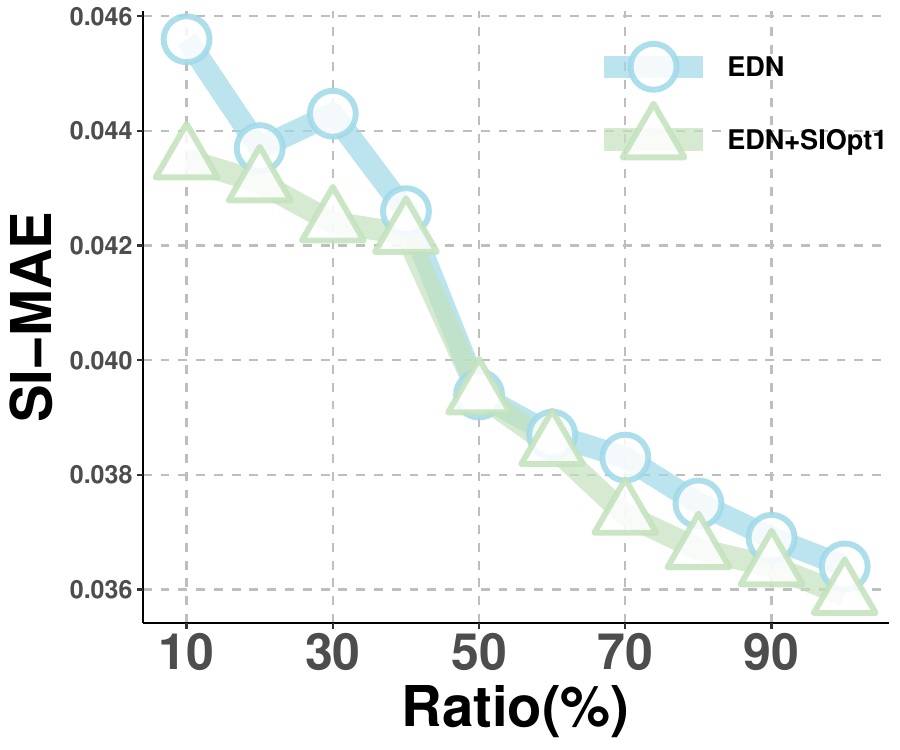}  
    \label{fig:EDN_ECSSD_ratio_line}
    \end{minipage}
    }
    \subfigure[DUT-OMRON]{   
    \begin{minipage}{0.185\linewidth}
    \includegraphics[width=\linewidth]{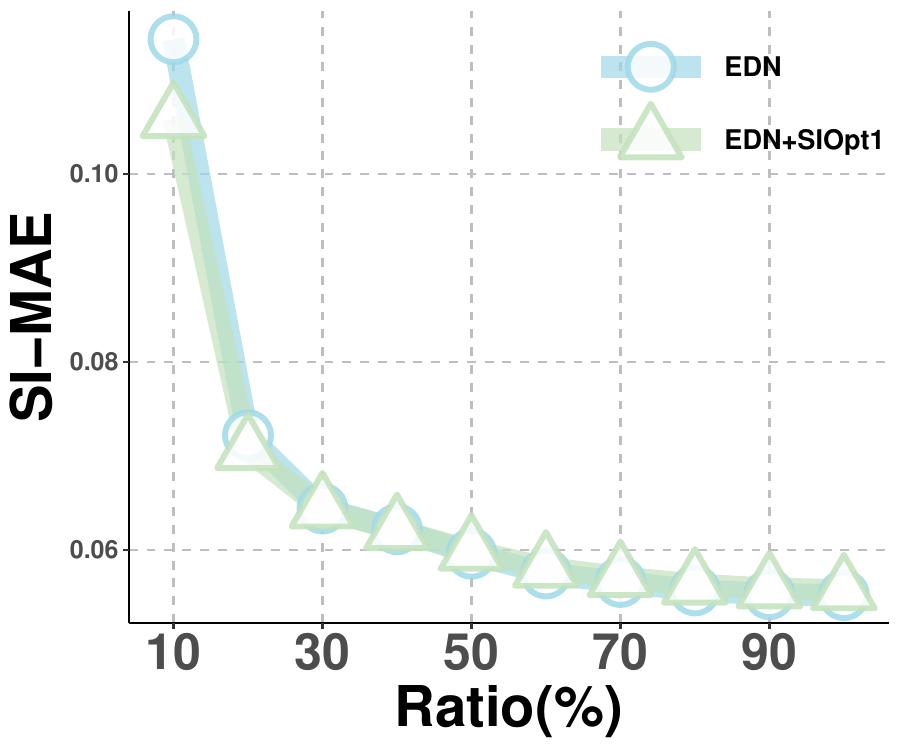}  
    \label{fig:EDN_DUT-OMRON_ratio_line}
    \end{minipage}
    }
    \subfigure[HKU-IS]{   
    \begin{minipage}{0.185\linewidth}
    \includegraphics[width=\linewidth]{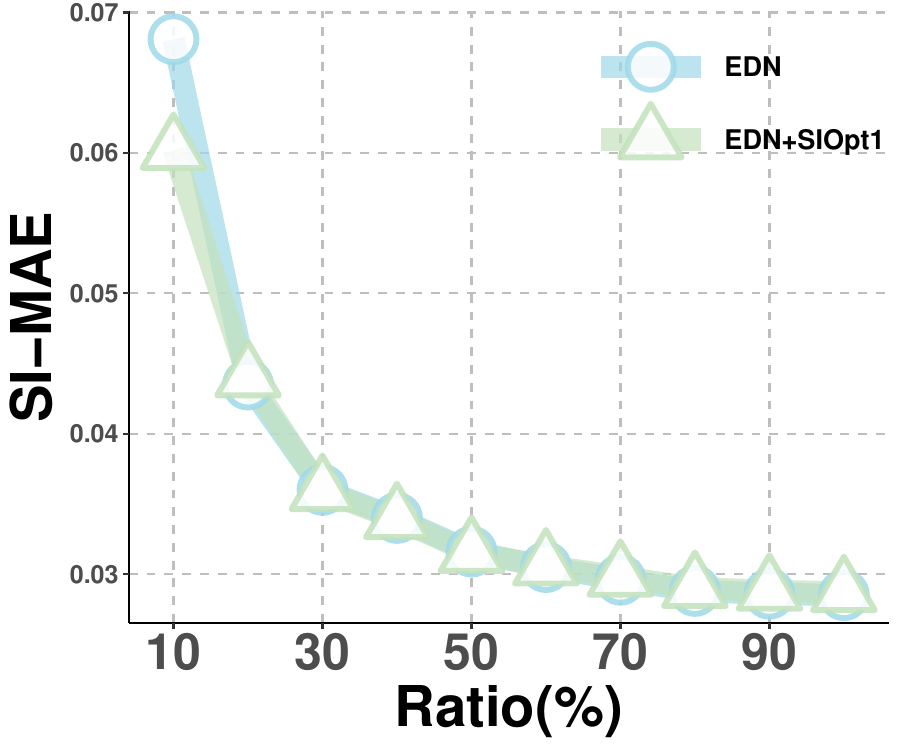}  
    \label{fig:EDN_HKU-IS_ratio_line}
    \end{minipage}
    }
    \caption{Fine-grained performance comparisons of varying object sizes, with EDN as the backbone.}
    \label{fig:fine-analysis-EDN-ratio-appendix}    
    \end{figure*}

\subsection{Fine-grained Performance Comparisons for Different Object Quantities} \label{number-fine-grained_appendix}
\noindent\textbf{Setups.} To further understand the effectiveness of our method, we conduct a fine-grained performance analysis based on the \textbf{number} of salient objects per image—a factor that can significantly influence detection outcomes in SOD tasks. This number-aware analysis is carried out on five widely adopted RGB SOD benchmark datasets: MSOD, DUTS, ECSSD, DUT-OMRON, and HKU-IS. We report results across five representative backbone models: PoolNet \cite{PoolNet,PoolNet+}, LDF \cite{LDF}, ICON \cite{ICON}, GateNet \cite{GateNet}, and EDN \cite{EDN}. Specifically, the number of salient objects in an image is determined by counting the connected components in the ground truth saliency mask. Based on this, we finally categorize each dataset into \textbf{six groups}, corresponding to the number of salient objects present: \([1], [2], [3], [4], [5], [6+]\), where the final group \([6+]\) indicates all images containing more than six salient objects.

\noindent\textbf{Results.} Comprehensive results are shown in \cref{fig:fine-analysis-Pool-num} (PoolNet), \cref{fig:fine-analysis-LDF-num} (LDF), \cref{fig:fine-analysis-ICON-num} (ICON), \cref{fig:fine-analysis-GateNet-num} (GateNet), \cref{fig:fine-analysis-EDN-num} (EDN), \cref{fig:fine-analysis-ADMNet-num} (ADMNet) and \cref{fig:fine-analysis-VST-num} (VST), respectively. We first observe that when an image contains only a single salient object, the performance of our method and the baselines is largely comparable. This observation is expected, as single-object scenarios tend to be easy to segment (consistent with our discussions in \cref{Revisting}), thereby minimizing the performance gap across different methods. However, as the number of salient objects increases, $\mathsf{SIOpt}$ framework could yield better performance than baselines in most cases, particularly on challenging datasets such as MSOD and HKU-IS. (e.g., MSOD and HKU-IS datasets). For instance, on the MSOD dataset, applying $\mathsf{SIOpt1}$ to the EDN backbone leads to a sharp improvement of approximately $0.007$ in $\SMAE$ for images containing two or more salient objects. These trends further underscore the effectiveness of our framework in handling complex scenes with multiple salient objects.

\begin{figure*}[ht]
    \centering
    \subfigure[MSOD]{   
    \begin{minipage}{0.185\linewidth}
    \includegraphics[width=\linewidth]{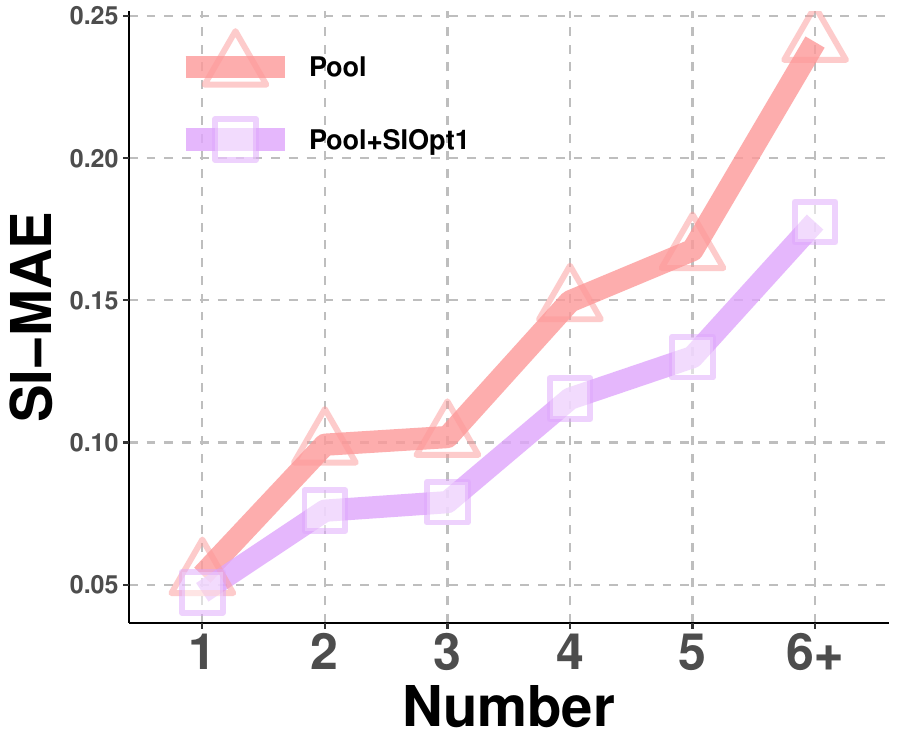}  
    \label{fig:Pool_msod_num_line}
    \end{minipage}
    }
    \subfigure[DUTS]{   
    \begin{minipage}{0.185\linewidth}
    \includegraphics[width=\linewidth]{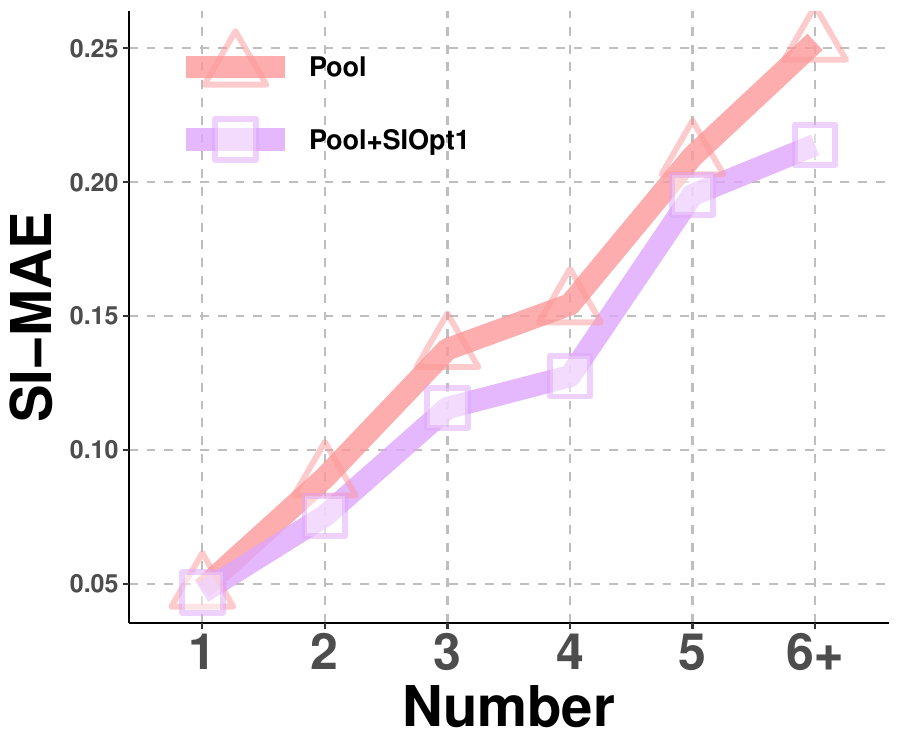}  
    \label{fig:Pool_DUTS_num_line}
    \end{minipage}
    }
    \subfigure[ECSSD]{   
    \begin{minipage}{0.185\linewidth}
    \includegraphics[width=\linewidth]{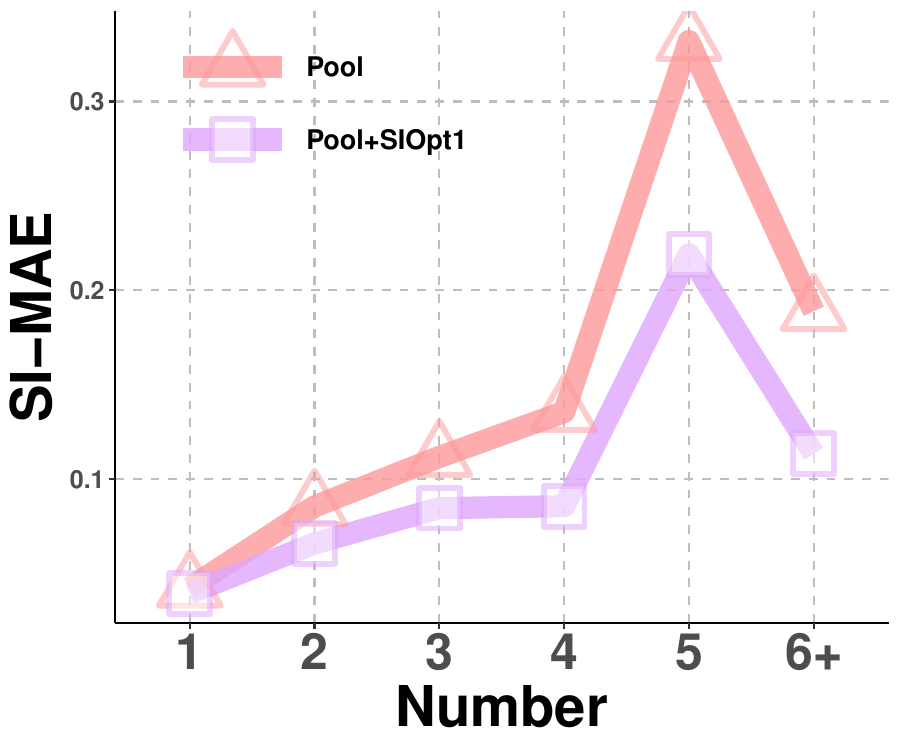}  
    \label{fig:Pool_ECSSD_num_line}
    \end{minipage}
    }
    \subfigure[DUT-OMRON]{   
    \begin{minipage}{0.185\linewidth}
    \includegraphics[width=\linewidth]{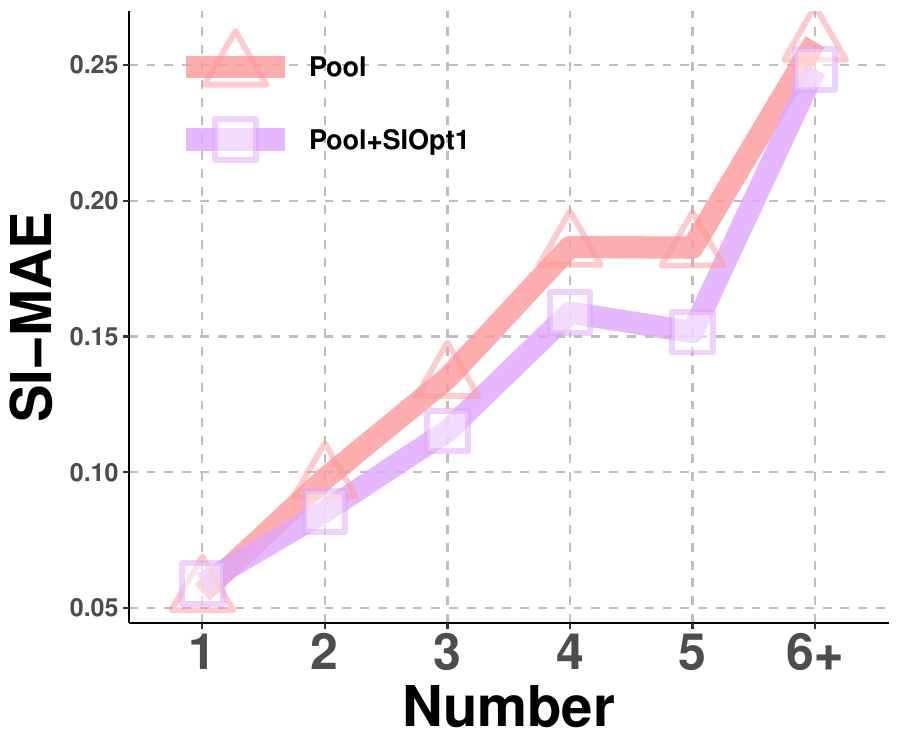}  
    \label{fig:Pool_DUT-OMRON_num_line}
    \end{minipage}
    }
    \subfigure[HKU-IS]{   
    \begin{minipage}{0.185\linewidth}
    \includegraphics[width=\linewidth]{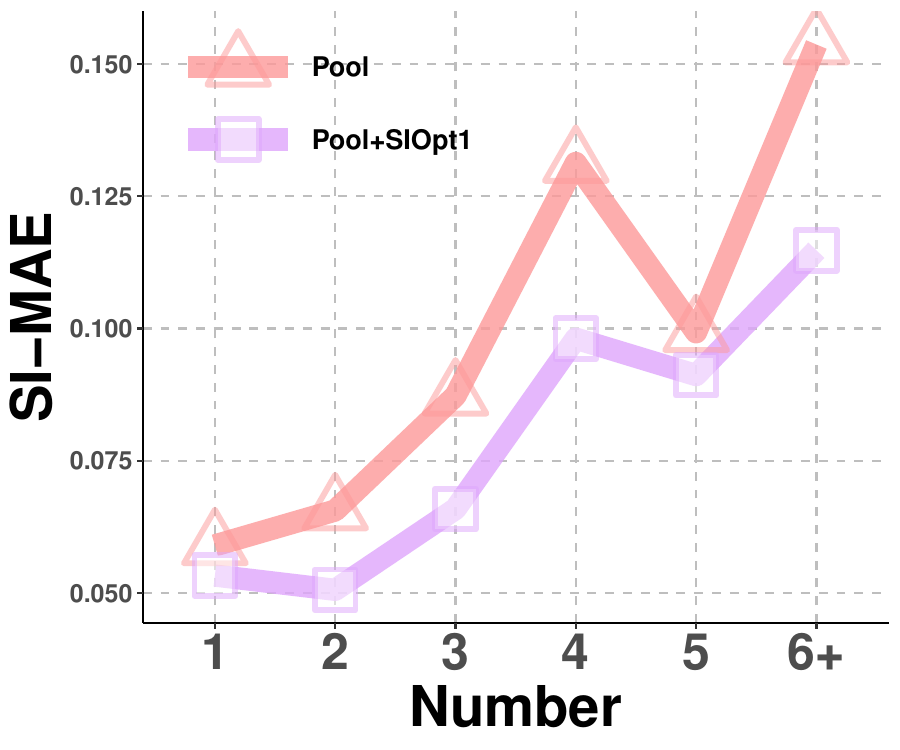}  
    \label{fig:Pool_HKU-IS_num_line}
    \end{minipage}
    }
    \caption{Fine-grained performance comparisons under different object numbers, with PoolNet as the backbone.}    
    \label{fig:fine-analysis-Pool-num}    
    \end{figure*}

\begin{figure*}[ht]
\centering
\subfigure[MSOD]{   
\begin{minipage}{0.185\linewidth}
\includegraphics[width=\linewidth]{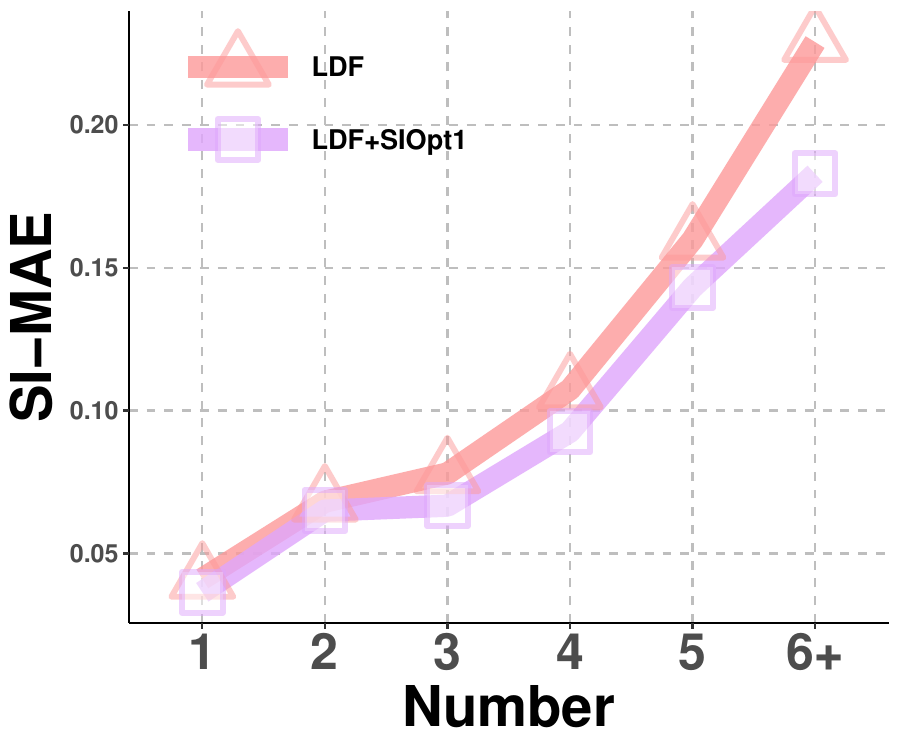}  
\label{fig:LDF_msod_num_line}
\end{minipage}
}
\subfigure[DUTS]{   
\begin{minipage}{0.185\linewidth}
\includegraphics[width=\linewidth]{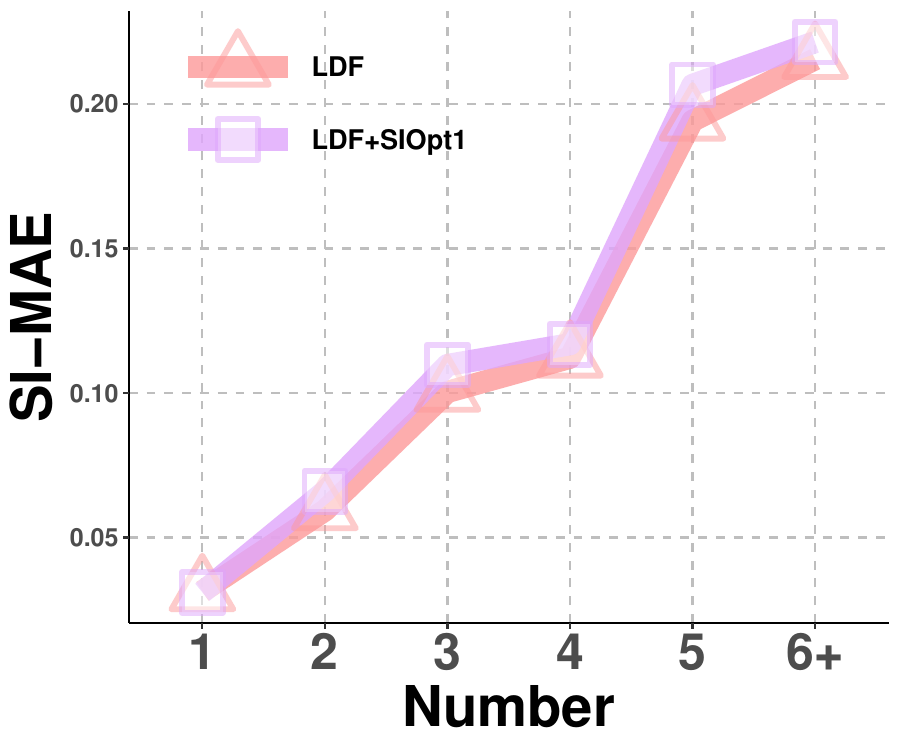}  
\label{fig:LDF_DUTS_num_line}
\end{minipage}
}
\subfigure[ECSSD]{   
\begin{minipage}{0.185\linewidth}
\includegraphics[width=\linewidth]{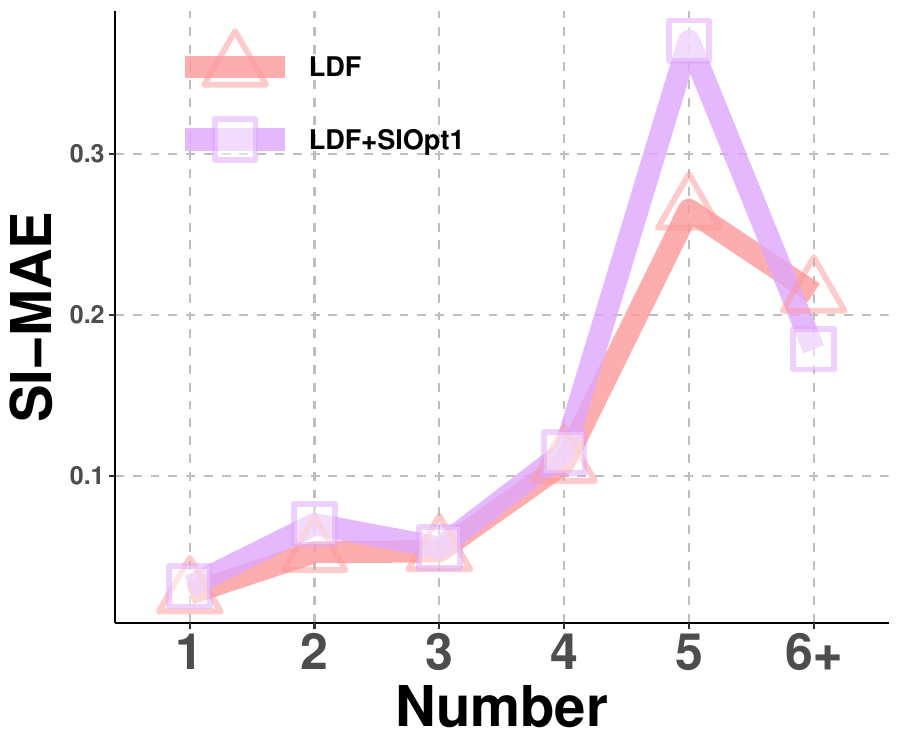}  
\label{fig:LDF_ECSSD_num_line}
\end{minipage}
}
\subfigure[DUT-OMRON]{   
\begin{minipage}{0.185\linewidth}
\includegraphics[width=\linewidth]{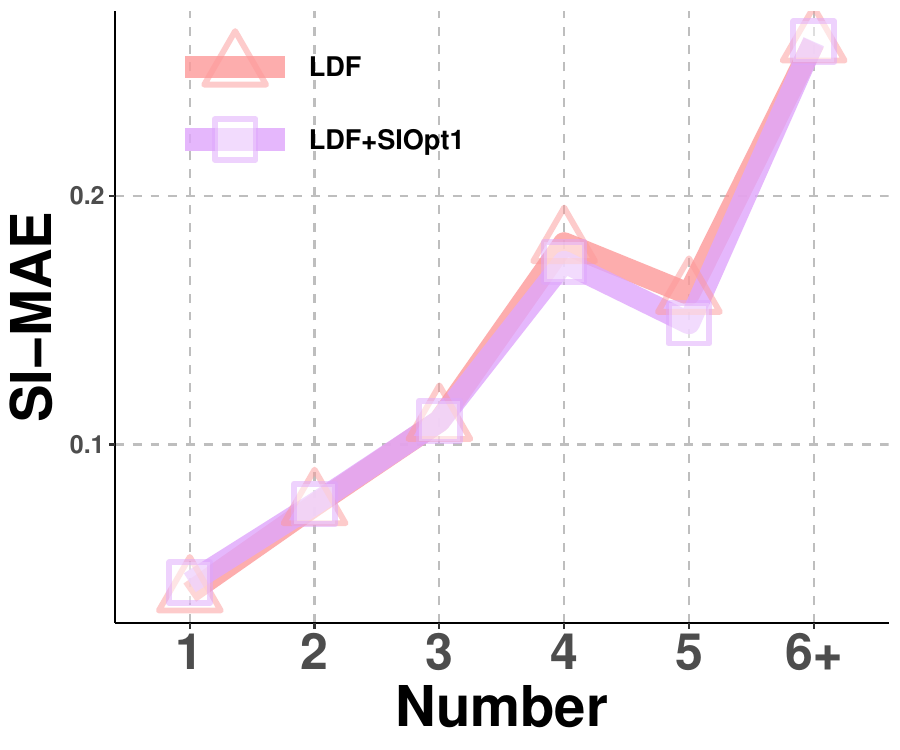}  
\label{fig:LDF_DUT-OMRON_num_line}
\end{minipage}
}
\subfigure[HKU-IS]{   
\begin{minipage}{0.185\linewidth}
\includegraphics[width=\linewidth]{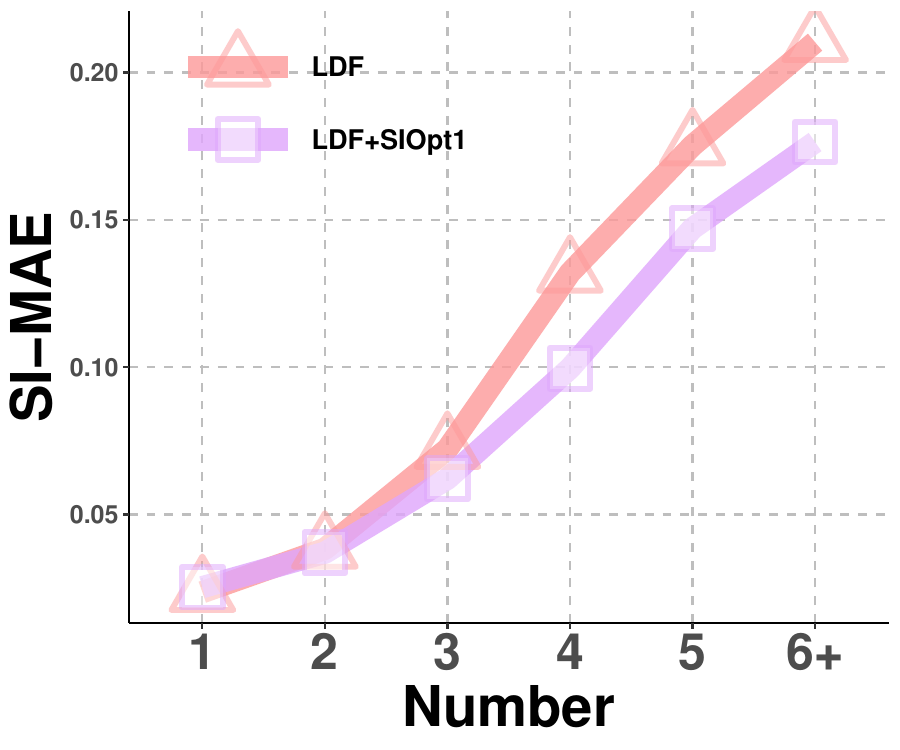}  
\label{fig:LDF_HKU-IS_num_line}
\end{minipage}
}
\caption{Fine-grained performance comparisons under different object numbers, with LDF as the backbone.}    
\label{fig:fine-analysis-LDF-num}    
\end{figure*}

\begin{figure*}[ht]
    \centering
    \subfigure[MSOD]{   
    \begin{minipage}{0.185\linewidth}
    \includegraphics[width=\linewidth]{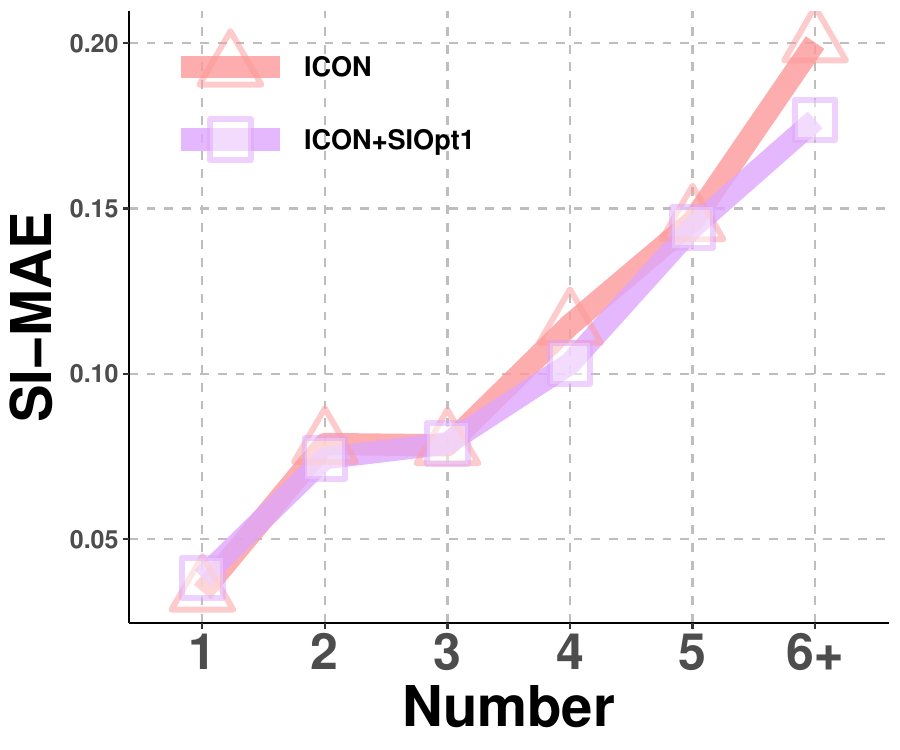}  
    \label{fig:ICON_msod_num_line}
    \end{minipage}
    }
    \subfigure[DUTS]{   
    \begin{minipage}{0.185\linewidth}
    \includegraphics[width=\linewidth]{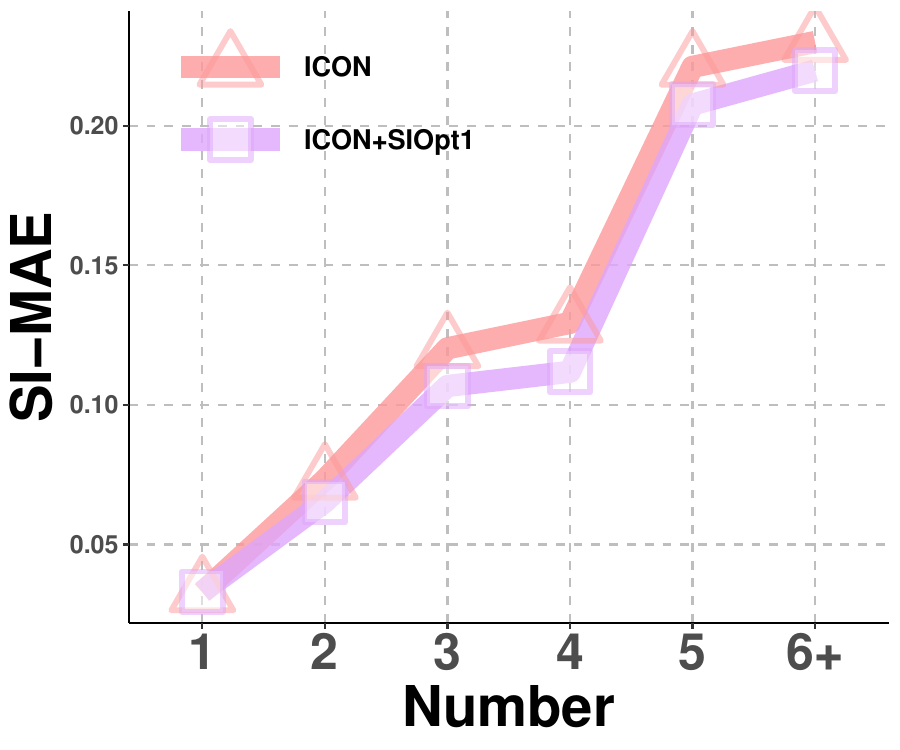}  
    \label{fig:ICON_DUTS_num_line}
    \end{minipage}
    }
    \subfigure[ECSSD]{   
    \begin{minipage}{0.185\linewidth}
    \includegraphics[width=\linewidth]{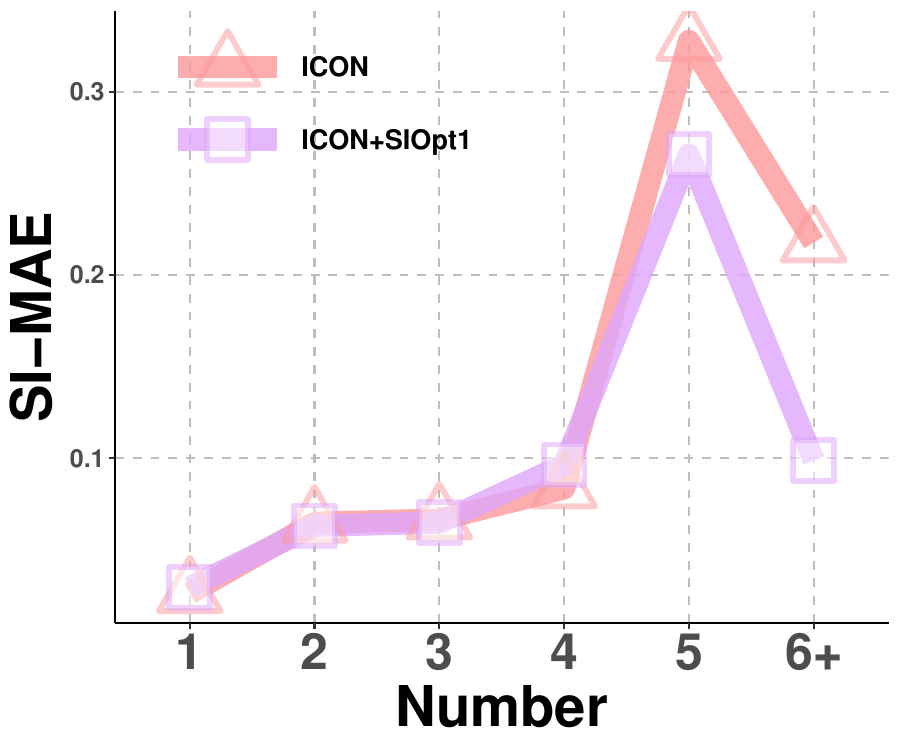}  
    \label{fig:ICON_ECSSD_num_line}
    \end{minipage}
    }
    \subfigure[DUT-OMRON]{   
    \begin{minipage}{0.185\linewidth}
    \includegraphics[width=\linewidth]{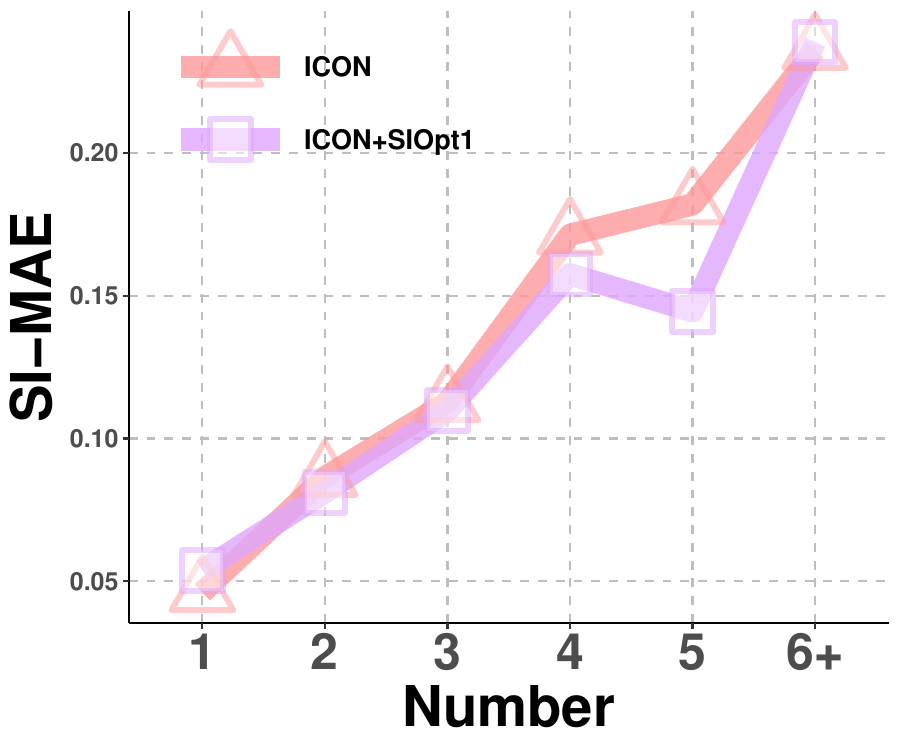}  
    \label{fig:ICON_DUT-OMRON_num_line}
    \end{minipage}
    }
    \subfigure[HKU-IS]{   
    \begin{minipage}{0.185\linewidth}
    \includegraphics[width=\linewidth]{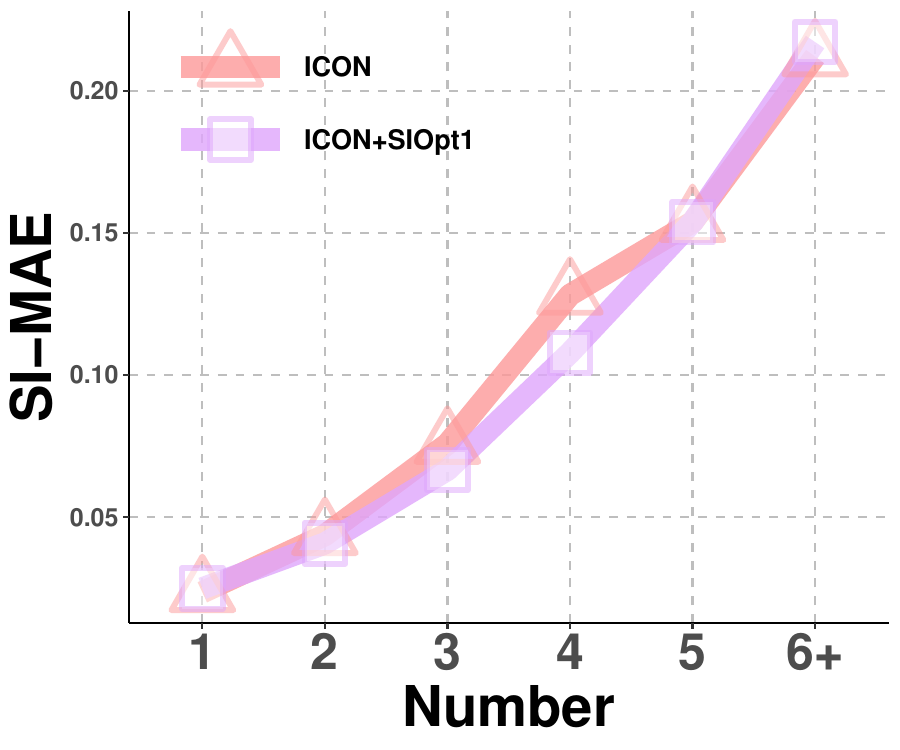}  
    \label{fig:ICON_HKU-IS_num_line}
    \end{minipage}
    }
    \caption{Fine-grained performance comparisons under different object numbers, with ICON as the backbone.}    
    \label{fig:fine-analysis-ICON-num}    
    \end{figure*}

    \begin{figure*}[!t]
        \centering
        \subfigure[MSOD]{   
        \begin{minipage}{0.185\linewidth}
        \includegraphics[width=\linewidth]{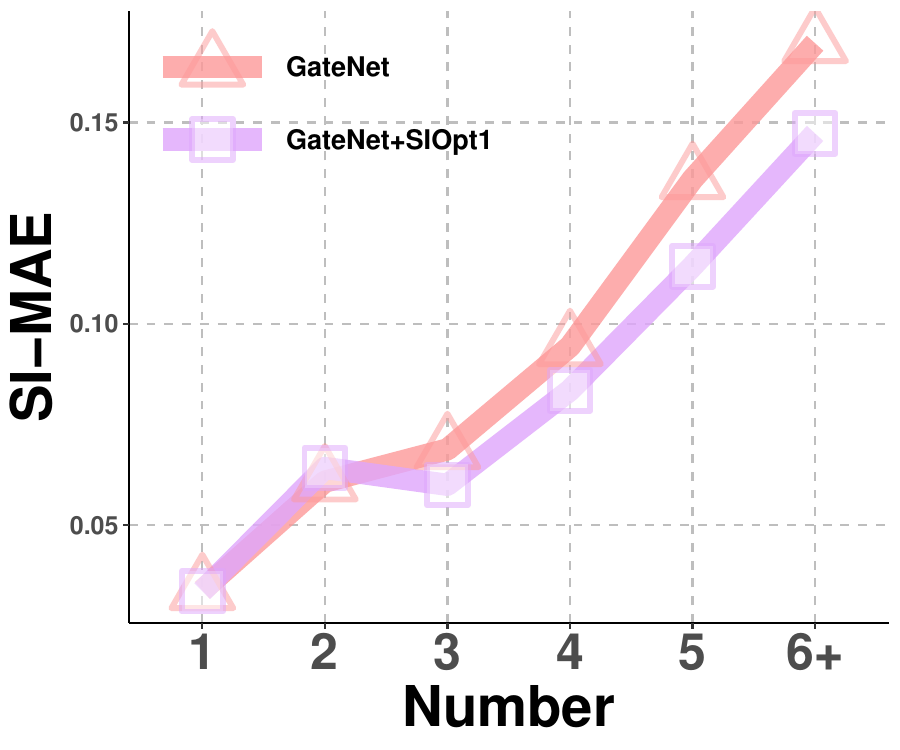}  
        \label{fig:GateNet_msod_num_line}
        \end{minipage}
        }
        \subfigure[DUTS]{   
        \begin{minipage}{0.185\linewidth}
        \includegraphics[width=\linewidth]{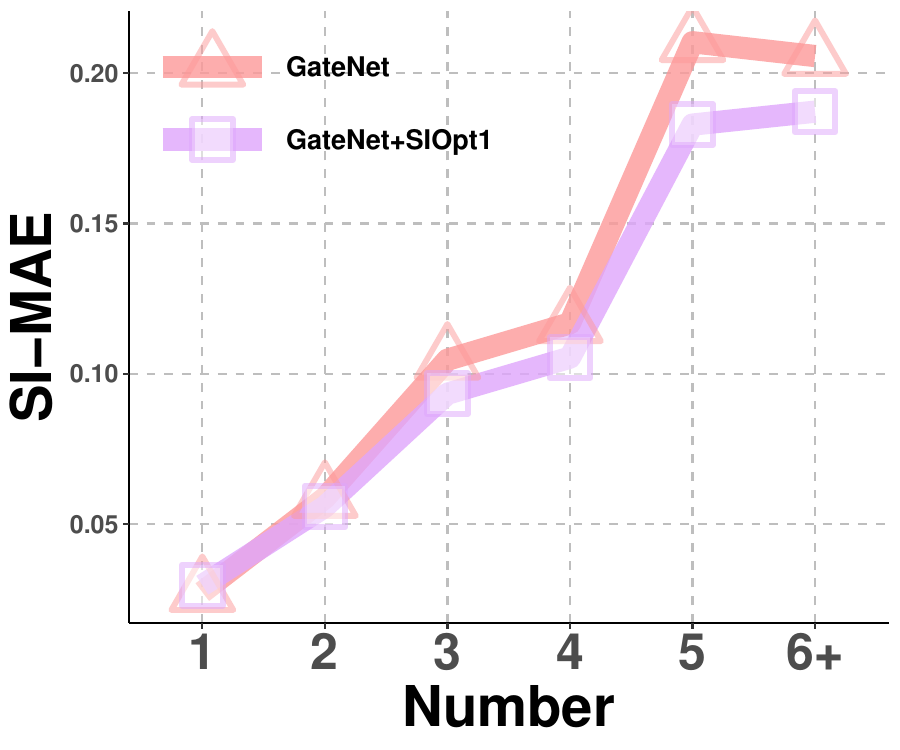}  
        \label{fig:GateNet_DUTS_num_line}
        \end{minipage}
        }
        \subfigure[ECSSD]{   
        \begin{minipage}{0.185\linewidth}
        \includegraphics[width=\linewidth]{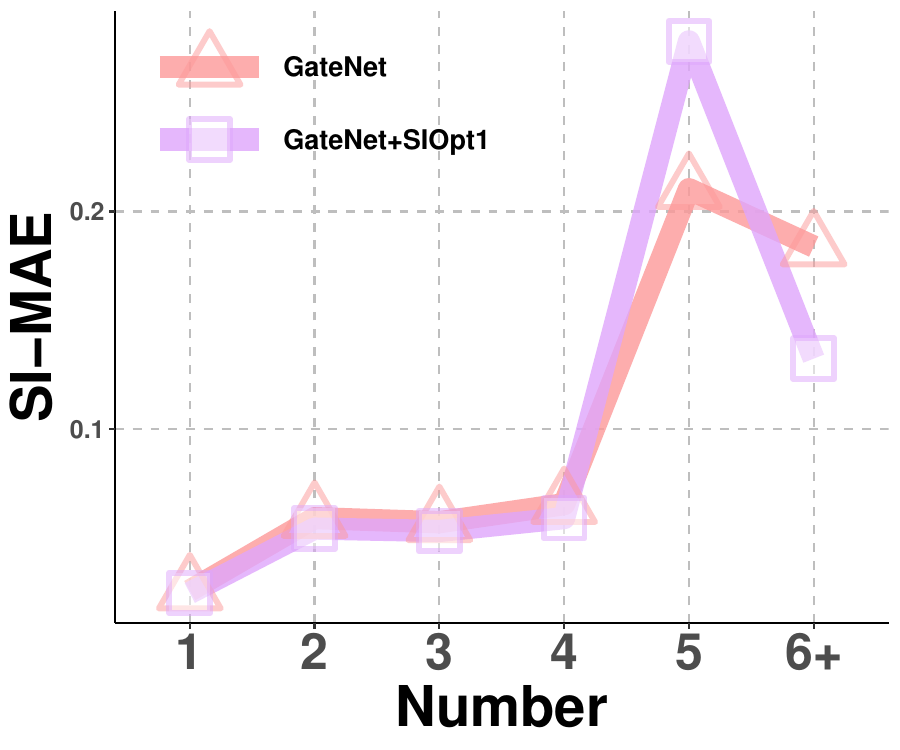}  
        \label{fig:GateNet_ECSSD_num_line}
        \end{minipage}
        }
        \subfigure[DUT-OMRON]{   
        \begin{minipage}{0.185\linewidth}
        \includegraphics[width=\linewidth]{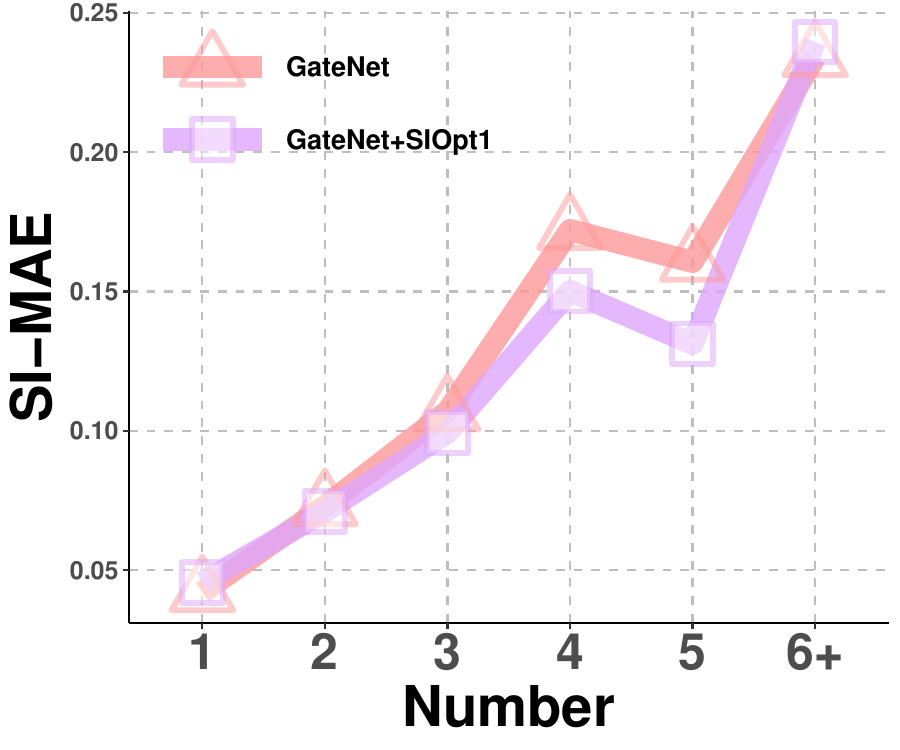}  
        \label{fig:GateNet_DUT-OMRON_num_line}
        \end{minipage}
        }
        \subfigure[HKU-IS]{   
        \begin{minipage}{0.185\linewidth}
        \includegraphics[width=\linewidth]{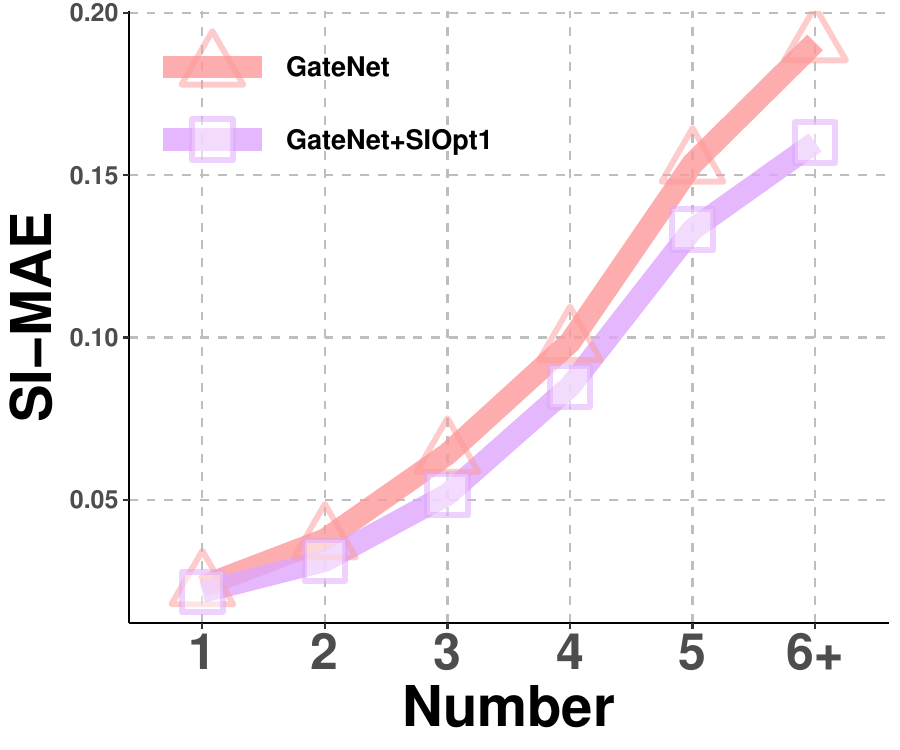}  
        \label{fig:GateNet_HKU-IS_num_line}
        \end{minipage}
        }
        \caption{Fine-grained performance comparisons under different object numbers, with GateNet as the backbone.}    
        \label{fig:fine-analysis-GateNet-num}    
        \end{figure*}

\begin{figure*}[!t]
    \centering
    \subfigure[MSOD]{   
    \begin{minipage}{0.185\linewidth}
    \includegraphics[width=\linewidth]{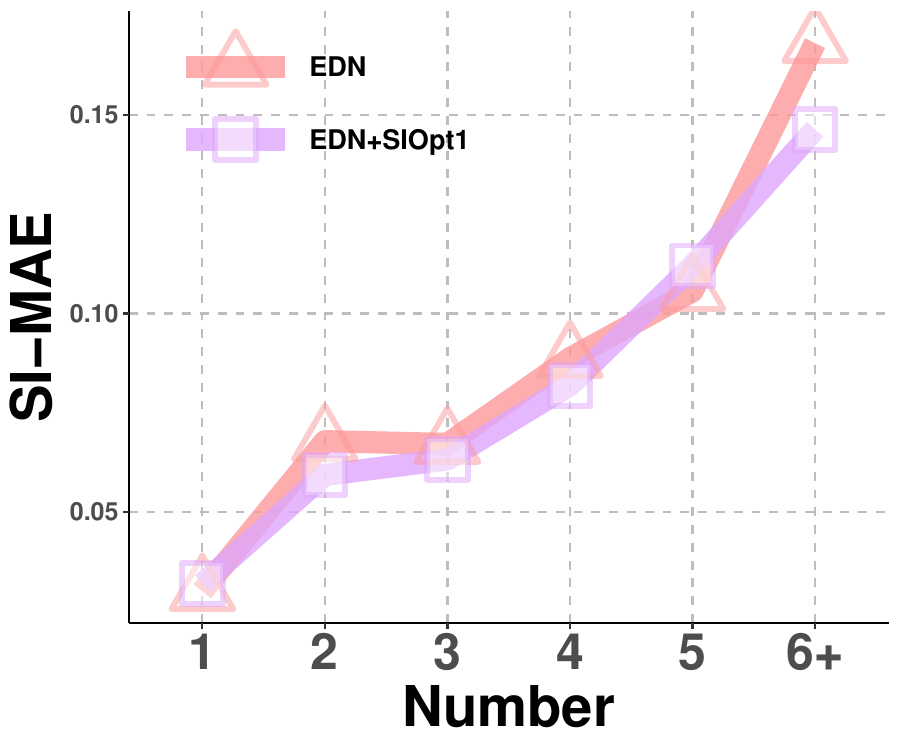}  
    \label{fig:EDN_msod_num_line}
    \end{minipage}
    }
    \subfigure[DUTS]{   
    \begin{minipage}{0.185\linewidth}
    \includegraphics[width=\linewidth]{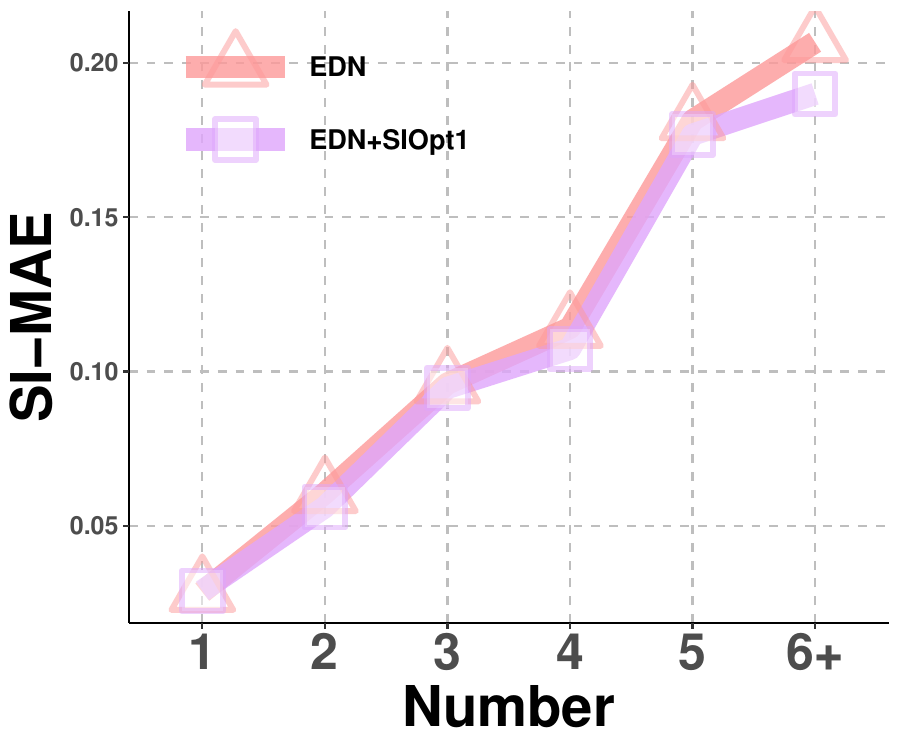}  
    \label{fig:EDN_DUTS_num_line}
    \end{minipage}
    }
    \subfigure[ECSSD]{   
    \begin{minipage}{0.185\linewidth}
    \includegraphics[width=\linewidth]{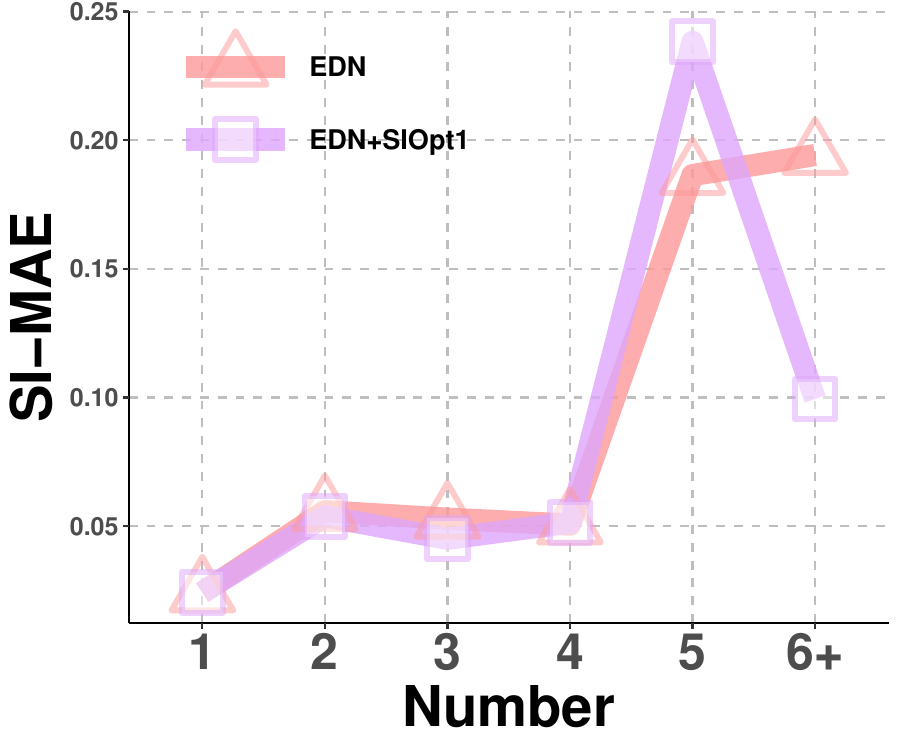}  
    \label{fig:EDN_ECSSD_num_line}
    \end{minipage}
    }
    \subfigure[DUT-OMRON]{   
    \begin{minipage}{0.185\linewidth}
    \includegraphics[width=\linewidth]{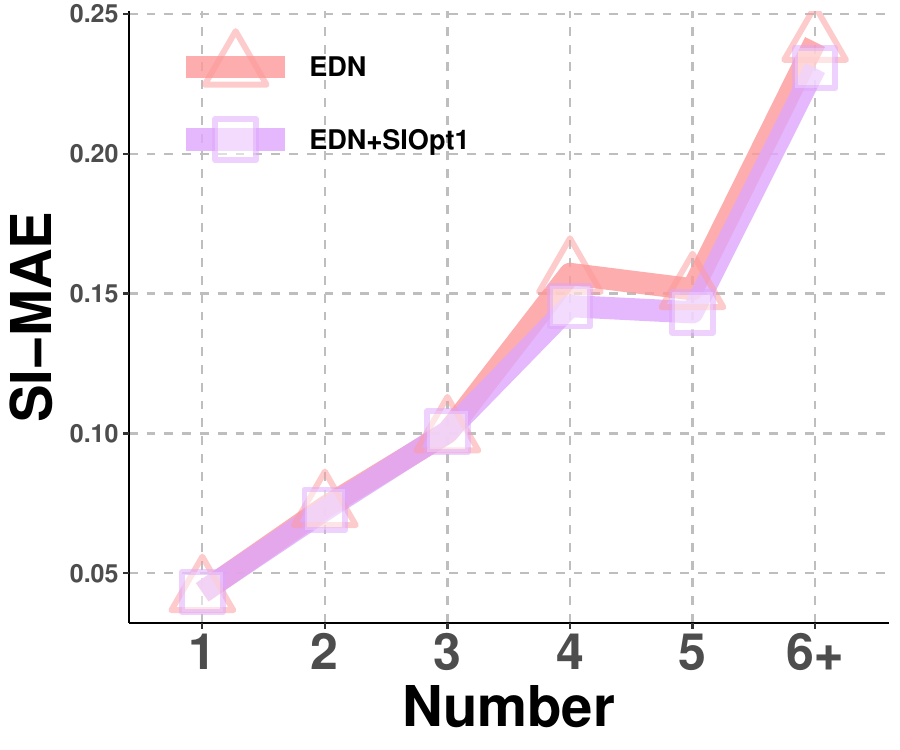}  
    \label{fig:EDN_DUT-OMRON_num_line}
    \end{minipage}
    }
    \subfigure[HKU-IS]{   
    \begin{minipage}{0.185\linewidth}
    \includegraphics[width=\linewidth]{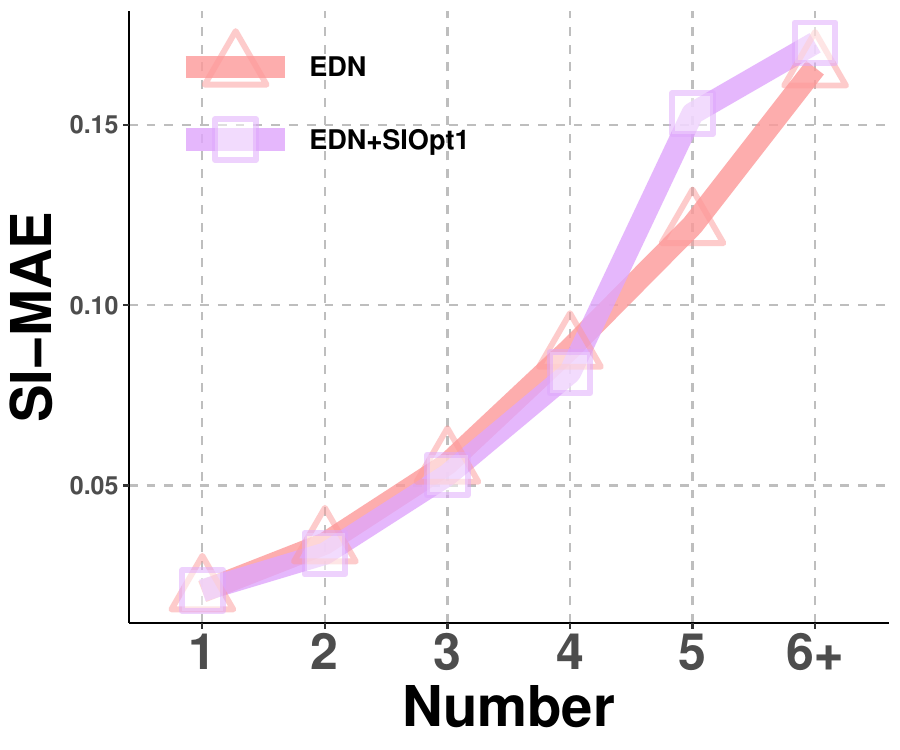}  
    \label{fig:EDN_HKU-IS_num_line}
    \end{minipage}
    }
    \caption{Fine-grained performance comparisons under different object numbers, with EDN as the backbone.}    
    \label{fig:fine-analysis-EDN-num}    
    \end{figure*}

    \begin{figure*}[!t]
        \centering
        \subfigure[DUTS]{   
        \begin{minipage}{0.21\linewidth}
        \includegraphics[width=\linewidth]{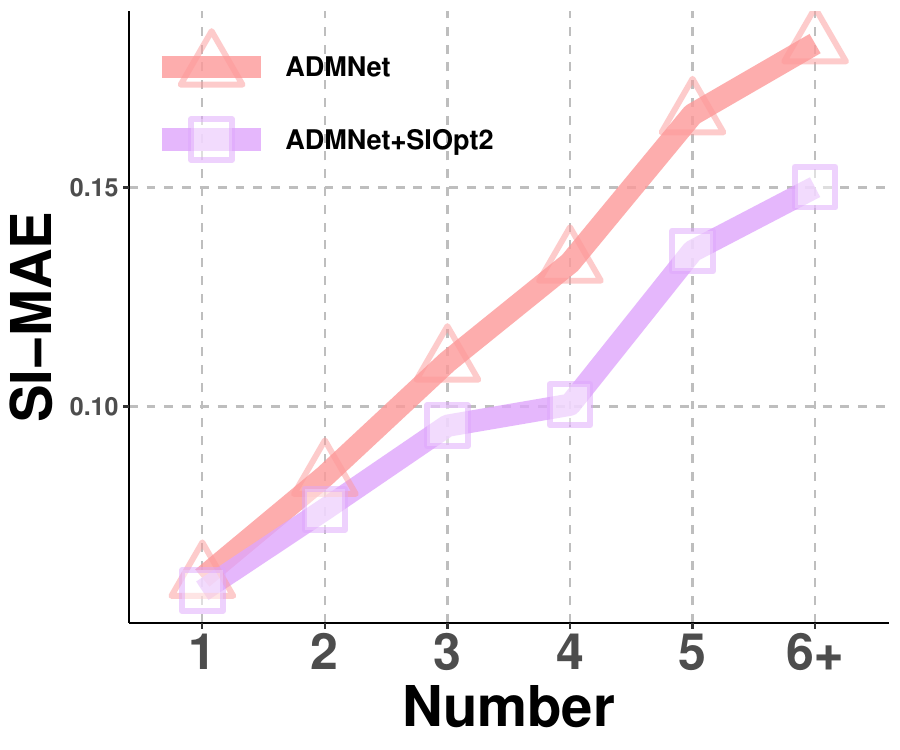}  
        \end{minipage}
        }
        \subfigure[DUTS]{   
        \begin{minipage}{0.21\linewidth}
        \includegraphics[width=\linewidth]{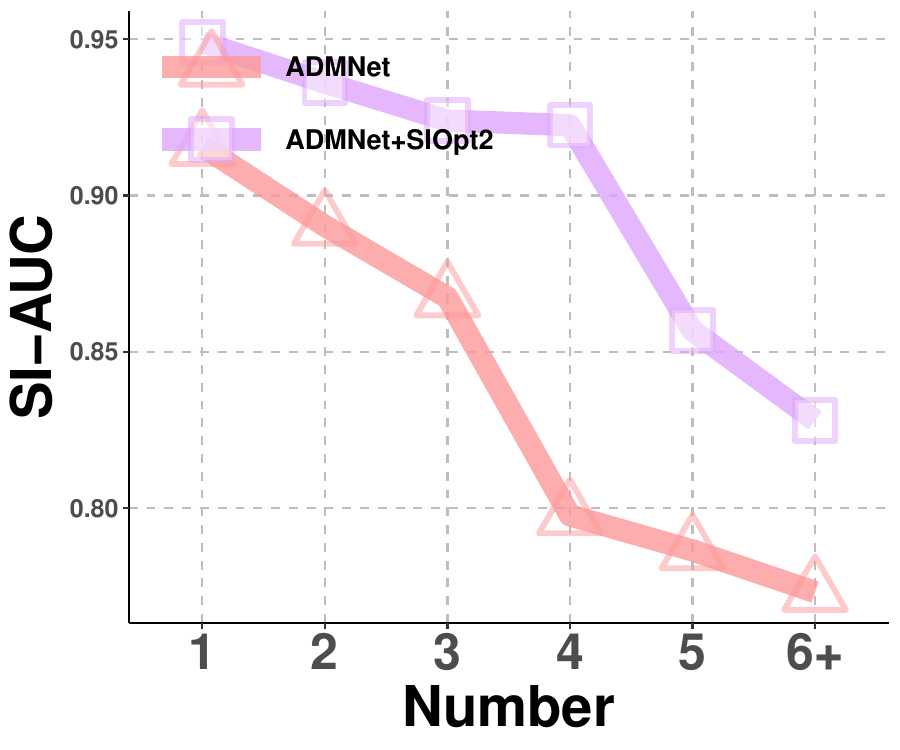}  
        \end{minipage}
        }
        \subfigure[DUTS]{   
        \begin{minipage}{0.21\linewidth}
        \includegraphics[width=\linewidth]{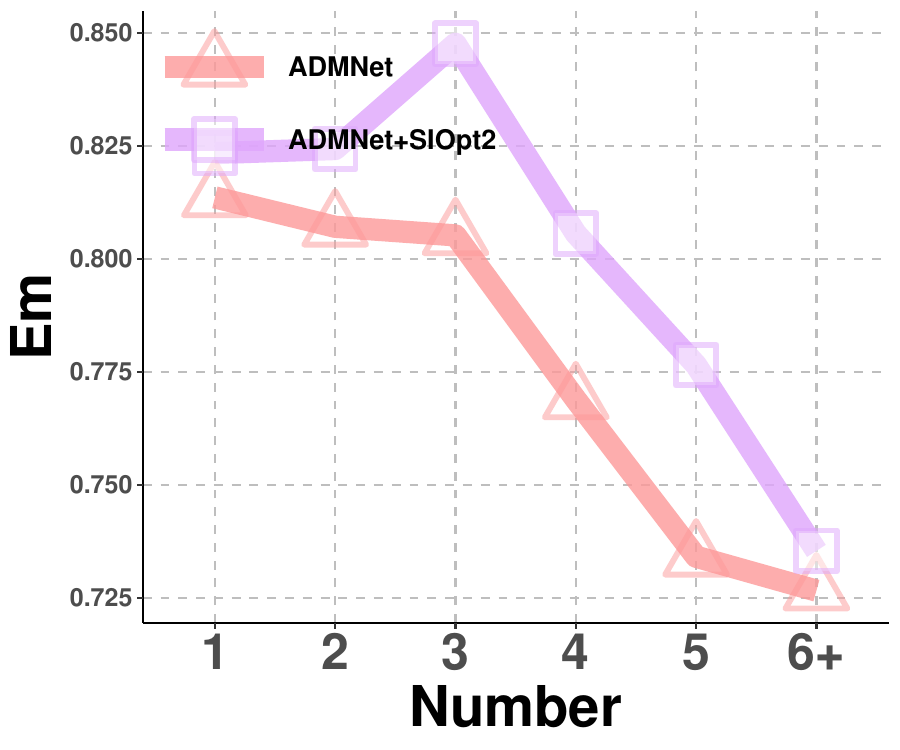}  
        \end{minipage}
        }
        \subfigure[DUTS]{   
        \begin{minipage}{0.21\linewidth}
        \includegraphics[width=\linewidth]{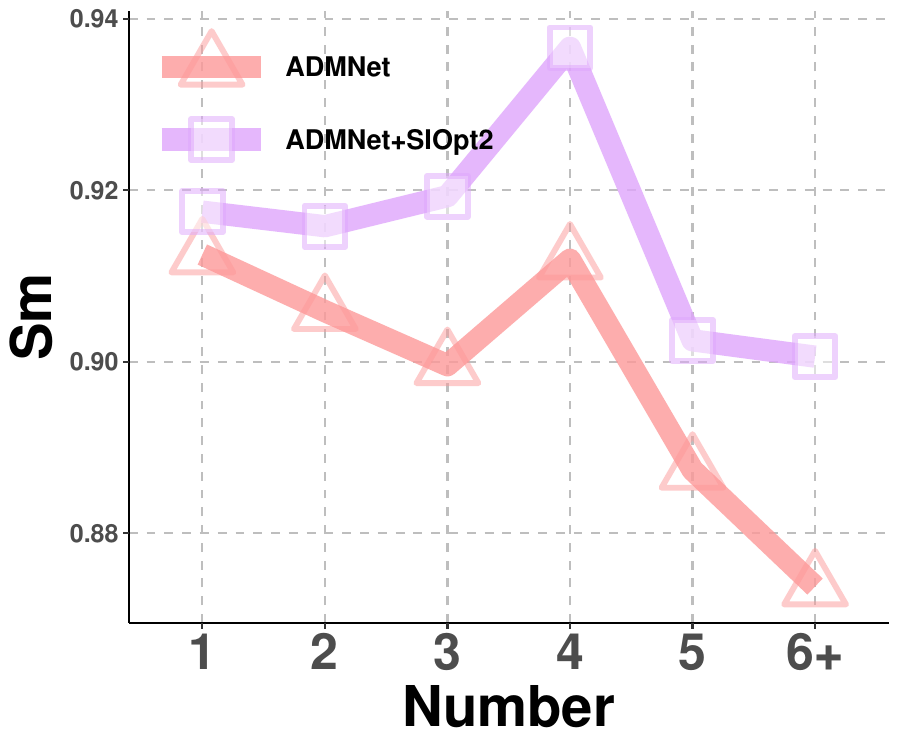}  
        \end{minipage}
        } \\
        \subfigure[MSOD]{   
        \begin{minipage}{0.21\linewidth}
        \includegraphics[width=\linewidth]{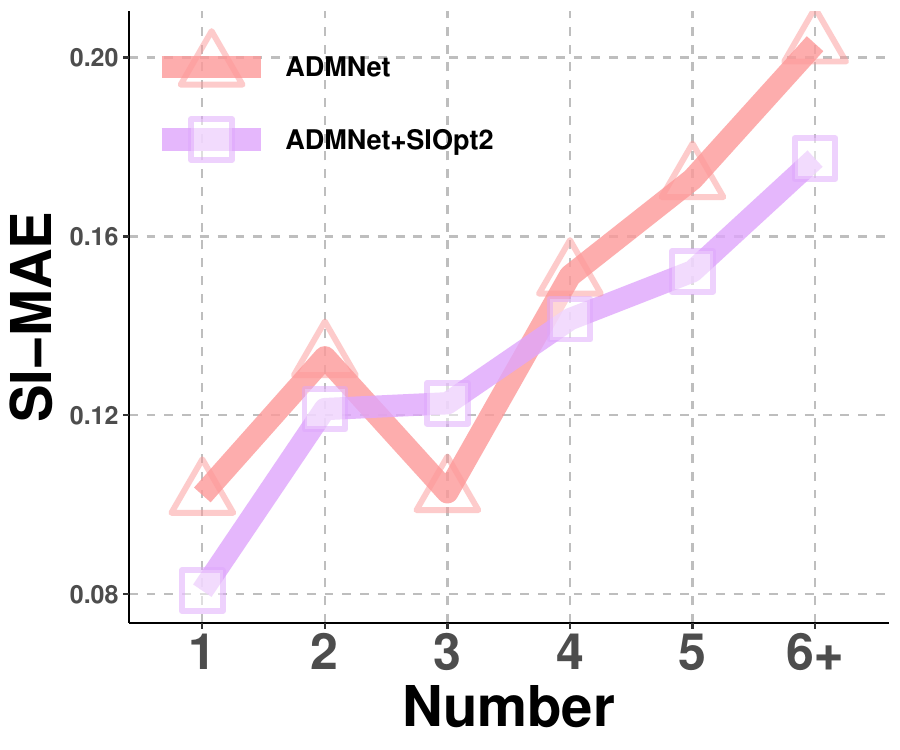}  
        \end{minipage}
        }
        \subfigure[MSOD]{   
        \begin{minipage}{0.21\linewidth}
        \includegraphics[width=\linewidth]{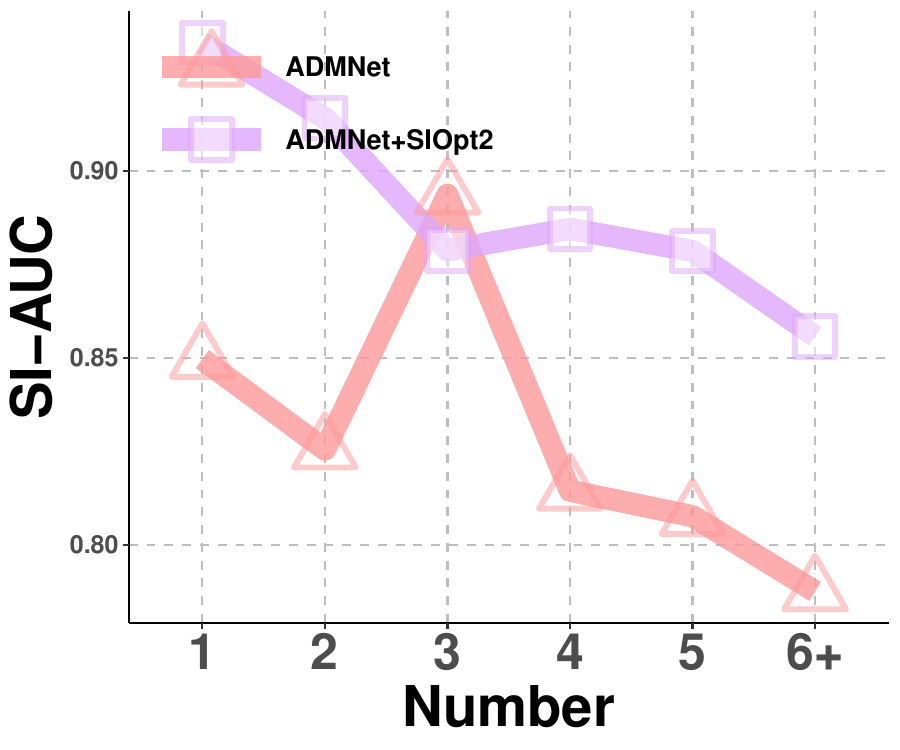}  
        \end{minipage}
        }
        \subfigure[MSOD]{   
        \begin{minipage}{0.21\linewidth}
        \includegraphics[width=\linewidth]{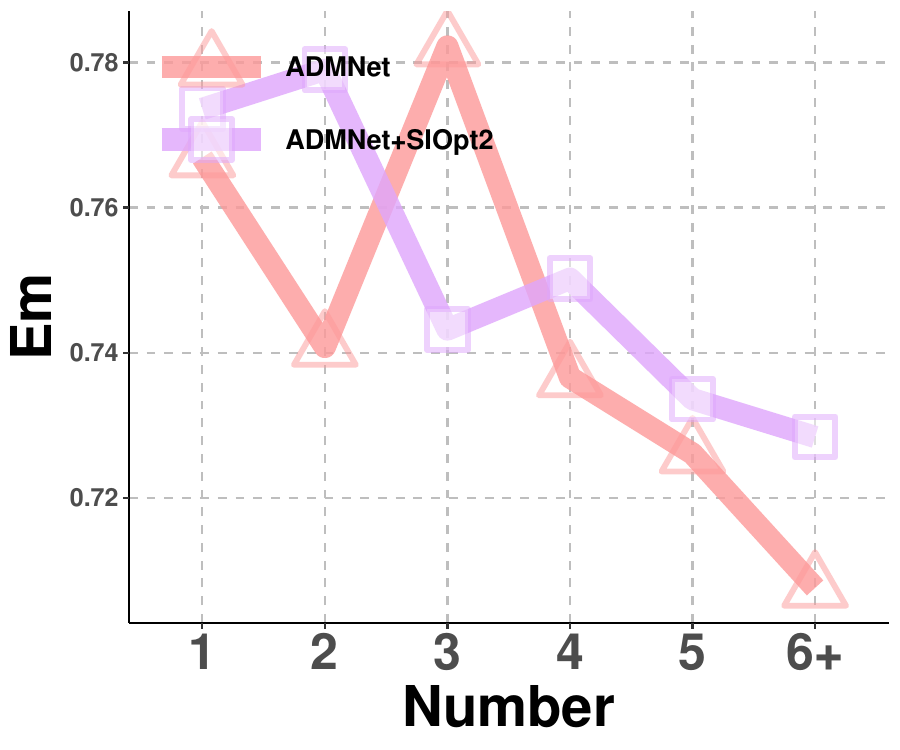}  
        \end{minipage}
        }
        \subfigure[MSOD]{   
        \begin{minipage}{0.21\linewidth}
        \includegraphics[width=\linewidth]{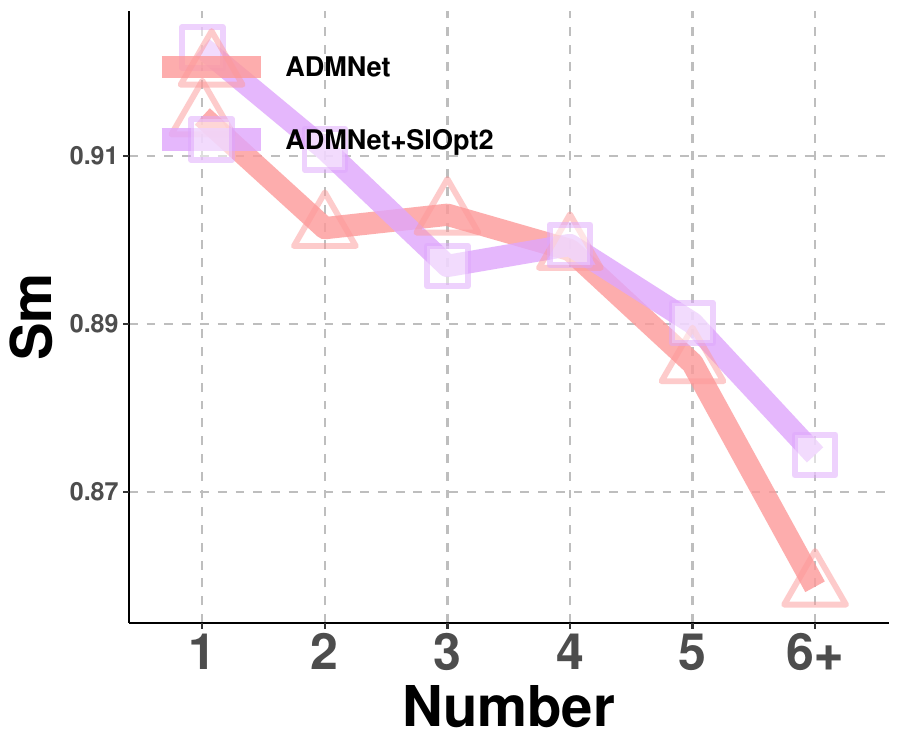}  
        \end{minipage}
        }
        \caption{Fine-grained performance comparisons under different object numbers, with ADMNet as the backbone.}    
        \label{fig:fine-analysis-ADMNet-num}    
\end{figure*}

\begin{figure*}[!t]
    \centering
    \subfigure[DUTS]{   
    \begin{minipage}{0.21\linewidth}
    \includegraphics[width=\linewidth]{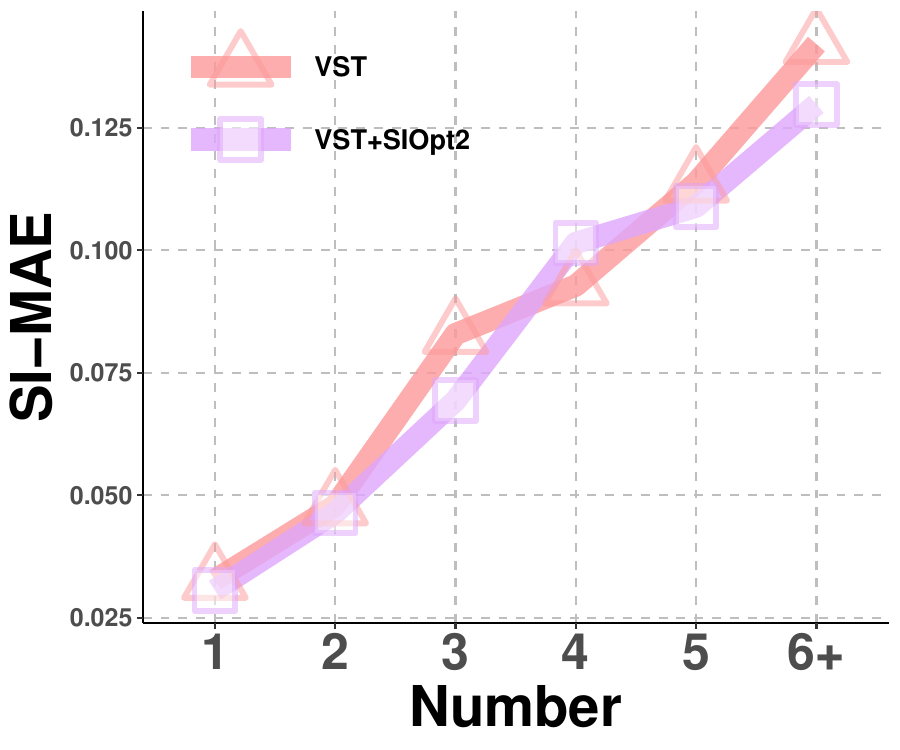}  
    \end{minipage}
    }
    \subfigure[DUTS]{   
    \begin{minipage}{0.21\linewidth}
    \includegraphics[width=\linewidth]{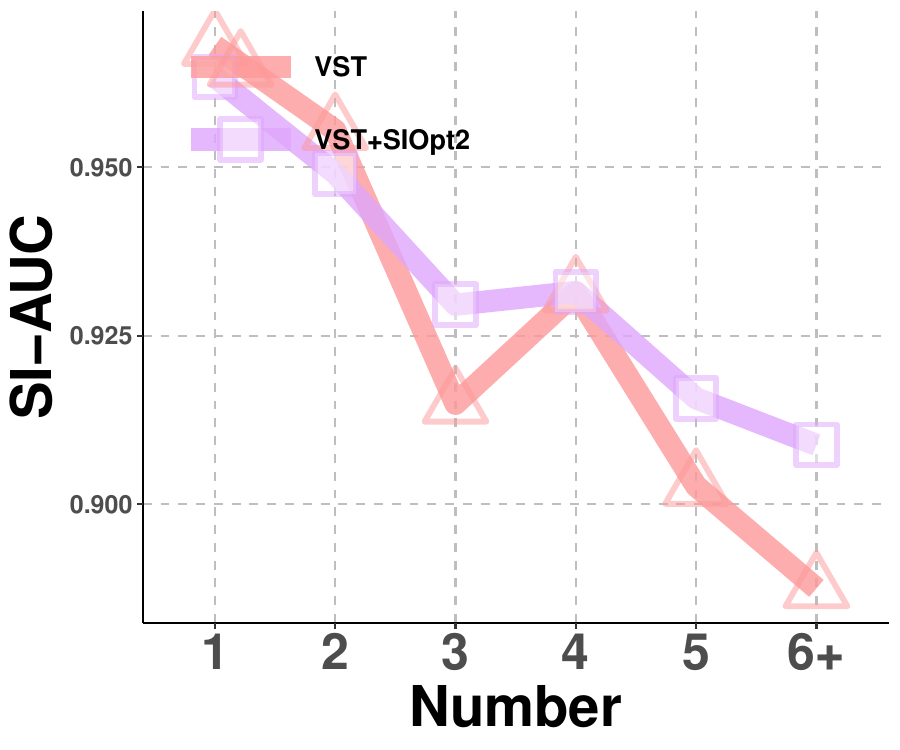}  
    \end{minipage}
    }
    \subfigure[DUTS]{   
    \begin{minipage}{0.21\linewidth}
    \includegraphics[width=\linewidth]{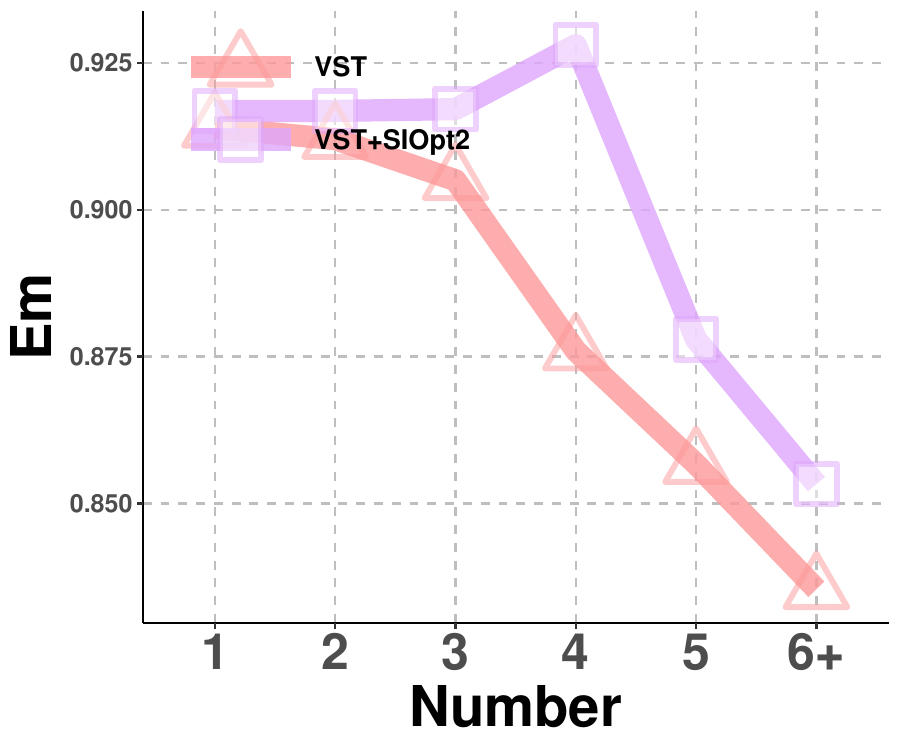}  
    \end{minipage}
    }
    \subfigure[DUTS]{   
    \begin{minipage}{0.21\linewidth}
    \includegraphics[width=\linewidth]{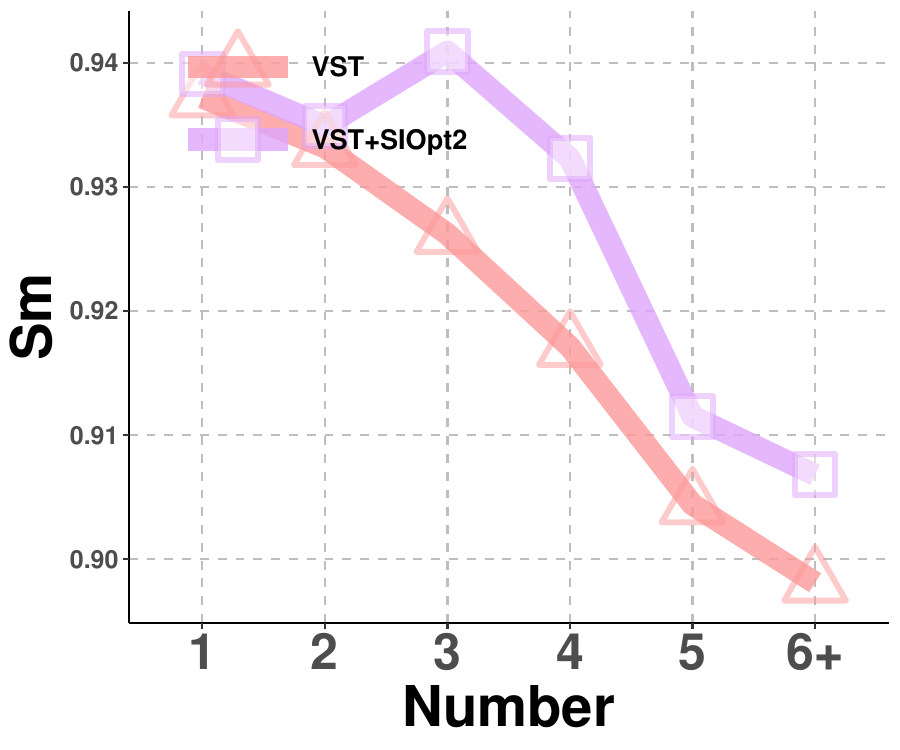}  
    \end{minipage}
    } \\
    \subfigure[MSOD]{   
    \begin{minipage}{0.21\linewidth}
    \includegraphics[width=\linewidth]{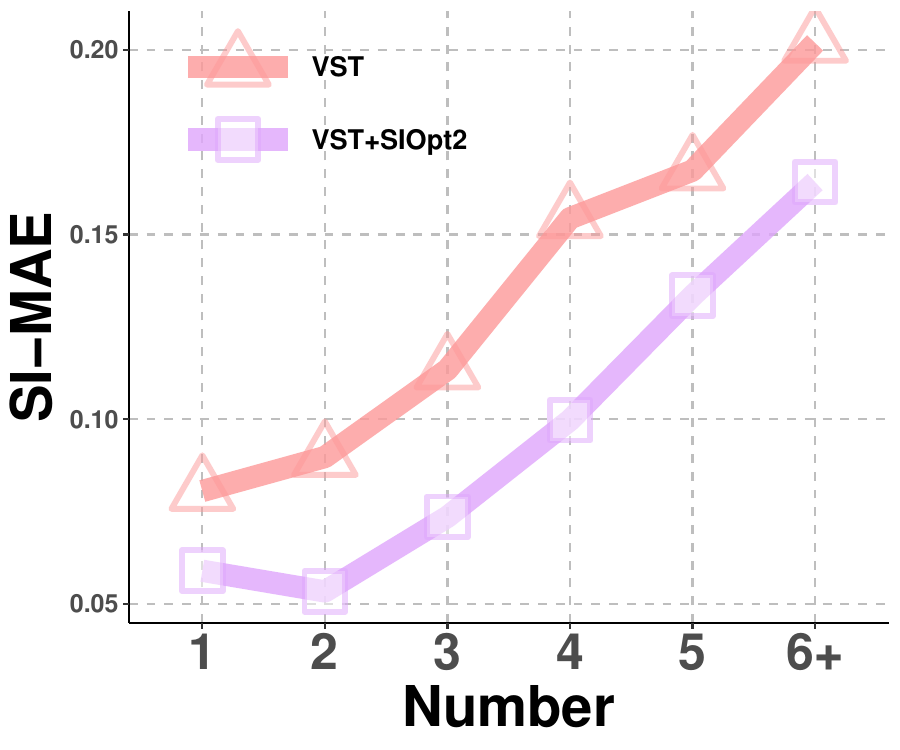}  
    \end{minipage}
    }
    \subfigure[MSOD]{   
    \begin{minipage}{0.21\linewidth}
    \includegraphics[width=\linewidth]{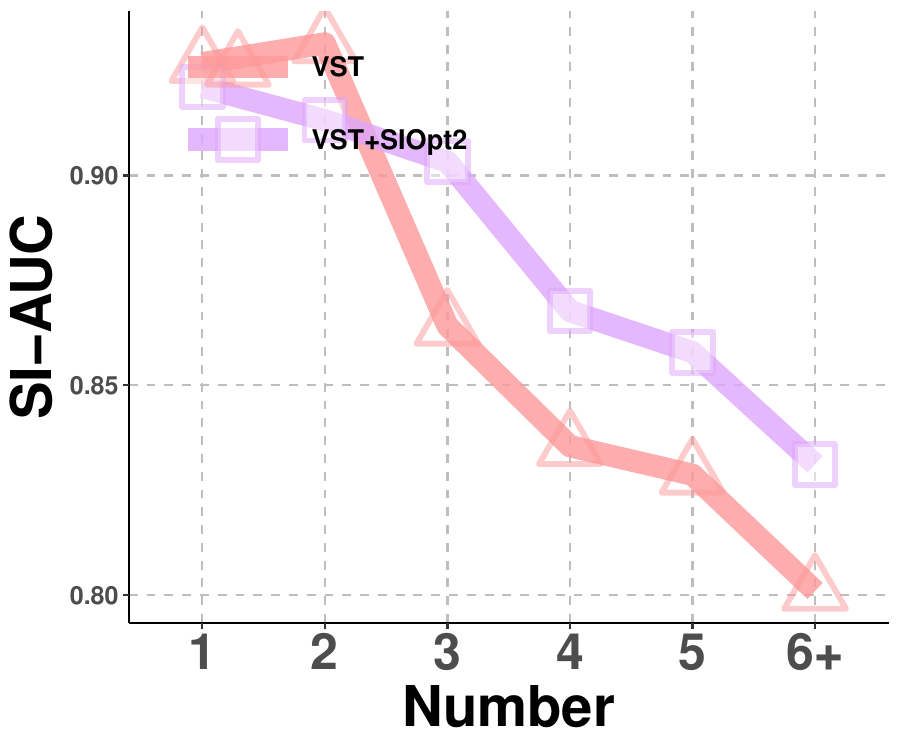}  
    \end{minipage}
    }
    \subfigure[MSOD]{   
    \begin{minipage}{0.21\linewidth}
    \includegraphics[width=\linewidth]{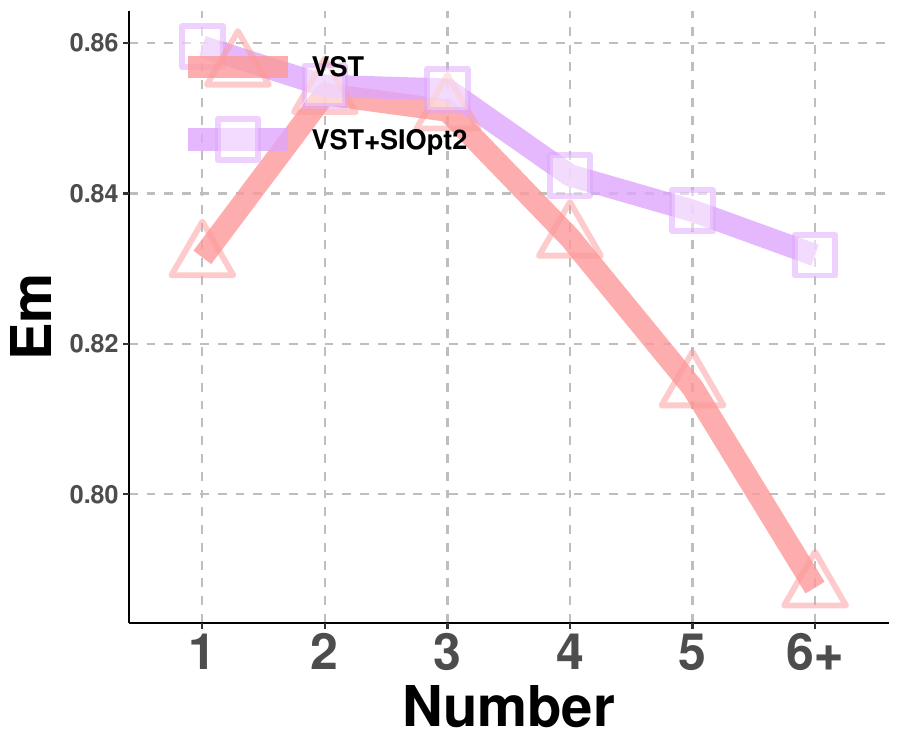}  
    \end{minipage}
    }
    \subfigure[MSOD]{   
    \begin{minipage}{0.21\linewidth}
    \includegraphics[width=\linewidth]{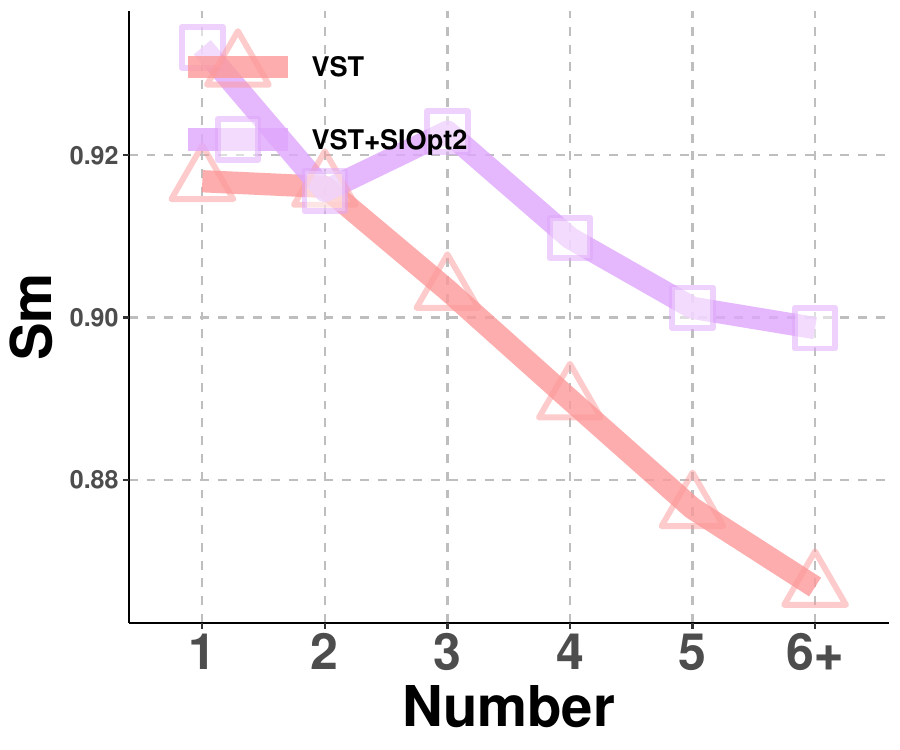}  
    \end{minipage}
    }
    \caption{Fine-grained performance comparisons under different object numbers, with VST as the backbone.}    
    \label{fig:fine-analysis-VST-num}    
\end{figure*}

\subsection{Ablation Studies of $\alpha_{SI}$} \label{ablation_appendix}
\noindent\textbf{Setups.} We conduct ablation studies on the hyperparameter $\alpha_{SI}$ to assess its impact in Eq.~\ref{eq:loss}. As previously described, $\alpha_{SI}$ controls the trade-off between the contributions of the background frame and the foreground frames in the classification-aware loss (e.g., $\BCE$). By default, we set $\alpha_{SI} = \frac{S^{back}_{M+1}}{\sum_{k=1}^{M} S_k^{fore}}$, which proportionally weights the background according to the total saliency of the foreground frames. In this ablation, we consider three variants of $\alpha_{SI}$ using EDN as the representative backbone:  

\textbf{(1)} $\alpha_{SI} = 0$, where the background frame $\boldsymbol{X}\setminus\mathcal{C}(\boldsymbol{X})$ in Eq.~\ref{20250418:eq100} is completely ignored, and only the foreground frames are used for optimization;

\textbf{(2)} $\alpha_{SI} = 1$, where the background frame is treated equally with each foreground frame; and  

\textbf{(3)} the default setting, which adaptively scales the background based on saliency proportions. 

All other experimental settings follow \cref{details_appendix}.

\noindent\textbf{Results.} The results are presented in \cref{fig:ablation_SIMAE}, \cref{tab:ablation_alpha_msod}, \cref{tab:ablation_alpha_DUTS} and \cref{tab:ablation_alpha_ECSSD}, respectively. We can see that the default of $\alpha_{SI}$ consistently delivers balanced and often superior performance across a wide range of evaluation metrics. This supports our arguments that proportionally weighting the background helps harmonize the learning signals from different regional contexts. Interestingly, while $\alpha_{SI} = 0$ sometimes yields strong results on $\F$-based metrics, it significantly underperforms on $\MAE$, $\SMAE$, and $\E_m$. The reason lies in that $\F$-based metrics emphasize the salient regions and are less sensitive to false positives in the background. In contrast, metrics like $\MAE$-based and $\E_m$ provide a more holistic assessment by accounting for both salient and non-salient areas. As a result, ignoring the background entirely (i.e., $\alpha_{SI}=0$) causes the model to overfit to the foreground regions, achieving high precision locally but poor accuracy globally. On the other hand, $\alpha_{SI} = 1$ shows better performance compared to $\alpha_{SI} = 0$, but still falls short of our adaptive strategy due to the slight penalty on the error within the background frame. These findings collectively underscore the effectiveness of our proposed $\alpha_{SI}$ formulation in achieving a well-balanced optimization objective across foreground and background regions.

\begin{table*}[!h]
    \centering
    \caption{Ablation on the parameter $\alpha_{SI}$ on MSOD (300 images). The best results are marked in bold.}
    \scalebox{0.80}{
    \begin{tabular}{c|c|ccccccc}
    \toprule
        Methods & $\alpha_{SI}$ & $\MAE \downarrow $ & $\SMAE \downarrow$ &  $\F_m^\beta \uparrow$ & $\SF_m^{\beta} \uparrow $ & $\F_{max}\uparrow$ & $\SF_{max} \uparrow$ & $\E_m \uparrow$  \\
    \midrule
        EDN (A) & ResNet50 & 0.0467 & 0.0788  & 0.7925 & 0.7635 & 0.8410 & 0.8321 & 0.8712 \\ 
        \textbf{+ $\mathsf{SIOpt1}$} (B) & 0 & 0.2340 & 0.2000  & 0.4790	& \textbf{0.8547} & 0.7695 & \textbf{0.9136} & 0.6199 \\
        \textbf{+ $\mathsf{SIOpt1}$} (C) & 1 & 0.0544 & 0.0812 & 0.7893	& 0.7850 & 0.8502 & 0.8912 & 0.8825 \\ 
        \textbf{+ $\mathsf{SIOpt1}$} (D) & $\frac{S^{fore}}{S^{back}}$ & \textbf{0.0453} & \textbf{0.0724}  & \textbf{0.8057}	& 0.7990	& \textbf{0.8555} & 0.8619 & \textbf{0.8936} \\ 
    \bottomrule
    \end{tabular}
   }
    \label{tab:ablation_alpha_msod}
\end{table*}
\begin{table*}[!h]
    \centering
    \caption{Ablation on the parameter $\alpha_{SI}$ on DUTS-TE (5,019 images). The best results are marked in bold.}
    \scalebox{0.85}{
    \begin{tabular}{c|c|ccccccc}
    \toprule
        Methods & $\alpha_{SI}$ & $\MAE \downarrow $ & $\SMAE \downarrow$  & $\F_m^\beta \uparrow$ & $\SF_m^{\beta} \uparrow $ & $\F_{max}\uparrow$ & $\SF_{max} \uparrow$ & $\E_m \uparrow$  \\
    \midrule
        EDN (A) & ResNet50 & \textbf{0.0389} & 0.0388  & \textbf{0.8288} & 0.8565 & 0.8752 & 0.9017 & 0.9033 \\
        \textbf{+ $\mathsf{SIOpt1}$} (B) & 0 & 0.2318 & 0.1975  & 0.4621 & \textbf{0.8730} & 0.7705 & \textbf{0.9235} & 0.6069 \\
        \textbf{+ $\mathsf{SIOpt1}$} (C) & 1 & 0.0489 & 0.0460  & 0.8146 & 0.8585 & 0.8807 & 0.9182 & 0.8954 \\
        \textbf{+ $\mathsf{SIOpt1}$} (D) & $\frac{S^{fore}}{S^{back}}$ & 0.0392 & \textbf{0.0381}  & 0.8260 & 0.8672 & \textbf{0.8765} & 0.9119 & \textbf{0.9072} \\ 
    \bottomrule
    \end{tabular}
   }
    \label{tab:ablation_alpha_DUTS}
\end{table*}
\begin{table*}[!h]
    \centering
    \caption{Ablation on the parameter $\alpha_{SI}$ on ECSSD (1,000 images). The best results are marked in bold.}
    \scalebox{0.85}{
    \begin{tabular}{c|c|cccccccc}    
    \toprule
        Methods & $\alpha_{SI}$ & $\MAE \downarrow $ & $\SMAE \downarrow$  & $\F_m^\beta \uparrow$ & $\SF_m^{\beta} \uparrow $ & $\F_{max}\uparrow$ & $\SF_{max} \uparrow$ & $\E_m \uparrow$  \\
    \midrule
        EDN (A) & ResNet50 & 0.0363 & 0.0271  & \textbf{0.9089} & 0.9147 & \textbf{0.9531} & 0.9560 & 0.9338 \\ 
        \textbf{+ $\mathsf{SIOpt1}$} (B) & 0 & 0.1656 & 0.1282  & 0.6557 & \textbf{0.9236} & 0.9043 & \textbf{0.9587} & 0.7431 \\
        \textbf{+ $\mathsf{SIOpt1}$} (C) & 1 & 0.0454 & 0.0340 & 0.8986 & 0.9164 & 0.9457 & 0.9556 & 0.9282 \\
        \textbf{+ $\mathsf{SIOpt1}$} (D) & $\frac{S^{fore}}{S^{back}}$ & \textbf{0.0358} & \textbf{0.0269} & 0.9084 & 0.9216 & 0.9456 & 0.9543 & \textbf{0.9375} \\
    \bottomrule
    \end{tabular}
   }
    \label{tab:ablation_alpha_ECSSD}
\end{table*}
\begin{table*}[!h]
    \centering
    \caption{Ablation on the parameter $\alpha_{SI}$ on DUT-OMRON (5,168 images). The best results are marked in bold.}
    \scalebox{0.85}{
    \begin{tabular}{c|c|ccccccc}
    \toprule
        Methods & $\alpha_{SI}$ & $\MAE \downarrow $ & $\SMAE \downarrow$  & $\F_m^\beta \uparrow$ & $\SF_m^{\beta} \uparrow $ & $\F_{max}\uparrow$ & $\SF_{max} \uparrow$ & $\E_m \uparrow$  \\
    \midrule
        EDN (A) & ResNet50 & \textbf{0.0514} & 0.0484  & 0.7529 & 0.8224 & 0.8117 & 0.8798 & 0.8514 \\
        \textbf{+ $\mathsf{SIOpt1}$} (B) & 0 & 0.2693 & 0.2305  & 0.4231 & \textbf{0.8708} & 0.7104 & \textbf{0.9231} & 0.5640 \\
        \textbf{+ $\mathsf{SIOpt1}$} (C) & 1 & 0.0642 & 0.0555 & 0.7442 & 0.8239 & \textbf{0.8190} & 0.9048 & 0.8508 \\
        \textbf{+ $\mathsf{SIOpt1}$} (D) & $\frac{S^{fore}}{S^{back}}$ & 0.0557 & \textbf{0.0483} & \textbf{0.7544} & 0.8381 & 0.8163 & 0.8912 & \textbf{0.8594} \\
    \bottomrule
    \end{tabular}
   }
    \label{tab:ablation_alpha_DUT-OMRON}
\end{table*}
\begin{table*}[!h]
    \centering
    \caption{Ablation on the parameter $\alpha_{SI}$ on HKU-IS(4,447 images). The best results are marked in bold.}
    \scalebox{0.85}{
    \begin{tabular}{c|c|cccccccc}
    \toprule
        Methods & $\alpha_{SI}$ & $\MAE \downarrow $ & $\SMAE \downarrow$ & $\F_m^\beta \uparrow$ & $\SF_m^{\beta} \uparrow $ & $\F_{max}\uparrow$ & $\SF_{max} \uparrow$ & $\E_m \uparrow$  \\
    \midrule
        EDN (A) & ResNet50 & \textbf{0.0279} & 0.0294  & \textbf{0.9004} & 0.9017 & \textbf{0.9417} & 0.9364 & 0.9429 \\
        \textbf{+ $\mathsf{SIOpt1}$} (B) & 0 & 0.1822 & 0.1431  & 0.6027 & \textbf{0.9132} & 0.8921 & \textbf{0.9541} & 0.7100 \\
        \textbf{+ $\mathsf{SIOpt1}$} (C) & 1 & 0.0383 & 0.0364 & 0.8880 & 0.9007 & 0.9377 & 0.9478 & 0.9347 \\
        \textbf{+ $\mathsf{SIOpt1}$} (D) & $\frac{S^{fore}}{S^{back}}$ & 0.0287 & \textbf{0.0289} & 0.8986 & 0.9072 & 0.9375 & 0.9443 & \textbf{0.9442} \\
    \bottomrule
    \end{tabular}
   }
    \label{tab:ablation_alpha_HKU-IS}
\end{table*}

\subsection{Time Cost Comparison} \label{time_cost}
\subsubsection{Training Time Comparison} 
We report the practical training time of different backbones when applying $\mathsf{SIOpt}$ in \cref{tab:eff}, where the results are reported as mean $\pm$ standard deviation (in seconds). Specifically, assume the time complexity per image of the baseline method (e.g., $\BCE$) is $\tilde{\mathcal{O}}(S)$. As discussed in \cref{20250411:Sec.4.2.1}, the theoretical complexity of $\mathsf{SIOpt1}$ can be nearly $\tilde{\mathcal{O}}((M+1)\bar{S})$, where $\tilde{\mathcal{O}}(\bar{S})$ represents the average computation cost per frame. This is corroborated by the empirical results in \cref{tab:eff} given that the average number of components $\bar{M}$ on the DUTS-TR training set is approximately $\bar{M} \approx 1.21$. As for $\mathsf{SIOpt2}$, although the complexity of computing $\mathcal{L}_{\SI\AUC}$ remains $\tilde{\mathcal{O}}(S)$, it typically requires combination with auxiliary losses to enhance training stability and performance (e.g., ICON) —a common practice in AUC-optimized learning \cite{MAUC,DBLP:conf/iccv/Yuan0SY21}. This combination introduces a slight increase in computational cost. In summary, the results suggest that our proposed $\mathsf{SIOpt}$ paradigm offers promising detection performance with an acceptable computational overhead.

\begin{table}[!t]
    \centering
    \caption{Practical pre-process time for each dataset.}
    \begin{tabular}{l|c|c|c}
    \toprule
        Dataset (Sample size) & Stage(a) & Stage(b) & Total \\
    \midrule
        DUTS-TE (5,019) & 474.0s & 188.2s & 658.2s \\ 
        DUT-OMRON (5,168) & 470.7s & 220.1s & 690.8s \\ 
        ECSSD (1,000) & 91.8s & 40.5s & 132.3s \\ 
        HKU-IS (4,447) & 508.8s & 190.1s & 698.9s \\ 
        XPIE (10,000) & 1432.2s & 316.5s & 1748.7s \\
    \bottomrule
    \end{tabular}
    \label{tab:pre_time}
\end{table}


\subsubsection{Pre-processing Time} It is important to note that the time required for computing the connected components of an image is excluded from the training time, as all such operations are performed during the data pre-processing stage. The pre-processing primarily involves two steps:

\textbf{Stage (a)} Identifying the connected components within each image.

\textbf{Stage (b)} Generating the corresponding weight masks based on the bounding boxes of these components.

Accordingly, we report the pre-processing time on several representative datasets in \cref{tab:pre_time}, including DUTS-TE, DUT-OMRON, ECSSD, HKU-IS and XPIE. The results demonstrate that both the connected components and bounding boxes can be computed with acceptable efficiency.

\subsubsection{Complexity of $\mathsf{SIOpt2}$ with and without PBAcc}  \label{PBAcc_appendix}
As discussed in the main paper, directly optimizing Eq.(\ref{20250323eq28}) for size-invariance presents a significant computational challenge, as each positive pixel must be paired with all negative pixels to compute \( \ell_{\SI\AUC}(f, \boldsymbol{X}_k^{\text{fore}, +}, \boldsymbol{X}^-) \). This results in a per-image complexity of approximately \(\mathcal{O}(S^2)\), which becomes prohibitive when working with high-resolution inputs, such as images of size $S = 384 \times 384$. To address this issue, we propose a Pixel-level Bipartite Acceleration (PBAcc) strategy to reduce the computational burden in \cref{20250419Sec4.2.3}. To evaluate the effectiveness of PBAcc, we has demonstrated that $\mathsf{SIOpt2}$ incurs a moderate training overhead with PBAcc in Sec.\cref{main:efficacy}. In this section, we further conduct a controlled experiment comparing $\mathsf{SIOpt2}$ with PBAcc (\textbf{w/ PBAcc}) and without PBAcc (\textbf{w/o PBAcc}). Specifically, we select several representative backbones (including PoolNet) and modify each by replacing Eq.~(\ref{20250419eq32}) with Eq.~(\ref{20250323eq28}) in \cref{loss_description}. Comprehensive experimental results are provided in \cref{tab:SIOpt2-comparability}. As shown, $\mathsf{SIOpt2}$ \textbf{w/ PBAcc} is substantially more memory-efficient than its naive counterpart (\textbf{w/o PBAcc}). For instance, on PoolNet, PBAcc reduces GPU memory consumption by approximately $15$ times, significantly enhancing the feasibility of $\mathsf{SIOpt2}$ optimization. In contrast, for other backbones such as ICON, GateNet, and EDN, even with a batch size of 1, $\mathsf{SIOpt2}$ cannot even be optimized without PBAcc due to memory constraints. These findings consistently highlight the efficiency and practicality of our proposed acceleration scheme.

\begin{table}[!htbp]
  \centering
\caption{Performance comparisons between our proposed size-invariant AUC optimization (i.e., $\mathsf{SIOpt2}$ and AUCSeg \cite{AUCSeg}). }
  \scalebox{0.85}{
    \begin{tabular}{c|c|c|cccccccccc}
    \toprule
    \multicolumn{1}{c}{Dataset} & \multicolumn{1}{|c|}{Backbone} & \multicolumn{1}{|c|}{Methods} & $\MAE \downarrow$ & $\SMAE \downarrow$ & $\AUC \uparrow $ & $\SAUC \uparrow $ & $\F_m^{\beta} \uparrow$ & $\SF_m^{\beta} \uparrow $ & $\F_{max}^{\beta} \uparrow$ & $\SF_{max}^{\beta} \uparrow$ & $\E_m \uparrow$ & $\Sm_m \uparrow$ \\
    \midrule
    \multirow{4}[4]{*}{DUTS-TE} & \multirow{2}[2]{*}{GateNet \cite{GateNet}} & AUCSeg \cite{AUCSeg} & 0.0430  & 0.0418  & \cellcolor[rgb]{ .796,  .769,  .91}0.9740  & \cellcolor[rgb]{ .796,  .769,  .91}\textbf{0.9687} & 0.8035  & 0.8326  & \cellcolor[rgb]{ .933,  .925,  .973}0.8841  & \cellcolor[rgb]{ .875,  .859,  .945}0.9100  & 0.8890  & \cellcolor[rgb]{ .886,  .871,  .949}0.9362  \\
          &       & SIOpt2 (Ours) & \cellcolor[rgb]{ .796,  .769,  .91}\textbf{0.0368} & \cellcolor[rgb]{ .796,  .769,  .91}\textbf{0.0367} & \cellcolor[rgb]{ .796,  .769,  .91}\textbf{0.9740} & \cellcolor[rgb]{ .812,  .788,  .918}0.9673  & \cellcolor[rgb]{ .804,  .776,  .914}\textbf{0.8282} & \cellcolor[rgb]{ .796,  .769,  .91}\textbf{0.8571} & \cellcolor[rgb]{ .796,  .769,  .91}\textbf{0.8956} & \cellcolor[rgb]{ .796,  .769,  .91}\textbf{0.9174} & \cellcolor[rgb]{ .804,  .776,  .914}\textbf{0.9087} & \cellcolor[rgb]{ .808,  .78,  .918}\textbf{0.9387} \\
\cmidrule{2-13}          & \multirow{2}[2]{*}{EDN \cite{EDN}} & AUCSeg \cite{AUCSeg} & \cellcolor[rgb]{ .898,  .886,  .953}0.0400  & \cellcolor[rgb]{ .929,  .922,  .969}0.0401  & 0.9596  & 0.9505  & \cellcolor[rgb]{ .796,  .769,  .91}\textbf{0.8290} & \cellcolor[rgb]{ .918,  .906,  .965}0.8428  & \cellcolor[rgb]{ .945,  .937,  .976}\textbf{0.8833} & 0.8979  & \cellcolor[rgb]{ .902,  .886,  .957}0.8989  & \cellcolor[rgb]{ .796,  .769,  .91}\textbf{0.9390} \\
          &       & SIOpt2 (Ours) & \cellcolor[rgb]{ .855,  .835,  .933}\textbf{0.0386} & \cellcolor[rgb]{ .847,  .827,  .929}\textbf{0.0380} & \cellcolor[rgb]{ .965,  .961,  .984}\textbf{0.9622} & \cellcolor[rgb]{ .957,  .953,  .98}\textbf{0.9545} & \cellcolor[rgb]{ .839,  .82,  .929}0.8237  & \cellcolor[rgb]{ .871,  .851,  .945}\textbf{0.8485} & 0.8784  & \cellcolor[rgb]{ .988,  .988,  .996}\textbf{0.8991} & \cellcolor[rgb]{ .796,  .769,  .91}\textbf{0.9091} & 0.9325  \\
    \midrule
    \multirow{4}[4]{*}{DUT-OMRON} & \multirow{2}[2]{*}{GateNet \cite{GateNet}} & AUCSeg \cite{AUCSeg} & 0.0580  & 0.0501  & \cellcolor[rgb]{ .914,  .969,  .961}0.9371  & \cellcolor[rgb]{ .906,  .969,  .957}0.9354  & 0.7211  & 0.7831  & \cellcolor[rgb]{ .98,  .992,  .992}0.8171  & \cellcolor[rgb]{ .965,  .988,  .984}0.8794  & 0.8344  & \cellcolor[rgb]{ .937,  .98,  .973}0.9235  \\
          &       & SIOpt2 (Ours) & \cellcolor[rgb]{ .847,  .945,  .929}\textbf{0.0523} & \cellcolor[rgb]{ .847,  .945,  .929}\textbf{0.0454} & \cellcolor[rgb]{ .847,  .945,  .929}\textbf{0.9417} & \cellcolor[rgb]{ .847,  .945,  .929}\textbf{0.9397} & \cellcolor[rgb]{ .898,  .965,  .953}\textbf{0.7499} & \cellcolor[rgb]{ .847,  .945,  .929}\textbf{0.8181} & \cellcolor[rgb]{ .847,  .945,  .929}\textbf{0.8302} & \cellcolor[rgb]{ .847,  .945,  .929}\textbf{0.8911} & \cellcolor[rgb]{ .867,  .953,  .937}\textbf{0.8586} & \cellcolor[rgb]{ .914,  .969,  .961}\textbf{0.9248} \\
\cmidrule{2-13}          & \multirow{2}[2]{*}{EDN \cite{EDN}} & AUCSeg \cite{AUCSeg} & \cellcolor[rgb]{ .847,  .945,  .929}\textbf{0.0522} & \cellcolor[rgb]{ .859,  .949,  .933}\textbf{0.0458} & \cellcolor[rgb]{ .984,  .996,  .996}\textbf{0.9319} & \cellcolor[rgb]{ .996,  1,  1}\textbf{0.9286} & \cellcolor[rgb]{ .847,  .945,  .929}\textbf{0.7631} & \cellcolor[rgb]{ .882,  .961,  .945}0.8106  & \cellcolor[rgb]{ .859,  .949,  .937}\textbf{0.8292} & \cellcolor[rgb]{ .973,  .992,  .988}\textbf{0.8786} & \cellcolor[rgb]{ .867,  .953,  .941}0.8583  & \cellcolor[rgb]{ .847,  .945,  .929}\textbf{0.9283} \\
          &       & SIOpt2 (Ours) & \cellcolor[rgb]{ .878,  .953,  .941}0.0534  & \cellcolor[rgb]{ .875,  .953,  .941}0.0463  & 0.9307  & 0.9282  & \cellcolor[rgb]{ .89,  .961,  .949}0.7522  & \cellcolor[rgb]{ .859,  .949,  .937}\textbf{0.8160} & 0.8148  & 0.8757  & \cellcolor[rgb]{ .847,  .945,  .929}\textbf{0.8616} & 0.9199  \\
    \midrule
    \multirow{4}[4]{*}{MSOD} & \multirow{2}[2]{*}{GateNet \cite{GateNet}} & AUCSeg \cite{AUCSeg} & \cellcolor[rgb]{ .996,  .996,  .996}0.0500  & \cellcolor[rgb]{ .996,  .992,  .992}0.0827  & \cellcolor[rgb]{ .996,  .988,  .984}0.9506  & \cellcolor[rgb]{ .996,  .996,  .996}0.9395  & \cellcolor[rgb]{ .988,  .894,  .827}0.7673  & \cellcolor[rgb]{ .988,  .898,  .835}0.7287  & \cellcolor[rgb]{ .988,  .925,  .878}0.8439  & \cellcolor[rgb]{ .996,  .988,  .98}0.8621  & \cellcolor[rgb]{ .988,  .894,  .827}0.8587  & \cellcolor[rgb]{ .996,  .996,  .992}0.9325  \\
          &       & SIOpt2 (Ours) & \cellcolor[rgb]{ .988,  .894,  .827}\textbf{0.0446} & \cellcolor[rgb]{ .988,  .894,  .827}\textbf{0.0716} & \textbf{0.9532} & \textbf{0.9403} & \textbf{0.8023} & \textbf{0.7736} & \textbf{0.8616} & \textbf{0.8682} & \textbf{0.8909} & \textbf{0.9327} \\
\cmidrule{2-13}          & \multirow{2}[2]{*}{EDN \cite{EDN}} & AUCSeg \cite{AUCSeg} & 0.0501  & 0.0832  & \cellcolor[rgb]{ .988,  .894,  .827}0.9224  & \cellcolor[rgb]{ .988,  .894,  .827}0.9006  & \cellcolor[rgb]{ .992,  .929,  .886}0.7796  & \cellcolor[rgb]{ .988,  .894,  .827}0.7262  & \cellcolor[rgb]{ .988,  .894,  .827}0.8362  & \cellcolor[rgb]{ .988,  .894,  .827}0.8118  & \cellcolor[rgb]{ .988,  .898,  .831}0.8600  & \cellcolor[rgb]{ .996,  .984,  .976}\textbf{0.9317} \\
          &       & SIOpt2 (Ours) & \cellcolor[rgb]{ .992,  .929,  .89}\textbf{0.0466} & \cellcolor[rgb]{ .988,  .91,  .851}\textbf{0.0734} & \cellcolor[rgb]{ .992,  .933,  .89}\textbf{0.9340} & \cellcolor[rgb]{ .992,  .937,  .902}\textbf{0.9181} & \cellcolor[rgb]{ .992,  .953,  .929}\textbf{0.7880} & \cellcolor[rgb]{ .996,  .969,  .953}\textbf{0.7611} & \cellcolor[rgb]{ .988,  .922,  .875}\textbf{0.8436} & \cellcolor[rgb]{ .992,  .933,  .894}\textbf{0.8344} & \cellcolor[rgb]{ .996,  .992,  .992}\textbf{0.8895} & \cellcolor[rgb]{ .988,  .894,  .827}0.9245  \\
    \midrule
    \multirow{4}[4]{*}{SOD} & \multirow{2}[2]{*}{GateNet \cite{GateNet}} & AUCSeg \cite{AUCSeg} & 0.1094  & 0.1018  & \cellcolor[rgb]{ .98,  .855,  .871}\textbf{0.9021} & \cellcolor[rgb]{ .98,  .855,  .871}\textbf{0.8938} & 0.7720  & 0.7369  & \cellcolor[rgb]{ .996,  .949,  .953}0.8621  & \cellcolor[rgb]{ .984,  .867,  .882}0.8495  & 0.7860  & \cellcolor[rgb]{ .996,  .949,  .957}0.8870  \\
          &       & SIOpt2 (Ours) & \cellcolor[rgb]{ .98,  .855,  .871}\textbf{0.0990} & \cellcolor[rgb]{ .98,  .855,  .871}\textbf{0.0925} & \cellcolor[rgb]{ .984,  .863,  .875}0.9014  & \cellcolor[rgb]{ .984,  .867,  .882}0.8916  & \cellcolor[rgb]{ .98,  .855,  .871}\textbf{0.8042} & \cellcolor[rgb]{ .98,  .855,  .871}\textbf{0.7668} & \cellcolor[rgb]{ .98,  .855,  .871}\textbf{0.8719} & \cellcolor[rgb]{ .98,  .855,  .871}\textbf{0.8524} & \cellcolor[rgb]{ .98,  .855,  .871}\textbf{0.8094} & \cellcolor[rgb]{ .98,  .855,  .871}\textbf{0.8942} \\
\cmidrule{2-13}          & \multirow{2}[2]{*}{EDN \cite{EDN}} & AUCSeg \cite{AUCSeg} & \cellcolor[rgb]{ .996,  .98,  .98}0.1081  & \cellcolor[rgb]{ .996,  .976,  .98}0.1005  & 0.8767  & 0.8627  & \cellcolor[rgb]{ .992,  .929,  .937}0.7881  & \cellcolor[rgb]{ 1,  .988,  .992}0.7395  & 0.8562  & 0.8157  & \cellcolor[rgb]{ .996,  .961,  .965}0.7926  & \cellcolor[rgb]{ .996,  .945,  .953}\textbf{0.8874} \\
          &       & SIOpt2 (Ours) & \cellcolor[rgb]{ .988,  .933,  .941}\textbf{0.1048} & \cellcolor[rgb]{ .984,  .902,  .914}\textbf{0.0957} & \textbf{0.8769} & \cellcolor[rgb]{ 1,  .992,  .992}\textbf{0.8650} & \cellcolor[rgb]{ .984,  .886,  .898}\textbf{0.7979} & \cellcolor[rgb]{ .988,  .91,  .922}\textbf{0.7559} & \cellcolor[rgb]{ .996,  .969,  .973}\textbf{0.8600} & \cellcolor[rgb]{ .996,  .973,  .976}\textbf{0.8231} & \cellcolor[rgb]{ .984,  .859,  .875}\textbf{0.8090} & 0.8830  \\
    \bottomrule
    \end{tabular}%
  \label{tab:AUCSeg}%
  }
  
\end{table}%

    \subsection{Performance Comparisons against other Pixel-level AUC-oriented Methods} \label{app:E.10}

    At the end of \cref{20250419Sec4.2.3}, we discussed the strengths and distinctions of our proposed size-invariant paradigm compared to traditional AUC-oriented approaches. To empirically validate these advantages, this section evaluates a recently proposed and effective AUC-based baseline for semantic segmentation—AUCSeg \cite{AUCSeg}—and compares its performance with our method, $\mathsf{SIOpt2}$, on RGB-based SOD tasks. 

    \noindent\textbf{Setups}. AUCSeg \cite{AUCSeg}\footnote{\url{https://github.com/boyuh/AUCSeg}} is an early work introducing AUC optimization into pixel-level long-tail semantic segmentation. It proposes a novel AUC loss function and incorporates a memory-efficient T-Memory Bank to address the high computational cost associated with large-scale datasets. To adapt AUCSeg for SOD tasks, we follow the original implementation, modifying it for binary classification, where the regularization weight $\lambda$ in AUCSeg is set to $1$. The implementation details of $\mathsf{SIOpt2}$ remain consistent with those summarized in \cref{details_appendix}. Following the RGB-based experimental settings in \cref{Experiments_appendix}, we adopt two backbone networks—GateNet \cite{GateNet} and EDN \cite{EDN}—and evaluate both AUCSeg and $\mathsf{SIOpt2}$ on four benchmarks: DUTS-TE, DUT-OMRON, MSOD, and SOD.

    \noindent \textbf{Results.} For clear comparison, we present qualitative visualizations in \cref{fig:AUCSeg} and report quantitative results in \cref{tab:AUCSeg}. Several key observations can be drawn from these results. First, $\mathsf{SIOpt2}$ consistently outperforms AUCSeg across most cases, with particularly significant improvements on the MSOD and SOD datasets, which contain multiple and imbalanced salient objects per image. Moreover, since both AUCSeg and $\mathsf{SIOpt2}$ are designed to directly optimize the AUC metric, evaluating their performance on other non-AUC metrics provides a more comprehensive view of their generalization capabilities. In this regard, $\mathsf{SIOpt2}$ demonstrates more balanced performance across various metrics compared to AUCSeg, highlighting the generalization advantages of its size-invariant design. Most importantly, the performance gap becomes sharper on size-invariant metrics. For example, using GateNet as the backbone, $\mathsf{SIOpt2}$ achieves an average improvement of $3.4\%$ on $\SF_m^{\beta}$ across the four datasets. Overall, these results consistently reinforce the superiority of our proposed size-invariant paradigm.

\stopcontents[sections]
\end{document}